\documentclass{article}

\PassOptionsToPackage{authoryear}{natbib}

\usepackage[final,main]{neurips_2026}

\input{preamble.tex}
\usepackage{longtable}

\makeatletter
\patchcmd{\@notice}{\enlargethispage{2\baselineskip}}{}{}{
  \PackageError{camera-ready}{Could not patch the NeurIPS notice}{Check the definition of \string\@notice.}
}
\makeatother

\let\sb=_

\renewcommand{\algname}[2][]{{\sffamily#1\color{mydarkgreen}#2}\xspace}

\definecolor{myrefcolor}{RGB}{165,30,55}
\definecolor{myurlpurple}{RGB}{94,53,177}
\hypersetup{
  colorlinks=true,
  linkcolor=myrefcolor,
  citecolor=mydarkblue,
  urlcolor=myurlpurple
}

\definecolor{bgcolor}{rgb}{0.8,1,1}
\definecolor{bgcolor2}{rgb}{0.8,1,0.8}

\title{Rennala NSGD: Asynchronous Stochastic Optimization Beyond Euclidean Geometry}

\author{%
  Igor Sokolov \\
  KAUST\thanks{Center of Excellence for Generative AI, King Abdullah University of Science and Technology (KAUST), Thuwal, Saudi Arabia. Correspondence to: Igor Sokolov \href{mailto:igor.sokolov.ml@gmail.com}{\texttt{igor.sokolov.ml@gmail.com}}. Website: \url{https://igor-sokolov.com}.}
  \And
  Alexander Tyurin \\
  AXXX\thanks{AXXX, Moscow, Russia.}, Applied AI Institute\thanks{Applied AI Institute, Moscow, Russia.}
  \And
  Peter Richt\'arik \\
  KAUST
}

\begin{document}

\maketitle

\begin{abstract}
The empirical success of matrix-aware optimizers, such as \algname{Muon}~\citep{jordan2024muon} and \algname{Shampoo}~\citep{gupta2018shampoo}, highlights the potential of non-Euclidean geometry for deep learning, including large language model (LLM) training. Despite this progress, the theoretical understanding of stochastic optimization in general normed spaces remains limited, particularly in distributed settings. We analyze minibatch stochastic gradient descent (\algname{SGD}) for smooth nonconvex optimization \emph{with respect to an arbitrary norm on a finite-dimensional space}. Using the framework of $\kappa$-regular normed spaces~\citep{juditsky2008large}, we derive high-probability oracle complexity bounds under a \emph{light-tail} condition stated directly in the dual norm induced by the update norm. The bounds make the dependence on geometry explicit through the regularity constant $\kappa$ of the dual space.
In the Euclidean case, our bound recovers the standard minibatch \algname{SGD} rate up to logarithmic factors. For some non-Euclidean geometries and noise models, it is sharper than bounds obtained by converting a Euclidean noise assumption, reducing a linear dimension dependence to a logarithmic one.
Building on this framework, we introduce \algname{Rennala NSGD}, a semi-asynchronous distributed algorithm, and establish wall-clock time complexity guarantees under fixed, heterogeneous worker computation times. In the Euclidean case ($\kappa=1$), our bound matches the parameter dependence of classical \algname{Rennala SGD} up to logarithmic factors.
\end{abstract}

\section{Introduction}\label{sec:intro}
\textbf{Modern optimizers are non-Euclidean in practice.}
Modern deep-learning training increasingly uses \emph{geometry-aware} primitives, including matrix preconditioning and orthogonalized updates.
Matrix-aware optimizers include \algname{Shampoo}~\citep{gupta2018shampoo}, which uses matrix preconditioning, and \Muon~\citep{jordan2024muon} and \algname{Dion}~\citep{ahn2025dion}, which use orthogonalized matrix updates.
Recent studies report improved training efficiency over \algname{AdamW} for \Muon~\citep{liu2025muon} and \algname{Dion}~\citep{ahn2025dion} in Transformer language-model experiments.
Their empirical advantages are often interpreted through the connection between these matrix updates and steepest descent with respect to non-Euclidean norms, such as the spectral norm and related Schatten norms~\citep{bernstein2024old, lau2025polargrad}. These norms respect the matrix structure of neural-network parameters, including the weight matrices in Transformer attention and feed-forward layers.
This geometric perspective motivates a stochastic theory developed directly in general normed spaces.

\textbf{Theory gap: stochastic non-Euclidean \algname[]{SGD} under the ``right'' noise model.}
A central modeling question in stochastic non-Euclidean optimization is how to quantify oracle noise.
Most existing analyses~\citep{kovalev2025understanding, gruntkowska2025error, li2025note, shen2025convergence, zhang2025provable, takezawa2025fedmuon}
assume Euclidean (or Frobenius) moment bounds or bounded variance and then transfer results to a \emph{target non-Euclidean norm} via norm equivalence.
Such conversions can introduce dimension-dependent inflation and may obscure the regimes in which a geometry-aware method should provably help.
We therefore seek guarantees under a noise model stated in the \emph{native norm}.

\textbf{Our lens: dual-norm light tails and $\kappa$-regularity.}
We measure oracle noise directly in the dual norm $\snorm{\cdot}$ induced by the chosen update norm $\norm{\cdot}$ (\Cref{as:non_euclidean_light_tail}).
Using the $\kappa$-regularity framework~\citep{juditsky2008large}, we obtain high-probability guarantees with the additional geometric dependence captured by the dual-space regularity constant $\kappa$.
Every nonzero finite-dimensional normed space admits such a constant, so this approach does not restrict the choice of update norm (\Cref{subsec:regularity}).
For suitable geometries and noise models, it yields sharper bounds than Euclidean/Frobenius conversions (\Cref{subsubsec:native_noise_improvement}).

\textbf{Problem setup.}
Throughout the paper, $E$ denotes a finite-dimensional space with $\dim(E)\ge1$.
We consider the distributed, non-convex stochastic optimization problem expressed as a finite sum over $N$ workers:
\begin{align}
    \label{eq:main_task_sum}
    \min_{x \in E} \left\{ f(x) \eqdef \frac{1}{N} \sum_{i=1}^N f_i(x) \right\},
\end{align}
where each $f_i$ is an expectation over a worker-specific data distribution $\cD_i$:
\begin{align}
    \label{eq:local_objective}
    f_i(x) \eqdef \ExpSub{\xi_i \sim \cD_i}{f_i(x; \xi_i)}.
\end{align}
The space $E$ may represent, for instance, vectors in $\R^d$ or matrices in $\R^{m\times n}$. For each worker $i$, the function $f_i:E\times\mathbb{S}_i\to\R$ is defined over a sample space $\mathbb{S}_i$, and $\xi_i$ is a random variable drawn from $\cD_i$.
The problem is \emph{unconstrained}: the methods we analyze contain no projection, so nothing confines the iterates to a proper subset of $E$, and every assumption below is stated on all of $E$.

This formulation is particularly relevant to Federated Learning (FL)~\citep{konevcny2016federated, mcmahan2017communication, kairouz2021advances}. In this context, two primary forms of heterogeneity arise: \textbf{device heterogeneity}, which concerns variations in system capabilities, and \textbf{statistical heterogeneity}, where data distributions $\cD_i$ differ across workers. To isolate the effects of device heterogeneity, we assume that all workers share the same sample loss $f(\cdot;\cdot)$ and sampling distribution $\cD$. Thus, $f_i(x;\xi)=f(x;\xi)$ and $\cD_i=\cD$ for every worker $i$, so all local objectives coincide with $f(x)=\ExpSub{\xi\sim\cD}{f(x;\xi)}$.
In this case, \eqref{eq:main_task_sum} reduces to
\begin{align}
    \label{eq:main_task}
    \min_{x \in E} \left\{f(x) \eqdef \ExpSub{\xi \sim \cD}{f(x;\xi)}\right \}.
\end{align}
Our algorithms access $f$ through a stochastic first-order oracle:
given a query point $x$, it returns a random vector $\nf{x;\xi}\in E^*$
that estimates $\nabla f(x)$.
In the sample-loss setting above, this response is the gradient of
$f(x;\xi)$ with respect to $x$.
More generally, we allow stochastic estimates that need not be gradients
of sample losses; the oracle assumptions are stated formally in
\Cref{subsec:stationarity_oracle}.  
In the parallel setting, all workers use the same oracle map and sampling
distribution $\cD$.

Throughout the paper, we assume that $f$ is $L$-smooth
(\Cref{as:non_euclidean_smoothness}) and bounded from below
(\Cref{as:lower_bound}), and that the oracle responses are unbiased and
satisfy a light-tail condition in the induced dual norm $\snorm{\cdot}$
(\Cref{as:non_euclidean_light_tail}).

We work in the following parallel computing setup:
\begin{enumerate}[label=(\roman*), leftmargin=*, nosep]
    \item A fixed number of $N$ workers are available to perform computations in parallel.
    \item Worker $i$ requires a fixed time $\tau_i$ seconds to compute one stochastic gradient after receiving a query point. These times remain fixed throughout the run, and we index the workers so that $0 < \tau_1 \le \tau_2 \le \cdots \le \tau_N$.
    \item Communication in both directions between the server and the workers, and server-side processing, including evaluation of the sharp update, are treated as instantaneous, as in the computation model of \Cref{sec:rennala_nsgd}.
\end{enumerate}
\textbf{From oracle complexity to wall-clock time: heterogeneous parallelism.}
On the systems side, asynchronous \algname{SGD} has been shown to improve wall-clock convergence guarantees over synchronous minibatch \algname{SGD} that waits for all workers, in settings with heterogeneous worker speeds \citep{mishchenko2022asynchronous}. This line of work emphasizes \textbf{time complexity}---the wall-clock time needed to reach a prescribed accuracy---as a natural measure of parallel performance. In the Euclidean setting, \citet{tyurin2023optimal} establish a lower bound for their zero-respecting algorithm class and a matching upper bound, up to constants, for \algname{Rennala SGD} under their specified unbiased bounded-variance oracle model with fixed worker computation times.
On the optimization side, non-Euclidean methods are motivated by modern matrix-aware updates, yet their interaction with computational heterogeneity remains poorly understood.
Bridging these two directions requires (i) tight control of stochastic gradient noise in the induced dual norm $\snorm{\cdot}$ associated with the chosen update norm $\norm{\cdot}$ and (ii) a wall-clock analysis that accounts for heterogeneous worker speeds under non-Euclidean updates.
Together, these two directions motivate the following question:

\begin{tcolorbox}[colback=black!5!white,colframe=black!75!black, sharp corners, boxrule=0.5pt]
\centering
\textit{Should stochastic gradient noise be controlled in a Euclidean or Frobenius norm and then converted, via norm equivalence, to the dual norm $\snorm{\cdot}$ induced by the update norm $\norm{\cdot}$? Or can a noise condition stated directly in $\snorm{\cdot}$ yield sharper oracle and wall-clock guarantees?}
\end{tcolorbox}

\textbf{Notation.} The notation used throughout the paper is collected in \Cref{app:notation}, which gives the main definitions followed by a consolidated reference table.

\textbf{Our contributions.}
We study non-convex stochastic optimization in a normed space $(E, \norm{\cdot})$ and focus on the statistically homogeneous but \emph{computationally heterogeneous} parallel regime. Our goal is to characterize both (i) the number of stochastic gradients needed to reach an $\eps$-stationary point and (ii) the \emph{wall-clock time} needed when workers have different speeds.

Our contributions are threefold:

\textbf{(i) Minibatch non-Euclidean \algname[]{SGD} under a dual-norm light-tail oracle.}
We establish high-probability oracle-complexity bounds for minibatch \algname{SGD} under the stated smoothness, lower-boundedness, and unbiased dual-norm light-tail assumptions.
The key ingredient is a minibatch concentration inequality for fresh i.i.d.\ samples in regular dual spaces (\Cref{lemma:non_euclidean_large_deviations}).
Combining it with non-Euclidean descent, we show that an $\eps$-stationary point is reached with probability at least $1-\delta$ using
\begin{align}\label{eq:complexity_dual_native_intro}
\widetilde{\mathcal O}\!\left(\frac{L\Delta}{\eps} + \frac{\kappa L\Delta\sigma^2}{\eps^2}\right)
\end{align}
stochastic gradient evaluations (\Cref{thm:stochastic_L_smooth_final} and \Cref{cor:complexity_dual_native_light_tail}), where $\Delta=f(x^0)-\finf$ is the initial functional gap and $\kappa$ is a regularity constant of the dual space $(E^*, \snorm{\cdot})$.
In the Euclidean case, where $\kappa=1$, this matches the optimal oracle complexity up to logarithmic factors for general unbiased light-tail oracles, in the regime described in \Cref{subsec:main_oracle_complexity}.

\textbf{(ii) Sharp regularity bounds and improvements over norm-equivalence conversions.}
We establish matching upper and lower bounds on the optimal $\kappa$-regularity parameter for $\ell_p^d$ and Schatten-$p$ spaces, exact for $1\le p<2$ and up to absolute constants for $2\le p\le\infty$ (\Cref{tab:regularity_examples} and \Cref{sec:reg_examples}).
For some non-Euclidean geometries and noise models, the native-norm
analysis yields significantly smaller oracle-complexity bounds than
those obtained by converting a Euclidean/Frobenius noise assumption. In particular, for the $\ell_1$ update norm under dense Rademacher perturbations, the statistical coefficient grows logarithmically rather than linearly with the dimension. We quantify these comparisons for $\ell_p^d$ and Schatten-$p$ dual geometries in \Cref{subsubsec:native_noise_improvement} and \Cref{sec:frobenius_light_tail_oracle}.

\textbf{(iii) Wall-clock guarantees for semi-asynchronous non-Euclidean \algname[]{SGD}: \algname[]{Rennala NSGD}.}
Building on the minibatch theory, we introduce \algname{Rennala NSGD}, a semi-asynchronous parallel method for spaces with non-Euclidean geometry, which collects fresh minibatches at the current iterate and discards stale gradients. In a system with worker computation times $\tau_i$, we prove (\Cref{thm:main_time_complexity}) the wall-clock time bound
\begin{align}\label{eq:time_complexity_rennala_intro}
    \widetilde{\mathcal O}\left(
    \min_{S\in[N]}
    \left(\frac{1}{S}\sum_{i=1}^S\frac{1}{\tau_i}\right)^{-1}
    \left(\frac{L\Delta}{\eps}+\frac{\kappa L\Delta \sigma^2}{S\eps^2}\right)
    \right).
\end{align}
In the Euclidean case, where $\kappa=1$, \eqref{eq:time_complexity_rennala_intro} matches the parameter dependence of the classical \algname{Rennala SGD} time-complexity bound up to logarithmic factors \citep{tyurin2023optimal} (\Cref{subsec:time_complexity}).



\section{Related Work}\label{sec:related_work}
Our work connects two research threads: (i) stochastic optimization in a general normed space $(E, \norm{\cdot})$ with \emph{non-Euclidean norm}, and (ii) time-aware parallel optimization under computational heterogeneity.

\textbf{Non-Euclidean Optimization.}
Recent work on large-scale training has explored matrix-aware optimizers, including the \algname{Muon} and \algname{Shampoo} families, that approximate steepest-descent directions in non-Euclidean norms~\citep{jordan2024muon,shah2025practical,an2025asgo}.
Several existing theoretical analyses of these methods in the non-distributed setting
\citep{li2025note,shen2025convergence,kim2026convergence,nagashima2026improved,kovalev2025understanding}
control stochastic gradient noise in Euclidean or Frobenius norms.
Some of these analyses use norm equivalence to obtain guarantees in a \emph{target non-Euclidean norm}.
Such conversions can introduce dimension- or rank-dependent factors.
Our analysis instead imposes a light-tail condition directly in the dual norm induced by the update geometry.

In distributed settings, several works have explored \algname{Muon}-style methods.
\citet{zhang2025provable} and \citet{takezawa2025fedmuon} analyze synchronous federated training and address the statistical bias of local updates, while \citet{gruntkowska2025error} focuses on communication efficiency via error feedback.
These studies primarily target communication overhead or local-update bias within synchronized rounds (typically under Frobenius/Euclidean noise models), and do not address computational heterogeneity or wall-clock bounds.

\textbf{Stochastic Optimization in Regular Spaces.}
The $\kappa$-regularity framework of \citet{juditsky2008large} provides large-deviation bounds for martingale sums in finite-dimensional normed spaces, with geometric dependence captured by the regularity constant $\kappa$.
Under conditional exponential-moment bounds on the squared norms of the increments, they obtain sub-Gaussian deviation bounds for the norm of the martingale sum.
These tools were subsequently incorporated into the analysis of stochastic first-order methods \citep[see, e.g.,][]{juditsky2012randomized}.
Related martingale large-deviation bounds have also been used to derive non-asymptotic confidence bounds for assessing the quality of solutions to stochastic optimization problems \citep{guigues2017multistep}.

\textbf{Time Complexity in Heterogeneous Parallel \algname[]{SGD}.}
In the distributed regime, parallelizing \algname{SGD} is a natural strategy for accelerating training. The simplest approach, \textbf{minibatch \algname{SGD}}, synchronizes $N$ workers at each step to compute gradients in parallel, after which a central server performs the update~\citep{gower2019sgd, cotter2011better, goyal2017accurate}:
$x^{k+1} = x^k - \frac{\gamma}{N} \sum_{i=1}^N \nabla f(x^k; \xi_i^k)$.
For a general minibatch size $B$ and a suitable stepsize, the standard Euclidean/Frobenius analysis gives the iteration-complexity bound
$\mathcal{T}_{\mathrm{SGD}}^{\opf}(B)=\mathcal{O}\!\left(\frac{L_{\opf}\Delta}{\eps_{\opf}} + \frac{\sigma_{\opf}^2 L_{\opf}\Delta}{B\eps_{\opf}^2}\right)$~\citep{ghadimi2013stochastic},%
\footnote{The bound assumes smoothness, lower boundedness, unbiased bounded-variance noise, and fresh independent samples. It ensures $\Exp{\sqfnorm{\nabla f(\hat{x})}}\le\eps_{\opf}$ for an iterate $\hat{x}$ selected uniformly from the iterates. The subscript $\opf$ denotes Euclidean/Frobenius parameters, and $\Delta=f(x^0)-\finf$.}
with $B=N$ for the synchronized method above.
Since each iteration waits for all workers, heterogeneous worker speeds make the slowest worker the bottleneck~\citep{dutta2018slow, chen2016revisiting}: if worker $i$ takes $\tau_i$ seconds per gradient, the runtime scales as $\tau_{\max}\mathcal{T}_{\mathrm{SGD}}^{\opf}(N)$, where $\tau_{\max}\eqdef\max_i\tau_i$.

To overcome this synchronization barrier, \textbf{Asynchronous \algname{SGD}} (\algname{ASGD}) was introduced~\citep{recht2011hogwild, feyzmahdavian2016asynchronous, nguyen2018sgd, arjevani2020tight}.
In its standard form, the server applies each arriving stochastic gradient immediately:
$x^{k+1} = x^k - \gamma_k \nabla f(x^{k-\delta_k}; \xi^{k-\delta_k})$,
where $\delta_k$ is the gradient delay and $\gamma_k$ may depend on it.
Recent theoretical advances~\citep{mishchenko2022asynchronous, koloskova2022sharper, cohen2021asynchronous} establish wall-clock bounds for suitable asynchronous methods with fixed worker computation times $\tau_i$ of order $\left(\frac{1}{N}\sum_{i=1}^N\frac{1}{\tau_i}\right)^{-1}\cdot\mathcal{T}_{\mathrm{SGD}}^{\opf}(N)$, governed by the harmonic-mean worker time rather than the slowest worker (\Cref{app:related_time_complexity}).

However, to the best of our knowledge, the available wall-clock bounds for classical \algname{ASGD} remain suboptimal in general.
\citet{tyurin2023optimal} close this gap with \algname{Rennala SGD}, a semi-asynchronous method that uses only gradients computed at the current iterate.
Its wall-clock bound is
$\min_{S\in[N]}\left\{\left(\frac{1}{S}\sum_{i=1}^S\frac{1}{\tau_i}\right)^{-1}\mathcal{T}_{\mathrm{SGD}}^{\opf}(S)\right\}$,
where $0<\tau_1\le\cdots\le\tau_N$ and $S$ counts the fastest workers considered in the bound.
This bound is minimax-optimal up to constants in their model.%
\footnote{The guarantee concerns Euclidean/Frobenius stationarity in expectation under their unbiased bounded-variance oracle model with fixed worker computation times; optimality is relative to their class of zero-respecting algorithms.}

Asynchrony with \emph{non-Euclidean} updates has been explored mostly in convex/online regimes~\citep{duchi2013estimation, feyzmahdavian2016asynchronous, joulani2019think}. These methods use mirror descent or dual averaging, with non-Euclidean geometry specified through a prox function and its associated Bregman divergence. Our approach instead uses steepest-descent steps defined by an update norm and a fixed sharp selection. These works focus on delay-robust iteration/regret guarantees, and do not provide time-optimal wall-clock bounds under computational heterogeneity, nor nonconvex stationarity guarantees in general normed spaces.

We combine our stochastic analysis of \algname{SGD} for arbitrary non-Euclidean norms with a \algname{Rennala}-style semi-asynchronous architecture to obtain oracle and wall-clock guarantees beyond the Euclidean/Frobenius regime. Additional literature review references are deferred to Appendix~\ref{app:extended_related_work}.

\section{\texorpdfstring{\algname[]{SGD}}{SGD} Under Non-Euclidean Norms}\label{sec:geometry_native_minibatch}
We derive high-probability convergence and oracle-complexity guarantees for minibatch \algname{SGD} in nonconvex optimization, using steepest-descent updates in an arbitrary norm $\norm{\cdot}$ and measuring oracle noise in its induced dual norm $\snorm{\cdot}$.
Notation is collected in \Cref{app:notation}.

\subsection{Non-Euclidean Update Rule and the Sharp Operator}\label{subsec:non_eu_update}
\textbf{Minibatch non-Euclidean \algname[]{SGD}.}
At each iteration $k$, draw a fresh minibatch of i.i.d.\ samples $\{\xi_j^k\}_{j=1}^B$ from $\cD$, independent of the past. Denote the minibatch stochastic gradient as
\begin{align}\label{eq:mini_gk_def}
    \gk \eqdef \frac{1}{B}\sum_{j=1}^B \nf{\xk;\xi_j^k}\in E^*.
\end{align}
Fix a Borel-measurable selection $s\mapsto s^\sharp$ with $s^\sharp\in\partial\phi^*(s)$ for every $s\in E^*$.
Given a prescribed sequence of positive stepsizes $\{\gamma_k\}$, the non-Euclidean minibatch \algname{SGD} update is
\begin{align}
\label{eq:non-eu-one-step-final}
    \xkpo = \xk - \gamma_k \rb{\gk}^{\sharp},
    \,\, \text{where } \rb{\gk}^{\sharp}\in\partial\phi^*(\gk).
\end{align}
The convergence and complexity bounds below use the constant choice $\gamma_k\equiv\frac{1}{2L}$.

\textbf{Steepest descent / LMO form.}
The update in \eqref{eq:non-eu-one-step-final} admits equivalent steepest-descent and Linear Minimization Oracle (LMO) characterizations in $(E, \norm{\cdot})$ (\Cref{app:non_euclidean_sgd_derivation}).%
\footnote{The LMO view is common in analyses of LMO-based optimizers (including \Muon\ and its federated variants); see, e.g., \citet{takezawa2025fedmuon}.
More broadly, the perspective that many widely used optimizers can be interpreted as steepest-descent methods under suitable norms has been developed in detail by~\citet{bernstein2024old}.}
The LMO form, with the minimizer determined by the fixed sharp selection, is
\begin{align}
\label{eq:lmo_form}
    \xkpo \in \lmo_{\cB(\xk, t_k)}(\gk),
    \,\,
    t_k \eqdef \gamma_k \snorm{\gk},
\end{align}
where $\lmo_{\cB(\xk, t_k)}(\gk)\eqdef \argmin_{x\in\cB(\xk, t_k)}\inner{\gk}{x}$ and
$\cB(\xk, t_k)\eqdef\{x\in E:\norm{x-\xk}\le t_k\}$.

\textbf{Special cases of the update \eqref{eq:non-eu-one-step-final}.}
Different choices of the update norm $\norm{\cdot}$ recover familiar update rules (up to stepsize reparameterizations).

$\bullet$ \textbf{$\ell_2$ norm (Euclidean case).}
If $\norm{\cdot}=\|\cdot\|_2$, then $\snorm{\cdot}=\|\cdot\|_2$ and $(\gk)^\sharp=\gk$, hence \eqref{eq:non-eu-one-step-final} reduces to classical \SGD \citep{gower2019sgd, gorbunov2020unified, demidovich2023guide}.

$\bullet$ \textbf{$\ell_\infty$ norm.}
    If $\norm{\cdot}=\|\cdot\|_\infty$ (so $\snorm{\cdot}=\|\cdot\|_1$), then
    $(\gk)^\sharp\in \|\gk\|_1\,\partial\|\gk\|_1$ and a canonical selection is
    $(\gk)^\sharp=\|\gk\|_1\cdot \sign(\gk)$, where $\sign$ is applied componentwise with $\sign(0)=0$.
    This yields a \emph{sign-type} update closely related to \algname{signSGD}~\citep{bernstein2018signsgd, sun2023momentum}.

$\bullet$ \textbf{$\ell_1$ norm.}
    If $\norm{\cdot}=\|\cdot\|_1$ (so $\snorm{\cdot}=\|\cdot\|_\infty$), then
    $(\gk)^\sharp\in \|\gk\|_\infty\,\partial\|\gk\|_\infty$, i.e., $(\gk)^\sharp$ is supported on (a subset of) the coordinates attaining $\|\gk\|_\infty$.
    For $\gk\neq 0$, an extreme-point selection yields a greedy single-coordinate step in \eqref{eq:non-eu-one-step-final} \citep{tseng1991relaxation, tseng2009block, nutini2015coordinate}.
    A general sharp selection may update several coordinates tied for the largest $|g_i^k|$.

A detailed derivation of the update \eqref{eq:non-eu-one-step-final} for the $\ell_1$ and $\ell_{\infty}$ update norms is deferred to \Cref{app:sharp_l1_linf}.

$\bullet$ \textbf{Spectral norm and orthogonalized matrix updates.}
    When $E=\R^{m\times n}$ and $\norm{\cdot}=\opnorm{\cdot}$ is the spectral norm, the dual norm is the nuclear norm $\nucnorm{\cdot}$.
    With the canonical sharp selection $(\bG^k)^\sharp=\nucnorm{\bG^k}\orth(\bG^k)$, the update takes the following orthogonalized form (derivation in Appendix~\ref{app:orth_derivation}):
    \begin{align}
    \label{eq:orth_update_main-final}
        \bX^{k+1} = \bX^k - \gamma_k \nucnorm{\bG^k}\cdot \bO^k,
        \,\, \bO^k=\orth(\bG^k),
    \end{align}
    where $\bG^k$ is the minibatch gradient in $E^*$ and
    \begin{align}
    \label{eq:orth_def-final}
        \orth(\bG) \eqdef (\bG\bG^\top)^{\dagger/2}\bG.
    \end{align}
    For a compact rank-$r$ SVD $\bG=\mathbf U_r\Sigma_r\mathbf V_r^\top$, where $r=\mathrm{rank}(\bG)$ and only positive singular values are retained, one has $\orth(\bG)=\mathbf U_r\mathbf V_r^\top$ (the polar factor / partial isometry).%
    \footnote{This direction is obtained from the unregularized \algname{Shampoo} update with accumulation disabled, interpreting inverse powers via pseudoinverses~\citep[Story~II, Eqs.~(15)--(16)]{bernstein2024old}.
\algname{Muon} uses a numerical approximation to the same polar map, computed by \algname{Newton--Schulz} iterations and applied to a momentum update~\citep{jordan2024muon}.}

\subsection{Oracle Model}\label{subsec:stationarity_oracle}

Following \citet{tyurin2023optimal}, we consider a jointly measurable stochastic gradient oracle $(x, \xi)\mapsto\nf{x;\xi}\in E^*$, with a fresh independent sample $\xi\sim\cD$ at each call.%
\footnote{The response $\nf{x;\xi}$ need not be the gradient of a sample loss; sample-loss gradients are a special case of this oracle model.}
We analyze \eqref{eq:non-eu-one-step-final} under the following assumptions.

\begin{assumption}[Non-Euclidean $L$-smoothness]
\label{as:non_euclidean_smoothness}
The function $f$ is $L$-smooth with respect to $\norm{\cdot}$ if
$\snorm{\nabla f(x)-\nabla f(y)} \le L\norm{x-y}$ for all $x, y\in E$.
\end{assumption}

\begin{assumption}[Lower bound]
\label{as:lower_bound}
There exists $\finf\in\R$ such that $f(x)\ge \finf$ for all $x\in E$.
\end{assumption}

\begin{assumption}[Dual-norm light-tail oracle induced by the update norm]
\label{as:non_euclidean_light_tail}
For all $x\in E$, the stochastic gradient $\nf{x;\xi}$ is unbiased,
$\ExpSub{\xi}{\nf{x;\xi}}=\nabla f(x)$, and satisfies the light-tail condition in the induced dual norm
\begin{align}
\label{eq:dual_light_tail}
    \ExpSub{\xi}{\exp\rb{\frac{\sqsnorm{\nf{x;\xi}-\nabla f(x)}}{\sigma^2}}} \le \exp(1)
\end{align}
for some $\sigma>0$.
\end{assumption}

Assumptions~\ref{as:non_euclidean_smoothness} and \ref{as:lower_bound} are standard in the analysis of stochastic first-order methods~\citep{ghadimi2013stochastic, gower2019sgd}. Assumption~\ref{as:non_euclidean_light_tail} represents a non-Euclidean generalization of the standard light-tail condition---typically formulated in the Euclidean or Frobenius norm---widely used in optimization theory~\citep{lan2020first, tyurin2024toward}.

Our bounds use the dual-norm noise scale $\sigma$ and regularity constant $\kappa$ directly, without Euclidean conversion.
Whether this improves the bound depends on both the geometry and the noise model (\Cref{subsubsec:native_noise_improvement}).
Further discussion of \Cref{as:non_euclidean_light_tail} appears in \Cref{app:light_tail_discussion}.

\subsection{\texorpdfstring{$\kappa$}{kappa}-Regularity and Norm-Dependent Constants}
\label{subsec:regularity}
To obtain stochastic guarantees under an arbitrary update norm, we rely on the framework of $\kappa$-smooth and $\kappa$-regular normed spaces introduced by \citet{juditsky2008large}. We briefly recall the key definitions.

\begin{definition}[$\kappa_V$-Smooth and $\kappa_V$-Regular Spaces]
\label{def:kappa_smooth_regular}
Let $(V, \|\cdot\|_V)$ be a finite-dimensional normed space and define $\phi(u) \eqdef \frac12\|u\|_V^2$.
\emph{(i)} We say that $(V, \|\cdot\|_V)$ is \textbf{$\kappa_V$-smooth} for some $\kappa_V\ge 1$ if $\phi$ is continuously differentiable and satisfies
$\phi(u+v) \le \phi(u) + \langle \nabla \phi(u), v\rangle + \frac{\kappa_V}{2}\|v\|_V^2$ for all $u, v\in V$.

\emph{(ii)} We say that $(V, \|\cdot\|_V)$ is \textbf{$\kappa_V$-regular} if there exists an auxiliary norm $\|\cdot\|_{V, +}$ on $V$ that is $\kappa_{V, +}$-smooth for some $\kappa_{V, +}\in[1, \kappa_V]$ and satisfies the compatibility chain
\begin{align}\label{eq:reg_kappa_def_gen}
    \|u\|_V^2 \le \|u\|_{V, +}^2 \le \frac{\kappa_V}{\kappa_{V, +}}\|u\|_V^2, \quad \forall u\in V.
\end{align}
The constant $\kappa_V\ge 1$ for which $(V, \|\cdot\|_V)$ is $\kappa_V$-regular is called the \emph{constant of regularity} of $(V, \|\cdot\|_V)$.
\end{definition}

\begin{remark}[How we use Definition~\ref{def:kappa_smooth_regular} in this paper]
\label{rem:regularity_dual_convention}
Definition~\ref{def:kappa_smooth_regular} is stated for an arbitrary finite-dimensional normed space $(V, \|\cdot\|_V)$.
In the remainder of the paper, we will apply it primarily to the \emph{dual space} $V=E^*$ equipped with the induced dual norm $\|\cdot\|_V=\snorm{\cdot}$, since stochastic gradients and noise live in $E^*$ and are measured in $\snorm{\cdot}$.
To simplify notation, we write $\kappa \eqdef \kappa_{E^*}$, $\kappa_+ \eqdef \kappa_{E^*, +}$, and denote the auxiliary norm on $E^*$ by $\spnorm{\cdot} \eqdef \|\cdot\|_{E^*, +}$.
With this convention, the compatibility chain \eqref{eq:reg_kappa_def_gen} specializes to
\begin{align}\label{eq:reg_kappa_def}
    \sqsnorm{s} \le \sqspnorm{s} \le \frac{\kappa}{\kappa_+}\sqsnorm{s}, \quad \forall s \in E^*.
\end{align}
\end{remark}

The condition \eqref{eq:reg_kappa_def} does \emph{not} restrict the choice of the update norm $\norm{\cdot}$ (and hence of the dual norm $\snorm{\cdot}$). Indeed, for any nonzero finite-dimensional space $E$, its dual $(E^*, \snorm{\cdot})$ is automatically $\kappa$-regular for some $\kappa\le\dim(E)$ by \Cref{example:1}, and this worst-case linear dependence on the dimension cannot be improved in general (\Cref{prop:tightness_universal_bound}).
For $\ell_p^d$ and Schatten-$p$ norms, however, the regularity parameter can often be much smaller.%
\footnote{We focus on these families because their special cases $p\in\{1, 2, \infty\}$ recover norms commonly used in practice; see \Cref{app:notation}.}
\Cref{sec:reg_examples} collects these bounds and their proofs, with a summary of the auxiliary norms in \Cref{tab:regularity_examples}.

\subsection{Regularity vs.\ Euclidean Norm Equivalence}\label{subsec:regularity_vs_norm_equivalence_main}

A common alternative is to compare the induced dual norm $\snorm{\cdot}$ with a Euclidean reference norm and then apply Euclidean concentration. Regularity offers a more flexible version of this approach: it allows a smooth auxiliary norm that is comparable to $\snorm{\cdot}$ and tailored to the target geometry. This freedom can substantially sharpen the geometric bounds. For dual $\ell_p^d$ spaces with $2\le p\le\infty$, tuning the auxiliary norm yields the optimal order $\kappa=\Theta(\min\{p, \ln(ed)\})$, compared with $\kappa=d^{1-2/p}$ for the Euclidean choice. At $p=\infty$, the gain is striking: linear dimension dependence becomes logarithmic. At $p=2$, both choices give $\kappa=1$; for $1\le p<2$, a scaled Euclidean proxy is already optimal.

These gains require no change to the algorithm: the auxiliary norm is used only in the analysis. For suitable noise models, they also translate into sharper oracle-complexity bounds (\Cref{subsubsec:native_noise_improvement}). Details and the Schatten-$p$ counterpart appear in \Cref{sec:regularity_vs_norm_equivalence}.

\subsection{Key Lemma: Minibatch Concentration in Regular Dual Spaces}
\label{subsec:minibatch_conc}
The following concentration inequality is the technical core of the analysis: it controls the minibatch error directly in the induced dual norm $\snorm{\cdot}$, with the geometry entering only through $\kappa$.

\begin{restatable}[Concentration in Regular Normed Spaces]{lemma}{lightTailConcentration}\label{lemma:non_euclidean_large_deviations}
  Let $\kappa$ be a regularity constant of the dual space $(E^*, \snorm{\cdot})$.
  Assume that $\{\eta^j\}_{j=1}^B$ is a sequence of i.i.d. random vectors in $E^*$ such that $\Exp{\eta^j} = 0$ and they satisfy the light-tail condition:
  $\Exp{\exp\rb{\frac{\sqsnorm{\eta^j}}{\sigma^2}}} \leq \exp(1)$.
  Then for any $\lambda \ge 0$, the average $\bar{\eta} = \frac{1}{B}\sum_{j=1}^B \eta^j$ satisfies the concentration inequality:
  $\Prob\cb{\snorm{\bar{\eta}} \geq \left(\sqrt{2\kappa} + \sqrt{2}\lambda\right) \frac{\sigma}{\sqrt{B}}} \leq \exp(-\nfr{\lambda^2}{3})$.
\end{restatable}

\textbf{Proof idea.}
$\kappa$-regularity provides a smooth auxiliary norm on $E^*$, enabling a martingale concentration argument.
The comparison \eqref{eq:reg_kappa_def} then yields concentration in $\snorm{\cdot}$ without a detour through Euclidean/Frobenius conversions, with geometric dependence captured by $\kappa$.
Further explanation and the proof are given in \Cref{app:proof_key_concentration}.

\begin{table}[t]
\centering
\begin{threeparttable}
\caption{Statistical oracle-cost terms, omitting the factor $\nu^2$ and the common leading term $\widetilde{\cO}(L\Delta/\eps)$. Formulas hold for $d\ge1$. Within each row, {\green green} and {\red red} mark smaller and larger asymptotic orders; {\blue blue} marks agreement. Colours do not imply strict improvement in every finite dimension.}
\label{tab:main_noise_comparison}
\small
\setlength{\tabcolsep}{5pt}
\begin{tabular}{@{}llll@{}}
\toprule
Space $(E^*, \snorm{\cdot})$
& Euclidean/Frobenius\tnote{a}
& Geometry-native, Rademacher\tnote{b}
& Geometry-native, Gaussian\tnote{b} \\
\midrule
$(\R^d, \norm{\cdot}_1)$
& {\green $\widetilde{\cO}\rb{\nfr{d^2L\Delta}{\eps^2}}$}
& {\red $\widetilde{\cO}\rb{\nfr{d^3L\Delta}{\eps^2}}$}
& {\red $\widetilde{\cO}\rb{\nfr{d^3L\Delta}{\eps^2}}$} \\[3pt]
$(\R^d, \norm{\cdot}_2)$
& {\blue $\widetilde{\cO}\rb{\nfr{dL\Delta}{\eps^2}}$}
& {\blue $\widetilde{\cO}\rb{\nfr{dL\Delta}{\eps^2}}$}
& {\blue $\widetilde{\cO}\rb{\nfr{dL\Delta}{\eps^2}}$} \\[3pt]
$(\R^d, \norm{\cdot}_\infty)$
& {\red $\widetilde{\cO}\rb{\nfr{dL\Delta}{\eps^2}}$}
& {\green $\widetilde{\cO}\rb{\nfr{\ln(ed)L\Delta}{\eps^2}}$}
& {\green $\widetilde{\cO}\rb{\nfr{(\ln(ed))^2L\Delta}{\eps^2}}$} \\
\bottomrule
\end{tabular}
\begin{tablenotes}[flushleft]
\footnotesize
\item[a] Euclidean/Frobenius concentration followed by norm conversion; see \eqref{eq:complexity_dual_native_cor_F}.
\item[b] Concentration directly in the dual norm; see \eqref{eq:complexity_dual_native_cor}.
\end{tablenotes}
\end{threeparttable}
\end{table}

\subsection{Main Theorem: High-Probability Oracle Complexity}\label{subsec:main_oracle_complexity}
We now combine the descent analysis under non-Euclidean smoothness with the minibatch concentration lemma.

\begin{restatable}[Minibatch non-Euclidean \texorpdfstring{\algname[]{SGD}}{SGD} under light tails]{theorem}{lightTailConvergence}
\label{thm:stochastic_L_smooth_final}
Let $\kappa$ be a regularity constant of the dual space $(E^*, \snorm{\cdot})$.
Suppose Assumptions~\ref{as:non_euclidean_smoothness}, \ref{as:lower_bound}, and \ref{as:non_euclidean_light_tail} hold.
Run the update \eqref{eq:non-eu-one-step-final} with the constant stepsize $\gamma_k\equiv\frac{1}{2L}$ for
$T = \left\lceil \frac{128L\Delta}{15\eps} \right\rceil$
iterations, with a minibatch size chosen as
    $B = \max\left\{ \left\lceil \frac{16\sigma^2 \rb{\sqrt{2\kappa} + \sqrt{6\log(T/\delta)}}^2}{\eps} \right\rceil, \,1 \right\}.$
Then, with probability at least $1-\delta$, the iterates satisfy
    $\min_{0\le k\le T-1}\sqsnorm{\nfxk}\le \eps.$
\end{restatable}

\Cref{thm:stochastic_L_smooth_final} sets $T$ using the descent argument and chooses $B$ to control the gradient-estimation error simultaneously over all $T$ iterations.
Since each iteration uses $B$ stochastic gradients, the following corollary combines these choices to bound the total oracle cost $BT$.

\begin{restatable}[Norm-native oracle complexity under dual-norm light tails]{corollary}{lightTailComplexity}
\label{cor:complexity_dual_native_light_tail}
Suppose the assumptions of
\Cref{thm:stochastic_L_smooth_final}
hold. Run the update \eqref{eq:non-eu-one-step-final} with the constant stepsize
$\gamma_k\equiv\frac{1}{2L}$, and choose $T$ and $B$ as in that theorem.
Let $N_* \eqdef BT$ denote the total number of stochastic-gradient evaluations.
Then with probability at least
$1-\delta$, the method produces an iterate satisfying
    $\min_{0\le k\le T-1}\sqsnorm{\nabla f(x^k)} \le \eps$
after at most
\begin{align}
\label{eq:complexity_dual_native_cor}
    N_{*}
    =
    \widetilde{\mathcal O}\!\left(\frac{L\Delta}{\eps}+\frac{\kappa L\Delta \sigma^2}{\eps^2}\right)
\end{align}
stochastic-gradient evaluations, where
$\widetilde{\mathcal{O}}(\cdot)$ hides logarithmic factors in $T$ and
$\nfr{1}{\delta}$.
\end{restatable}

\textbf{The rate is optimal in $\eps$ and $\sigma$.}
\citet[Theorem~4.1 and the proof of Proposition~8.9]{jin2026tight}
establishes a lower bound of
$\Omega\!\left(L\Delta/\eps+L\Delta\sigma^2/\eps^2\right)$
oracle calls for general unbiased first-order oracles with uniformly
bounded noise, which satisfy our oracle model and light-tail condition
(\Cref{subsec:stationarity_oracle}).%
\footnote{The lower bound applies to our squared-gradient accuracy
criterion at fixed confidence, in the worst case over dimensions,
for $0<\eps\le c\min\{\sigma^2, L\Delta\}$ with a universal constant $c>0$.}
In the Euclidean case, where $\snorm{\cdot}=\norm{\cdot}_2$
and $\kappa=1$, this matches \eqref{eq:complexity_dual_native_cor},
so our rate is optimal up to logarithmic factors over the general
unbiased light-tail oracle class.
For a general update norm the two are not literally comparable,
since $\sigma$ is then measured in $\snorm{\cdot}$ rather than
in $\norm{\cdot}_2$; what \eqref{eq:complexity_dual_native_cor}
adds in that case is the single geometry factor $\kappa$.

\textbf{On the size of the minibatch.}
\Cref{thm:stochastic_L_smooth_final} uses
$B=\widetilde{\cO}\rb{1+\nfr{\kappa\sigma^2}{\eps}}$, with $B\ge1$.
The Euclidean analyses of \citet{tyurin2023optimal} and
\citet{tyurin2024toward} have the same dependence on
$\sigma^2/\eps$, up to logarithmic factors.%
\footnote{The lower bound of \citet{jin2026tight} constrains
the total cost $BT$, but not its division between iterations
and batch size; it therefore does not establish that our
particular batch size is necessary.}
For exact gradients, we set $B=1$ and retain
$\gamma_k\equiv\frac{1}{2L}$.%
\footnote{This case is treated separately from
\Cref{as:non_euclidean_light_tail}, which assumes $\sigma>0$.}
Deterministic descent then gives $N_*=\mathcal O(L\Delta/\eps)$,
explaining the noise-independent leading term in
\eqref{eq:complexity_dual_native_cor}.

\begin{figure}[t]
\centering
\begin{minipage}[t]{0.485\textwidth}
\begin{algorithm}[H]
\caption{\algname{Rennala NSGD} (server)}
\label{alg:gluon_server}
\small
\begin{algorithmic}[1]
    \STATE \textbf{Input:} initial point $x^0\in E$, stepsizes $\{\gamma_k>0\}$, minibatch size $B$, number of iterations $T$.
    \STATE Start \Cref{alg:gluon_worker} on all $N$ workers.
    \FOR{$k = 0, 1, \dots, T - 1$}
        \STATE Set $\gk = 0$ and $b = 0$.
        \WHILE{$b < B$}
            \STATE Receive the next gradient message $(g, k')$ from any worker.
            \IF{$k' = k$}
                \STATE $\gk \leftarrow \gk + \frac{1}{B}\, g$ and $b \leftarrow b + 1$.
            \ENDIF
            \IF{$b < B$}
                \STATE Send $(\xk, k)$ to that worker.
            \ENDIF
        \ENDWHILE
        \STATE $\xkpo = \xk - \gamma_k \rb{\gk}^{\sharp}$.
        \STATE Send $(\xkpo, k+1)$ to the worker whose gradient completed the minibatch.
    \ENDFOR
\end{algorithmic}
\end{algorithm}
\end{minipage}\hfill
\begin{minipage}[t]{0.485\textwidth}
\begin{algorithm}[H]
\caption{\algname{Rennala NSGD} (worker loop)}
\label{alg:gluon_worker}
\small
\begin{algorithmic}[1]
    \STATE Set $g = 0$ and $k' = -1$.
    \LOOP
        \STATE Send $(g, k')$ to the server.
        \STATE Receive $(\xk, k)$ from the server.
        \STATE Set $k' = k$.
        \STATE Sample $\xi \sim \cD$.
        \STATE Compute $g = \nf{\xk;\xi}$.
    \ENDLOOP
\end{algorithmic}
\end{algorithm}
\end{minipage}
\end{figure}

\subsubsection{Light-Tail Assumption in the Native Norm Can Lead to Substantial Improvement}
\label{subsubsec:native_noise_improvement}

Measuring noise directly in the induced dual norm can substantially sharpen oracle bounds.
Our native guarantee \eqref{eq:complexity_dual_native_cor} uses the statistical coefficient $\mathsf C_*=\kappa\sigma^2$.
For comparison, consider a Euclidean/Frobenius light-tail oracle with scale $\sigma_{\opf}$ and $\alpha>0$ satisfying $\alpha\snorm{z}\le\fnorm{z}$.
Concentrating the minibatch error in that norm and then converting it yields
$N_{\opf}=\widetilde{\cO}\rb{\frac{L\Delta}{\eps}+\frac{L\Delta}{\eps^2}\underbrace{\alpha^{-2}\sigma_{\opf}^2}_{\mathsf C_{\opf}}}$.

The two bounds have the same leading term, so we compare their statistical coefficients for two dense-noise models in dual $\ell_p^d$ geometries.
Both use $\nf{x;\xi}=\nabla f(x)+\nu u$, where $\nu>0$ and each call draws a fresh independent $u$ with independent Rademacher or standard Gaussian entries.%
\footnote{See \Cref{def:rademacher_oracle_alt} and \Cref{def:gaussian_oracle_alt}.}
\Cref{tab:main_noise_comparison} reports the resulting statistical terms.

For these models at dual $p=\infty$, corresponding to the $\ell_1$ update norm, the native coefficient grows as $\ln(ed)$ for Rademacher noise and $(\ln(ed))^2$ for Gaussian noise, while the converted coefficient grows as $d$. At $p=2$ the bounds agree; at $p=1$, the native coefficient has the larger asymptotic order, $d^3$ instead of $d^2$. These conclusions depend on the noise model: $p$ alone does not determine which analysis gives a smaller bound. Full derivations, intermediate $p$ regimes, and the diagonal Schatten counterparts appear in \Cref{sec:frobenius_light_tail_oracle}.

\section{\texorpdfstring{\algname[]{Rennala NSGD}}{Rennala NSGD}: Semi-Asynchronous Parallelism Beyond Euclidean Norms}
\label{sec:rennala_nsgd}

We present \algname{Rennala NSGD}, a semi-asynchronous method combining the minibatch non-Euclidean update \eqref{eq:non-eu-one-step-final} with a modified collection protocol of \algname{Rennala SGD}~\citep{tyurin2023optimal}.%
\footnote{The original Euclidean method is time-optimal up to constants within their zero-respecting algorithm class, under their specified unbiased bounded-variance oracle model with fixed worker computation times.}
A central server coordinates $N$ workers with heterogeneous computation speeds, using only gradients computed at the current iterate and discarding stale ones (\Cref{alg:gluon_server} and \Cref{alg:gluon_worker}).

\textbf{Computation model.}
We use the fixed computation model of \citet{tyurin2023optimal}.
All $N$ workers draw fresh independent samples from the same distribution $\cD$, and worker $i$ takes a fixed time $\tau_i$ per stochastic gradient, with $0<\tau_1\le\cdots\le\tau_N$.
Communication and server processing are instantaneous.
The server processes messages one at a time, including any update and reply, and handles simultaneous returns in a fixed worker order.

\textbf{Fresh gradients.}
Each gradient carries its query index $k'$.
At iteration $k$, the server accepts only fresh gradients, with $k'=k$, and discards stale ones, with $k'<k$.
Conditional on the history, accepted gradients are i.i.d.\ oracle responses at $\xk$, so the oracle model of \Cref{subsec:stationarity_oracle} and the serial convergence analysis apply.

\textbf{Algorithm: \algname[]{Rennala NSGD}.}
\Cref{alg:gluon_server} and \Cref{alg:gluon_worker} give the pseudocode.
At iteration $k$, the server collects the first $B$ fresh gradients, possibly several from the same worker, averages them as in \eqref{eq:mini_gk_def}, and applies \eqref{eq:non-eu-one-step-final}.
After each message, the server replies with the current iterate and its index.
If the message completes the minibatch, it first updates and then sends $(\xkpo, k+1)$ to that worker; otherwise, it sends $(\xk, k)$.%
\footnote{This changes the reply-before-update order in Method~4 of \citet{tyurin2023optimal}. Details and the proof that the wall-clock bound is preserved appear in \Cref{sec:non_eu_updaten}.}
Other workers finish their pending computations and receive the latest iterate when their results are processed.

\subsection{Main Theorem: Wall-Clock Time Complexity}\label{subsec:time_complexity}

The following theorem converts the high-probability oracle guarantee of \Cref{thm:stochastic_L_smooth_final} into a wall-clock guarantee under heterogeneous worker speeds.

\begin{samepage}
\begin{restatable}[Wall-clock time complexity of \texorpdfstring{\algname[]{Rennala NSGD}}{Rennala NSGD}]{theorem}{lightTailWallClock}
\label{thm:main_time_complexity}
Suppose the assumptions of \Cref{thm:stochastic_L_smooth_final} hold, and use the fixed computation model of \Cref{sec:rennala_nsgd}.
Run \algname{Rennala NSGD} with the stepsize, iteration count $T$, and minibatch size $B$ prescribed in that theorem.
For $0<\eps\le L\Delta$ and $\delta\in(0, 1)$, the iterates satisfy $\min_{0\le k<T}\sqsnorm{\nfxk}\le\eps$ with probability at least $1-\delta$ within total wall-clock time
\begin{align}\label{eq:time_simplified_bound}
    T_{\mathrm{total}} = \widetilde{\cO}\left(\min_{S\in[N]}\left[\left(\frac{1}{S}\sum_{i=1}^S\frac{1}{\tau_i}\right)^{-1}\left(\frac{L\Delta}{\eps}+\frac{\kappa L\Delta\sigma^2}{S\eps^2}\right)\right]\right).
\end{align}
\end{restatable}
\end{samepage}

The minimum in \eqref{eq:time_simplified_bound} balances the harmonic-mean time of the $S$ fastest workers against a statistical term reduced by $S$.
Neither the worker times nor the minimizing $S$ are needed to run the algorithm.
The update norm affects $L$, $\sigma$, $\kappa$, and the stationarity criterion, so bounds for different update norms certify different notions of accuracy.

For a fixed update norm, accuracy, and worker times, the native-versus-converted comparison depends on $\kappa\sigma^2$ and $\alpha^{-2}\sigma_{\opf}^2$.
The smaller coefficient gives a wall-clock bound that is no larger and can be significantly smaller.

In the Euclidean case, $\kappa=1$ recovers the parameter dependence of classical \algname{Rennala SGD} up to logarithmic factors~\citep{tyurin2023optimal}.
Our guarantee holds with high probability under light tails rather than in expectation under bounded variance; this agreement does not establish optimality in our oracle model or for general update norms.
See \Cref{sec:non_eu_updaten} for the proof and \Cref{app:rennala_time_discussion} for the full discussion.

\section{Conclusion}\label{sec:conclusion}
We analyzed minibatch \algname{SGD} in finite-dimensional normed spaces under dual-norm light-tailed noise. Concentration in regular dual spaces yields high-probability oracle-complexity bounds with geometric dependence captured by $\kappa$ (\Cref{lemma:non_euclidean_large_deviations} and \Cref{cor:complexity_dual_native_light_tail}). We established matching upper and lower regularity bounds for $\ell_p$ and Schatten-$p$ spaces (\Cref{tab:regularity_examples}) and compared native and Euclidean/Frobenius noise analyses. For example, dense Rademacher noise with the $\ell_1$ update norm gives logarithmic rather than linear dimension dependence in the statistical coefficient (\Cref{prop:frobenius_vs_dual_lp_rademacher_alt}).

Building on these results, \algname{Rennala NSGD} combines non-Euclidean updates with fresh-minibatch collection, yielding a high-probability wall-clock bound under fixed, heterogeneous worker computation times (\Cref{thm:main_time_complexity}). In the Euclidean case, our oracle and time bounds match the parameter dependence of standard minibatch \algname{SGD} and classical \algname{Rennala SGD}, respectively, up to logarithmic factors. Together, these results connect update geometry and stochastic noise to parallel optimization time. \Cref{sec:limitations} discusses the limitations of our assumptions and computation model.

\newpage

\begin{ack}
The research reported in this publication was supported by funding from King Abdullah University of Science and Technology (KAUST): i) KAUST Baseline Research Scheme, ii) CRG Grant ORFS-CRG12-2024-6460, and iii) Center of Excellence for Generative AI, under award number 5940.

The authors declare no competing interests.
\end{ack}

\bibliography{bib}

\begin{thebibliography}{91}
\providecommand{\natexlab}[1]{#1}
\providecommand{\url}[1]{\texttt{#1}}
\expandafter\ifx\csname urlstyle\endcsname\relax
  \providecommand{\doi}[1]{doi: #1}\else
  \providecommand{\doi}{doi: \begingroup \urlstyle{rm}\Url}\fi

\bibitem[Ahn et~al.(2025)Ahn, Xu, Abreu, and Langford]{ahn2025dion}
Kwangjun Ahn, Byron Xu, Natalie Abreu, and John Langford.
\newblock {D}ion: Distributed orthonormalized updates.
\newblock \emph{arXiv preprint arXiv:2504.05295}, 2025.

\bibitem[Ajalloeian and Stich(2020)]{ajalloeian2020convergence}
Ahmad Ajalloeian and Sebastian~U Stich.
\newblock On the convergence of {SGD} with biased gradients.
\newblock \emph{arXiv preprint arXiv:2008.00051}, 2020.

\bibitem[Alistarh et~al.(2017)Alistarh, Grubic, Li, Tomioka, and
  Vojnovic]{alistarh2017qsgd}
Dan Alistarh, Demjan Grubic, Jerry Li, Ryota Tomioka, and Milan Vojnovic.
\newblock {QSGD}: Communication-efficient {SGD} via gradient quantization and
  encoding.
\newblock \emph{Advances in neural information processing systems}, 30, 2017.

\bibitem[Alistarh et~al.(2018)Alistarh, Hoefler, Johansson, Konstantinov,
  Khirirat, and Renggli]{alistarh2018convergence}
Dan Alistarh, Torsten Hoefler, Mikael Johansson, Nikola Konstantinov, Sarit
  Khirirat, and C{\'e}dric Renggli.
\newblock The convergence of sparsified gradient methods.
\newblock \emph{Advances in Neural Information Processing Systems}, 31, 2018.

\bibitem[Amsel et~al.(2025)Amsel, Persson, Musco, and Gower]{amsel2025polar}
Noah Amsel, David Persson, Christopher Musco, and Robert~M Gower.
\newblock The polar express: Optimal matrix sign methods and their application
  to the {M}uon algorithm.
\newblock \emph{arXiv preprint arXiv:2505.16932}, 2025.

\bibitem[An et~al.(2025)An, Liu, Pan, Ren, Ma, Goldfarb, and Zhang]{an2025asgo}
Kang An, Yuxing Liu, Rui Pan, Yi~Ren, Shiqian Ma, Donald Goldfarb, and Tong
  Zhang.
\newblock {ASGO}: Adaptive structured gradient optimization.
\newblock \emph{arXiv preprint arXiv:2503.20762}, 2025.

\bibitem[Arjevani et~al.(2020)Arjevani, Shamir, and Srebro]{arjevani2020tight}
Yossi Arjevani, Ohad Shamir, and Nathan Srebro.
\newblock A tight convergence analysis for stochastic gradient descent with
  delayed updates.
\newblock In \emph{Algorithmic Learning Theory}, pages 111--132. PMLR, 2020.

\bibitem[Beck(2017)]{beck2017first}
Amir Beck.
\newblock \emph{First-order methods in optimization}.
\newblock SIAM, 2017.

\bibitem[Bernstein and Newhouse(2024)]{bernstein2024old}
Jeremy Bernstein and Laker Newhouse.
\newblock Old optimizer, new norm: An anthology.
\newblock \emph{arXiv preprint arXiv:2409.20325}, 2024.

\bibitem[Bernstein et~al.(2018)Bernstein, Wang, Azizzadenesheli, and
  Anandkumar]{bernstein2018signsgd}
Jeremy Bernstein, Yu-Xiang Wang, Kamyar Azizzadenesheli, and Animashree
  Anandkumar.
\newblock sign{SGD}: Compressed optimisation for non-convex problems.
\newblock In \emph{International conference on machine learning}, pages
  560--569. PMLR, 2018.

\bibitem[Beznosikov et~al.(2023)Beznosikov, Horv{\'a}th, Richt{\'a}rik, and
  Safaryan]{beznosikov2023biased}
Aleksandr Beznosikov, Samuel Horv{\'a}th, Peter Richt{\'a}rik, and Mher
  Safaryan.
\newblock On biased compression for distributed learning.
\newblock \emph{Journal of Machine Learning Research}, 24\penalty0
  (276):\penalty0 1--50, 2023.

\bibitem[Chen et~al.(2016)Chen, Pan, Monga, Bengio, and
  Jozefowicz]{chen2016revisiting}
Jianmin Chen, Xinghao Pan, Rajat Monga, Samy Bengio, and Rafal Jozefowicz.
\newblock Revisiting distributed synchronous {SGD}.
\newblock \emph{arXiv preprint arXiv:1604.00981}, 2016.

\bibitem[Cohen et~al.(2021)Cohen, Daniely, Drori, Koren, and
  Schain]{cohen2021asynchronous}
Alon Cohen, Amit Daniely, Yoel Drori, Tomer Koren, and Mariano Schain.
\newblock Asynchronous stochastic optimization robust to arbitrary delays.
\newblock \emph{Advances in Neural Information Processing Systems},
  34:\penalty0 9024--9035, 2021.

\bibitem[Cotter et~al.(2011)Cotter, Shamir, Srebro, and
  Sridharan]{cotter2011better}
Andrew Cotter, Ohad Shamir, Nati Srebro, and Karthik Sridharan.
\newblock Better mini-batch algorithms via accelerated gradient methods.
\newblock \emph{Advances in neural information processing systems}, 24, 2011.

\bibitem[Demidovich et~al.(2023)Demidovich, Malinovsky, Sokolov, and
  Richt{\'a}rik]{demidovich2023guide}
Yury Demidovich, Grigory Malinovsky, Igor Sokolov, and Peter Richt{\'a}rik.
\newblock A guide through the zoo of biased {SGD}.
\newblock \emph{Advances in Neural Information Processing Systems},
  36:\penalty0 23158--23171, 2023.

\bibitem[Demidovich et~al.(2024)Demidovich, Ostroukhov, Malinovsky, Horváth,
  Takáč, Richtárik, and Gorbunov]{demidovich2024methodslocalstepsrandom}
Yury Demidovich, Petr Ostroukhov, Grigory Malinovsky, Samuel Horváth, Martin
  Takáč, Peter Richtárik, and Eduard Gorbunov.
\newblock Methods with local steps and random reshuffling for generally smooth
  non-convex federated optimization, 2024.
\newblock URL \url{https://arxiv.org/abs/2412.02781}.

\bibitem[Devolder et~al.(2014)Devolder, Glineur, and
  Nesterov]{devolder2014first}
Olivier Devolder, Fran{\c{c}}ois Glineur, and Yurii Nesterov.
\newblock First-order methods of smooth convex optimization with inexact
  oracle.
\newblock \emph{Mathematical Programming}, 146\penalty0 (1):\penalty0 37--75,
  2014.

\bibitem[Duchi et~al.(2013)Duchi, Jordan, and McMahan]{duchi2013estimation}
John Duchi, Michael~I Jordan, and Brendan McMahan.
\newblock Estimation, optimization, and parallelism when data is sparse.
\newblock \emph{Advances in Neural Information Processing Systems}, 26, 2013.

\bibitem[D{\"u}mbgen et~al.(2010)D{\"u}mbgen, van~de Geer, Veraar, and
  Wellner]{dumbgen2010nemirovski}
Lutz D{\"u}mbgen, Sara~A van~de Geer, Mark~C Veraar, and Jon~A Wellner.
\newblock Nemirovski's inequalities revisited.
\newblock \emph{The American Mathematical Monthly}, 117\penalty0 (2):\penalty0
  138--160, 2010.

\bibitem[Dutta et~al.(2018)Dutta, Joshi, Ghosh, Dube, and
  Nagpurkar]{dutta2018slow}
Sanghamitra Dutta, Gauri Joshi, Soumyadip Ghosh, Parijat Dube, and Priya
  Nagpurkar.
\newblock Slow and stale gradients can win the race: Error-runtime trade-offs
  in distributed {SGD}.
\newblock In \emph{International conference on artificial intelligence and
  statistics}, pages 803--812. PMLR, 2018.

\bibitem[Fatkhullin et~al.(2021)Fatkhullin, Sokolov, Gorbunov, Li, and
  Richt{\'a}rik]{fatkhullin2021ef21}
Ilyas Fatkhullin, Igor Sokolov, Eduard Gorbunov, Zhize Li, and Peter
  Richt{\'a}rik.
\newblock {EF21} with bells \& whistles: Practical algorithmic extensions of
  modern error feedback.
\newblock \emph{arXiv preprint arXiv:2110.03294}, 2021.

\bibitem[Feyzmahdavian et~al.(2016)Feyzmahdavian, Aytekin, and
  Johansson]{feyzmahdavian2016asynchronous}
Hamid~Reza Feyzmahdavian, Arda Aytekin, and Mikael Johansson.
\newblock An asynchronous mini-batch algorithm for regularized stochastic
  optimization.
\newblock \emph{IEEE Transactions on Automatic Control}, 61\penalty0
  (12):\penalty0 3740--3754, 2016.

\bibitem[Gasnikov(2017)]{gasnikov2017universal}
Alexander Gasnikov.
\newblock Universal gradient descent.
\newblock \emph{arXiv preprint arXiv:1711.00394}, 2017.

\bibitem[Ghadimi and Lan(2013)]{ghadimi2013stochastic}
Saeed Ghadimi and Guanghui Lan.
\newblock Stochastic first-and zeroth-order methods for nonconvex stochastic
  programming.
\newblock \emph{SIAM journal on optimization}, 23\penalty0 (4):\penalty0
  2341--2368, 2013.

\bibitem[Gorbunov et~al.(2020)Gorbunov, Hanzely, and
  Richt{\'a}rik]{gorbunov2020unified}
Eduard Gorbunov, Filip Hanzely, and Peter Richt{\'a}rik.
\newblock A unified theory of {SGD}: Variance reduction, sampling, quantization
  and coordinate descent.
\newblock In \emph{International Conference on Artificial Intelligence and
  Statistics}, pages 680--690. PMLR, 2020.

\bibitem[Gower et~al.(2019)Gower, Loizou, Qian, Sailanbayev, Shulgin, and
  Richt{\'a}rik]{gower2019sgd}
Robert~Mansel Gower, Nicolas Loizou, Xun Qian, Alibek Sailanbayev, Egor
  Shulgin, and Peter Richt{\'a}rik.
\newblock {SGD}: General analysis and improved rates.
\newblock In Kamalika Chaudhuri and Ruslan Salakhutdinov, editors,
  \emph{Proceedings of the 36th International Conference on Machine Learning},
  volume~97 of \emph{Proceedings of Machine Learning Research}, pages
  5200--5209. PMLR, 09--15 Jun 2019.
\newblock URL \url{https://proceedings.mlr.press/v97/qian19b.html}.

\bibitem[Goyal et~al.(2017)Goyal, Doll{\'a}r, Girshick, Noordhuis, Wesolowski,
  Kyrola, Tulloch, Jia, and He]{goyal2017accurate}
Priya Goyal, Piotr Doll{\'a}r, Ross Girshick, Pieter Noordhuis, Lukasz
  Wesolowski, Aapo Kyrola, Andrew Tulloch, Yangqing Jia, and Kaiming He.
\newblock Accurate, large minibatch {SGD}: Training {ImageNet} in 1 hour.
\newblock \emph{arXiv preprint arXiv:1706.02677}, 2017.

\bibitem[Grishina et~al.(2025)Grishina, Smirnov, and
  Rakhuba]{grishina2025accelerating}
Ekaterina Grishina, Matvey Smirnov, and Maxim Rakhuba.
\newblock Accelerating {N}ewton-{S}chulz iteration for orthogonalization via
  {C}hebyshev-type polynomials.
\newblock \emph{arXiv preprint arXiv:2506.10935}, 2025.

\bibitem[Gruntkowska et~al.(2023)Gruntkowska, Tyurin, and
  Richt{\'a}rik]{EF21-P}
Kaja Gruntkowska, Alexander Tyurin, and Peter Richt{\'a}rik.
\newblock {EF21-P} and friends: Improved theoretical communication complexity
  for distributed optimization with bidirectional compression.
\newblock In \emph{International Conference on Machine Learning}, pages
  11761--11807. PMLR, 2023.

\bibitem[Gruntkowska et~al.(2025)Gruntkowska, Gaponov, Tovmasyan, and
  Richt{\'a}rik]{gruntkowska2025error}
Kaja Gruntkowska, Alexander Gaponov, Zhirayr Tovmasyan, and Peter
  Richt{\'a}rik.
\newblock Error feedback for {M}uon and friends.
\newblock \emph{arXiv preprint arXiv:2510.00643}, 2025.

\bibitem[Guigues et~al.(2017)Guigues, Juditsky, and
  Nemirovski]{guigues2017multistep}
Vincent Guigues, Anatoli Juditsky, and Arkadi Nemirovski.
\newblock Non-asymptotic confidence bounds for the optimal value of a
  stochastic program.
\newblock \emph{Optimization Methods and Software}, 32\penalty0 (5):\penalty0
  1033--1058, 2017.

\bibitem[Gupta et~al.(2018)Gupta, Koren, and Singer]{gupta2018shampoo}
Vineet Gupta, Tomer Koren, and Yoram Singer.
\newblock {S}hampoo: Preconditioned stochastic tensor optimization.
\newblock In \emph{International Conference on Machine Learning}, pages
  1842--1850. PMLR, 2018.

\bibitem[Horv{\'a}th et~al.(2022)Horv{\'a}th, Ho, Horv{\'a}th, Sahu, Canini,
  and Richt{\'a}rik]{horvath2022natural}
Samuel Horv{\'a}th, Chen-Yu Ho, Ludovit Horv{\'a}th, Atal~Narayan Sahu, Marco
  Canini, and Peter Richt{\'a}rik.
\newblock Natural compression for distributed deep learning.
\newblock In \emph{Mathematical and Scientific Machine Learning}, pages
  129--141. PMLR, 2022.

\bibitem[Hu et~al.(2021)Hu, Shen, Wallis, Allen-Zhu, Li, Wang, Wang, and
  Chen]{hu2021loralowrankadaptationlarge}
Edward~J. Hu, Yelong Shen, Phillip Wallis, Zeyuan Allen-Zhu, Yuanzhi Li, Shean
  Wang, Lu~Wang, and Weizhu Chen.
\newblock {LoRA}: Low-rank adaptation of large language models, 2021.
\newblock URL \url{https://arxiv.org/abs/2106.09685}.

\bibitem[Huang et~al.(2022)Huang, Chen, Yin, and Yuan]{huang2022lower}
Xinmeng Huang, Yiming Chen, Wotao Yin, and Kun Yuan.
\newblock Lower bounds and nearly optimal algorithms in distributed learning
  with communication compression.
\newblock \emph{Advances in Neural Information Processing Systems},
  35:\penalty0 18955--18969, 2022.

\bibitem[Jin(2026)]{jin2026tight}
Jikai Jin.
\newblock A tight lower bound for smooth nonconvex stochastic optimization with
  bounded gradient noise.
\newblock \emph{arXiv preprint arXiv:2608.09004}, 2026.

\bibitem[John(1948)]{john1948extremum}
Fritz John.
\newblock Extremum problems with inequalities as subsidiary conditions, studies
  and essays presented to r. {C}ourant on his 60th birthday, {J}anuary 8, 1948,
  1948.

\bibitem[Jordan et~al.(2024)Jordan, Jin, Boza, You, Cesista, Newhouse, and
  Bernstein]{jordan2024muon}
Keller Jordan, Yuchen Jin, Vlado Boza, Jiacheng You, Franz Cesista, Laker
  Newhouse, and Jeremy Bernstein.
\newblock {M}uon: An optimizer for hidden layers in neural networks, 2024.
\newblock URL \url{https://kellerjordan.github.io/posts/muon/}.

\bibitem[Joulani et~al.(2019)Joulani, Gy{\"o}rgy, and
  Szepesv{\'a}ri]{joulani2019think}
Pooria Joulani, Andr{\'a}s Gy{\"o}rgy, and Csaba Szepesv{\'a}ri.
\newblock Think out of the" box": Generically-constrained asynchronous
  composite optimization and hedging.
\newblock \emph{Advances in Neural Information Processing Systems}, 32, 2019.

\bibitem[Juditsky and Nemirovski(2008)]{juditsky2008large}
Anatoli Juditsky and Arkadii~S Nemirovski.
\newblock Large deviations of vector-valued martingales in 2-smooth normed
  spaces.
\newblock \emph{arXiv preprint arXiv:0809.0813}, 2008.

\bibitem[Juditsky et~al.(2013)Juditsky, K{\i}l{\i}n{\c{c}}~Karzan, and
  Nemirovski]{juditsky2012randomized}
Anatoli Juditsky, Fatma K{\i}l{\i}n{\c{c}}~Karzan, and Arkadi Nemirovski.
\newblock Randomized first order algorithms with applications to
  $\ell_1$-minimization.
\newblock \emph{Mathematical Programming}, 142\penalty0 (1):\penalty0 269--310,
  2013.

\bibitem[Kairouz et~al.(2021)Kairouz, McMahan, Avent, Bellet, Bennis, Bhagoji,
  Bonawitz, Charles, Cormode, Cummings, et~al.]{kairouz2021advances}
Peter Kairouz, H~Brendan McMahan, Brendan Avent, Aur{\'e}lien Bellet, Mehdi
  Bennis, Arjun~Nitin Bhagoji, Kallista Bonawitz, Zachary Charles, Graham
  Cormode, Rachel Cummings, et~al.
\newblock Advances and open problems in federated learning.
\newblock \emph{Foundations and trends{\textregistered} in machine learning},
  14\penalty0 (1--2):\penalty0 1--210, 2021.

\bibitem[Karimireddy et~al.(2019)Karimireddy, Rebjock, Stich, and
  Jaggi]{karimireddy2019error}
Sai~Praneeth Karimireddy, Quentin Rebjock, Sebastian Stich, and Martin Jaggi.
\newblock Error feedback fixes sign{SGD} and other gradient compression
  schemes.
\newblock In \emph{International Conference on Machine Learning}, pages
  3252--3261. PMLR, 2019.

\bibitem[Karimireddy et~al.(2020)Karimireddy, Kale, Mohri, Reddi, Stich, and
  Suresh]{karimireddy2020scaffold}
Sai~Praneeth Karimireddy, Satyen Kale, Mehryar Mohri, Sashank Reddi, Sebastian
  Stich, and Ananda~Theertha Suresh.
\newblock {SCAFFOLD}: Stochastic controlled averaging for federated learning.
\newblock In \emph{International conference on machine learning}, pages
  5132--5143. PMLR, 2020.

\bibitem[Kim and Oh(2026)]{kim2026convergence}
Gyu~Yeol Kim and Min-hwan Oh.
\newblock Convergence of {M}uon with {N}ewton-{S}chulz.
\newblock \emph{arXiv preprint arXiv:2601.19156}, 2026.

\bibitem[Koloskova et~al.(2020)Koloskova, Loizou, Boreiri, Jaggi, and
  Stich]{koloskova2020unified}
Anastasia Koloskova, Nicolas Loizou, Sadra Boreiri, Martin Jaggi, and Sebastian
  Stich.
\newblock A unified theory of decentralized {SGD} with changing topology and
  local updates.
\newblock In \emph{International Conference on Machine Learning}, pages
  5381--5393. PMLR, 2020.

\bibitem[Koloskova et~al.(2022)Koloskova, Stich, and
  Jaggi]{koloskova2022sharper}
Anastasiia Koloskova, Sebastian~U Stich, and Martin Jaggi.
\newblock Sharper convergence guarantees for asynchronous {SGD} for distributed
  and federated learning.
\newblock \emph{Advances in Neural Information Processing Systems},
  35:\penalty0 17202--17215, 2022.

\bibitem[Kone\v{c}n{\'y} et~al.(2016)Kone\v{c}n{\'y}, McMahan, Yu,
  Richt{\'a}rik, Suresh, and Bacon]{konevcny2016federated}
Jakub Kone\v{c}n{\'y}, H.~B. McMahan, Felix~X. Yu, Peter Richt{\'a}rik,
  Ananda~Theertha Suresh, and Dave Bacon.
\newblock Federated learning: Strategies for improving communication
  efficiency.
\newblock \emph{ArXiv}, abs/1610.05492, 2016.

\bibitem[Kovalev(2025)]{kovalev2025understanding}
Dmitry Kovalev.
\newblock Understanding gradient orthogonalization for deep learning via
  non-{E}uclidean trust-region optimization.
\newblock \emph{arXiv preprint arXiv:2503.12645}, 2025.

\bibitem[Kovalev and Borodich(2025)]{kovalevborodich2025}
Dmitry Kovalev and Ekaterina Borodich.
\newblock Non-{E}uclidean {SGD} for structured optimization: Unified analysis
  and improved rates.
\newblock \emph{arXiv preprint arXiv:2511.11466}, 2025.

\bibitem[Lan(2020)]{lan2020first}
Guanghui Lan.
\newblock \emph{First-order and stochastic optimization methods for machine
  learning}, volume~1.
\newblock Springer, 2020.

\bibitem[Lau et~al.(2025)Lau, Long, and Su]{lau2025polargrad}
Tim Tsz-Kit Lau, Qi~Long, and Weijie Su.
\newblock {PolarGrad}: A class of matrix-gradient optimizers from a unifying
  preconditioning perspective.
\newblock \emph{arXiv preprint arXiv:2505.21799}, 2025.

\bibitem[Li and Hong(2025)]{li2025note}
Jiaxiang Li and Mingyi Hong.
\newblock A note on the convergence of {M}uon.
\newblock \emph{arXiv preprint arXiv:2502.02900}, 2025.

\bibitem[Liu et~al.(2025)Liu, Su, Yao, Jiang, Lai, Du, Qin, Xu, Lu, Yan,
  et~al.]{liu2025muon}
Jingyuan Liu, Jianlin Su, Xingcheng Yao, Zhejun Jiang, Guokun Lai, Yulun Du,
  Yidao Qin, Weixin Xu, Enzhe Lu, Junjie Yan, et~al.
\newblock {M}uon is scalable for {LLM} training.
\newblock \emph{arXiv preprint arXiv:2502.16982}, 2025.

\bibitem[Malinovsky et~al.(2022)Malinovsky, Yi, and
  Richt{\'a}rik]{malinovsky2022variance}
Grigory Malinovsky, Kai Yi, and Peter Richt{\'a}rik.
\newblock Variance reduced {ProxSkip}: Algorithm, theory and application to
  federated learning.
\newblock In Alice~H. Oh, Alekh Agarwal, Danielle Belgrave, and Kyunghyun Cho,
  editors, \emph{Advances in Neural Information Processing Systems}, 2022.
\newblock URL \url{https://openreview.net/forum?id=edkno3SvKo}.

\bibitem[McMahan et~al.(2017)McMahan, Moore, Ramage, Hampson, and
  y~Arcas]{mcmahan2017communication}
Brendan McMahan, Eider Moore, Daniel Ramage, Seth Hampson, and Blaise~Aguera
  y~Arcas.
\newblock Communication-efficient learning of deep networks from decentralized
  data.
\newblock In \emph{Artificial intelligence and statistics}, pages 1273--1282.
  Pmlr, 2017.

\bibitem[Mishchenko et~al.(2022)Mishchenko, Bach, Even, and
  Woodworth]{mishchenko2022asynchronous}
Konstantin Mishchenko, Francis Bach, Mathieu Even, and Blake~E Woodworth.
\newblock Asynchronous {SGD} beats minibatch {SGD} under arbitrary delays.
\newblock \emph{Advances in Neural Information Processing Systems},
  35:\penalty0 420--433, 2022.

\bibitem[Nagashima and Iiduka(2026)]{nagashima2026improved}
Shuntaro Nagashima and Hideaki Iiduka.
\newblock Improved convergence rates of {M}uon optimizer for nonconvex
  optimization.
\newblock \emph{arXiv preprint arXiv:2601.19400}, 2026.

\bibitem[Nguyen et~al.(2018)Nguyen, Nguyen, van Dijk, Richt{\'a}rik,
  Scheinberg, and Tak{\'a}c]{nguyen2018sgd}
Lam Nguyen, Phuong~Ha Nguyen, Marten van Dijk, Peter Richt{\'a}rik, Katya
  Scheinberg, and Martin Tak{\'a}c.
\newblock {SGD and Hogwild!} convergence without the bounded gradients
  assumption.
\newblock In \emph{International Conference on Machine Learning}, pages
  3750--3758. PMLR, 2018.

\bibitem[Nutini et~al.(2015)Nutini, Schmidt, Laradji, Friedlander, and
  Koepke]{nutini2015coordinate}
Julie Nutini, Mark Schmidt, Issam Laradji, Michael Friedlander, and Hoyt
  Koepke.
\newblock Coordinate descent converges faster with the {Gauss--Southwell} rule
  than random selection.
\newblock In \emph{International Conference on Machine Learning}, pages
  1632--1641. PMLR, 2015.

\bibitem[Recht et~al.(2011)Recht, Re, Wright, and Niu]{recht2011hogwild}
Benjamin Recht, Christopher Re, Stephen Wright, and Feng Niu.
\newblock {H}ogwild: A lock-free approach to parallelizing stochastic gradient
  descent.
\newblock In \emph{Advances in Neural Information Processing Systems}, 2011.

\bibitem[Refael et~al.(2025)Refael, Smorodinsky, Tirer, and
  Lindenbaum]{refael2025sumo}
Yehonathan Refael, Guy Smorodinsky, Tom Tirer, and Ofir Lindenbaum.
\newblock {SUMO}: Subspace-aware moment-orthogonalization for accelerating
  memory-efficient {LLM} training.
\newblock \emph{arXiv preprint arXiv:2505.24749}, 2025.

\bibitem[Richt{\'a}rik et~al.(2021)Richt{\'a}rik, Sokolov, and
  Fatkhullin]{richtarik2021ef21}
Peter Richt{\'a}rik, Igor Sokolov, and Ilyas Fatkhullin.
\newblock {EF21}: A new, simpler, theoretically better, and practically faster
  error feedback.
\newblock \emph{Advances in Neural Information Processing Systems},
  34:\penalty0 4384--4396, 2021.

\bibitem[Richt{\'a}rik et~al.(2022)Richt{\'a}rik, Sokolov, Gasanov, Fatkhullin,
  Li, and Gorbunov]{richtarik20223pc}
Peter Richt{\'a}rik, Igor Sokolov, Elnur Gasanov, Ilyas Fatkhullin, Zhize Li,
  and Eduard Gorbunov.
\newblock {3PC}: Three point compressors for communication-efficient
  distributed training and a better theory for lazy aggregation.
\newblock In \emph{International Conference on Machine Learning}, pages
  18596--18648. PMLR, 2022.

\bibitem[Sadiev et~al.(2022)Sadiev, Malinovsky, Gorbunov, Sokolov, Khaled,
  Burlachenko, and Richt{\'a}rik]{sadiev2022federated}
Abdurakhmon Sadiev, Grigory Malinovsky, Eduard Gorbunov, Igor Sokolov, Ahmed
  Khaled, Konstantin Burlachenko, and Peter Richt{\'a}rik.
\newblock Federated optimization algorithms with random reshuffling and
  gradient compression.
\newblock \emph{arXiv preprint arXiv:2206.07021}, 2022.

\bibitem[Sadiev et~al.(2024)Sadiev, Malinovsky, Gorbunov, Sokolov, Khaled,
  Burlachenko, and Richt{\'a}rik]{sadiev2024dont}
Abdurakhmon Sadiev, Grigory Malinovsky, Eduard Gorbunov, Igor Sokolov, Ahmed
  Khaled, Konstantin~Pavlovich Burlachenko, and Peter Richt{\'a}rik.
\newblock Don't compress gradients in random reshuffling: Compress gradient
  differences.
\newblock In \emph{The Thirty-eighth Annual Conference on Neural Information
  Processing Systems}, 2024.
\newblock URL \url{https://openreview.net/forum?id=CzPtBzgfae}.

\bibitem[Sadiev et~al.(2025)Sadiev, Demidovich, Sokolov, Malinovsky, Khirirat,
  and Richt{\'a}rik]{sadiev2025improved}
Abdurakhmon Sadiev, Yury Demidovich, Igor Sokolov, Grigory Malinovsky, Sarit
  Khirirat, and Peter Richt{\'a}rik.
\newblock Improved convergence in parameter-agnostic error feedback through
  momentum.
\newblock \emph{arXiv preprint arXiv:2511.14501}, 2025.

\bibitem[Sadiev et~al.(2026)Sadiev, Maranjyan, Ilin, and
  Richt{\'a}rik]{sadiev2026ringmaster}
Abdurakhmon Sadiev, Artavazd Maranjyan, Ivan Ilin, and Peter Richt{\'a}rik.
\newblock {R}ingmaster {LMO}: Asynchronous linear minimization oracle momentum
  method.
\newblock \emph{arXiv preprint arXiv:2605.18174}, 2026.

\bibitem[Seide et~al.(2014)Seide, Fu, Droppo, Li, and Yu]{Seide2014}
Frank Seide, Hao Fu, Jasha Droppo, Gang Li, and Dong Yu.
\newblock 1-bit stochastic gradient descent and its application to
  data-parallel distributed training of speech {DNN}s.
\newblock In \emph{Fifteenth Annual Conference of the International Speech
  Communication Association}, 2014.

\bibitem[Shah et~al.(2025)Shah, Polloreno, Stratos, Monk, Chaluvaraju, Hojel,
  Ma, Thomas, Tanwer, Shah, et~al.]{shah2025practical}
Ishaan Shah, Anthony~M Polloreno, Karl Stratos, Philip Monk, Adarsh
  Chaluvaraju, Andrew Hojel, Andrew Ma, Anil Thomas, Ashish Tanwer, Darsh~J
  Shah, et~al.
\newblock Practical efficiency of {M}uon for pretraining.
\newblock \emph{arXiv preprint arXiv:2505.02222}, 2025.

\bibitem[Shen et~al.(2025)Shen, Huang, Huang, Shen, and
  Zhang]{shen2025convergence}
Wei Shen, Ruichuan Huang, Minhui Huang, Cong Shen, and Jiawei Zhang.
\newblock On the convergence analysis of {M}uon.
\newblock \emph{arXiv preprint arXiv:2505.23737}, 2025.

\bibitem[Sokolov(2022)]{sokolov2022non}
Igor Sokolov.
\newblock Non-convex stochastic optimization with biased gradient estimators.
\newblock {MS} thesis, King Abdullah University of Science and Technology,
  2022.
\newblock URL \url{https://doi.org/10.25781/KAUST-5EAM7}.

\bibitem[Sokolov and
  Richtárik(2024)]{sokolov2024marinapsuperiorperformancenonsmooth}
Igor Sokolov and Peter Richtárik.
\newblock {MARINA}-p: Superior performance in non-smooth federated optimization
  with adaptive stepsizes, 2024.
\newblock URL \url{https://arxiv.org/abs/2412.17082}.

\bibitem[Sokolov et~al.(2025)Sokolov, Sadiev, Demidovich, Al-Qahtani, and
  Richt{\'a}rik]{sokolov2025bernoulli}
Igor Sokolov, Abdurakhmon Sadiev, Yury Demidovich, Fawaz~S Al-Qahtani, and
  Peter Richt{\'a}rik.
\newblock {Bernoulli-LoRA}: A theoretical framework for randomized low-rank
  adaptation.
\newblock \emph{arXiv preprint arXiv:2508.03820}, 2025.

\bibitem[Sokolov et~al.(2026)Sokolov, Condat, and
  Richt{\'a}rik]{sokolov2026silage}
Igor Sokolov, Laurent Condat, and Peter Richt{\'a}rik.
\newblock {SILAGE}: Memory-efficient, full-gradient-free nonconvex optimization
  for nested finite sums.
\newblock \emph{arXiv preprint arXiv:2606.15832}, 2026.

\bibitem[Stich and Karimireddy(2019)]{stich2019error}
Sebastian~U Stich and Sai~Praneeth Karimireddy.
\newblock The error-feedback framework: {B}etter rates for {SGD} with delayed
  gradients and compressed communication.
\newblock \emph{arXiv preprint arXiv:1909.05350}, 2019.

\bibitem[Stich et~al.(2018)Stich, Cordonnier, and Jaggi]{stich2018sparsified}
Sebastian~U Stich, Jean-Baptiste Cordonnier, and Martin Jaggi.
\newblock Sparsified {SGD} with memory.
\newblock \emph{Advances in neural information processing systems}, 31, 2018.

\bibitem[Stonyakin et~al.(2023)Stonyakin, Kuruzov, and
  Polyak]{stonyakin2023stopping}
Fedor Stonyakin, Ilya Kuruzov, and Boris Polyak.
\newblock Stopping rules for gradient methods for non-convex problems with
  additive noise in gradient.
\newblock \emph{Journal of Optimization Theory and Applications}, 198\penalty0
  (2):\penalty0 531--551, 2023.

\bibitem[Sun et~al.(2023)Sun, Wang, Li, and Wang]{sun2023momentum}
Tao Sun, Qingsong Wang, Dongsheng Li, and Bao Wang.
\newblock Momentum ensures convergence of sign{SGD} under weaker assumptions.
\newblock In \emph{International Conference on Machine Learning}, pages
  33077--33099. PMLR, 2023.

\bibitem[Takezawa et~al.(2025)Takezawa, Koloskova, Jiang, and
  Stich]{takezawa2025fedmuon}
Yuki Takezawa, Anastasia Koloskova, Xiaowen Jiang, and Sebastian~U Stich.
\newblock {FedMuon}: Federated learning with bias-corrected {LMO}-based
  optimization.
\newblock \emph{arXiv preprint arXiv:2509.26337}, 2025.

\bibitem[Tseng and Bertsekas(1991)]{tseng1991relaxation}
Paul Tseng and Dimitri~P Bertsekas.
\newblock Relaxation methods for problems with strictly convex costs and linear
  constraints.
\newblock \emph{Mathematics of operations research}, 16\penalty0 (3):\penalty0
  462--481, 1991.

\bibitem[Tseng and Yun(2009)]{tseng2009block}
Paul Tseng and Sangwoon Yun.
\newblock Block-coordinate gradient descent method for linearly constrained
  nonsmooth separable optimization.
\newblock \emph{Journal of optimization theory and applications}, 140\penalty0
  (3):\penalty0 513--535, 2009.

\bibitem[Tyurin(2024)]{tyurin2024toward}
Alexander Tyurin.
\newblock Toward a unified theory of gradient descent under generalized
  smoothness.
\newblock \emph{arXiv preprint arXiv:2412.11773}, 2024.

\bibitem[Tyurin and Richt{\'a}rik(2023)]{tyurin2023optimal}
Alexander Tyurin and Peter Richt{\'a}rik.
\newblock Optimal time complexities of parallel stochastic optimization methods
  under a fixed computation model.
\newblock \emph{Advances in Neural Information Processing Systems},
  36:\penalty0 16515--16577, 2023.

\bibitem[Vogels et~al.(2019)Vogels, Karimireddy, and Jaggi]{vogels2019powersgd}
Thijs Vogels, Sai~Praneeth Karimireddy, and Martin Jaggi.
\newblock {PowerSGD}: Practical low-rank gradient compression for distributed
  optimization.
\newblock \emph{Advances in Neural Information Processing Systems}, 32, 2019.

\bibitem[Wang et~al.(2018)Wang, Sievert, Liu, Charles, Papailiopoulos, and
  Wright]{wang2018atomo}
Hongyi Wang, Scott Sievert, Shengchao Liu, Zachary Charles, Dimitris
  Papailiopoulos, and Stephen Wright.
\newblock {ATOMO}: Communication-efficient learning via atomic sparsification.
\newblock \emph{Advances in neural information processing systems}, 31, 2018.

\bibitem[Wangni et~al.(2018)Wangni, Wang, Liu, and Zhang]{wangni2018gradient}
Jianqiao Wangni, Jialei Wang, Ji~Liu, and Tong Zhang.
\newblock Gradient sparsification for communication-efficient distributed
  optimization.
\newblock \emph{Advances in Neural Information Processing Systems}, 31, 2018.

\bibitem[Xie et~al.(2025)Xie, Wang, Wu, and Li]{xie2025tale}
Shuo Xie, Tianhao Wang, Beining Wu, and Zhiyuan Li.
\newblock A tale of two geometries: Adaptive optimizers and non-{E}uclidean
  descent.
\newblock \emph{arXiv preprint arXiv:2511.20584}, 2025.

\bibitem[Yi et~al.(2024)Yi, Kharisov, Sokolov, and Richt{\'a}rik]{yi2024cohort}
Kai Yi, Timur Kharisov, Igor Sokolov, and Peter Richt{\'a}rik.
\newblock Cohort squeeze: Beyond a single communication round per cohort in
  cross-device federated learning.
\newblock \emph{arXiv preprint arXiv:2406.01115}, 2024.

\bibitem[Yuan and Li(2022)]{yuan2022convergence}
Xiaotong Yuan and Ping Li.
\newblock On convergence of {FedProx}: Local dissimilarity invariant bounds,
  non-smoothness and beyond.
\newblock \emph{Advances in Neural Information Processing Systems},
  35:\penalty0 10752--10765, 2022.

\bibitem[Zhang and Gao(2025)]{zhang2025provable}
Xinwen Zhang and Hongchang Gao.
\newblock On provable benefits of {Muon} in federated learning.
\newblock \emph{arXiv preprint arXiv:2510.03866}, 2025.

\end{thebibliography}
\bibliographystyle{plainnat}

\newpage
\appendix
\onecolumn

\tableofcontents

\clearpage

\section{Notation}\label{app:notation}
We collect here the notation used throughout the paper; the objects introduced are developed further in \Cref{sec:geometry_native_minibatch}.

\textbf{Spaces and norms.} Our analysis takes place in a finite-dimensional normed space $(E, \norm{\cdot})$ with dual space $(E^*, \snorm{\cdot})$. We use the canonical dual pairing $\inner{y}{x}\eqdef y(x)$ for $y\in E^*$ and $x\in E$, and define the dual norm by
$\snorm{y} \eqdef \sup_{\norm{x}\le 1}\inner{y}{x}$, $y\in E^*$.
We refer to $\norm{\cdot}$ as the \emph{update norm}: it determines $\snorm{\cdot}$ and, for a given gradient and positive stepsize, the set of admissible steepest-descent steps. The fixed sharp selection described below specifies a concrete step from this set.
Gradients and stochastic gradients are elements of $E^*$.
Our analysis uses the primal norm $\norm{\cdot}$ for parameter displacements and its induced dual norm $\snorm{\cdot}$ for gradient smoothness, oracle noise, and stationarity.
We also use auxiliary norms in the proofs and Euclidean or Frobenius norms when comparing oracle models.
We write $\kappa$ for a regularity constant of $(E^*, \snorm{\cdot})$, and $\spnorm{\cdot}$ and $\kappa_+$ for a corresponding auxiliary norm and its smoothness constant (\Cref{def:kappa_smooth_regular}).

\textbf{The sharp operator.} Let $\phi(x)\eqdef\frac12\sqnorm{x}$ for $x\in E$, with Fenchel conjugate $\phi^*(s)\eqdef\frac12\sqsnorm{s}$ for $s\in E^*$. The subdifferential $\partial\phi^*(s)\subseteq E$ may contain more than one element. A \emph{sharp selection} assigns to each $s\in E^*$ an element $s^\sharp\in\partial\phi^*(s)$. For the stochastic algorithm, we fix a Borel-measurable selection $s\mapsto s^\sharp$ as in \Cref{subsec:non_eu_update} and call this map the \emph{sharp operator}. Every such selected vector satisfies $\sqnorm{s^\sharp}=\sqsnorm{s}$ and $\inner{s}{s^\sharp}=\sqsnorm{s}$. In the Euclidean case $s^\sharp=s$.
In the quadratic case, $\sqnorm{x}=\inner{Ax}{x}$ for a symmetric positive-definite linear map $A:E\to E^*$, so $\nabla\phi(x)=Ax$, $\nabla^2\phi=A$, and $s^\sharp=\nabla\phi^*(s)=A^{-1}s$.
Here the sharp operator is the fixed linear map $A^{-1}$.
For a general norm, the squared-norm functions $\phi$ and $\phi^*$ need not be differentiable everywhere, and the sharp operator need not admit such a linear representation.
The subdifferential formulation remains valid without differentiability assumptions: $\partial\phi^*$ is positively homogeneous and may contain multiple vectors at a given point.
Basic properties of the sharp operator are collected in \Cref{lemma:sharp_flat_properties}, and \Cref{subsec:non_eu_update} derives the resulting update together with its steepest-descent and LMO forms.

\textbf{Vector geometries.} For vector-valued parameters we write $\ell_p^d$ for $\R^d$ equipped with the $\ell_p$ norm, $\norm{x}_p\eqdef\bigl(\sum_{i=1}^d|x_i|^p\bigr)^{1/p}$ for $1\le p<\infty$ and $\norm{x}_\infty\eqdef\max_{1\le i\le d}|x_i|$, with $\ell_2$ the Euclidean norm. The $\ell_p$ and $\ell_q$ norms are dual to one another whenever $\nfr{1}{p}+\nfr{1}{q}=1$, so an $\ell_p$ update norm induces an $\ell_q$ dual norm; in particular $\ell_1$ and $\ell_\infty$ form a dual pair, which is the source of the coordinate-type and sign-type updates in \Cref{subsec:non_eu_update}.

\textbf{Matrix geometries.} When the parameter space contains matrix-valued blocks, as in modern neural networks, non-Euclidean update norms are naturally expressed through Schatten norms.
For $\bX\in\R^{m\times n}$ with $m, n\ge1$, set $r\eqdef\min\{m, n\}$ and write $\sigma(\bX)=(\sigma_1(\bX), \dots, \sigma_r(\bX))\in\R^r$ for the full vector of singular values, ordered as $\sigma_1(\bX)\ge\dots\ge\sigma_r(\bX)\ge0$.
Zero singular values are included, so $\sigma(0)=0\in\R^r$.
For $1\le p<\infty$, the Schatten-$p$ norm is $\norm{\bX}_{S_p}\eqdef\norm{\sigma(\bX)}_p=\bigl(\sum_{i=1}^r\sigma_i(\bX)^p\bigr)^{1/p}$; for $p=\infty$, we define $\norm{\bX}_{S_\infty}\eqdef\norm{\sigma(\bX)}_\infty=\sigma_1(\bX)$.
Its special cases are the nuclear norm $\nucnorm{\bX}=\norm{\bX}_{S_1}$, the Frobenius norm $\fnorm{\bX}=\norm{\bX}_{S_2}$, and the spectral norm $\opnorm{\bX}=\norm{\bX}_{S_\infty}$. The spectral and nuclear norms are dual to one another.
The dimension-dependent bounds use $r=\min\{m, n\}$, whereas the actual rank $\mathrm{rank}(\bX)$ is the number of positive singular values. In compact-SVD arguments, $r$ is explicitly redefined locally as the actual matrix rank.

\textbf{Accuracy and complexity.} We write $\Delta\eqdef f(x^0)-\finf$ for the initial functional gap. A point $x\in E$ is \emph{$\eps$-stationary} if $\sqsnorm{\nabla f(x)}\le\eps$; stationarity is thus measured in the dual norm induced by the chosen update norm. Our light-tail guarantees are of high-probability type: for a failure probability $\delta\in(0, 1)$ they assert that, with probability at least $1-\delta$, at least one iterate among $x^0, \dots, x^{T-1}$ is $\eps$-stationary. By \emph{iteration complexity} we mean the number $T$ of algorithmic iterations required, by \emph{oracle complexity} the number of stochastic gradient evaluations, and by \emph{time complexity} the wall-clock time. For serial minibatch \algname{SGD}, the oracle count is $N_*=BT$, with $B$ evaluations per iteration; $N$ denotes the number of workers in the parallel setting. In wall-clock bounds, $S\in[N]$ denotes the number of fastest workers considered in the minimization. We write $[N]\eqdef\{1, \dots, N\}$, and $\widetilde{\cO}(\cdot)$ for $\cO(\cdot)$ up to factors logarithmic in $T$ and $\nfr{1}{\delta}$.

Throughout the convergence and complexity analysis, we take a positive smoothness bound $L>0$, an accuracy $\eps>0$, and consider the nontrivial case $\Delta>0$. If $\Delta=0$, then $x^0$ is a global minimizer and hence $\nabla f(x^0)=0$, so it can be returned without any iterations. The simplified oracle and wall-clock complexity bounds are stated in the accuracy regime $0<\eps\le L\Delta$.

\textbf{Notation reference.} The following table collects the notation used in the main text and the principal constructions in the appendices. The last column points to the definition or the relevant discussion. Symbols reused locally are listed with their scope; temporary variables in individual calculations retain their local definitions.

\begingroup
\setlength{\tabcolsep}{4pt}
\renewcommand{\arraystretch}{1.15}
\begin{longtable}{@{}>{\raggedright\arraybackslash}p{0.21\textwidth}>{\raggedright\arraybackslash}p{0.55\textwidth}>{\raggedright\arraybackslash}p{0.20\textwidth}@{}}
\caption{Notation reference. Definitions with a stated local scope do not replace the conventions used elsewhere.}\label{tab:notation_reference}\\
\toprule
\textbf{Symbol} & \textbf{Meaning} & \textbf{Reference / scope} \\
\midrule
\endfirsthead
\multicolumn{3}{@{}l}{\textit{Notation reference (continued)}}\\
\toprule
\textbf{Symbol} & \textbf{Meaning} & \textbf{Reference / scope} \\
\midrule
\endhead
\midrule
\multicolumn{3}{r@{}}{\textit{Continued on the next page}}\\
\endfoot
\bottomrule
\endlastfoot

\multicolumn{3}{@{}l}{\textbf{Spaces and norms}}\\*
$E$, $E^*$ & Parameter space and its dual; $1\le\dim(E)<\infty$. Gradients belong to $E^*$. & \Cref{app:notation} \\
$E^{**}$ & Bidual space, canonically identified with $E$ in finite dimensions. & \Cref{thm:bidual_norm_equals_primal} \\
$x$, $s$ & Generic primal vector $x\in E$ and dual vector $s\in E^*$. & \Cref{app:notation} \\
$\inner{s}{x}$ & Dual pairing: the value of the linear functional $s$ at $x$. & \Cref{app:notation} \\
$\norm{\cdot}$ & Primal norm, called the update norm. & \Cref{subsec:non_eu_update} \\
$\snorm{\cdot}$ & Induced dual norm, $\snorm{s}=\sup_{\norm{x}\le1}\inner{s}{x}$. & \Cref{lemma:dual_norm_is_a_norm} \\
$\norm{\cdot}_2$, $\fnorm{\cdot}$ & Euclidean and Frobenius reference norms; in the oracle comparison, the subscript $\opf$ also denotes the Euclidean norm for vectors. & \Cref{sec:frobenius_light_tail_oracle} \\
$\mathcal C$, $\cB(y, t)$ & A constraint set and the closed update-norm ball $\{x\in E:\norm{x-y}\le t\}$. & \Cref{subsec:lmo_form} \\
$\mathbb B$, $\mathbb{B}_*$ & Primal and dual unit balls, respectively. The notation $B$ also denotes a convex body locally in the ellipsoid proof. & \Cref{lemma:fenchel_of_V}; \Cref{lem:fenchel_of_norm} \\
$\|\cdot\|_a$, $\|\cdot\|_b$ & Two auxiliary norms with neutral labels $a, b$; $\|\cdot\|_{a, *}$ and $\|\cdot\|_{b, *}$ are their dual norms. & \Cref{lemma:duality_of_norm_equivalence} \\

\midrule
\multicolumn{3}{@{}l}{\textbf{Standard notation from convex analysis}}\\*
$\dom f$, $\dom\partial f$ & Effective domain $\{x:f(x)<+\infty\}$ and the set of points where the subdifferential is nonempty. In convex-analysis statements, $f$ may be extended-valued. & \eqref{eq:subgradient_def} \\
$\operatorname{int}(C)$, $\ri(C)$ & Interior in the ambient space and relative interior in the affine hull of $C$. & Existence of subgradients in the auxiliary results \\
$\partial f(x)$ & Convex subdifferential: the set of $s$ satisfying $f(y)\ge f(x)+\inner{s}{y-x}$ for every $y$, at a finite-valued point $x$. & \eqref{eq:subgradient_def} \\
$f^*$ & Fenchel conjugate, $f^*(s)=\sup_x\{\inner{s}{x}-f(x)\}$. The star on a function denotes conjugacy, whereas the star on a norm denotes the dual norm. & \Cref{thm:fenchel_young_inequality} \\
$\iota_C$ & Extended-valued indicator: zero on $C$ and $+\infty$ outside $C$. & \Cref{lem:fenchel_of_norm} \\
$\partial\norm{x}$, $\partial\snorm{s}$ & Subdifferentials of the norm functions evaluated at $x$ and $s$, respectively; the short notation does not differentiate a scalar value. & \Cref{lem:subdifferential_of_the_norm} \\
$\phi$, $\phi^*$ & Squared-norm functions $\phi(x)=\frac12\sqnorm{x}$ and $\phi^*(s)=\frac12\sqsnorm{s}$, which are Fenchel conjugates. & \Cref{lemma:subdiff_phi_from_notes} \\
$\argmin$, $\argmax$ & Sets of minimizers and maximizers; they need not be singletons. & \Cref{subsec:lmo_form} \\
$\operatorname{conv}(C)$, $\partial C$ & Convex hull and boundary of a set $C$. The boundary notation is distinct from the subdifferential of a function. & \Cref{thm:john_ellipsoid} \\

\midrule
\multicolumn{3}{@{}l}{\textbf{Objective, oracle, and accuracy}}\\*
$f$, $f_i$ & Global objective and worker objectives; in the finite-sum formulation, $f=N^{-1}\sum_{i=1}^N f_i$. Under statistical homogeneity all $f_i$ coincide. & \eqref{eq:main_task_sum} \\
$f(x;\xi)$, $f_i(x;\xi_i)$ & Sample losses in the expectation formulation, which is a special case of the general oracle model. & \eqref{eq:local_objective} and \eqref{eq:main_task} \\
$\mathbb S_i$, $\cD_i$, $\cD$ & Worker sample space, worker sampling distribution, and common distribution used in the homogeneous setting. & Problem setup in \Cref{sec:intro} \\
$\nabla f(x)$ & Exact derivative of $f$ at $x$, viewed as an element of $E^*$. & \Cref{subsec:stationarity_oracle} \\
$\nf{x;\xi}$ & Response of the jointly measurable stochastic oracle. It is an unbiased estimate of $\nabla f(x)$, but need not be the gradient of a sample loss. & \Cref{subsec:stationarity_oracle} \\
$\xi$, $\xi_j^k$ & Fresh oracle sample; the indexed version is sample $j$ of iteration $k$. Each new minibatch is independent of the preceding history. & \Cref{subsec:non_eu_update} \\
$L$ & Positive smoothness bound satisfying $\snorm{\nabla f(x)-\nabla f(y)}\le L\norm{x-y}$. & \Cref{as:non_euclidean_smoothness} \\
$\finf$ & A finite lower bound on $f$; it need not be attained. & \Cref{as:lower_bound} \\
$\Delta$ & Initial functional gap, $f(x^0)-\finf$. The nontrivial complexity regime assumes $\Delta>0$. & \Cref{app:notation} \\
$\sigma$, $\sigma^2$ & Positive dual-norm light-tail scale and its square: $\Exp{\exp(\sqsnorm{\nf{x;\xi}-\nabla f(x)}/\sigma^2)}\le e$. It also bounds the noise second moment. & \Cref{as:non_euclidean_light_tail} \\
$G$ & Uniform bound on the dual norm of the oracle error in the bounded-error example; distinct from a matrix gradient $\bG$. & \Cref{subsec:stationarity_oracle} \\
$\eps$ & Squared-gradient stationarity tolerance: $\sqsnorm{\nabla f(x)}\le\eps$. Simplified complexity bounds use $0<\eps\le L\Delta$. & \Cref{app:notation} \\
$\delta$ & Failure probability, $0<\delta<1$; the main guarantees hold with probability at least $1-\delta$. & \Cref{thm:stochastic_L_smooth_final} \\
$\hat x$ & An iterate selected uniformly from the prescribed finite sequence, used for the companion expectation criterion. & \Cref{app:light_tail_discussion} \\
$L_{\opf}$, $\eps_{\opf}$ & Smoothness constant and squared-gradient accuracy measured in the Euclidean/Frobenius geometry in the literature comparison. & \Cref{sec:related_work} \\
$\sigma_{\opf}^2$ & Reference-norm noise scale: a variance bound in the bounded-variance literature discussion, and an exponential-moment scale in the light-tail oracle comparison. & \Cref{sec:related_work}; \Cref{as:stoch_frobenius_light_tail_alt} \\

\midrule
\multicolumn{3}{@{}l}{\textbf{Updates and iteration counts}}\\*
$x^0$, $x^k$ & Initial point and iterate at index $k$; $\bX^k$ denotes the corresponding matrix iterate. & \eqref{eq:non-eu-one-step-final} \\
$g^k$, $\bG^k$ & Minibatch gradient $g^k=B^{-1}\sum_{j=1}^B\nf{x^k;\xi_j^k}$ and its matrix counterpart. & \eqref{eq:mini_gk_def} \\
$s^\sharp$ & Fixed Borel-measurable selection from $\partial\phi^*(s)$. It satisfies $\norm{s^\sharp}=\snorm{s}$ and $\inner{s}{s^\sharp}=\sqsnorm{s}$. & \Cref{lemma:sharp_flat_properties} \\
$\gamma_k$ & Positive stepsize; the main convergence and time bounds use $\gamma_k\equiv1/(2L)$. & \Cref{thm:stochastic_L_smooth_final} \\
$h^k$, $h_k$ & Step displacement, selected as $-\gamma_k(g^k)^\sharp$ to agree with the fixed sharp rule. Both index conventions occur in the derivation. & \Cref{app:steepest_derivation} \\
$t_k$ & Radius $\gamma_k\snorm{g^k}$ of the equivalent LMO ball; not a wall-clock time. & \eqref{eq:lmo_form} \\
$\lmo_{\mathcal C}(s)$ & Set-valued linear minimization oracle, $\argmin_{x\in\mathcal C}\inner{s}{x}$. & \Cref{subsec:lmo_form} \\
$B$ & Minibatch size, a positive integer; independent of the number of workers. & \Cref{thm:stochastic_L_smooth_final} \\
$T$ & Prescribed iteration count; in the main bounds, $T=\lceil128L\Delta/(15\eps)\rceil$. & \Cref{thm:stochastic_L_smooth_final} \\
$N_*$ & Serial minibatch oracle count $BT$; written $TB_*$ when explicitly comparing native and converted batch choices. & \Cref{cor:complexity_dual_native_light_tail} \\
$\mathcal T_{\mathrm{SGD}}^{\opf}(B)$ & Euclidean/Frobenius iteration-complexity bound for minibatch size $B$; the corresponding oracle bound is $B\mathcal T_{\mathrm{SGD}}^{\opf}(B)$. & \Cref{sec:related_work} \\
$\mathcal O$, $\Omega$, $\Theta$, $o$ & Asymptotic upper, lower, matching-order, and strictly smaller-order notation, respectively. & Complexity and dimension comparisons \\
$\widetilde{\mathcal O}$ & Upper-bound notation suppressing factors logarithmic in $T$ and $1/\delta$. & \Cref{app:notation} \\
$I(s)$, $e_i$ & Maximizing-coordinate set $I(s)=\{i:|s_i|=\|s\|_\infty\}$ and coordinate unit vector. & \Cref{app:sharp_l1_coordinate} \\
$\sign$ & Componentwise sign function, with $\sign(0)=0$ for the canonical sign selection. & \Cref{app:sharp_linf_sign} \\
$A$ & In the quadratic-norm example, a symmetric positive-definite map $E\to E^*$ with $\sqnorm{x}=\inner{Ax}{x}$ and $s^\sharp=A^{-1}s$. & \Cref{app:notation} \\

\midrule
\multicolumn{3}{@{}l}{\textbf{Regularity and norm comparisons}}\\*
$V$, $\|\cdot\|_V$ & Generic finite-dimensional normed space in the regularity results; usually specialized to $V=E^*$. Distinct from the potential $V(s)$ below. & \Cref{def:kappa_smooth_regular} \\
$\kappa_V$, $\kappa_{V, +}$ & A regularity constant of $V$ and a smoothness constant of its auxiliary norm, with $1\le\kappa_{V, +}\le\kappa_V$. The same constant notation is used when discussing smoothness alone. & \Cref{def:kappa_smooth_regular} \\
$\|\cdot\|_{V, +}$ & Auxiliary norm with $\|u\|_V^2\le\|u\|_{V, +}^2\le(\kappa_V/\kappa_{V, +})\|u\|_V^2$. & \eqref{eq:reg_kappa_def_gen} \\
$\|\cdot\|_{V, +, *}$ & Norm dual to $\|\cdot\|_{V, +}$, used in smoothness and conjugacy characterizations. & \Cref{prop:kappa_smooth_characterizations} \\
$\kappa$, $\kappa_+$ & Abbreviations for $\kappa_{E^*}$ and $\kappa_{E^*, +}$ in the stochastic analysis. They are admissible constants, not necessarily the smallest possible ones. & \Cref{rem:regularity_dual_convention} \\
$\spnorm{\cdot}$ & Auxiliary norm on $E^*$, distinct from the update norm. The concentration proof chooses a $\kappa$-smooth proxy with squared-norm distortion at most $2$. & \eqref{eq:reg_kappa_def}; \eqref{eq:reg_to_smooth_equiv_dual} \\
$\|\cdot\|_+$, $\phi_+$ & Local shorthand for an auxiliary norm and its half-squared function, $\phi_+(u)=\frac12\|u\|_+^2$. & \Cref{lem:second_moment_kappa_regular} \\
$C$ & In the regularity comparison, squared-norm distortion with $\kappa=\kappa_+C$. In the two-norm equivalence lemma it instead bounds the norms themselves. & \Cref{sec:regularity_vs_norm_equivalence}; \Cref{lemma:duality_of_norm_equivalence} \\
$\alpha$, $\beta$ & Positive equivalence constants satisfying $\alpha\snorm{s}\le\fnorm{s}\le\beta\snorm{s}$. Noise conversion uses $\alpha^{-2}$; the Euclidean proxy gives $\kappa=(\beta/\alpha)^2$. & \eqref{eq:norm_equiv_2_light_tail_alt} \\
$\rho$ & Finite auxiliary-norm exponent in the regularity constructions; for $p\ge2$, the optimization uses $2\le\rho\le p$ and $\rho<\infty$, also when $p=\infty$. & \Cref{prop:lp_regular_vs_equiv}; \Cref{prop:schatten_regular_vs_equiv} \\
$\rho_*$, $\hat\rho$ & Tuned auxiliary exponent and the threshold used to express it. The vector and Schatten choices differ; their closed forms and dimension ranges are stated separately. & \Cref{prop:lp_regular_vs_equiv}; \Cref{prop:schatten_regular_vs_equiv} \\
$\kappa_p(d)$ & Constructed regularity parameter for $\ell_p^d$, equal to $d^{2/p-1}$ for $p<2$ and obtained by minimizing $(\rho-1)d^{2/\rho-2/p}$ for $p\ge2$. & \Cref{example:2} \\
$\kappa_p(m, n)$ & Constructed Schatten regularity parameter using $r=\min\{m, n\}$; the uniform smoothness bound used in the construction is $\max\{2, \rho-1\}$. & \Cref{example:3} \\
$\kappa_{\ell_p}(d)$, $\kappa_{S_p}(r)$ & Explicit admissible parameters chosen for the oracle comparisons. They need not coincide when $d=r$; the latter uses the exact value $1$ at $p=2$. & \eqref{eq:kappa_lp_alt} and \eqref{eq:kappa_schatten_alt} \\
$\rho$ (related work) & Conversion constant $\sup_{z\ne0}\snorm{z}/\|z\|_2$ in the discussion of \algname{Ringmaster LMO}; unrelated to the auxiliary exponent. & \Cref{app:related_time_complexity} \\

\midrule
\multicolumn{3}{@{}l}{\textbf{Vector and matrix geometries}}\\*
$d$, $m$, $n$ & Vector dimension and numbers of matrix rows and columns. In general-space statements, $d$ may denote $\dim(V)$. & \Cref{app:notation}; \Cref{example:1} \\
$p$, $q$ & Norm exponents; when paired as dual exponents, $1/p+1/q=1$, with $1/\infty=0$. In oracle comparisons, $p$ indexes the dual norm. & \Cref{app:notation} \\
$\ell_p^d$, $\|\cdot\|_p$ & $\R^d$ equipped with the norm $(\sum_i|x_i|^p)^{1/p}$ for finite $p$, and $\|x\|_\infty=\max_i|x_i|$. & \Cref{app:notation} \\
$r$ (dimension) & Matrix dimension parameter $\min\{m, n\}$ in Schatten norm comparisons and regularity bounds; not the actual rank. & \Cref{app:notation} \\
$r$ (compact SVD) & Locally, the actual rank $\operatorname{rank}(\bG)$, counting only positive singular values. & \Cref{app:orth_derivation} \\
$\sigma_i(\bX)$, $\sigma(\bX)$ & Singular values in nonincreasing order and their full vector, including zeros. This use of $\sigma$ is distinct from the noise scale. & \Cref{app:notation} \\
$\|\cdot\|_{S_p}$ & Schatten-$p$ norm, $\|\bX\|_{S_p}=\|\sigma(\bX)\|_p$. & \Cref{app:notation} \\
$\opnorm{\cdot}$, $\nucnorm{\cdot}$ & Spectral norm $S_\infty$ and nuclear norm $S_1$, a dual pair. The Frobenius norm is $S_2$. & \Cref{app:notation} \\
$\inner{\bG}{\bX}_F$ & Frobenius pairing $\operatorname{tr}(\bG^\top\bX)$, identifying the matrix space with its dual. & \Cref{lem:subdiff_nuc_norm} \\
$\bU_r$, $\Sigma_r$, $\bV_r$ & Compact-SVD factors, $\bG=\bU_r\Sigma_r\bV_r^\top$, retaining only positive singular values; the columns of $\bU_r$ and $\bV_r$ are orthonormal. & \Cref{app:orth_derivation} \\
$\bW$ & Residual matrix satisfying $\bU_r^\top\bW=0$, $\bW\bV_r=0$, and $\opnorm{\bW}\le1$ in the nuclear-subgradient and spectral-LMO formulas. & \Cref{lem:subdiff_nuc_norm}; \Cref{lem:lmo_spectral_norm} \\
$M^\dagger$, $M^{\dagger/2}$ & Moore--Penrose pseudoinverse and, for positive semidefinite $M$, the positive semidefinite square root of $M^\dagger$. & \Cref{subseq:equiv_orth} \\
$\orth(\bG)$, $\bO^k$ & Canonical polar map $\orth(\bG)=\bU_r\bV_r^\top$, with $\orth(0)=0$; $\bO^k=\orth(\bG^k)$. & \eqref{eq:orth_def-final} \\
$(\bG^k)^\sharp$ & Canonical spectral-norm sharp selection $\nucnorm{\bG^k}\orth(\bG^k)$; distinct from the polar factor alone. & \eqref{eq:orth_update_main-final} \\
$\eta_k$ & State-dependent stepsize $\gamma_k\nucnorm{\bG^k}$ in the orthogonalized update; distinct from a noise vector $\eta^j$. & \Cref{app:orth_derivation} \\
$I_d$, $I_r$ & Identity matrices in the specified dimensions, distinct from the maximizing-coordinate set $I(s)$ and random coordinate index $I$. & \Cref{thm:john_ellipsoid}; \Cref{lem:subdiff_nuc_norm} \\
$\operatorname{tr}$, $\det$, $\operatorname{rank}$ & Matrix trace, determinant, and rank. & \Cref{thm:john_ellipsoid}; \Cref{app:orth_derivation} \\
$E_{ii}$ & Rectangular coordinate matrix with entry $1$ at $(i, i)$ and zeros elsewhere. & \Cref{prop:tightness_schatten_bounds} \\

\midrule
\multicolumn{3}{@{}l}{\textbf{Noise models and statistical coefficients}}\\*
$\zeta$, $Z$ & Additive, zero-mean oracle perturbation; its distribution and norm scale depend on the stated model. & \Cref{subsec:stationarity_oracle}; \Cref{sec:frobenius_light_tail_oracle} \\
$\nu$ & Positive amplitude of the Rademacher or Gaussian perturbations. & \Cref{def:rademacher_oracle_alt}; \Cref{def:gaussian_oracle_alt} \\
$\chi_i$, $\chi$ & Independent Rademacher signs and their vector; in the one-sparse model, $\chi$ denotes a single sign. & \Cref{def:rademacher_oracle_alt}; \Cref{rem:one_sparse_noise_alt} \\
$\omega_i$, $\omega$ & Independent standard Gaussian coordinates and their vector in the Gaussian oracle model. In regularity lower-bound proofs, similar symbols locally denote signs instead. & \Cref{def:gaussian_oracle_alt}; \Cref{prop:tightness_lp_bounds} \\
$\Diag_{m, n}(u)$ & Rectangular diagonal embedding: entry $u_i$ at $(i, i)$ for $i\le\min\{m, n\}$ and zero elsewhere. & \eqref{eq:rectangular_diagonal_embedding_alt} \\
$D$ & Number of independent noise coordinates: $D=d$ for dense vector noise and $D=r$ for diagonal matrix noise. & \Cref{sec:frobenius_light_tail_oracle} \\
$\vartheta_p(D)$ & Gaussian shape factor, equal to $1$ for $1\le p\le2$ and $\min\{p, \ln(eD)\}$ for $p>2$. & \eqref{eq:gaussian_shape_factor_alt} \\
$m_p$, $X_p$, $H_p(D)$ & In the Gaussian proof for $p>2$, $m_p=\Exp{\|\omega\|_p}$, $X_p=\|\omega\|_p-m_p$, and $H_p(D)=\vartheta_p(D)D^{2/p}$. & \Cref{lem:gaussian_light_tail_scale_alt}, proof \\
$\alpha_p(d)$, $\alpha_{S_p}(r)$ & Sharp one-sided conversion constants for Euclidean-to-$\ell_p$ and Frobenius-to-Schatten-$p$ norms. & \eqref{eq:alpha_lp_alt} and \eqref{eq:alpha_schatten_alt} \\
$B_*$, $B_{\opf}$ & Batch sizes prescribed by the native and converted analyses, respectively, for the same update rule and iteration count. & \eqref{eq:native_batch_alt} and \eqref{eq:frobenius_batch_alt} \\
$N_{\opf}$ & Converted-analysis oracle count $TB_{\opf}$. & \eqref{eq:complexity_frobenius_alt} \\
$\mathsf C_*$, $\mathsf C_{\opf}$ & Statistical coefficients $\kappa\sigma^2$ and $\alpha^{-2}\sigma_{\opf}^2$. Superscripts $\mathrm{Rad}$ and $\mathrm{Gau}$ distinguish the two perturbation models. & \eqref{eq:statistical_coefficients_alt} \\
$I$ (one-sparse noise) & Random coordinate index in $\zeta=\nu\chi e_I$, independent of the Rademacher sign. Distinct from the set $I(s)$. & \Cref{rem:one_sparse_noise_alt} \\

\midrule
\multicolumn{3}{@{}l}{\textbf{Ellipsoid construction for the universal regularity bound}}\\*
$B$, $\mathcal B_2$ & Centrally symmetric convex body and Euclidean unit ball. Here $B$ is a set, not a minibatch size. & \Cref{thm:john_ellipsoid} \\
$\cE_{\mathrm{in}}$ & Maximal-volume ellipsoid contained in $B$, with $B\subseteq\sqrt d\,\cE_{\mathrm{in}}$. & \Cref{thm:john_ellipsoid} \\
$S_+^d$, $S_{++}^d$ & Cones of real symmetric positive semidefinite and positive definite matrices, respectively. & \Cref{thm:john_ellipsoid}, proof \\
$c$, $Q$, $Q^*$ & Center and positive-definite shape matrix in $c+Q\mathcal B_2$, and the maximizing centered shape matrix $Q^*$. & \Cref{thm:john_ellipsoid}, proof \\
$\mathcal Q$ & Feasible shape matrices $\{Q\in S_+^d:Q\mathcal B_2\subseteq B\}$. & \Cref{thm:john_ellipsoid}, proof \\
$T$ (ellipsoid proof) & Normalizing linear map $(Q^*)^{-1}$, distinct from the iteration count. & \Cref{thm:john_ellipsoid}, proof \\
$\widetilde B$, $\widetilde B^\circ$ & Normalized body $T(B)$ and its polar, $\{y:\langle y, x\rangle\le1\text{ for all }x\in\widetilde B\}$. & \Cref{thm:john_ellipsoid}, proof \\
$g(Q)$, $\mu$, $\mathcal L$ & Convex constraint function, nonnegative multiplier, and Lagrangian in the finite-dimensional shape optimization; unrelated to stochastic gradient responses. & \Cref{thm:john_ellipsoid}, proof \\
$y_i$, $\lambda_i$ & Unit contact points and positive weights with $I_d=\sum_i\lambda_i y_i y_i^\top$ and $\sum_i\lambda_i=d$. & \eqref{eq:kkt_identity_john_detailed} \\

\midrule
\multicolumn{3}{@{}l}{\textbf{Probability and concentration proofs}}\\*
$\Omega$, $\cF$, $\Prob$ & Sample space, sigma-algebra of events, and probability measure. The symbol $\Omega$ also denotes asymptotic lower bounds elsewhere. & \Cref{def:measurable_space} \\
$\Exp{\cdot}$, $\ExpSub{\xi}{\cdot}$ & Expectation, with a subscript specifying the sampled variable when useful; conditioning on a history is written with a vertical bar. & \Cref{as:non_euclidean_light_tail} \\
$\sigma(\xi)$ & Sigma-algebra generated by a random variable; unrelated to the light-tail scale or singular values. & \Cref{def:F_measurable_rv} and the following remark \\
$\cF_t$, $\cF_k$ & Filtration or history available before the next sample or iteration; the Rennala proof includes pending worker states in this history. & \Cref{def:filtration}; \Cref{sec:non_eu_updaten} \\
$\eta^j$, $\eta_j^k$ & Centered noise vectors; at an iterate, $\eta_j^k=\nf{x^k;\xi_j^k}-\nabla f(x^k)$. & \Cref{lemma:non_euclidean_large_deviations}; \Cref{sub:proof_of_the_theorem_stochastic_l_smooth_final} \\
$\bar\eta$, $S_j$ & Average $B^{-1}\sum_{j=1}^B\eta^j$ and partial sum $\sum_{t=1}^j\eta^t$. In martingale preliminaries, $S_t=\sum_{i=1}^t\xi_i$. & \Cref{lem:frobenius_to_dual_conc_alt}; \Cref{app:proof_key_concentration} \\
$\xi_t$ (martingales) & A vector-valued martingale difference in the auxiliary results; distinct from the oracle's sample variable. & \Cref{def:mds} \\
$\lambda$ & Nonnegative tail parameter in the concentration bound $e^{-\lambda^2/3}$; for the finite-horizon union bound, $\lambda=\sqrt{3\log(T/\delta)}$. & \Cref{lemma:non_euclidean_large_deviations} \\
$\sigma_+^2$ & Auxiliary-norm light-tail scale $2\sigma^2$ in the concentration proof. & \eqref{eq:light_tail_plus_norm} \\
$\widehat\phi$ & Half-squared primal norm restricted to the unit ball, and $+\infty$ outside it. & \Cref{lemma:fenchel_of_V} \\
$V(s)$ & Potential $(\widehat\phi)^*(s)$: $\frac12\snorm{s}^2$ when $\snorm{s}\le1$, and $\snorm{s}-\frac12$ otherwise. The concentration proof uses the same construction with $\spnorm{\cdot}$. & \Cref{lemma:fenchel_of_V}; \Cref{app:proof_key_concentration} \\
$\beta$, $V_\beta(s)$ & Positive potential-scaling parameter and $V_\beta(s)=\beta V(s/\beta)$. This $\beta$ is unrelated to the norm-equivalence constant. & \Cref{lemma:V_beta_descent} \\
$\psi_j$, $\Psi_B$ & Scalar increments $V_\beta(S_j)-V_\beta(S_{j-1})$ and their sum, which telescopes to $V_\beta(S_B)$. & \Cref{app:proof_key_concentration} \\
$\mu_j$, $\nu_j$ & Conditional mean upper bounds and scalar exponential-moment scales. The application sets $\mu_j=\kappa\sigma_+^2/(2\beta)$ and $\nu_j=\sigma_+$. & \Cref{lemma:scalar_concentration} \\
$\psi(x)$ & Locally in the vector smoothness proof, $\|x\|_\rho^2=2\phi_+(x)$; distinct from the scalar random increments $\psi_j$. & \Cref{example:2}, proof \\
$\eps_i$, $u_i$, $U_i$ & Rademacher signs, deterministic vector directions, and matrix directions in the regularity lower-bound constructions. Indexed $\eps_i$ are not stationarity tolerances. & \Cref{prop:tightness_lp_bounds}; \Cref{prop:tightness_schatten_bounds} \\
$n$, $N$ (lower bounds) & Local lengths of sign-sequence constructions in the vector and matrix proofs, respectively; distinct from matrix dimensions and the worker count. & \Cref{prop:tightness_lp_bounds}; \Cref{prop:tightness_schatten_bounds} \\

\midrule
\multicolumn{3}{@{}l}{\textbf{Parallel computation and collection-time proof}}\\*
$N$, $[N]$ & Number of workers and the index set $\{1, \ldots, N\}$. & \Cref{sec:rennala_nsgd} \\
$\tau_i$, $\tau_{\max}$ & Fixed time for worker $i$ to compute one stochastic gradient, ordered as $0<\tau_1\le\cdots\le\tau_N$; $\tau_{\max}=\max_i\tau_i$. & \Cref{sec:related_work}; \Cref{sec:rennala_nsgd} \\
$S$ & Number of fastest workers considered in the time-bound minimization, $S\in[N]$; not a batch size. & \Cref{thm:main_time_complexity} \\
$T_{\mathrm{total}}$ & Wall-clock time to complete the prescribed $T$ iterations under the fixed computation model. & \eqref{eq:time_exact_bound} \\
$g$, $k'$ & Gradient response and the index of its query point in a worker message. The server accepts it only when $k'=k$. & \Cref{alg:gluon_server}; \Cref{alg:gluon_worker} \\
$b$ & Number of gradients accepted so far in the current minibatch. & \Cref{alg:gluon_server} \\
$\delta_k$ & Gradient delay in the related-work description of classical asynchronous updates; distinct from failure probability $\delta$. & \Cref{sec:related_work} \\
$B_i$ & Sufficient target quota for worker $i$ in the collection-time proof, not necessarily its actual accepted count. Quotas are relaxed to reals and then rounded up. & \Cref{sec:non_eu_updaten} \\
$t_i$ & Sufficient time $\tau_i(1+B_i)$ when $B_i>0$, and zero when $B_i=0$. & \Cref{sec:non_eu_updaten} \\
$t'$, $t'(S)$ & Relaxed batch-collection time and the candidate $t'(S)=(B+S)/(\sum_{i=1}^S1/\tau_i)$. & \Cref{sec:non_eu_updaten} \\
$S^\star$ & Smallest minimizer of $t'(S)$ over $S\in[N]$, used only in the runtime proof. & \Cref{sec:non_eu_updaten} \\
$A_S$, $\bar r_S$ & Aggregate rate $A_S=\sum_{i=1}^S1/\tau_i$ and average rate $\bar r_S=A_S/S$ of the fastest $S$ workers. Its reciprocal is their harmonic-mean computation time. & \Cref{sec:non_eu_updaten} \\

\end{longtable}
\endgroup

\clearpage

\section{Limitations}
\label{sec:limitations}

\textbf{Noise assumptions.}
Our high-probability guarantees rely on a dual-norm light-tail condition, which is stronger than bounded variance. Under bounded variance, the counterpart discussed in \Cref{app:light_tail_discussion} gives an expectation guarantee, while the high-probability analysis requires stronger control of the noise.

\textbf{Dependence on geometry and noise.}
The benefits of the native-norm formulation also depend on the problem: the regularity parameter $\kappa$ can be as large as $\dim(E)$, and whether our bounds improve on Euclidean/Frobenius conversions depends jointly on the geometry and the noise model.

\textbf{Statistical homogeneity.}
In the parallel setting, statistical homogeneity allows gradients from different workers to form an unbiased minibatch at the current iterate. When local distributions differ, accepting the first available fresh gradients can instead favor faster workers and bias the estimate of the global gradient. Extending our guarantees to this setting therefore requires revisiting minibatch collection and its time analysis.

\textbf{Computation and communication model.}
The current analysis also assumes fixed worker computation times and instantaneous communication and server processing, leaving arbitrary delays and communication costs outside its scope.

\textbf{Cost of the sharp operator.}
This paper does not account for the cost of computing the update direction. The computation model charges the workers only for stochastic gradients and treats the work of the server, including the evaluation of the sharp operator, as free. For a non-Euclidean update norm this cost need not be negligible: for the spectral norm, the sharp operator requires an orthogonalization of $\bG^k$ (\Cref{subsec:non_eu_update}). Accordingly, \eqref{eq:time_simplified_bound} should be read as the time spent on computing stochastic gradients. Our guarantees assume an exact sharp selection, so numerical approximations require additional error control. Reducing the cost of the sharp operator is an important problem in its own right, which lies outside the scope of this work; existing approaches are reviewed in \Cref{app:sharp_operator_cost}.

\textbf{Experimental validation.}
The present contribution is theoretical; assessing its practical benefits and computational tradeoffs through experiments remains future work.

\clearpage

\section{Extended Related Work}
\label{app:extended_related_work}
\subsection{Non-Euclidean Optimization}
\label{app:non_euclidean_related_work}
\citet{kovalevborodich2025} develop a unified analysis of non-Euclidean \algname{SGD} with momentum and weight decay under structured smoothness and gradient-noise assumptions, including variants with extrapolation and momentum variance reduction.
Their Assumption~2 bounds the noise second moment in a weighted Euclidean norm, whereas our light-tail assumption is formulated directly in the dual norm induced by the update geometry.
In the convex setting, they compare the leading complexity terms for \algname{signSGD} and \algname{Muon} with those of normalized \algname{SGD}.
Their comparisons identify favorable regimes involving dense solutions together with strongly nonuniform coordinatewise smoothness and noise scales, or solutions of high effective rank together with approximately low-rank matrices controlling curvature and gradient noise.
Although they provide explicit comparisons in terms of structural parameters, they do not construct specific objective–oracle pairs realizing the favorable regimes.
Moreover, the conditions yielding an improvement may be restrictive when the smoothness and noise parameters have many components or singular values of comparable magnitude.

Concurrent work by \citet{xie2025tale} studies non-Euclidean noise through an \emph{adaptive variance} model, whereas we use a light-tail condition in the dual norm induced by the update geometry. They prove that Normalized Steepest Descent (\algname{NSD}) with momentum admits a \emph{dimension-free} convergence rate for nonconvex optimization, in the sense that it avoids the norm-distortion factor that arises when controlling noise in a Euclidean reference norm and then converting to the \emph{target dual norm}.

\subsection{Time Complexity in Heterogeneous Parallel \texorpdfstring{\algname[]{SGD}}{SGD}}
\label{app:related_time_complexity}
In the asynchronous update described in \Cref{sec:related_work}, the gradient applied at iteration $k$ is evaluated at the earlier iterate $x^{k-\delta_k}$. It therefore need not be an unbiased estimator of $\nabla f(x^k)$; the analysis must account for this bias, as in frameworks for biased gradient estimators~\citep{demidovich2023guide}.

The harmonic-mean wall-clock bounds cited in \Cref{sec:related_work} rely on suitable stepsize or filtering rules. The results cover \algname{ASGD} with delay-adaptive stepsizes and filtering variants such as \algname{Picky SGD}, which skips updates when the query point is too far from the current iterate~\citep{mishchenko2022asynchronous, koloskova2022sharper, cohen2021asynchronous}.

\citet{sadiev2026ringmaster} propose \algname{Ringmaster LMO}, an asynchronous LMO-based momentum method for unconstrained nonconvex stochastic optimization, and translate its iteration complexity into wall-clock time under heterogeneous worker speeds, recovering the optimal time-complexity scaling of \citet{tyurin2023optimal} in the standard smooth Euclidean case. This is the method closest to ours, and the two are best read as addressing different questions.
The key difference lies in the \emph{noise model}, which is the subject of the present paper: they assume bounded variance in the Euclidean norm and transfer it to the update geometry through the norm-equivalence constant $\rho \eqdef \sup_{z\neq 0}\nfr{\snorm{z}}{\|z\|_2}$, so that the product $\rho\sigma$ enters their complexity bounds---precisely the conversion that \Cref{as:non_euclidean_light_tail} is designed to avoid by measuring noise directly in the induced dual norm $\snorm{\cdot}$.

\subsection{Federated Learning: Challenges and Algorithmic Remedies}
\label{app:fl_challenges}

Problem \eqref{eq:main_task_sum} is the standard formulation of federated learning (FL), in which a server coordinates a large number of clients that keep their data local~\citep{konevcny2016federated, mcmahan2017communication, kairouz2021advances}.

This setting presents several distinct challenges. \emph{Statistical heterogeneity} means that the local distributions $\cD_i$ differ across clients, while \emph{device} or \emph{system heterogeneity} refers to differences in compute speed, availability, and reliability. With \emph{partial participation}, only a sampled cohort is active in each round. \emph{Privacy} motivates keeping data local, while the \emph{communication bottleneck} arises because transmitting model updates is often much more expensive than computing them~\citep{sokolov2022non}.

The same nested averaging structure also appears in centralized settings: the objective averages over groups, and each group averages over its own samples. Examples include pooled data silos, out-of-core learning, and deliberate stratification. There, the aim is to exploit this structure to reduce memory use and oracle calls, with complexity depending on similarity both across groups and within each group~\citep{sokolov2026silage}.

The challenge closest to the agenda of the present paper is how computation is organized in a server-based data-parallel setting, where workers repeatedly exchange gradients or model updates with a coordinating server. Two widely studied algorithmic levers aim to reduce the communication cost of these exchanges.

\textbf{Communication compression.} The first lever is to reduce the size of each transmitted message by quantizing, sparsifying, or otherwise compressing the update while controlling the resulting error. Unbiased compressors~\citep{alistarh2017qsgd, wangni2018gradient, horvath2022natural} simplify the analysis: conditional on the vectors being compressed, independent compression randomness across workers ensures that the Euclidean second moment of the average compression error equals $1/N$ times the average of the individual error second moments.

Biased compression requires additional care. Naive compression of local gradients can lead to exponential divergence in heterogeneous distributed settings, as shown by the Top-$1$ counterexample of \citet[Example~1]{beznosikov2023biased}. Error feedback is a standard way to obtain convergence guarantees with biased compressors~\citep{Seide2014, stich2018sparsified, alistarh2018convergence, karimireddy2019error, stich2019error}. For smooth nonconvex objectives, \algname{EF21} matches the $\cO(1/T)$ iteration dependence of gradient descent, with compression-dependent constants and without bounded-gradient assumptions~\citep[Theorem~2]{richtarik2021ef21}. Subsequent developments include practical variants~\citep{fatkhullin2021ef21} and normalized methods that combine momentum with time-varying stepsizes requiring no tuning from problem constants~\citep{sadiev2025improved}. The framework of \emph{three-point compressors} connects error feedback with lazy aggregation~\citep{richtarik20223pc}.

Other work studies which messages to compress and how to exploit their structure. This includes server-to-worker compression~\citep{EF21-P, sokolov2024marinapsuperiorperformancenonsmooth} and low-rank or structured compression of transmitted gradients~\citep{vogels2019powersgd, wang2018atomo}. Compression can also be combined with parameter-efficient fine-tuning, where the update itself is restricted to low rank~\citep{hu2021loralowrankadaptationlarge}, with quantization or error feedback applied on top~\citep{sokolov2025bernoulli}.

Compression also interacts with without-replacement sampling. Its variance can mask the benefits of random reshuffling in \algname{Q-RR}, whereas \algname{DIANA-RR} improves the accuracy dependence in the strongly convex setting by compressing differences between gradients and learned shifts~\citep{sadiev2022federated, sadiev2024dont}. Lower bounds characterize the limits of communication compression~\citep{huang2022lower}.

\textbf{Local steps.} The second lever is to communicate less often, letting each client take several local update steps between synchronizations, as in \algname{FedAvg}~\citep{mcmahan2017communication}. The obstacle is client drift: local progress is made towards \emph{local} optima, and the effects of the resulting bias must be controlled~\citep{karimireddy2020scaffold, yuan2022convergence, koloskova2020unified}. Subsequent work combines local steps with probabilistic synchronization~\citep{malinovsky2022variance}, with random reshuffling~\citep{demidovich2024methodslocalstepsrandom}, and with compression itself~\citep{sadiev2022federated}. A complementary question is how much can be extracted from a single cohort: instead of the customary single communication round per sampled cohort, one may run several proximal-point rounds inside a cohort and thereby cut the total communication cost substantially~\citep{yi2024cohort}.

Our work addresses a complementary question: how \emph{computational} heterogeneity interacts with the geometry of the update. We study the wall-clock time of a fresh-minibatch protocol when workers differ in speed and updates follow steepest-descent directions in a non-Euclidean norm. Our setting is \emph{statistically homogeneous} (\Cref{sec:intro}), which removes heterogeneity-induced client drift, although local stochastic trajectories can still differ.

Our time model counts stochastic-gradient computation and treats communication in both directions and server-side processing, including evaluation of the sharp update, as instantaneous. We therefore do not analyze communication compression or local-step methods here. Combining these approaches with our treatment of worker speeds and update geometry is a natural direction for future work.

\subsection{Computational Cost of the Sharp Operator}
\label{app:sharp_operator_cost}

The two approaches discussed above reduce communication costs by changing the size or frequency of messages exchanged between workers and the server. A separate computational cost arises from evaluating the non-Euclidean update direction itself. As shown in \Cref{subsec:non_eu_update}, for the spectral update norm the canonical sharp selection is $\bG^\sharp=\nucnorm{\bG}\orth(\bG)$, where $\orth(\bG)$ is the polar factor. A direct implementation uses a compact rank-$r$ singular value decomposition (SVD), where $r=\operatorname{rank}(\bG)$, retaining only the nonzero singular values and their corresponding singular vectors. The complete orthogonal factors of a full SVD are unnecessary. This SVD-based implementation can be computationally expensive at the scale of modern models, motivating alternative algorithms for computing the polar factor. The analysis in this paper assumes exact access to the chosen sharp selection; extending its guarantees to numerical approximations requires separate control of the approximation error.

Several lines of work aim to reduce this cost. The original \Muon algorithm replaces the SVD by the computationally cheaper \algname{Newton--Schulz} iteration. Subsequent work has proposed efficient polar-decomposition subroutines~\citep{lau2025polargrad, amsel2025polar}, acceleration of the \algname{Newton--Schulz} iteration with Chebyshev-type polynomials~\citep{grishina2025accelerating}, and SVD-based orthogonalization of momentum within an adaptively updated low-dimensional subspace~\citep{refael2025sumo}. A different approach, \Dion, replaces the full-rank step by a rank-$r$ distributed orthonormalization built on power iteration and randomized Cholesky QR~\citep{ahn2025dion}.


\clearpage

\section{Additional Discussion of \texorpdfstring{\Cref{as:non_euclidean_light_tail}}{Assumption 3}}
\label{app:light_tail_discussion}

This section gives examples satisfying \Cref{as:non_euclidean_light_tail}, explains the choice of norm and tail condition, and compares the resulting guarantees with those under bounded variance.

\subsection{Examples Satisfying the Assumption}
\label{app:light_tail_examples}
\textbf{Additive perturbations.} \Cref{as:non_euclidean_light_tail} is satisfied, for instance, by the classical perturbed oracle $\nf{x;\xi}=\nabla f(x)+Z$ with $Z$ zero-mean and sub-Gaussian in the induced dual norm, that is, $\Exp{\exp\rb{\frac{\sqsnorm{Z}}{\sigma^2}}}\le\exp(1)$. More generally, our framework applies to the additively perturbed and inexact-gradient models studied in the literature~\citep{ajalloeian2020convergence, gasnikov2017universal, stonyakin2023stopping, devolder2014first} only when their errors have zero mean and satisfy the dual-norm light-tail condition \eqref{eq:dual_light_tail} with a common $\sigma>0$ for all $x\in E$.

\textbf{Bounded errors.} An unbiased oracle with almost surely bounded error, $\snorm{\nf{x;\xi}-\nabla f(x)}\le G$ for every $x\in E$ with a common $G\ge0$, also satisfies \Cref{as:non_euclidean_light_tail}. If $G>0$, this pointwise bound is sufficient for the light-tail condition with $\sigma=G$, since
$\ExpSub{\xi}{\exp\rb{\frac{\sqsnorm{\nf{x;\xi}-\nabla f(x)}}{G^2}}}\le\exp(1)$.
If $G=0$, then for each fixed $x\in E$ the oracle error vanishes almost surely. Hence the expectation in \eqref{eq:dual_light_tail} equals $1$ for any $\sigma>0$.

\subsection{Noise in the Induced Dual Norm}
\label{app:light_tail_dual_norm}
Stationarity is measured in $\snorm{\cdot}$, and the update direction $(\gk)^\sharp$ is produced from $\gk$ through the dual norm, so $\snorm{\cdot}$ is the norm in which the estimation error actually enters the descent inequality. A reference Euclidean/Frobenius light-tail assumption also suffices: first concentrate the minibatch error in that norm, then use $\snorm{z}\le\alpha^{-1}\fnorm{z}$ to transfer the error bound. This yields the statistical coefficient $\mathsf C_{\opf}=\alpha^{-2}\sigma_{\opf}^2$ in \eqref{eq:complexity_dual_native_cor_F}. The native analysis instead uses $\mathsf C_*=\kappa\sigma^2$; which coefficient is smaller depends on the geometry and the noise model (\Cref{sec:frobenius_light_tail_oracle}).

\subsection{Exponential Moments and High-Probability Guarantees}
\label{app:light_tail_high_probability}
For a prescribed finite horizon $T\ge1$, our analysis requires
$\snorm{\gk-\nfxk}\le\nfr{\sqrt{\eps}}{4}$ to hold \emph{simultaneously} for $0\le k<T$ with probability at least $1-\delta$.
We apply the concentration bound conditionally at each iterate with failure probability at most $\delta/T$, then take a union bound over the $T$ iterations.
The resulting minibatch size contains $\log(T/\delta)=\log T+\log(1/\delta)$: the union bound contributes the horizon dependence $\log T$, while the concentration tail supplies the failure-probability dependence $\log(1/\delta)$.

\subsection{Comparison with Bounded-Variance Oracles}
\label{app:light_tail_bounded_variance}
For comparison, replace the exponential-moment condition by the dual-norm variance bound
$\ExpSub{\xi}{\sqsnorm{\nf{x;\xi}-\nabla f(x)}}\le\sigma^2$ for all $x\in E$.
Retain unbiasedness, fresh i.i.d.\ samples independent of the past at each iteration, and the smoothness and lower-bound assumptions \Cref{as:non_euclidean_smoothness} and \Cref{as:lower_bound}. As above, $\kappa$ denotes a regularity constant of the dual space $(E^*, \snorm{\cdot})$.
At any fixed query point, \Cref{lem:second_moment_kappa_regular} gives
$\Exp{\sqsnorm{\frac1B\sum_{j=1}^B\eta_j}}\le\nfr{\kappa\sigma^2}{B}$ for the centered minibatch noise vectors $\eta_j$; fresh sampling gives the same bound conditionally on the history at each iterate.

With $\gamma_k\equiv\frac{1}{2L}$ and suitable choices of $T$ and $B$, the corresponding analysis guarantees
$\Exp{\sqsnorm{\nabla f(\hat x)}}\le\eps$, where $\hat x=x^R$ and $R$ is uniform on $\{0, \dots, T-1\}$ independently of the run.
For $L>0$, $\Delta>0$, and $0<\eps\le L\Delta$, the resulting oracle complexity is
$N_*=BT=\cO(\nfr{L\Delta}{\eps}+\nfr{\kappa L\Delta\sigma^2}{\eps^2})$.
This gives an expectation counterpart to our light-tail guarantees; the present paper focuses on high-probability bounds.

\clearpage
\section{Useful Auxiliary Results}
This section collects definitions and auxiliary results used throughout the proofs. We begin with convex analysis and properties of primal and dual norms, then describe squared norms and their conjugates. We next give a generalized Young inequality and compare Schatten norms. Finally, we introduce the probabilistic definitions needed for vector-valued martingales and prove a second-moment bound for martingale sums in $\kappa$-regular spaces.

Let $E$ be a finite-dimensional vector space, and let $E^*$ be its dual. For $s \in E^*$, we write $\langle s, x\rangle$ to denote the action of $s$ on $x$.

\subsection{Foundations of Convex Analysis}

To analyze non-smooth optimization objectives and generalized updates, we require the standard definition of the subgradient.
\begin{definition}[Subgradient]
Let $f: E \to \R \cup \{+\infty\}$ be proper. A vector $s \in E^*$ is a subgradient of $f$ at a point $x \in E$ if
\begin{align}
    \label{eq:subgradient_def}
    f(y) \ge f(x) + \lin{s, y-x} \quad \forall y \in E.
\end{align}
\end{definition}

The collection of all subgradients forms the subdifferential.
\begin{definition}[Subdifferential]
Let $f: E \to \R \cup \{+\infty\}$ be proper. The set of all subgradients of $f$ at $x$ is denoted $\partial f(x)$, and called the \textbf{subdifferential} of $f$ at $x$. We can think of the subdifferential as a set-valued operator $x \mapsto \partial f(x)$. If $\partial f(x) \neq \emptyset$, we say that $f$ is \textbf{subdifferentiable} at $x$. The set of all such points $x$ is denoted $\dom \partial f$.
\end{definition}

To ensure our algorithm is well-defined, we rely on the fact that the subdifferential is non-empty in the relative interior of the domain.
\begin{theorem}[Existence of subgradients]
Let $E$ be a finite-dimensional real vector space and let $f: E \to \R \cup \{+\infty\}$ be a proper convex function.
\begin{enumerate}[label=(\roman*)]
    \item For every $x \in \ri(\dom f)$,\footnote{The relative interior of a set $S$, denoted $\ri(S)$, is the interior of the set when viewed within the smallest affine subspace that contains it.} the subdifferential $\partial f(x)$ is non-empty, closed, and convex.
    \item If $x \in \operatorname{int}(\dom f)$, where the interior is taken in $E$, then $\partial f(x)$ is also compact.
\end{enumerate}
\end{theorem}

Note that inequality \eqref{eq:subgradient_def} can be rearranged into
\begin{align*}
    f(x) - \lin{s, x} \le f(y) - \lin{s, y} \quad \forall y \in E,
\end{align*}
which can be interpreted as follows: a vector $s$ is a subgradient of $f$ at $x$ if and only if $x$ minimizes the function $y \mapsto f(y) - \lin{s, y}$. We have arrived at the following result:

\begin{lemma}
Let $f: E \to \R \cup \{+\infty\}$ be proper and consider a point $x \in \dom f$. Then
\begin{align}
    s \in \partial f(x) \quad \Leftrightarrow \quad x \in \argmin_{y \in E} \{f(y) - \lin{s, y}\}.
\end{align}
\end{lemma}

We note that the set of minimizers of a function is identical to the set of maximizers of the negative of the function. Applying this simple principle, we observe that
\begin{align}
    \argmin_{y \in E} \{f(y) - \lin{s, y}\} = \argmax_{y \in E} \{\lin{s, y} - f(y)\}.
\end{align}

We also recall the fundamental relationship between a function and its convex conjugate, which allows us to invert the subdifferential mapping.
\begin{theorem}[Fenchel-Moreau-Rockafellar]
\label{thm:fenchel_moreau_rockafellar}
Let $f: E \to \R \cup \{+\infty\}$ be proper, closed and convex. Then
\begin{align*}
    s \in \partial f(x) \quad \Leftrightarrow \quad x \in \partial f^*(s).
\end{align*}
\end{theorem}

\subsection{Properties of the Primal and Dual Norms}

Our analysis depends heavily on the interplay between the primal norm $\norm{\cdot}$ and the dual norm $\snorm{\cdot}$. We first verify that the dual norm satisfies the standard axioms.
\begin{lemma}[The Dual Norm is a Norm]
\label{lemma:dual_norm_is_a_norm}
Let $(E, \norm{\cdot})$ be a finite-dimensional normed space. The dual norm $\snorm{\cdot}$, defined for any $s \in E^*$ as
\begin{align*}
    \snorm{s} \eqdef \sup_{\norm{x} \le 1} \inner{s}{x},
\end{align*}
satisfies the three axioms of a norm: definiteness, absolute homogeneity, and the triangle inequality.
\end{lemma}

\begin{proof}
We first show that $\snorm{s}$ is finite for every $s \in E^*$.
Since $E$ is finite-dimensional, its closed unit ball is compact,
and every linear functional on $E$ is continuous. Therefore,
the supremum defining the dual norm is attained and finite:
\begin{align*}
    \snorm{s}
    = \sup_{\norm{x}\le 1}\inner{s}{x}
    = \max_{\norm{x}\le 1}\inner{s}{x}
    < +\infty.
\end{align*}
We now verify the three norm axioms.

\begin{proofof}{Property (i): Definiteness}
We must show that $\snorm{s} \ge 0$, with equality if and only if $s=0$.
\begin{itemize}
    \item \textbf{Non-negativity:} 
The supremum is taken over the unit ball $\{x \in E \mid \norm{x} \le 1\}$, which contains the origin $0_E$. At this point, the dual pairing evaluates to $\inner{s}{0_E}=0$. Since the supremum over a set must be greater than or equal to the value of the objective at any point in that set, we conclude that $\snorm{s} \ge 0$.
    \item \textbf{Positive Definiteness:} 
First, if $s=0$, its dual norm is clearly zero: $\snorm{0} = \sup_{\norm{x}\le 1} \inner{0}{x} = 0$.

For the converse, assume that $\snorm{s} = 0$. By definition of the supremum, this implies that $\inner{s}{x} \le 0$ for all $x$ in the unit ball $\{x \in E \mid \norm{x} \le 1\}$. We will now show that this can only be true if $s=0$.

Assume for contradiction that $s \neq 0$. As a non-zero linear functional, $s$ cannot be zero for all inputs, so there must exist some vector $y \in E$ such that $\inner{s}{y} \neq 0$. Let's construct a specific vector $x_0$ in the unit ball. If $\inner{s}{y} > 0$, we can simply set $x_0 = \nfr{y}{\norm{y}}$. If $\inner{s}{y} < 0$, we can set $x_0 = -\nfr{y}{\norm{y}}$. In either case, we have a vector $x_0$ such that $\norm{x_0}=1$ and, crucially, $\inner{s}{x_0} = \nfr{|\inner{s}{y}|}{\norm{y}} > 0$.

The existence of such an $x_0$ contradicts our initial deduction that $\inner{s}{x} \le 0$ for all vectors in the unit ball. Therefore, the initial assumption that $s \neq 0$ must be false. We conclude that if $\snorm{s}=0$, then $s=0$.

\end{itemize}
\end{proofof}

\begin{proofof}{Property (ii): Absolute Homogeneity}
We need to show that $\snorm{ts} = |t|\snorm{s}$ for any scalar $t \in \R$.
\begin{align*}
    \snorm{ts} = \sup_{\norm{x} \le 1} \inner{ts}{x} = \sup_{\norm{x} \le 1} t\inner{s}{x}.
\end{align*}
We consider three cases for the scalar $t$:
\begin{itemize}
    \item If $t > 0$, we can pull the positive constant out of the supremum:
    \begin{align*}
        \snorm{ts} = t \sup_{\norm{x} \le 1} \inner{s}{x} = t\snorm{s} = |t|\snorm{s}.
    \end{align*}
    \item If $t < 0$, let $t = -|t|$. Then:
    \begin{align*}
        \snorm{ts} = \sup_{\norm{x} \le 1} (-|t|\inner{s}{x}) = |t| \sup_{\norm{x} \le 1} (-\inner{s}{x}).
    \end{align*}
    Let $y = -x$. The set $\{y \in E \mid \norm{y} \le 1\}$ is the same as $\{x \in E \mid \norm{x} \le 1\}$ because $\norm{-x} = \norm{x}$. Therefore, we can change the variable in the supremum:
    \begin{align*}
        \sup_{\norm{x} \le 1} (-\inner{s}{x}) = \sup_{\norm{y} \le 1} \inner{s}{y} = \snorm{s}.
    \end{align*}
    Combining these gives $\snorm{ts} = |t|\snorm{s}$.
    \item If $t = 0$, we have $\snorm{0 \cdot s} = \snorm{0} = 0$, and $|0|\snorm{s} = 0$.
\end{itemize}
The identity holds for all $t \in \R$.
\end{proofof}

\begin{proofof}{Property (iii): Triangle Inequality}
We need to show that $\snorm{s_1 + s_2} \le \snorm{s_1} + \snorm{s_2}$. The proof follows directly from the definition of the dual norm and the properties of the supremum.
\begin{align*}
    \snorm{s_1 + s_2} &= \sup_{\norm{x} \le 1} \inner{s_1 + s_2}{x} \\
    &= \sup_{\norm{x} \le 1} \left( \inner{s_1}{x} + \inner{s_2}{x} \right) \quad (\text{by linearity of the dual pairing}) \\
    &\le \sup_{\norm{x} \le 1} \inner{s_1}{x} + \sup_{\norm{x} \le 1} \inner{s_2}{x} \quad (\text{by subadditivity of supremum}) \\
    &= \snorm{s_1} + \snorm{s_2}.
\end{align*}
This completes the proof.
\end{proofof}
\end{proof}

To relate geometric properties (like regularity) between primal and dual spaces, we need to understand how norm equivalence translates to the dual space.

\begin{lemma}[Duality of Norm Equivalence]
\label{lemma:duality_of_norm_equivalence}
Let $\norm{\cdot}_a$ and $\norm{\cdot}_b$ be two norms on a finite-dimensional real vector space $E$.
The subscripts $a$ and $b$ are neutral labels for the two auxiliary norms considered in this lemma.
Let $\norm{\cdot}_{a, *}$ and $\norm{\cdot}_{b, *}$ denote their respective dual norms on $E^*$.
If
\begin{align*}
    \norm{x}_a \le \norm{x}_b \le C \norm{x}_a \quad \forall x \in E,
\end{align*}
for some constant $C \ge 1$, then their dual norms satisfy the reverse equivalence inequality:
\begin{align*}
    \frac{1}{C}\norm{s}_{a, *} \le \norm{s}_{b, *} \le \norm{s}_{a, *} \quad \forall s \in E^*.
\end{align*}
\end{lemma}
\begin{proof}
The proof relies on the definition of the dual norm as the supremum of the dual pairing over the unit ball of the primal norm. Let $B_a = \{x \in E \mid \norm{x}_a \le 1\}$ and $B_b = \{x \in E \mid \norm{x}_b \le 1\}$ be the unit balls of the respective primal norms.

\textbf{Step 1: Proving $\norm{s}_{b, *} \le \norm{s}_{a, *}$.}
The first part of the primal norm inequality is $\norm{x}_a \le \norm{x}_b$. This inequality has a direct geometric interpretation for the unit balls. If a point $y$ is in $B_b$, then $\norm{y}_b \le 1$. From the inequality, this implies that $\norm{y}_a \le 1$, which means $y$ must also be in $B_a$. Therefore, we have the set inclusion:
\begin{align*}
    B_b \subseteq B_a.
\end{align*}
Now we use the definition of the dual norms. For any $s \in E^*$:
\begin{align*}
    \norm{s}_{b, *} = \sup_{x \in B_b} \inner{s}{x}.
\end{align*}
Since we are taking the supremum of the same linear function $\inner{s}{x}$ over a smaller set ($B_b$) compared to the set $B_a$, its value cannot be larger:
\begin{align*}
    \sup_{x \in B_b} \inner{s}{x} \le \sup_{x \in B_a} \inner{s}{x}.
\end{align*}
By definition, the right-hand side is $\norm{s}_{a, *}$. Thus, we have shown:
\begin{align*}
    \norm{s}_{b, *} \le \norm{s}_{a, *}.
\end{align*}

\textbf{Step 2: Proving $\frac{1}{C}\norm{s}_{a, *} \le \norm{s}_{b, *}$.}
The second part of the primal norm inequality is $\norm{x}_b \le C \norm{x}_a$. This can be rewritten as $\frac{1}{C}\norm{x}_b \le \norm{x}_a$.
This implies a relationship between the unit balls in the opposite direction. If a point $y$ is in $B_a$, then $\norm{y}_a \le 1$. From the inequality, we have $\norm{y}_b \le C \cdot \norm{y}_a \le C \cdot 1 = C$. This means that any point in $B_a$ is also in the ball of radius $C$ for the $\norm{\cdot}_b$ norm. This gives the set inclusion:
\begin{align*}
    B_a \subseteq C \cdot B_b = \{x \in E \mid \norm{x}_b \le C \}.
\end{align*}
Now we consider the definition of the dual norm $\norm{s}_{a, *}$:
\begin{align*}
    \norm{s}_{a, *} = \sup_{x \in B_a} \inner{s}{x}.
\end{align*}
Since $B_a$ is a subset of $C \cdot B_b$, the supremum over $B_a$ cannot be larger than the supremum over $C \cdot B_b$:
\begin{align*}
    \sup_{x \in B_a} \inner{s}{x} \le \sup_{x \in C \cdot B_b} \inner{s}{x}.
\end{align*}
The supremum of a linear functional over a scaled ball is the scaling factor times the supremum over the unscaled ball:
\begin{align*}
    \sup_{x \in C \cdot B_b} \inner{s}{x} = \sup_{y \in B_b} \inner{s}{Cy} = C \sup_{y \in B_b} \inner{s}{y} = C \norm{s}_{b, *}.
\end{align*}
Combining these steps, we have shown:
\begin{align*}
    \norm{s}_{a, *} \le C \norm{s}_{b, *}.
\end{align*}
Since $C \ge 1$, we can divide by $C$ to get the final inequality:
\begin{align*}
    \frac{1}{C}\norm{s}_{a, *} \le \norm{s}_{b, *}.
\end{align*}
Putting the two parts together completes the proof of the dual norm equivalence.
\end{proof}

We often require the existence of a dual vector that "certifies" the norm of a primal vector. This is guaranteed by the Hahn-Banach theorem.

\begin{lemma}[Consequence of Hahn-Banach theorem]
\label{lem:hahn_banach_consequence}
Let $(E, \norm{\cdot})$ be a nonzero finite-dimensional real normed space.
For every $x \in E$, there exists $s \in E^*$ such that
\begin{align*}
    \inner{s}{x} = \norm{x}
    \quad\text{and}\quad
    \snorm{s} = 1.
\end{align*}
\end{lemma}

For completeness, we also state the reflexivity property for finite-dimensional normed spaces.

\begin{theorem}[Bidual norm equals primal norm]
\label{thm:bidual_norm_equals_primal}
Let $\norm{\cdot}_{**}$ be the dual of the dual norm $\snorm{\cdot}$, called the bidual norm. Then
\begin{align*}
    \norm{\cdot}_{**} = \norm{\cdot}.
\end{align*}
\end{theorem}
\begin{proof}
The proof consists of two parts, showing the inequality in both directions.
\begin{enumerate}[label=(\roman*)]
    \item We first show that $\norm{x}_{**} \le \norm{x}$. Using the definition of the bidual norm and applying the duality inequality, we get
    \begin{align*}
        \norm{x}_{**} = \sup_{\snorm{s} \le 1} \lin{s, x} \le \sup_{\snorm{s} \le 1} \snorm{s}\norm{x} = \norm{x}.
    \end{align*}

    \item We will now prove that $\norm{x}_{**} \ge \norm{x}$. This holds trivially for $x=0$. So, let us consider $x \neq 0$. The definition of the bidual norm is
    \begin{align*}
        \norm{x}_{**} = \sup_{\snorm{s} \le 1} \lin{s, x}.
    \end{align*}
    Using \Cref{lem:hahn_banach_consequence}, we know that there exists $s' \in E^*$ such that $\snorm{s'} = 1$ and $\lin{s', x} = \norm{x}$. Since this specific $s'$ is in the set over which the supremum is taken, the supremum must be at least the value obtained at $s'$. Therefore,
    \begin{align*}
        \norm{x}_{**} = \sup_{\snorm{s} \le 1} \lin{s, x} \ge \lin{s', x} = \norm{x}.
    \end{align*}
\end{enumerate}
Combining both inequalities proves the theorem.
\end{proof}

The following two lemmas characterize the subdifferentials of the norm function itself. These are building blocks for deriving the subdifferentials of squared norms.

\begin{lemma}[Norm of the Subgradient of a Norm]
\label{lemma:norm_of_subgradient}
Let $(E, \norm{\cdot})$ be a finite-dimensional real normed space, and let $\snorm{\cdot}$ be its dual norm.
For every nonzero $s \in E^*$ and every $x \in \partial\snorm{s}$,\footnote{We write $\partial\snorm{s}$ for $\partial(\snorm{\cdot})(s)$, the subdifferential of the function $u \mapsto \snorm{u}$ evaluated at $s$. Similarly, $\partial\norm{x}$ denotes $\partial(\norm{\cdot})(x)$.}
we have $\norm{x}=1$.
\end{lemma}
\begin{proof}
The proof relies on two fundamental inequalities that characterize the subdifferential of a convex function.

Let $s \in E^*$ be a non-zero vector. By definition, a vector $x \in E$ is a subgradient of the dual norm $\snorm{\cdot}$ at $s$ if the following inequality holds for all $z \in E^*$:
\begin{align}
    \label{eq:subgradient_inequality}
    \snorm{z} \ge \snorm{s} + \inner{x}{z-s}.
\end{align}

\textbf{Step 1: Proving $\norm{x} \le 1$.}
We can choose a specific value for $z$ in the subgradient inequality. Let us choose $z = s+h$ for any $h \in E^*$. The inequality becomes $\snorm{s+h} \ge \snorm{s} + \inner{x}{h}$. By the triangle inequality, $\snorm{s+h} \le \snorm{s} + \snorm{h}$. This gives:
\begin{align*}
    \snorm{s} + \snorm{h} \ge \snorm{s} + \inner{x}{h} \implies \inner{x}{h} \le \snorm{h}.
\end{align*}
This inequality must hold for all $h \in E^*$. By the definition of the primal norm as the dual of the dual norm, we have:
\begin{align*}
    \norm{x} = \sup_{h \in E^*, \snorm{h} \le 1} \inner{x}{h}.
\end{align*}
Since $\inner{x}{h} \le \snorm{h}$, this implies that for any $h$ with $\snorm{h} \le 1$, we have $\inner{x}{h} \le 1$. Therefore, the supremum must be less than or equal to 1, which proves $\norm{x} \le 1$.

\textbf{Step 2: Proving $\norm{x} \ge 1$.}
Now we choose a different value for $z$ in the subgradient inequality. Let us choose $z = 0$. Since $\snorm{0}=0$, substituting into \eqref{eq:subgradient_inequality} gives:
\begin{align*}
    0 \ge \snorm{s} + \inner{x}{0-s} \\
    0 \ge \snorm{s} - \inner{x}{s}.
\end{align*}
Rearranging this gives the inequality:
\begin{align*}
    \inner{s}{x} \ge \snorm{s}.
\end{align*}
Now, by the duality inequality, we know that $\inner{s}{x} \le \snorm{s}\norm{x}$. Combining these two results:
\begin{align*}
    \snorm{s} \le \inner{s}{x} \le \snorm{s}\norm{x}.
\end{align*}
This chain of inequalities implies $\snorm{s} \le \snorm{s}\norm{x}$. Since we assumed $s$ is a non-zero vector, its norm $\snorm{s}$ is strictly positive. We can therefore divide both sides by $\snorm{s}$ without changing the direction of the inequality:
\begin{align*}
    1 \le \norm{x}.
\end{align*}
Combining the results from Step 1 ($\norm{x} \le 1$) and Step 2 ($\norm{x} \ge 1$), we conclude that we must have exactly $\norm{x}=1$.
\end{proof}

\begin{lemma}[Subdifferential of the norm]
\label{lem:subdifferential_of_the_norm}
The subdifferentials of $\norm{\cdot}$ and $\snorm{\cdot}$ are given by
\begin{align}
    \partial\norm{x} =
    \begin{cases}
        \{s \in E^* \mid \snorm{s} \le 1\} & \text{if } x = 0 \\
        \{s \in E^* \mid \inner{s}{x} = \norm{x} \text{ and } \snorm{s} = 1\} & \text{if } x \neq 0
    \end{cases}, \label{eq:subdiff_norm_primal} \\
    \partial\snorm{s} =
    \begin{cases}
        \{x \in E \mid \norm{x} \le 1\} & \text{if } s = 0 \\
        \{x \in E \mid \inner{s}{x} = \snorm{s} \text{ and } \norm{x} = 1\} & \text{if } s \neq 0
    \end{cases}. \label{eq:subdiff_norm_dual}
\end{align}
\end{lemma}
\begin{proof}
We prove \eqref{eq:subdiff_norm_primal}. The dual formula follows by applying the primal formula to $E^*$ and identifying the bidual norm with the primal norm using \Cref{thm:bidual_norm_equals_primal}. By definition,
\begin{align}
    s \in \partial \norm{x} \quad &\Leftrightarrow \quad \norm{y} \ge \norm{x} + \inner{s}{y-x} \quad \forall y \in E \nonumber \\
    &\Leftrightarrow \quad \inner{s}{x} - \norm{x} \ge \inner{s}{y} - \norm{y} \quad \forall y \in E \nonumber \\
    &\Leftrightarrow \quad x \in \argmax_{y \in E} \{\inner{s}{y} - \norm{y}\}. \label{eq:argmax_char_subdiff_norm}
\end{align}
We know from the proof of Lemma \ref{lem:fenchel_of_norm} that
\begin{align*}
    \sup_{y \in E} \{\inner{s}{y} - \norm{y}\} = 
    \begin{cases}
        +\infty & \text{when } \snorm{s} > 1 \\
        0 & \text{when } \snorm{s} \le 1
    \end{cases}.
\end{align*}
Hence, for $x$ to be a maximizer, we must have that the maximum is finite (so $\snorm{s} \le 1$) and that the value $\inner{s}{x} - \norm{x}$ is equal to this maximum value of 0. Therefore,
\begin{align*}
    x \in \argmax_{y \in E} \{\inner{s}{y} - \norm{y}\} \quad \Leftrightarrow \quad \inner{s}{x} = \norm{x} \quad \text{and} \quad \snorm{s} \le 1.
\end{align*}
We now consider two cases for $x$.

If $x=0$, the condition $\inner{s}{x} = \norm{x}$ is automatically satisfied. Thus, the characterization becomes:
\begin{align}
    0 \in \argmax_{y \in E} \{\inner{s}{y} - \norm{y}\} \quad \Leftrightarrow \quad \snorm{s} \le 1. \label{eq:subdiff_norm_case_x_zero}
\end{align}
If $x \neq 0$, by the duality inequality, it is impossible for $\inner{s}{x} = \norm{x}$ to hold unless $\snorm{s} \ge 1$. Combining this with the condition $\snorm{s} \le 1$, we must have $\snorm{s} = 1$. Hence,
\begin{align}
    x \in \argmax_{y \in E} \{\inner{s}{y} - \norm{y}\} \quad \Leftrightarrow \quad \inner{s}{x} = \norm{x} \quad \text{and} \quad \snorm{s} = 1. \label{eq:subdiff_norm_case_x_nonzero}
\end{align}
It only remains to combine \eqref{eq:argmax_char_subdiff_norm} with \eqref{eq:subdiff_norm_case_x_zero} and \eqref{eq:subdiff_norm_case_x_nonzero} to establish the result.
\end{proof}

It is also useful to identify the convex conjugate of a norm, which corresponds to the indicator function of the dual unit ball.

\begin{lemma}[Fenchel conjugate of the norm]
\label{lem:fenchel_of_norm}
The Fenchel conjugate of the norm $\norm{\cdot}$ is the indicator function of the unit ball with respect to the dual norm. That is,
\begin{align*}
    (\norm{\cdot})^*(s) = \iota_{\mathbb{B}_*}(s) :=
    \begin{cases}
        0 & \text{if } \snorm{s} \le 1 \\
        +\infty & \text{if } \snorm{s} > 1
    \end{cases},
\end{align*}
where $\mathbb{B}_* := \{s \in E^* \mid \snorm{s} \le 1\}$.
\end{lemma}
\begin{proof}
We consider two cases based on the value of $\snorm{s}$.

\textbf{Case 1: $\snorm{s} \le 1$}. By the definition of the Fenchel conjugate and the duality inequality, we have:
\begin{align*}
    (\norm{\cdot})^*(s) &= \sup_{y \in E} \{\inner{s}{y} - \norm{y}\} \\
    &\le \sup_{y \in E} \{\snorm{s}\norm{y} - \norm{y}\} \\
    &\le \sup_{y \in E} \{\norm{y} - \norm{y}\} = 0.
\end{align*}
On the other hand, by choosing $y=0$, we see that the expression $\inner{s}{y} - \norm{y}$ evaluates to $0$. Since the supremum must be at least this value, we have $(\norm{\cdot})^*(s) \ge 0$. Combining these, we conclude that $(\norm{\cdot})^*(s) = 0$.

\textbf{Case 2: $\snorm{s} > 1$}. Since $\snorm{s} = \sup_{\norm{x} \le 1} \inner{s}{x}$, there must exist some $x \in E$ such that $\norm{x} \le 1$ and $\inner{s}{x} > 1$. For such an $x$, it follows that $\inner{s}{x} - \norm{x} > 0$. Therefore, we can lower-bound the supremum:
\begin{align*}
    (\norm{\cdot})^*(s) &= \sup_{y \in E} \{\inner{s}{y} - \norm{y}\} \\
    &\ge \sup_{t > 0} \{\inner{s}{tx} - \norm{tx}\} \\
    &= \sup_{t > 0} \{t(\inner{s}{x} - \norm{x})\}.
\end{align*}
Since the term $(\inner{s}{x} - \norm{x})$ is a positive constant, the supremum over $t>0$ is $+\infty$.
\end{proof}

\subsection{Properties of Squared Norms and their Conjugates}\label{subsec:sqnorm_conj}

In this section, we establish fundamental properties relating the squared norm function $\phi(x) \eqdef \frac{1}{2}\sqnorm{x}$ and its Fenchel conjugate $\phi^*(s) \eqdef \frac{1}{2}\sqsnorm{s}$ to their respective subgradients. These identities are central to the analysis of non-Euclidean methods and will be used extensively in our proofs.

First, we state the classic inequality bounding the dual pairing by the sum of a function and its conjugate.

\begin{theorem}[Fenchel-Young inequality]\label{thm:fenchel_young_inequality}
Let $f: E \to \R \cup \{+\infty\}$ be proper, and define its Fenchel conjugate by
\begin{align*}
    f^*(s) \eqdef \sup_{y \in E}\{\inner{s}{y} - f(y)\},
    \qquad s \in E^*.
\end{align*}
Then, for any $x \in \dom f$ and $s \in E^*$,
\begin{align}\label{eq:fenchel_young_inequality}
    \lin{s, x} \le f^*(s) + f(x).
\end{align}
Moreover,
\begin{align}\label{eq:fenchel_young_equality}
    s \in \partial f(x) \quad \Leftrightarrow \quad \lin{s, x} = f^*(s) + f(x).
\end{align}
\end{theorem}

We next characterize the subdifferentials of the squared primal and dual norms. These are generally set-valued; the sharp operator selects an element of $\partial\phi^*(s)$ for each $s\in E^*$, where $\phi^*(s)=\frac12\sqsnorm{s}$.

\begin{lemma}[Subdifferential of $\phi$ and $\phi^*$]
\label{lemma:subdiff_phi_from_notes}
The subdifferentials of $\phi(x) \eqdef \frac{1}{2}\sqnorm{x}$ and $\phi^*(s) \eqdef \frac{1}{2}\sqsnorm{s}$ are given by:
\begin{enumerate}[label=(\roman*)]
    \item $\partial\phi(x) = \{s \in E^* \mid \inner{s}{x} = \snorm{s}\norm{x} \text{ and } \snorm{s} = \norm{x}\}$.
    \item $\partial\phi^*(s) = \{x \in E \mid \inner{s}{x} = \snorm{s}\norm{x} \text{ and } \norm{x} = \snorm{s}\}$.
\end{enumerate}
\end{lemma}

\begin{proof}
We first verify that $\phi^*$ is the Fenchel conjugate of $\phi$.
By the duality inequality, for every $s \in E^*$,
\begin{align*}
    \sup_{x \in E}\left\{\inner{s}{x} - \frac12\sqnorm{x}\right\}
    &\le \sup_{r\ge0}\left\{r\snorm{s} - \frac12 r^2\right\}
    = \frac12\sqsnorm{s}.
\end{align*}
If $s=0$, this upper bound is attained at $x=0$.
If $s\neq0$, compactness of the unit ball and continuity of the linear functional give a maximizer $u$ with $\inner{s}{u}=\snorm{s}>0$.
This maximizer satisfies $\norm{u}=1$, since otherwise scaling it to have norm one would increase $\inner{s}{u}$.
Taking $x=\snorm{s}u$, we obtain
\begin{align*}
    \inner{s}{x} - \frac12\sqnorm{x}
    = \sqsnorm{s} - \frac12\sqsnorm{s}
    = \frac12\sqsnorm{s}.
\end{align*}
Thus the supremum equals $\phi^*(s)=\frac12\sqsnorm{s}$ for every $s \in E^*$.
We now prove the two subdifferential characterizations.

\begin{enumerate}[label=(\roman*)]
    \item We first characterize the subdifferential $\partial\phi(x)$. By the Fenchel-Young equality (\Cref{thm:fenchel_young_inequality}), an element $s \in E^*$ is in $\partial\phi(x)$ if and only if the equality holds in the Fenchel-Young inequality:
    \begin{align*}
        \inner{s}{x} = \phi(x) + \phi^*(s).
    \end{align*}
    Substituting the definitions of $\phi$ and $\phi^*$ gives the equivalent condition:
    \begin{align}
        \label{eq:subdiff_phi_condition}
        \inner{s}{x} = \frac{1}{2}\sqnorm{x} + \frac{1}{2}\sqsnorm{s}.
    \end{align}
    We now establish a general chain of inequalities. By the duality inequality, we have:
    \begin{align*}
        \inner{s}{x} \le \snorm{s}\norm{x}.
    \end{align*}
    Furthermore, the standard Young's inequality for scalars, $ab \le \frac{1}{2}a^2 + \frac{1}{2}b^2$, applied with $a=\snorm{s}$ and $b=\norm{x}$, yields:
    \begin{align*}
        \snorm{s}\norm{x} \le \frac{1}{2}\sqsnorm{s} + \frac{1}{2}\sqnorm{x}.
    \end{align*}
    Combining these two results gives the chain:
    \begin{align}
        \label{eq:subdiff_phi_ineq_chain}
        \inner{s}{x} \le \snorm{s}\norm{x} \le \frac{1}{2}\sqsnorm{s} + \frac{1}{2}\sqnorm{x}.
    \end{align}
    For $s$ to be an element of $\partial\phi(x)$, the condition in \eqref{eq:subdiff_phi_condition} must hold, which means the first and last terms of the chain in \eqref{eq:subdiff_phi_ineq_chain} must be equal. This can only occur if both inequalities in the chain hold with equality.
    \begin{itemize}
        \item The first inequality becomes an equality, $\inner{s}{x} = \snorm{s}\norm{x}$, which is the equality condition in the duality inequality.
        \item The second inequality becomes an equality, $\snorm{s}\norm{x} = \frac{1}{2}\sqsnorm{s} + \frac{1}{2}\sqnorm{x}$. Rearranging gives $\sqsnorm{s} - 2\snorm{s}\norm{x} + \sqnorm{x} = 0$, which is equivalent to $(\snorm{s} - \norm{x})^2 = 0$. This holds if and only if $\snorm{s} = \norm{x}$.
    \end{itemize}
    Therefore, $s \in \partial\phi(x)$ if and only if both conditions are satisfied, which proves the first statement.

    \item The proof for $\partial\phi^*(s)$ follows from a symmetric argument. We apply the result from part (i) to the function $\phi^*(s)$ defined on the space $E^*$. The dual of $E^*$ is the bidual space $E^{**}$, and the corresponding dual norm is the bidual norm, $\norm{\cdot}_{**}$.
    
    Applying the characterization from (i) to $\phi^*$ (with the roles of the spaces and norms interchanged) gives that $x \in \partial\phi^*(s)$ if and only if:
    \begin{align*}
        x \in E^{**}, \quad \inner{x}{s} = \norm{x}_{**}\snorm{s}, \quad \text{and} \quad \norm{x}_{**} = \snorm{s}.
    \end{align*}
    As stated in Theorem \ref{thm:bidual_norm_equals_primal}, for a finite-dimensional space, there is a canonical isomorphism between $E$ and $E^{**}$, and the bidual norm is equal to the primal norm, i.e., $\norm{\cdot}_{**} = \norm{\cdot}$. We can thus identify $x \in E^{**}$ with $x \in E$ and replace the bidual norm with the primal norm. This substitution yields the desired characterization for $\partial\phi^*(s)$:
    \begin{align*}
        \partial\phi^*(s) = \{x \in E \mid \inner{s}{x} = \snorm{s}\norm{x} \text{ and } \norm{x} = \snorm{s}\}.
    \end{align*}
\end{enumerate}
This completes the proof.
\end{proof}

Nonemptiness follows directly from \Cref{lem:hahn_banach_consequence}.
For $x\neq0$, choose $s_0\in E^*$ with $\snorm{s_0}=1$ and $\inner{s_0}{x}=\norm{x}$; then $s=\norm{x}s_0$ satisfies the conditions in part (i).
For $x=0$, take $s=0$.
The same argument on $E^*$ establishes nonemptiness of $\partial\phi^*(s)$.

For $s\neq0$, every $x\in\partial\phi^*(s)$ has $\norm{x}=\snorm{s}>0$.
Writing $x=\snorm{s}u$ gives
\begin{align}
    \partial\phi^*(s)
    &= \{x \in E \mid \inner{s}{x} = \sqsnorm{s}, \ \norm{x} = \snorm{s}\} \nonumber \\
    &= \snorm{s}\cdot\{u \in E \mid \norm{u}=1, \ \inner{s}{u}=\snorm{s}\} \nonumber \\
    &= \snorm{s}\cdot\partial\snorm{s}, \label{eq:subdiff_phi_star_vs_eta_star}
\end{align}
where the last equality follows from \Cref{lem:subdifferential_of_the_norm}.
At $s=0$, the characterization above gives $\partial\phi^*(0)=\{0\}$, and $\partial\snorm{0}$ is the primal unit ball, so \eqref{eq:subdiff_phi_star_vs_eta_star} also holds at zero.

\begin{lemma}[Subdifferential of the Squared Dual Norm]
\label{lemma:subdiff_squared_norm}
Let $\phi^*(s) \eqdef \frac{1}{2}\sqsnorm{s}$. For any $s \in E^*$, its subdifferential is given by
\begin{align*}
    \partial \phi^*(s) = \snorm{s} \cdot \partial \snorm{s},
\end{align*}
where $\partial \snorm{s}$ is the subdifferential of the dual norm at $s$. If the dual norm is differentiable at $s$, then this relationship simplifies to $\nabla \phi^*(s) = \snorm{s} \cdot \nabla \snorm{s}$.
\end{lemma}

\begin{proof}
For $s\neq0$, the identity follows from \eqref{eq:subdiff_phi_star_vs_eta_star}, obtained by normalizing the subgradients characterized in \Cref{lemma:subdiff_phi_from_notes}.

At $s=0$, \Cref{lemma:subdiff_phi_from_notes} gives $\partial\phi^*(0)=\{0\}$, while \Cref{lem:subdifferential_of_the_norm} gives
\begin{align*}
    \partial\snorm{0}=\{x\in E:\norm{x}\le1\}.
\end{align*}
Consequently,
\begin{align*}
    \partial\phi^*(0)=\{0\}=\snorm{0}\cdot\partial\snorm{0}.
\end{align*}
Thus the subdifferential identity holds for every $s\in E^*$.

If the dual norm is differentiable at $s$, the ordinary chain rule applied to $\phi^*(s)=\frac12\sqsnorm{s}$ yields
\begin{align*}
    \nabla\phi^*(s)=\snorm{s}\cdot\nabla\snorm{s}.
\end{align*}
\end{proof}

These subdifferential mappings satisfy useful scaling properties.

\begin{lemma}[Homogeneity of $\partial\phi$ and $\partial\phi^*$]
\label{lem:homogeneity_subdifferential}
The subdifferentials of $\phi$ and $\phi^*$ satisfy the relations
\begin{align*}
    \partial\phi(tx) &= t\partial\phi(x), \quad \forall x \in E, t \in \R, \\
    \partial\phi^*(ts) &= t\partial\phi^*(s), \quad \forall s \in E^*, t \in \R.
\end{align*}
In particular, $\partial\phi(0) = \{0\}$ and $\partial\phi^*(0) = \{0\}$.
\end{lemma}

Finally, we derive key identities involving the sharp operator $(\cdot)^\sharp$, which are used directly in the derivation of the descent lemma.

\begin{lemma}[Properties of the Sharp Operators]
\label{lemma:sharp_flat_properties}
Let $x \in E$ and $s \in E^*$. The following four identities hold for any element selected from the respective subdifferentials:
\begin{enumerate}[label=(\roman*)]
    \item For any $u \in \partial\phi^*(s)$, we have $\inner{s}{u} = \sqsnorm{s}$.
    \item For any $u \in \partial\phi^*(s)$, we have $\norm{u} = \snorm{s}$.
    \item For any $v \in \partial\phi(x)$, we have $\inner{v}{x} = \sqnorm{x}$.
    \item For any $v \in \partial\phi(x)$, we have $\snorm{v} = \norm{x}$.
\end{enumerate}
\end{lemma}

\begin{proof}
Let $u \in \partial\phi^*(s)$ and $v \in \partial\phi(x)$ be arbitrary.
By \Cref{lemma:subdiff_phi_from_notes}(ii),
\begin{align*}
    \norm{u} = \snorm{s},
    \qquad
    \inner{s}{u} = \snorm{s}\norm{u} = \sqsnorm{s},
\end{align*}
which proves (i) and (ii).
Similarly, \Cref{lemma:subdiff_phi_from_notes}(i) gives
\begin{align*}
    \snorm{v} = \norm{x},
    \qquad
    \inner{v}{x} = \snorm{v}\norm{x} = \sqnorm{x},
\end{align*}
which proves (iii) and (iv).
\end{proof}

\subsection{Inequalities and Specific Geometries}

To bound dual pairings in our convergence analysis, we require the following generalized version of Young's inequality.

\begin{lemma}[Generalized Young inequality]
Let $\|\cdot\|$ be any norm on $\R^d$ with dual norm
$\|u\|_*=\sup_{\|x\|\le 1}\langle u, x\rangle$. Then for all $u, v\in\R^d$
and every ${s}>0$,
\[
\langle u, v\rangle \;\le\; \|u\|_*\,\|v\|
\;\le\; \frac{1}{2{s}}\,\|u\|_*^2 + \frac{{s}}{2}\,\|v\|^2 .
\]
\end{lemma}
\begin{proof}
\textbf{(i) Duality (H\"older) inequality.}
If $v=0$ the claim is trivial. Otherwise, set $w:=\nfr{v}{\|v\|}$, so $\|w\|=1$.
By the definition of the dual norm,
\[
\langle u, w\rangle \le \sup_{\|x\|\le 1}\langle u, x\rangle = \|u\|_* .
\]
Multiplying by $\|v\|$ gives $\langle u, v\rangle = \|v\|\langle u, w\rangle
\le \|u\|_*\,\|v\|$.

\textbf{(ii) Young’s inequality with a parameter.}
For any $a, b\in\R$ and any ${s}>0$,
\[
0\le \Bigl(\frac{a}{\sqrt{2{s}}}-\sqrt{\frac{{s}}{2}}\,b\Bigr)^2
= \frac{a^2}{2{s}} + \frac{{s}}{2}b^2 - ab,
\]
hence $ab \le \frac{a^2}{2{s}} + \frac{{s}}{2}b^2$.
Applying this with $a=\|u\|_*$ and $b=\|v\|$ yields
\[
\|u\|_*\,\|v\| \le \frac{1}{2{s}}\,\|u\|_*^2 + \frac{{s}}{2}\,\|v\|^2 .
\]
Combining (i) and (ii) proves the statement.
\end{proof}

For our examples involving matrix optimization, we utilize the comparison inequalities between different Schatten norms.

\begin{lemma}[Schatten norm comparison]\label{lem:schatten_norm_comparison}
Let $X\in\R^{m\times n}$ and $r\eqdef\min\{m, n\}$. For any $1\le \rho\le p\le\infty$,
\begin{align*}
    \norm{X}_{S_p} \le \norm{X}_{S_{\rho}} \le r^{\frac{1}{\rho}-\frac{1}{p}}\norm{X}_{S_p}.
\end{align*}
\end{lemma}
\begin{proof}
Let $\sigma(X)\in\R_+^{r}$ be the vector of singular values of $X$. By definition, $\norm{X}_{S_q}=\norm{\sigma(X)}_q$ for every $q\in[1, \infty]$. Applying the standard $\ell_p$ comparison inequality in $\R^{r}$ to $\sigma(X)$ yields the claim.
\end{proof}

\subsection{Probabilistic Preliminaries for Martingale Bounds}
\label{sub:prob_prelim_martingales}

To rigorously analyze the behavior of stochastic gradients in non-Euclidean norms, we need tools from the theory of vector-valued martingales. We begin by recalling the necessary probabilistic definitions.

\begin{definition}[Measurable space]
\label{def:measurable_space}
A \textbf{measurable space} is a pair $(S, \cS)$, where $S$ is a set and $\cS$ is a
$\sigma$-algebra of subsets of $S$, i.e., a collection $\cS\subseteq 2^S$ such that:
\begin{enumerate}[label=(\roman*)]
    \item $S\in \cS$;
    \item if $A\in\cS$, then its complement $S\setminus A\in\cS$;
    \item if $A_1, A_2, \dots\in\cS$, then the countable union $\bigcup_{n=1}^\infty A_n\in\cS$.
\end{enumerate}
Elements of $\cS$ are called \textbf{measurable sets}.
\end{definition}

\begin{definition}[$\cF$-measurable random variable]
\label{def:F_measurable_rv}
Let $(\Omega, \cF, \bb P)$ be a probability space and let $(S, \cS)$ be a measurable space.
A map $\xi:\Omega\to S$ is called \textbf{$\cF$-measurable} (or simply \textbf{measurable})
if for every measurable set $A\in\cS$, the preimage
\[
\xi^{-1}(A)\eqdef \{\omega\in\Omega:\ \xi(\omega)\in A\}
\]
belongs to $\cF$, i.e.,
\[
\forall A\in\cS:\quad \xi^{-1}(A)\in\cF.
\]
In particular, when $S=\R$ and $\cS=\cB(\R)$ is the Borel $\sigma$-algebra,
this becomes:
\[
\forall B\in\cB(\R):\quad \{\omega\in\Omega:\ \xi(\omega)\in B\}\in\cF.
\]
\end{definition}

\begin{remark}[Generated $\sigma$-algebra]
Let $\xi:(\Omega, \cF)\to(S, \cS)$ be a measurable map. The $\sigma$-algebra generated by $\xi$ is
\[
\sigma(\xi)\eqdef \{\xi^{-1}(A):\ A\in\cS\}.
\]
It is the smallest $\sigma$-algebra $\cG\subseteq\cF$ such that $\xi$ is $\cG$-measurable.
\end{remark}

\begin{definition}[Sub-$\sigma$-algebra]
Let $(\Omega, \cF, \bb P)$ be a probability space. A collection of sets $\cG$ is called a
\textbf{sub-$\sigma$-algebra of $\cF$} if:
\begin{enumerate}[label=(\roman*)]
    \item $\cG \subseteq \cF$ (i.e., $\forall A\subseteq\Omega:\ A\in\cG \Rightarrow A\in\cF$),
    \item $\cG$ is a $\sigma$-algebra on $\Omega$, i.e.,
    $\Omega\in\cG$, $A\in\cG \Rightarrow A^c\in\cG$, and
    $A_1, A_2, \dots\in\cG \Rightarrow \bigcup_{n=1}^\infty A_n\in\cG$.
\end{enumerate}
\end{definition}

\begin{definition}[Filtration]
\label{def:filtration}
Let $(\Omega, \cF, \bb P)$ be a probability space and let $T$ be a totally ordered index set
(e.g., $T=\N_0$ or $T=[0, \infty)$). A family of $\sigma$-algebras $\{\cF_t\}_{t\in T}$ is called a
\textbf{filtration} if:
\begin{enumerate}[label=(\roman*)]
    \item $\cF_t \subseteq \cF$ for every $t\in T$ (each $\cF_t$ is a sub-$\sigma$-algebra of $\cF$);
    \item for all $s, t\in T$ with $s\le t$, one has $\cF_s \subseteq \cF_t$.
\end{enumerate}
The inclusion $\cF_s\subseteq\cF_t$ is an inclusion of \emph{collections of sets}:
it means $\forall A\subseteq\Omega$, $A\in\cF_s \Rightarrow A\in\cF_t$.

\medskip
\noindent\textbf{Generated filtrations (common construction).}
Given random variables $\{X_u\}_{u\in T}$ on $(\Omega, \cF, \bb P)$, one can define the
\textbf{filtration generated by $(X_u)$} by
\begin{align*}
    \cF_t \eqdef \sigma\big(\{X_u:\; u\le t\}\big),
\end{align*}
i.e., $\cF_t$ is the smallest $\sigma$-algebra with respect to which all $X_u$ for $u\le t$
are measurable. In particular, if $s\le t$ then $\{X_u:u\le s\}\subseteq \{X_u:u\le t\}$, and therefore
$\cF_s=\sigma(\{X_u:u\le s\})\subseteq\sigma(\{X_u:u\le t\})=\cF_t$.

We interpret $\cF_t$ as the information available up to time $t$.
\end{definition}

\begin{definition}[Martingale-difference sequence]
\label{def:mds}
Let $(\Omega, \cF, \bb P)$ be a probability space, let $\{\cF_t\}_{t\ge 0}$ be a filtration ($\cF_{t-1}\subseteq \cF_t\subseteq \cF$), and let $(V, \norm{\cdot}_V)$ be a finite-dimensional normed space. A sequence of $V$-valued random variables $\{\xi_t\}_{t\ge 1}$ is called an \textbf{$\{\cF_t\}$-martingale-difference sequence} if for every $t\ge 1$:
\begin{enumerate}[label=(\roman*)]
    \item $\xi_t$ is $\cF_t$-measurable,
    \item $\Exp{\norm{\xi_t}_V} < \infty$ (integrability),
    \item $\Exp{\xi_t \mid \cF_{t-1}} = 0$ almost surely.
\end{enumerate}
In this case, the partial sums $S_t \eqdef \sum_{i=1}^t \xi_i$ form a $V$-valued martingale with respect to $\{\cF_t\}$.
\end{definition}

The next lemma converts $\kappa$-regularity into a second-moment bound for vector-valued martingale sums.
Intuitively, $\kappa$ controls how large $\E\|S_T\|^2$ can be relative to the accumulated $\sum_{t=1}^T \E\|\xi_t\|^2$.

\begin{lemma}[Second-moment bound from $\kappa_V$-regularity]
\label{lem:second_moment_kappa_regular}
Assume that the normed space $(V, \|\cdot\|_V)$ is $\kappa_V$-regular in the sense of Definition~\ref{def:kappa_smooth_regular}.
Let $\{\xi_t\}_{t=1}^T$ be a $V$-valued $\{\cF_t\}$-martingale-difference sequence such that
\begin{align*}
    \Exp{\|\xi_t\|_V^2} < \infty,
    \qquad t=1, \dots, T.
\end{align*}
Then, for
\begin{align*}
    S_T \eqdef \sum_{t=1}^T \xi_t,
\end{align*}
one has
\begin{align}
\label{eq:second_moment_kappa_regular}
    \Exp{\|S_T\|_V^2} \le \kappa_V \sum_{t=1}^T \Exp{\|\xi_t\|_V^2}.
\end{align}
\end{lemma}

\begin{proof}
To simplify notation throughout this proof, we write
\begin{align*}
    \|x\| \eqdef \|x\|_V,
    \qquad
    \|x\|_+ \eqdef \|x\|_{V, +},
    \qquad
    \kappa \eqdef \kappa_V,
    \qquad
    \kappa_+ \eqdef \kappa_{V, +}.
\end{align*}

By $\kappa$-regularity, there exist $\kappa_+\in[1, \kappa]$ and a $\kappa_+$-smooth auxiliary norm $\|\cdot\|_+$ on $V$ such that
\begin{align}
\label{eq:plus_equivalence}
    \|x\|^2 \le \|x\|_+^2 \le \frac{\kappa}{\kappa_+}\|x\|^2,
    \qquad \forall x\in V.
\end{align}
Define the auxiliary function
\begin{align*}
    \phi_+(x) \eqdef \frac12\|x\|_+^2.
\end{align*}
Since $(V, \|\cdot\|_+)$ is $\kappa_+$-smooth, for all $x, h\in V$ we have
\begin{align}
\label{eq:smooth_ineq_phi_plus}
    \phi_+(x+h)
    \le
    \phi_+(x) + \langle \nabla \phi_+(x), h\rangle + \frac{\kappa_+}{2}\|h\|_+^2.
\end{align}

For each $t=0, 1, \dots, T$, let
\begin{align*}
    S_t \eqdef \sum_{i=1}^t \xi_i,
    \qquad
    S_0 \eqdef 0.
\end{align*}
Applying \eqref{eq:smooth_ineq_phi_plus} with $x=S_{t-1}$ and $h=\xi_t$, we get
\begin{align*}
    \phi_+(S_t)
    \le
    \phi_+(S_{t-1})
    + \langle \nabla \phi_+(S_{t-1}), \xi_t\rangle
    + \frac{\kappa_+}{2}\|\xi_t\|_+^2.
\end{align*}
Before taking conditional expectations, we verify integrability.
By \eqref{eq:plus_equivalence}, $\Exp{\|\xi_i\|_+^2}<\infty$ for every $i$.
The triangle inequality and the scalar Cauchy-Schwarz inequality give
\begin{align*}
    \Exp{\|S_t\|_+^2}
    \le t\sum_{i=1}^t\Exp{\|\xi_i\|_+^2}
    <\infty,
    \qquad t=1, \dots, T,
\end{align*}
and $S_0=0$.
Let $\|\cdot\|_{+, *}$ denote the dual norm of $\|\cdot\|_+$.
By \Cref{lemma:sharp_flat_properties}(iv), applied to $\phi_+$,
\begin{align*}
    \|\nabla\phi_+(z)\|_{+, *}=\|z\|_+,
    \qquad z\in V.
\end{align*}
Consequently, the duality inequality and Cauchy-Schwarz yield
\begin{align*}
    \Exp{\left|\langle\nabla\phi_+(S_{t-1}), \xi_t\rangle\right|}
    &\le \Exp{\|S_{t-1}\|_+\|\xi_t\|_+} \\
    &\le \left(\Exp{\|S_{t-1}\|_+^2}\right)^{1/2}
           \left(\Exp{\|\xi_t\|_+^2}\right)^{1/2}
    <\infty.
\end{align*}
Thus all terms in the preceding smoothness inequality are integrable.
Since $\nabla\phi_+$ is continuous and $S_{t-1}$ is $\cF_{t-1}$-measurable, $\nabla\phi_+(S_{t-1})$ is also $\cF_{t-1}$-measurable.
We may therefore take conditional expectations and use $\Exp{\xi_t\mid\cF_{t-1}}=0$ a.s. to obtain
\begin{align*}
    \Exp{\langle \nabla \phi_+(S_{t-1}), \xi_t\rangle \mid \cF_{t-1}}
    &=
    \left\langle \nabla \phi_+(S_{t-1}), \Exp{\xi_t\mid \cF_{t-1}}\right\rangle \\
    &= 0
    \qquad \text{a.s.}
\end{align*}
Therefore,
\begin{align*}
    \Exp{\phi_+(S_t)\mid \cF_{t-1}}
    \le
    \phi_+(S_{t-1})
    + \frac{\kappa_+}{2}\Exp{\|\xi_t\|_+^2\mid \cF_{t-1}}.
\end{align*}
Taking full expectation and summing over $t=1, \dots, T$ yields
\begin{align}
\label{eq:telescoping_phi_plus}
    \Exp{\phi_+(S_T)}
    \le
    \frac{\kappa_+}{2}\sum_{t=1}^T \Exp{\|\xi_t\|_+^2}.
\end{align}
Since $\phi_+(x)=\frac12\|x\|_+^2$, \eqref{eq:telescoping_phi_plus} is equivalent to
\begin{align}
\label{eq:second_moment_plus}
    \Exp{\|S_T\|_+^2}
    \le
    \kappa_+\sum_{t=1}^T \Exp{\|\xi_t\|_+^2}.
\end{align}

Finally, using \eqref{eq:plus_equivalence} and then \eqref{eq:second_moment_plus}, we obtain
\begin{align*}
    \Exp{\|S_T\|^2}
    &\le
    \Exp{\|S_T\|_+^2} \\
    &\le
    \kappa_+\sum_{t=1}^T \Exp{\|\xi_t\|_+^2} \\
    &\le
    \kappa_+\sum_{t=1}^T \Exp{\frac{\kappa}{\kappa_+}\|\xi_t\|^2} \\
    &=
    \kappa \sum_{t=1}^T \Exp{\|\xi_t\|^2},
\end{align*}
which is exactly \eqref{eq:second_moment_kappa_regular}.
\end{proof}


\clearpage
\section{Sharp Operator for \texorpdfstring{$\ell_1$}{l1} and \texorpdfstring{$\ell_\infty$}{linf} Update Norms}\label{app:sharp_l1_linf}

In this section, we specialize the abstract definition of the sharp operator
$s^\sharp \in \partial \phi^*(s)$, where $\phi^*(s)\eqdef \frac12\snorm{s}^2$,
to the nonsmooth update norms $\onenorm{\cdot}$ and $\infnorm{\cdot}$ on $\R^d$.
Throughout, we identify $E=\R^d$ and $E^*=\R^d$ via the standard pairing
$\inner{s}{x}=\sum_{i=1}^d s_i x_i$, and we write $\{e_i\}_{i=1}^d$ for the canonical basis of $\R^d$.

\subsection{\texorpdfstring{$\ell_1$}{l1} Update Norm: Coordinate-Type Steps}\label{app:sharp_l1_coordinate}

Assume the update norm is $\norm{\cdot}=\onenorm{\cdot}$, so the induced dual norm is $\snorm{\cdot}=\infnorm{\cdot}$.
Then the squared dual-norm potential is $\phi^*(s)=\frac12\infnorm{s}^2$.
By \Cref{lemma:subdiff_squared_norm}, for every $s\in\R^d$,
\begin{align}\label{eq:subdiff_phi_linf_basic}
    \partial \phi^*(s) = \infnorm{s} \cdot \partial \infnorm{s} .
\end{align}

An explicit description of $\partial\infnorm{s}$ (and hence of $\partial\phi^*(s)$) follows from a standard result in convex analysis; for completeness, we state it below.

\begin{proposition}[Subdifferential of $\|\cdot\|_\infty$]\label{prop:subdiff_linf_explicit}
For $s\in\R^d$, define the active set
\begin{align*}
I(s)\eqdef \left\{ i\in[d] : |s_i|=\infnorm{s} \right\}.
\end{align*}
If $s=\mathbf 0$, then $\partial\infnorm{s}=\{x\in\R^d:\onenorm{x}\le 1\}$.
If $s\neq \mathbf 0$, then
\begin{align}\label{eq:subdiff_linf_explicit}
    \partial\infnorm{s} &= \operatorname{conv}\left\{\sign(s_i)\,e_i:\ i\in I(s)\right\} \nonumber \\
    &= \left\{ \sum_{i\in I(s)} \alpha_i \sign(s_i)\, e_i : \alpha_i\ge 0, \ \sum_{i\in I(s)}\alpha_i=1 \right\}.
\end{align}
\end{proposition}
\begin{proof}
See \citet[Example~3.52]{beck2017first}.
\end{proof}

\begin{corollary}[Sharp operator under $\|\cdot\|_1$]\label{cor:sharp_l1_coordinate}
Assume $\norm{\cdot}=\onenorm{\cdot}$ and $\snorm{\cdot}=\infnorm{\cdot}$.
For $s\in\R^d$ with active set $I(s)$ as above,
\begin{align}\label{eq:sharp_l1_coordinate_set}
    \partial\phi^*(s) = \infnorm{s}\,\partial\infnorm{s} = \infnorm{s}\cdot \operatorname{conv}\left\{\sign(s_i)\,e_i:\ i\in I(s)\right\}.
\end{align}
In particular, if $s\neq \mathbf 0$, then every selection $s^\sharp\in\partial\phi^*(s)$ satisfies:
\begin{enumerate}[label=(\roman*)]
    \item $s^\sharp_i=0$ for all $i\notin I(s)$ (support contained in the maximizers of $|s_i|$),
    \item $\onenorm{s^\sharp}=\infnorm{s}$,
    \item $s^\sharp_i\,s_i\ge 0$ for all $i$ (sign alignment on the active set).
\end{enumerate}
Moreover, any extreme-point selection has the coordinate form
\begin{align}\label{eq:sharp_l1_coordinate_extreme}
    s^\sharp=\infnorm{s}\,\sign(s_{i^\star})\,e_{i^\star}, \qquad i^\star\in I(s),
\end{align}
and yields a single-coordinate update when substituted into \eqref{eq:non-eu-one-step-final} (with $s=\gk$).
\end{corollary}

\begin{proof}
The identity \eqref{eq:sharp_l1_coordinate_set} follows from \eqref{eq:subdiff_phi_linf_basic} and \Cref{prop:subdiff_linf_explicit}.
Properties (i)--(iii) follow immediately from the convex-hull representation in \eqref{eq:sharp_l1_coordinate_set}.
Finally, the extreme points of the simplex $\{\alpha\ge 0:\sum_{i\in I(s)}\alpha_i=1\}$ correspond to placing all mass on one index $i^\star\in I(s)$,
which gives \eqref{eq:sharp_l1_coordinate_extreme}.
\end{proof}

\subsection{\texorpdfstring{$\ell_\infty$}{linf} Update Norm: Sign-Type Steps}\label{app:sharp_linf_sign}

Assume the update norm is $\norm{\cdot}=\infnorm{\cdot}$, so the induced dual norm is $\snorm{\cdot}=\onenorm{\cdot}$.
Then $\phi^*(s)=\frac12\onenorm{s}^2$ and \Cref{lemma:subdiff_squared_norm} yields
\begin{align}\label{eq:subdiff_phi_l1_basic}
    \partial\phi^*(s)=\onenorm{s}\cdot \partial\onenorm{s}.
\end{align}

Similarly, one can characterize $\partial\onenorm{s}$ explicitly.

\begin{proposition}[Subdifferential of $\|\cdot\|_1$]
\leavevmode\phantomsection 
\label{prop:subdiff_l1_explicit}
Let $s\in\R^d$. If $s=\mathbf 0$, then $\partial\onenorm{s}=\{x\in\R^d:\infnorm{x}\le 1\}$.
If $s\neq \mathbf 0$, then
\begin{align}\label{eq:subdiff_l1_explicit}
    \partial\onenorm{s} = \left\{x\in\R^d: x_i=\sign(s_i)\ \text{for all } i \text{ with } s_i\neq 0, \ \ |x_i|\le 1\ \text{for all } i \text{ with } s_i=0 \right\}.
\end{align}
Equivalently, $\partial\onenorm{s}$ consists of the componentwise sign vectors of $s$, with an arbitrary choice in $[-1, 1]$ on coordinates where $s_i=0$.
\end{proposition}
\begin{proof}
See \citet[Example~3.41]{beck2017first}.
\end{proof}

\begin{corollary}[Sharp operator under $\|\cdot\|_\infty$]\label{cor:sharp_linf_sign}
Assume $\norm{\cdot}=\infnorm{\cdot}$ and $\snorm{\cdot}=\onenorm{\cdot}$.
Then for any $s\in\R^d$,
\begin{align}\label{eq:sharp_linf_sign_set}
    \partial\phi^*(s)=\onenorm{s}\cdot\partial\onenorm{s}.
\end{align}
In particular, if $s\neq \mathbf 0$ then any selection $s^\sharp\in\partial\phi^*(s)$ satisfies
\begin{align}\label{eq:sharp_linf_sign_form}
    s^\sharp = \onenorm{s} \cdot v, \qquad v_i=\sign(s_i)\ \text{for } s_i\neq 0, \ \ |v_i|\le 1\ \text{for } s_i=0.
\end{align}
A canonical selection is $v=\sign(s)$, where $\sign$ is applied componentwise with $\sign(0)=0$. This yields the sign-type update
\begin{align*}
x^{k+1} = x^k - \gamma_k \onenorm{\gk}\,\sign(\gk),
\end{align*}
closely related to \algname{signSGD}-style directions~\citep{bernstein2018signsgd,sun2023momentum}.
\end{corollary}

\begin{proof}
This follows immediately from \eqref{eq:subdiff_phi_l1_basic} and \Cref{prop:subdiff_l1_explicit}.
\end{proof}

\clearpage
\section{Derivation of the Non-Euclidean \texorpdfstring{\algname[]{SGD}}{SGD} Update}
\label{app:non_euclidean_sgd_derivation}
In this section, we provide a mathematical justification for the non-Euclidean \algname{SGD} update rule used in the main text:
\begin{align}
    \label{eq:app_update_rule}
    \xkpo = \xk - \gamma_k \rb{\gk}^{\sharp}.
\end{align}
We show that the steepest-descent and Linear Minimization Oracle (LMO) formulations characterize the same set of admissible steps, from which the fixed sharp selection determines the update.

\subsection{Derivation from Steepest Descent}\label{app:steepest_derivation}
The steepest descent method generates the next iterate $\xkpo = \xk + h_k$ by minimizing a linearized model regularized by the squared norm. In minibatch \SGD, the linear term uses the stochastic gradient $\gk$, and the stepsize $\gamma_k$ sets the regularization weight $1/(2\gamma_k)$. The step $h_k$ is a solution to the following optimization problem:
\begin{align}
    \label{eq:app_steepest_descent_problem}
    h_k \in \argmin_{h \in E} \left\{ \inner{\gk}{h} + \frac{1}{2\gamma_k}\sqnorm{h} \right\},
\end{align}
where $\gamma_k > 0$ and $\norm{\cdot}$ is an arbitrary norm on the space $E$. The objective function in \eqref{eq:app_steepest_descent_problem} is convex, and its first-order optimality condition states that the zero vector must be in the subdifferential of the objective at $h_k$. Letting $\phi(h) \eqdef \frac{1}{2}\sqnorm{h}$, this condition is:
\begin{align*}
    0 &\in \gk + \frac{1}{\gamma_k} \partial\phi(h_k).
\end{align*}
Rearranging this expression, we get:
\begin{align*}
    -\gamma_k \gk \in \partial\phi(h_k).
\end{align*}
This is a condition on the subdifferential of the primal squared norm $\phi$. We can invert this relationship using the Fenchel-Moreau-Rockafellar theorem (Theorem \ref{thm:fenchel_moreau_rockafellar}), which states that $s \in \partial\phi(x) \Leftrightarrow x \in \partial\phi^*(s)$. Applying this gives:
\begin{align*}
    h_k \in \partial\phi^*(-\gamma_k \gk).
\end{align*}
Using the homogeneity of the subdifferential of the squared norm (Lemma \ref{lem:homogeneity_subdifferential}), which states that $\partial\phi^*(ts) = t\partial\phi^*(s)$ for all $t \in \R$, we can factor out the stepsize $\gamma_k$:
\begin{align*}
    h_k \in -\gamma_k \partial\phi^*(\gk).
\end{align*}
All preceding conditions are equivalent, so the full set of minimizers is $-\gamma_k\partial\phi^*(\gk)$. For the fixed sharp selection from the main text, choose $h_k=-\gamma_k(\gk)^\sharp$; this is a minimizer because $(\gk)^\sharp\in\partial\phi^*(\gk)$. Substituting this choice into $\xkpo=\xk+h_k$ gives the specified non-Euclidean \algname{SGD} update:
\begin{align*}
    \xkpo = \xk - \gamma_k (\gk)^{\sharp}.
\end{align*}

\subsection{Linear Minimization Oracle (LMO) Formulation}\label{subsec:lmo_form}
We now show that a Linear Minimization Oracle (LMO) gives the same set of admissible next iterates as the steepest-descent formulation. An LMO is an operator which, for a given set $\mathcal{C} \subset E$ and a vector $s \in E^*$, solves the linear problem:
\begin{align*}
    \lmo_{\mathcal{C}}(s) \eqdef \argmin_{x \in \mathcal{C}} \inner{s}{x}.
\end{align*}
Our derivation will rely on two key properties, which we first establish as lemmas.

\begin{lemma}[Characterization of the Sharp Operator via LMO]
\label{lem:app_sharp_as_lmo}
For any non-zero $s \in E^*$, the subdifferential $\partial\phi^*(s)$ can be expressed in terms of an LMO over the unit ball $\cB(0, 1) = \{x \in E \mid \norm{x} \le 1\}$ as follows:
\begin{align*}
    \partial\phi^*(s) = -\snorm{s} \cdot \lmo_{\cB(0, 1)}(s).
\end{align*}
\end{lemma}
\begin{proof}
From \Cref{eq:subdiff_phi_star_vs_eta_star} and the variational characterization of the subdifferential of the norm from \Cref{lem:subdifferential_of_the_norm}, we have for any $s \neq 0$:
\begin{align*}
    \partial\phi^*(s) &= \snorm{s} \cdot \partial\snorm{s} \\
    &= \snorm{s} \cdot \argmax_{x \in E} \{\inner{s}{x} \mid \norm{x} \le 1\} \\
    &= -\snorm{s} \cdot \argmin_{x \in E} \{\inner{s}{x} \mid \norm{x} \le 1\} \\
    &\equiv -\snorm{s} \cdot \lmo_{\cB(0, 1)}(s).
\end{align*}
This completes the proof.
\end{proof}

\begin{lemma}[Translation Property of the LMO]
\label{lem:app_lmo_translation}
Let $\cB(y, t) = \{x \in E \mid \norm{x-y} \le t\}$ be a ball of radius $t>0$ centered at $y \in E$. The LMO over this ball can be expressed in terms of the LMO over the unit ball $\cB(0, 1)$ as:
\begin{align*}
    \lmo_{\cB(y, t)}(s) = y + t \cdot \lmo_{\cB(0, 1)}(s).
\end{align*}
\end{lemma}
\begin{proof}
The proof follows from a change of variables. The LMO problem is to find a minimizer $x^*$ for $\inner{s}{x}$ over the set $\cB(y, t)$.
\begin{align*}
    \lmo_{\cB(y, t)}(s) &= \argmin_{x \in E: \norm{x-y}\le t} \inner{s}{x} \\
    &= y + \argmin_{x' \in E: \norm{x'}\le t} \inner{s}{y+x'} \quad (\text{letting } x' = x-y) \\
    &= y + \argmin_{x' \in E: \norm{x'}\le t} \inner{s}{x'} \\
    &= y + t \cdot \argmin_{x'' \in E: \norm{x''}\le 1} \inner{s}{tx''} \quad (\text{letting } x' = tx'') \\
    &= y + t \cdot \argmin_{x'' \in E: \norm{x''}\le 1} \inner{s}{x''} \\
    &= y + t \cdot \lmo_{\cB(0, 1)}(s).
\end{align*}
This completes the derivation.
\end{proof}

For $\gk\neq0$, define $t_k\eqdef\gamma_k\snorm{\gk}>0$. Applying \Cref{lem:app_lmo_translation} and then \Cref{lem:app_sharp_as_lmo} gives the equality of solution sets:
\begin{align*}
    \lmo_{\cB(\xk, t_k)}(\gk)
    &= \xk+t_k\cdot\lmo_{\cB(0, 1)}(\gk) \\
    &= \xk-\gamma_k\partial\phi^*(\gk).
\end{align*}
The right-hand side is exactly the set of next iterates obtained from all steepest-descent minimizers. Thus the LMO formulation is
\begin{align}
    \xkpo \in \lmo_{\cB(\xk, t_k)}(\gk), \quad \text{ where } \quad t_k = \gamma_k \snorm{\gk}.
\end{align}
To reproduce the algorithm with the fixed sharp selection, choose the LMO minimizer $\xkpo=\xk-\gamma_k(\gk)^\sharp$. The equality of solution sets ensures that this choice is valid; an arbitrary different minimizer may produce a different admissible step. The radius of the ball, $t_k$, depends on the stepsize and the dual norm of the stochastic gradient.

\noindent
We note that if $\gk = 0$, then $t_k = 0$ and both formulations reduce to the identity update $\xkpo = \xk$: in this case the ball $\cB(\xk, 0)$ is the singleton $\{\xk\}$, and the LMO formulation still coincides with the steepest-descent solution.

\clearpage
\section{Derivation of the Orthogonalized Gradient Update}
\label{app:orth_derivation}

In this section, we derive the spectral-norm update \eqref{eq:non-eu-one-step-final} for the canonical sharp selection and characterize the full family of updates allowed by the LMO formulation.

\subsection{The Non-Euclidean Update with the Spectral Norm}

The non-Euclidean \algname{SGD} update is given by
\begin{align*}
    \bX^{k+1} = \bX^k - \gamma_k (\mathbf{G}^k)^{\sharp},
\end{align*}
where $(\mathbf{G}^k)^{\sharp} \in \partial\phi^*(\mathbf{G}^k)$. In our setting, the primal space is the space of matrices equipped with the spectral norm, $\norm{\cdot} = \opnorm{\cdot}$. The corresponding dual norm is the nuclear norm, $\snorm{\cdot} = \nucnorm{\cdot}$. The function $\phi^*$ is therefore the squared nuclear norm:
\begin{align*}
    \phi^*(\mathbf{G}) = \frac{1}{2}\nucnorm{\mathbf{G}}^2.
\end{align*}
To characterize the sharp operator $(\mathbf{G}^k)^{\sharp}$, we must find the subdifferential of $\phi^*$. By \Cref{lemma:subdiff_squared_norm}, we have $\partial\phi^*(\mathbf{G}) = \nucnorm{\mathbf{G}} \cdot \partial\nucnorm{\mathbf{G}}$. We therefore first characterize the subdifferential of the nuclear norm.


\begin{lemma}[Subdifferential of the Nuclear Norm]
\label{lem:subdiff_nuc_norm}
Let $\mathbf{G} \in \R^{m \times n}$ have thin singular value decomposition $\mathbf{G} = \mathbf{U}_r \Sigma_r \mathbf{V}_r^\top$, where $r = \mathrm{rank}(\mathbf{G})$, $\mathbf{U}_r \in \R^{m \times r}$, $\mathbf{V}_r \in \R^{n \times r}$, and $\Sigma_r = \mathrm{diag}(\sigma_1, \dots, \sigma_r) \in \R^{r \times r}$ with $\sigma_i > 0$. The subdifferential of the nuclear norm $\nucnorm{\cdot}$ at $\mathbf{G}$ is given by
\begin{align*}
    \partial\nucnorm{\mathbf{G}} = \left\{ \mathbf{U}_r \mathbf{V}_r^\top + \mathbf{W} \mid \mathbf{U}_r^\top\mathbf{W} = 0, \mathbf{W}\mathbf{V}_r=0, \text{ and } \opnorm{\mathbf{W}} \le 1 \right\}.
\end{align*}
A canonical and valid selection for an element of the subdifferential is the partial isometry $\mathbf{U}_r \mathbf{V}_r^\top$.
\end{lemma}
\begin{proof}
We will derive this result by applying the characterization of the subdifferential of the dual norm from \Cref{lem:subdifferential_of_the_norm}.

We apply the general norm characterization with the following identification:
\begin{itemize}
    \item The primal norm is the spectral norm: $\norm{\cdot} \equiv \opnorm{\cdot}$.
    \item The corresponding dual norm is the nuclear norm: $\snorm{\cdot} \equiv \nucnorm{\cdot}$.
\end{itemize}
We consider the two cases from \eqref{eq:subdiff_norm_dual}.

\textbf{Case 1: $\bG = 0$}.
According to \eqref{eq:subdiff_norm_dual}, the subdifferential of the dual norm at the origin is the unit ball of the primal norm. In our case, this is the unit ball of the spectral norm:
\begin{align*}
    \partial\nucnorm{0} = \{\bS \in \R^{m \times n} \mid \opnorm{\bS} \le 1\}.
\end{align*}
This matches the general formula: when $\bG=0$, we have $r=0$, and the empty factors $\bU_0\in\R^{m\times0}$ and $\bV_0\in\R^{n\times0}$ satisfy $\bU_0\bV_0^\top=0$. The constraints $\bU_0^\top\bW=0$ and $\bW\bV_0=0$ are then vacuous, so the formula reduces to $\{\bW\in\R^{m\times n}\mid\opnorm{\bW}\le1\}$.

\textbf{Case 2: $\bG \neq 0$}.
According to \eqref{eq:subdiff_norm_dual}, an element $\bS \in \R^{m \times n}$ belongs to the subdifferential $\partial\nucnorm{\bG}$ if and only if it satisfies two conditions:
\begin{enumerate}
    \item $\inner{\bS}{\bG} = \nucnorm{\bG}$
    \item $\opnorm{\bS} = 1$
\end{enumerate}
Let $u_i$ and $v_i$ be the columns of the fixed thin SVD factors $\bU_r$ and $\bV_r$. Suppose first that $\bS\in\partial\nucnorm{\bG}$. The two conditions above give
\begin{align*}
    \sum_{i=1}^r \sigma_i u_i^\top\bS v_i
    = \inner{\bS}{\bG}_{\mathrm F}
    = \nucnorm{\bG}
    = \sum_{i=1}^r\sigma_i.
\end{align*}
Each $u_i^\top\bS v_i\le\|\bS v_i\|_2\le1$, and all $\sigma_i$ are positive, so $u_i^\top\bS v_i=1$ for every $i\le r$. Since $\opnorm{\bS^\top}=\opnorm{\bS}=1$, equality in the Euclidean Cauchy--Schwarz inequality, applied to $\bS v_i$ and to $\bS^\top u_i$, yields
\begin{align*}
    \bS\bV_r=\bU_r,
    \qquad
    \bS^\top\bU_r=\bV_r.
\end{align*}
Thus $\bW\eqdef\bS-\bU_r\bV_r^\top$ satisfies $\bU_r^\top\bW=0$ and $\bW\bV_r=0$. Moreover,
\begin{align*}
    \bW=(\mathbf I_m-\bU_r\bU_r^\top)\bS(\mathbf I_n-\bV_r\bV_r^\top),
\end{align*}
so $\opnorm{\bW}\le\opnorm{\bS}=1$, because both orthogonal projections have spectral norm at most one. This proves that every subgradient has the stated form.

Conversely, let $\bW$ satisfy the three conditions in the lemma and set $\bS=\bU_r\bV_r^\top+\bW$. For any $z\in\R^n$, write $z=\bV_r a+b$ with $\bV_r^\top b=0$. Then $\bS z=\bU_r a+\bW b$, and the two terms are orthogonal because $\bU_r^\top\bW=0$. Hence
\begin{align*}
    \|\bS z\|_2^2
    =\|a\|_2^2+\|\bW b\|_2^2
    \le\|a\|_2^2+\|b\|_2^2
    =\|z\|_2^2.
\end{align*}
Therefore $\opnorm{\bS}\le1$. Since $r\ge1$ and $\bS v_i=u_i$ for every $i\le r$, we also have $\opnorm{\bS}\ge1$ and
\begin{align*}
    \inner{\bS}{\bG}_{\mathrm F}
    =\sum_{i=1}^r\sigma_i u_i^\top\bS v_i
    =\sum_{i=1}^r\sigma_i
    =\nucnorm{\bG}.
\end{align*}
Both conditions from \eqref{eq:subdiff_norm_dual} hold, so $\bS\in\partial\nucnorm{\bG}$. This proves the reverse inclusion.

Choosing $\bW=0$ yields the canonical element $\bS=\bU_r\bV_r^\top$.
\end{proof}

With this result, we can characterize the sharp operator at iteration $k$. Let $\mathbf{G}^k$ have thin singular value decomposition
\[
    \mathbf{G}^k = \mathbf{U}_r^k \Sigma_r^k (\mathbf{V}_r^k)^\top,
\]
where $r = \mathrm{rank}(\mathbf{G}^k)$. By Lemma~\ref{lem:subdiff_nuc_norm}, a canonical choice of subgradient of the nuclear norm is the partial isometry
\[
    \mathbf{O}^k \eqdef \mathbf{U}_r^k (\mathbf{V}_r^k)^\top \in \partial\nucnorm{\mathbf{G}^k}.
\]
Using the chain rule for subdifferentials, we have
\begin{align*}
    (\mathbf{G}^k)^{\sharp} \in \partial\phi^*(\mathbf{G}^k)
    = \nucnorm{\mathbf{G}^k} \cdot \partial\nucnorm{\mathbf{G}^k},
\end{align*}
and we select
\begin{align*}
    (\mathbf{G}^k)^{\sharp} \eqdef \nucnorm{\mathbf{G}^k} \cdot \mathbf{O}^k.
\end{align*}
Substituting this back into the \algname{SGD} update rule gives
\begin{align*}
    \bX^{k+1} = \bX^k - \gamma_k \nucnorm{\mathbf{G}^k} \mathbf{O}^k.
\end{align*}
For $\mathbf{G}^k\neq 0$, the direction $\mathbf{O}^k$ has spectral norm one. If we define a state-dependent stepsize $\eta_k \eqdef \gamma_k \nucnorm{\mathbf{G}^k}$, the update becomes
\begin{align*}
    \bX^{k+1} = \bX^k - \eta_k \mathbf{O}^k.
\end{align*}
This is the canonical orthogonalized update used in \Cref{subsec:non_eu_update}.

\subsection{Equivalence of Orthogonalization Definitions}\label{subseq:equiv_orth}

We now prove that the partial isometry $\mathbf{U}_r\mathbf{V}_r^\top$ obtained from the thin SVD agrees with the explicit formula given in \eqref{eq:orth_def-final}.

\begin{lemma}
\label{lem:orth_partial_isometry}
Let $\mathbf{G} \in \R^{m \times n}$ have thin singular value decomposition $\mathbf{G} = \mathbf{U}_r \Sigma_r \mathbf{V}_r^\top$, where $r = \mathrm{rank}(\mathbf{G})$, $\mathbf{U}_r \in \R^{m \times r}$, $\mathbf{V}_r \in \R^{n \times r}$, and $\Sigma_r = \mathrm{diag}(\sigma_1, \dots, \sigma_r) \in \R^{r \times r}$ with $\sigma_i > 0$. Then the orthogonalization function yields the partial isometry:
\begin{align*}
    \operatorname{orth}(\mathbf{G}) = (\mathbf{G}\mathbf{G}^\top)^{\dagger/2} \mathbf{G} = \mathbf{U}_r \mathbf{V}_r^\top.
\end{align*}
\end{lemma}
\begin{proof}
If $\mathbf{G}=0$, both sides are zero, with the empty product $\mathbf{U}_0\mathbf{V}_0^\top$ interpreted as the zero matrix. Assume henceforth that $\mathbf{G}\neq 0$.
Using the thin SVD, we have
\begin{align*}
    \mathbf{G}\mathbf{G}^\top
    &= (\mathbf{U}_r \Sigma_r \mathbf{V}_r^\top)(\mathbf{U}_r \Sigma_r \mathbf{V}_r^\top)^\top \\
    &= (\mathbf{U}_r \Sigma_r \mathbf{V}_r^\top)(\mathbf{V}_r \Sigma_r^\top \mathbf{U}_r^\top)
     = \mathbf{U}_r \Sigma_r^2 \mathbf{U}_r^\top,
\end{align*}
where $\Sigma_r^2 = \mathrm{diag}(\sigma_1^2, \dots, \sigma_r^2)$. The Moore--Penrose pseudoinverse is then
\begin{align*}
    (\mathbf{G}\mathbf{G}^\top)^\dagger
    = \mathbf{U}_r \Sigma_r^{-2} \mathbf{U}_r^\top,
\end{align*}
where $\Sigma_r^{-2} = \mathrm{diag}(\sigma_1^{-2}, \dots, \sigma_r^{-2})$. Since $(\mathbf{G}\mathbf{G}^\top)^\dagger$ is positive semidefinite, its unique positive semidefinite square root is
\begin{align*}
    (\mathbf{G}\mathbf{G}^\top)^{\dagger/2}
    = \mathbf{U}_r \Sigma_r^{-1} \mathbf{U}_r^\top,
\end{align*}
with $\Sigma_r^{-1} = \mathrm{diag}(\sigma_1^{-1}, \dots, \sigma_r^{-1})$.

Multiplying by $\mathbf{G}$ gives
\begin{align*}
    \operatorname{orth}(\mathbf{G})
    &= (\mathbf{G}\mathbf{G}^\top)^{\dagger/2} \mathbf{G} \\
    &= \mathbf{U}_r \Sigma_r^{-1} \mathbf{U}_r^\top (\mathbf{U}_r \Sigma_r \mathbf{V}_r^\top) \\
    &= \mathbf{U}_r \Sigma_r^{-1} (\mathbf{U}_r^\top \mathbf{U}_r) \Sigma_r \mathbf{V}_r^\top \\
    &= \mathbf{U}_r (\Sigma_r^{-1} \Sigma_r) \mathbf{V}_r^\top
     = \mathbf{U}_r \mathbf{V}_r^\top,
\end{align*}
since $\Sigma_r^{-1} \Sigma_r = \mathbf{I}_r$. This proves the claim.
\end{proof}

\subsection{LMO and the Canonical Orthogonalized Update for the Spectral Norm}

We characterize all solutions of the spectral-norm LMO and recover the canonical orthogonalized update by selecting the negative partial isometry.

The LMO formulation of the update is:
\begin{align*}
    \bX^{k+1} \in \lmo_{\cB(\bX^k, t_k)}(\bG^k),
\end{align*}
where the geometry is defined by the spectral norm, $\norm{\cdot} = \opnorm{\cdot}$. Consequently, the dual norm is the nuclear norm, $\snorm{\cdot} = \nucnorm{\cdot}$, and the radius of the ball is $t_k = \gamma_k \nucnorm{\bG^k}$ with $\gamma_k>0$. We first consider $\bG^k\neq 0$, so that $t_k>0$. The LMO problem is to solve:
\begin{align*}
    \bX^{k+1} \in \argmin_{\bX \in \R^{m \times n}} \left\{ \inner{\bG^k}{\bX}_{\mathrm{F}} \quad \text{s.t.} \quad \opnorm{\bX - \bX^k} \le t_k \right\}.
\end{align*}
Using the translation property of the LMO (\Cref{lem:app_lmo_translation}), we can simplify this problem:
\begin{align*}
    \bX^{k+1} \in \bX^k + t_k \cdot \lmo_{B(0, 1)}(\bG^k),
\end{align*}
where $B(0, 1)$ is the unit ball of the spectral norm, i.e., $B(0, 1) = \{\bX \in \R^{m \times n} \mid \opnorm{\bX} \le 1\}$. The following lemma describes the full set of minimizers over this ball.

\begin{lemma}[LMO over the Spectral Norm Ball]
\label{lem:lmo_spectral_norm}
Let $\bG \in \R^{m \times n}$ be non-zero with thin singular value decomposition $\bG = \bU_r\Sigma_r\bV_r^\top$, where $r=\operatorname{rank}(\bG)$ and $\Sigma_r=\operatorname{diag}(\sigma_1, \dots, \sigma_r)$ has positive diagonal entries. Then
\begin{align*}
    \lmo_{B(0, 1)}(\bG)
    &= \argmin_{\opnorm{\bX} \le 1} \inner{\bG}{\bX}_{\mathrm{F}} \\
    &= \left\{-\bU_r\bV_r^\top+\bW :
        \bU_r^\top\bW=0, \quad \bW\bV_r=0, \quad \opnorm{\bW}\le 1\right\},
\end{align*}
where $\bW\in\R^{m\times n}$. In particular, $-\bU_r\bV_r^\top=-\orth(\bG)$ is a canonical minimizer. The minimizer is unique if and only if $r=\min\{m, n\}$. For $\bG=0$, the LMO solution set is the entire ball $B(0, 1)$.
\end{lemma}
\begin{proof}
Let $u_i$ and $v_i$ be the columns of $\bU_r$ and $\bV_r$. For every feasible $\bX$, the bound $\opnorm{\bX}\le 1$ gives $u_i^\top\bX v_i\ge -1$. Hence
\begin{align*}
    \inner{\bG}{\bX}_{\mathrm{F}}
    = \sum_{i=1}^r \sigma_i u_i^\top\bX v_i
    \ge -\sum_{i=1}^r\sigma_i
    = -\nucnorm{\bG}.
\end{align*}
The matrix $\bX_0=-\bU_r\bV_r^\top$ is feasible and attains this bound.

If $\bX$ is any minimizer, positivity of all $\sigma_i$ implies $u_i^\top\bX v_i=-1$ for every $i\le r$. Since $u_i$ and $v_i$ are unit vectors and $\opnorm{\bX}=\opnorm{\bX^\top}\le 1$, equality in the Euclidean Cauchy--Schwarz inequality gives
\begin{align*}
    \bX\bV_r=-\bU_r,
    \qquad
    \bX^\top\bU_r=-\bV_r.
\end{align*}
Therefore $\bW\eqdef\bX+\bU_r\bV_r^\top$ satisfies $\bU_r^\top\bW=0$ and $\bW\bV_r=0$. Moreover,
\begin{align*}
    \bW=(\mathbf I_m-\bU_r\bU_r^\top)\bX(\mathbf I_n-\bV_r\bV_r^\top),
\end{align*}
so $\opnorm{\bW}\le\opnorm{\bX}\le1$ because both orthogonal projections have spectral norm at most one.

Conversely, let $\bW$ satisfy the three stated conditions and set $\bX=-\bU_r\bV_r^\top+\bW$. For any $z=\bV_r a+b$ with $\bV_r^\top b=0$, the vectors $\bU_r a$ and $\bW b$ are orthogonal, so
\begin{align*}
    \|\bX z\|_2^2
    =\|-\bU_r a+\bW b\|_2^2
    =\|a\|_2^2+\|\bW b\|_2^2
    \le\|a\|_2^2+\|b\|_2^2
    =\|z\|_2^2.
\end{align*}
Thus $\opnorm{\bX}\le1$, and $\bX v_i=-u_i$ implies $\inner{\bG}{\bX}_{\mathrm F}=-\nucnorm{\bG}$. This proves the full characterization. The identity $\bU_r\bV_r^\top=\orth(\bG)$ follows from \Cref{lem:orth_partial_isometry}.

If $r=\min\{m, n\}$, one of the two orthogonal complements is trivial, forcing $\bW=0$. If $r<\min\{m, n\}$, choose unit vectors $u$ and $v$ with $\bU_r^\top u=0$ and $\bV_r^\top v=0$. Then $\bW=uv^\top\neq0$ is admissible, so the minimizer is not unique. Finally, if $\bG=0$, the objective is identically zero on $B(0, 1)$.
\end{proof}

For $\bG^k\neq0$, \Cref{lem:lmo_spectral_norm} shows that the general LMO update has the form
\begin{align*}
    \bX^{k+1}=\bX^k-t_k\orth(\bG^k)+t_k\bW^k,
\end{align*}
where $r=\operatorname{rank}(\bG^k)$ and $\bW^k$ satisfies
\begin{align*}
    (\bU_r^k)^\top\bW^k=0,
    \qquad \bW^k\bV_r^k=0,
    \qquad \opnorm{\bW^k}\le1.
\end{align*}
By \Cref{lem:app_sharp_as_lmo}, this family is exactly the family of sharp updates, with corresponding selections
\begin{align*}
    (\bG^k)^\sharp=\nucnorm{\bG^k}\bigl(\orth(\bG^k)-\bW^k\bigr).
\end{align*}
Selecting $\bW^k=0$ gives the canonical update
\begin{align*}
    \bX^{k+1}
    &=\bX^k-t_k\orth(\bG^k) \\
    &=\bX^k-\gamma_k\nucnorm{\bG^k}\orth(\bG^k),
\end{align*}
which agrees with the canonical sharp selection used in the steepest-descent derivation and in \eqref{eq:orth_update_main-final}. With the state-dependent stepsize $\eta_k\eqdef\gamma_k\nucnorm{\bG^k}$ and direction $\bO^k\eqdef\orth(\bG^k)$, it takes the form $\bX^{k+1}=\bX^k-\eta_k\bO^k$. Other admissible LMO selections can produce different updates when $\operatorname{rank}(\bG^k)<\min\{m, n\}$.

\noindent
When $\bG^k=0$, we have $t_k=0$, $\partial\phi^*(0)=\{0\}$, and $\orth(0)=0$ by \eqref{eq:orth_def-final}. The translated ball is then $\cB(\bX^k, 0)=\{\bX^k\}$, so both the sharp and LMO formulations reduce to the identity update $\bX^{k+1}=\bX^k$.

\clearpage
\section{Regularity vs. Norm Equivalence: Why Not Just Use Euclidean Smoothness?}
\label{sec:regularity_vs_norm_equivalence}

To connect non-Euclidean geometries with the standard Euclidean setting, existing analyses
(e.g.,~\citep{kovalev2025understanding,gruntkowska2025error})
often proceed via \emph{norm equivalence}. Concretely, when the geometry is expressed in terms of
a dual norm $\snorm{\cdot}$, one compares it to a Euclidean reference norm (Frobenius for matrices,
$\ell_2$ for vectors). In the matrix case, there exist constants $\alpha, \beta>0$ (typically
dimension- and/or rank-dependent) such that for all $X\in\R^{m\times n}$,
\begin{align}\label{eq:norm_equiv_lit}
    \alpha\snorm{X} \le \fnorm{X} \le \beta\snorm{X}.
\end{align}
Analogously, in the vector case there exist $\alpha, \beta>0$ such that for all $x\in\R^{d}$,
\begin{align}\label{eq:norm_equiv_vec}
    \alpha\snorm{x} \le \norm{x}_2 \le \beta\snorm{x}.
\end{align}
These inequalities allow one to translate smoothness and noise assumptions stated in Euclidean/Frobenius norms
into the target geometry. However, since $\alpha$ and $\beta$ typically scale with the ambient dimension (or matrix
dimension/rank), such conversions can introduce large (often polynomial) inflation factors and thus obscure the
true benefits of the chosen geometry.

For an $E^*$-valued noise vector $\zeta$ with finite Euclidean second moment, only the left inequality in \eqref{eq:norm_equiv_vec} is needed to obtain
$\Exp{\sqsnorm{\zeta}}\le\alpha^{-2}\Exp{\norm{\zeta}_2^2}$.
Thus the noise second-moment conversion uses $\alpha^{-2}$; the Frobenius comparison in \eqref{eq:norm_equiv_lit} gives the same factor for matrices.

\medskip
\noindent
The $\kappa$-regularity framework allows an auxiliary norm that need not be Euclidean. As in \Cref{def:kappa_smooth_regular}, this norm is $\kappa_{V, +}$-smooth, meaning that its squared-norm function $\frac12\|\cdot\|_{V, +}^2$ is continuously differentiable and satisfies the stated smoothness inequality. It also satisfies the comparison \eqref{eq:reg_kappa_def_gen}:
\begin{align*}
    \|u\|_V^2
    \le \|u\|_{V, +}^2
    \le C_V\|u\|_V^2,
    \qquad
    C_V=\frac{\kappa_V}{\kappa_{V, +}},
    \qquad u\in V,
\end{align*}
where $1\le\kappa_{V, +}\le\kappa_V$.
Smoothness supports the analysis in the auxiliary norm, while the comparison transfers the resulting bounds to the original norm. Choosing the auxiliary norm therefore requires balancing its smoothness parameter $\kappa_{V, +}$ against the factor $C_V$ comparing the squared norms. To obtain a small regularity parameter, we minimize their product $\kappa_V=\kappa_{V, +}C_V$, rather than either factor alone. The two propositions below illustrate this choice for $\ell_p$ and Schatten-$p$ spaces.

\paragraph{Which space carries the norm.}
As in \Cref{def:kappa_smooth_regular}, \Cref{example:1}, \Cref{example:2} and \Cref{example:3}, the results of this section are
stated for a \emph{generic} finite-dimensional normed space $(V, \|\cdot\|_V)$, with constant of regularity
$\kappa_V$, auxiliary norm $\|\cdot\|_{V, +}$ and smoothness parameter $\kappa_{V, +}$. Following
\Cref{rem:regularity_dual_convention}, we apply them with $V=E^*$ and $\|\cdot\|_V=\snorm{\cdot}$, in which case
$\kappa_V$, $\|\cdot\|_{V, +}$ and $\kappa_{V, +}$ become $\kappa$, $\spnorm{\cdot}$ and $\kappa_+$, and the generic
compatibility chain \eqref{eq:reg_kappa_def_gen} specializes to \eqref{eq:reg_kappa_def}. It is therefore the
regularity constant of the \emph{dual} space that enters our bounds, since that is where stochastic gradients live and where the
noise is measured; an $\ell_p$ or Schatten-$p$ norm appearing below is a norm on $V=E^*$, that is, the norm
\emph{induced} on $E^*$ by the update norm on $E$.

\paragraph{The Euclidean proxy and the role of the auxiliary norm.}
Constructing a Euclidean proxy uses both sides of \eqref{eq:norm_equiv_vec}: the scaled norm
$\spnorm{\cdot}\eqdef\alpha^{-1}\norm{\cdot}_2$ is $1$-smooth and satisfies
$\sqsnorm{s}\le\sqspnorm{s}\le(\beta/\alpha)^2\sqsnorm{s}$.
It therefore gives the admissible regularity parameter $\kappa=(\beta/\alpha)^2$, distinct from the noise-conversion factor $\alpha^{-2}$. The same construction applies to the Frobenius norm.
Allowing other auxiliary norms can reduce the product $\kappa=\kappa_+C$, where $C=\kappa/\kappa_+$ is the squared-norm comparison factor in \eqref{eq:reg_kappa_def}.
The auxiliary norm is used only in the analysis and enters the guarantee through $\kappa$; the update \eqref{eq:non-eu-one-step-final} is determined by the update norm $\norm{\cdot}$ and the fixed sharp selection, and never evaluates $\spnorm{\cdot}$.

\paragraph{What this means for oracle complexity.}
When these constructions are applied to $V=E^*$, their effect on oracle complexity depends on the full statistical
coefficients for the chosen noise model. Suppressing logarithmic factors, the relevant comparison is between
$\mathsf C_*=\kappa\sigma^2$ for the native analysis and $\mathsf C_{\opf}=\alpha^{-2}\sigma_{\opf}^2$ for the
Euclidean/Frobenius conversion, as in \eqref{eq:statistical_coefficients_alt}. Here $\sigma$ and $\sigma_{\opf}$
measure the noise of the same oracle in the induced dual norm and the Euclidean/Frobenius norm, respectively.
A smaller $\kappa$ alone does not guarantee a sharper oracle-complexity bound; any improvement depends on these
full coefficients, as illustrated by the noise models in \Cref{sec:frobenius_light_tail_oracle}.
With this distinction in place, we compare the auxiliary norms for vectors and matrices below, before stating and proving the corresponding propositions.

\subsection{\texorpdfstring{$\ell_p^d$ Spaces}{ell-p Spaces in Dimension d}}
\label{subsec:regularity_vs_equiv_lp}

We first specialize the preceding discussion to $(E^*, \snorm{\cdot})=(\R^d, \norm{\cdot}_p)$. Whether a non-Euclidean auxiliary norm improves the regularity bound depends on $p$.

\textbf{The regime $1\le p<2$.}
A scaled Euclidean auxiliary norm is already optimal: taking
$\spnorm{\cdot}=d^{\frac{1}{p}-\frac{1}{2}}\norm{\cdot}_2$ reduces \eqref{eq:reg_kappa_def} to the classical $\ell_p$--$\ell_2$ equivalence,
$\norm{s}_2^2\le\norm{s}_p^2\le d^{\frac{2}{p}-1}\norm{s}_2^2$,
and gives $\kappa=d^{\frac{2}{p}-1}$ for every $d\ge1$.
This parameter cannot be improved by choosing another auxiliary norm; see \Cref{rem:optimality_euclidean_aux_p_lt_2} and case~1 of \Cref{prop:lp_regular_vs_equiv} below.

\textbf{The regime $2\le p\le\infty$.}
Here we may use $\spnorm{\cdot}=\norm{\cdot}_\rho$ with $2\le\rho\le p$ and $\rho<\infty$.
Its smoothness constant $\kappa_+=\rho-1$ and squared-norm comparison factor $d^{2/\rho-2/p}$ give the regularity bound
$\kappa(\rho)=(\rho-1)d^{2/\rho-2/p}$.
The minimizing choice is
$\rho_*\eqdef2$ for $1\le d<8$ and
$\rho_*\eqdef\min\bigl\{p, \,\ln d+\sqrt{(\ln d)^2-2\ln d}\,\bigr\}$ for $d\ge8$; in particular, $\rho_*$ remains finite when $p=\infty$.
For $1\le d<8$, it gives $\kappa=d^{1-2/p}\le7$.%
\footnote{Indeed, for $\rho\ge2$, we have $(\rho-1)/\rho^2\le1/4$, so $\kappa'(\rho)=d^{2/\rho-2/p}\bigl[1-2(\ln d)(\rho-1)/\rho^2\bigr]\ge d^{2/\rho-2/p}(1-\tfrac12\ln d)>0$ when $1\le d<8$.}
For $d\ge8$, the comparison takes the following form:
\begin{center}
\small
\begin{tabular}{@{}c@{\hspace{3.5em}}c@{}}
\emph{Euclidean proxy:}\ \ $\spnorm{\cdot}=\norm{\cdot}_2$
& \emph{Tuned proxy:}\ \ $\spnorm{\cdot}=\norm{\cdot}_{\rho_*}$ \\[5pt]
$\sqsnorm{s}\le\norm{s}_2^2\le d^{1-\frac{2}{p}}\sqsnorm{s}$
& $\sqsnorm{s}\le\norm{s}_{\rho_*}^2\le\cO(1)\,\sqsnorm{s}$ \\[5pt]
$\kappa=d^{1-\frac{2}{p}}$
& $\kappa=\Theta\rb{\min\{p, \ln(ed)\}}$
\end{tabular}
\end{center}
Both choices of $\spnorm{\cdot}$ are $\kappa_+$-smooth and satisfy the required norm-comparison inequalities, so each establishes $\kappa$-regularity according to \Cref{def:kappa_smooth_regular}, with its corresponding values of $\kappa_+$ and $\kappa$.
Since $\rho=2$ is allowed in the minimization of $\kappa(\rho)$, the tuned regularity bound is no larger than the Euclidean bound $d^{1-2/p}$ and can be significantly smaller.
The constant comparison factor has a simple explanation: when $p\le\hat\rho(d)$, where $\hat\rho(d)=\ln d+\sqrt{(\ln d)^2-2\ln d}$, we have $\rho_*=p$ and the factor is $1$; when $p>\hat\rho(d)$, it is at most $e^2$.

\textbf{Interpretation.}
At $p=2$, both proxies coincide with $\norm{\cdot}_2$ and give $\kappa=1$.
For $2\le p\le\infty$, the tuned bound has the optimal order $\Theta(\min\{p, \ln(ed)\})$ in all positive dimensions: case~2 of \Cref{prop:lp_regular_vs_equiv} establishes the comparison for $d\ge8$, and \Cref{rem:uniform_regularity_orders} covers $1\le d<8$.
Thus fixed finite $p\ge2$ admits a bound independent of $d$, while at $p=\infty$, the dual of the $\ell_1$ update norm, tuning reduces the regularity parameter from $d$ to $\Theta(\ln(ed))$.
These are comparisons of regularity parameters; the oracle-complexity comparison also depends on the noise scales, as explained above.
The following proposition gives the formal statements and their proof in the generic notation $(V, \|\cdot\|_V)$.

\begin{proposition}[Regularity vs.\ Euclidean equivalence in $\ell_p^d$ spaces]
\label{prop:lp_regular_vs_equiv}
Fix $d\in\N$ and $p\in[1, \infty]$. Let $V\eqdef \R^d$ and equip it with the norm $\|\cdot\|_V\eqdef \|\cdot\|_p$.
Then the following hold.
\begin{enumerate}
    \item \textbf{Case $1\le p<2$: regularity recovers the classical $\ell_p$--$\ell_2$ equivalence.}
    Take the auxiliary norm
    \begin{align*}
        \|x\|_{V, +} \eqdef d^{\frac{1}{p}-\frac{1}{2}}\|x\|_2 .
    \end{align*}
    Then $(V, \|\cdot\|_{V, +})$ is $1$-smooth, so $\kappa_{V, +}=1$, and the compatibility chain
    \eqref{eq:reg_kappa_def_gen} holds with
    \begin{align*}
        \kappa_V \eqdef \kappa_p(d) = d^{\frac{2}{p}-1}.
    \end{align*}
    For this choice, \eqref{eq:reg_kappa_def_gen} is nothing but the standard $\ell_p$--$\ell_2$ norm equivalence:
    \begin{align}\label{eq:lp_equiv}
        \norm{x}_{2} \le \norm{x}_{p} \le d^{\frac{1}{p}-\frac{1}{2}}\norm{x}_{2}
        \qquad\Longleftrightarrow\qquad
        \norm{x}_{2}^{2} \le \norm{x}_{p}^{2} \le d^{\frac{2}{p}-1}\norm{x}_{2}^{2}.
    \end{align}
    Moreover, this value of $\kappa_V$ cannot be improved: the matching lower bound is
    \eqref{eq:lower_bound_lp_p_lt_2} in \Cref{prop:tightness_lp_bounds}.

    \item \textbf{Case $2\le p\le\infty$: comparison with the Euclidean proxy.}
    For any $\rho\in[2, p]$ with $\rho<\infty$, take the auxiliary norm $\|\cdot\|_{V, +}\eqdef \|\cdot\|_\rho$. Then
    $(V, \|\cdot\|_\rho)$ is $(\rho-1)$-smooth, so $\kappa_{V, +}=\rho-1$, and \eqref{eq:reg_kappa_def_gen} holds
    with compatibility factor and admissible regularity parameter
    \begin{align*}
        \frac{\kappa_V}{\kappa_{V, +}}=d^{\frac{2}{\rho}-\frac{2}{p}}
        \qquad\text{and}\qquad
        \kappa_V(\rho)=(\rho-1)\,d^{\frac{2}{\rho}-\frac{2}{p}},
    \end{align*}
    respectively. Optimizing over $\rho\in[2, p]$ with $\rho<\infty$ yields the best regularity parameter obtainable this way,
    \begin{align}\label{eq:kappa_p_opt_rho}
        \kappa_V \eqdef \kappa_p(d)=\min_{\substack{2\le \rho\le p\\ \rho<\infty}} (\rho-1)\,d^{\frac{2}{\rho}-\frac{2}{p}}.
    \end{align}
    When $p=\infty$, we use $1/p=0$ and optimize over $\rho\in[2, \infty)$.
    Since $\rho=2$ is admissible, $\kappa_p(d)\le d^{1-2/p}$, the parameter supplied by the Euclidean proxy.
    At $p=2$, both choices coincide with $\norm{\cdot}_2$ and give $\kappa_p(d)=1$.
    Assume in addition that $\ln d\ge 2$ (e.g., $d\ge 8$) and set
    \begin{align*}
        \hat{\rho}(d)\eqdef \ln d + \sqrt{(\ln d)^2-2\ln d},
        \qquad
        \rho_* \eqdef \min\{p, \hat{\rho}(d)\} .
    \end{align*}
    The Euclidean choice $\rho=2$ produces the classical compatibility factor $d^{1-\frac{2}{p}}$, whereas the
    tuned choice $\rho=\rho_*$ produces a factor $d^{\frac{2}{\rho_*}-\frac{2}{p}}$ that is never larger and is in
    fact \emph{dimension-free} in both regimes $p\le\hat\rho(d)$ and $p>\hat\rho(d)$:
    \begin{align}\label{eq:lp_ge2_o1_equiv}
        \norm{x}_p^2 \le \norm{x}_{\rho_*}^2 \le \cO(1)\,\norm{x}_p^2
        \qquad\forall x\in\R^d .
    \end{align}
    Finally, the scaling $\kappa_p(d)=\Theta\bigl(\min\{p, \ln(ed)\}\bigr)$ is unavoidable for $d\ge 8$: the matching
    lower bound is \eqref{eq:lower_bound_lp_p_ge_2} in \Cref{prop:tightness_lp_bounds}.
\end{enumerate}
\end{proposition}

\begin{proof}
Both cases instantiate \Cref{def:kappa_smooth_regular} on $V=\R^d$; the constructions and the underlying
smoothness estimates are those of \Cref{example:2}, and the only new content is the explicit comparison with the
Euclidean conversion.

\smallskip
\noindent\emph{Case $1\le p<2$.}
The norm $\|\cdot\|_{V, +}$ is a scaled Euclidean norm and is therefore $1$-smooth, so $\kappa_{V, +}=1$.
The standard inequalities $\norm{x}_2\le \norm{x}_p \le d^{\frac{1}{p}-\frac{1}{2}}\norm{x}_2$
give $\norm{x}_p\le \|x\|_{V, +}$ and
$\|x\|_{V, +}^2=d^{\frac{2}{p}-1}\norm{x}_2^2\le d^{\frac{2}{p}-1}\norm{x}_p^2$,
which is \eqref{eq:reg_kappa_def_gen} with $\kappa_V=d^{\frac{2}{p}-1}$ and $\kappa_{V, +}=1$; squaring the same
two inequalities gives \eqref{eq:lp_equiv}. The lower bound is \eqref{eq:lower_bound_lp_p_lt_2}.

\smallskip
\noindent\emph{Case $2\le p\le\infty$.}
Fix $2\le\rho\le p$ with $\rho<\infty$. That $(\R^d, \norm{\cdot}_\rho)$ is $(\rho-1)$-smooth for $2\le\rho<\infty$ is shown in Step~1 of the proof of
\Cref{example:2}. The standard $\ell_p$--$\ell_\rho$ inequalities give
$\norm{x}_p\le \norm{x}_\rho \le d^{\frac{1}{\rho}-\frac{1}{p}}\norm{x}_p$, and squaring yields
\eqref{eq:reg_kappa_def_gen} with $\nfr{\kappa_V}{\kappa_{V, +}}=d^{\frac{2}{\rho}-\frac{2}{p}}$, hence
$\kappa_V(\rho)=(\rho-1)d^{\frac{2}{\rho}-\frac{2}{p}}$. Minimizing $\kappa_V(\rho)$ over $\rho\in[2, p]$ with $\rho<\infty$ gives
\eqref{eq:kappa_p_opt_rho}. When $p=\infty$, this continuous function satisfies $\kappa_V(\rho)=(\rho-1)d^{2/\rho}\to\infty$ as $\rho\to\infty$, so the minimum on $[2, \infty)$ is attained at a finite $\rho$.
The expression for $\hat{\rho}(d)$ and the bound $d^{2/\hat{\rho}(d)}\le e^2$ (valid
whenever $\ln d\ge 2$) are derived in Step~3 of the proof of \Cref{example:2}; in particular
$\hat\rho(d)\ge \ln d\ge 2$, so $2\le\rho_*<\infty$ even when $p=\infty$, and the tuned compatibility factor
$d^{\frac{2}{\rho_*}-\frac{2}{p}}$ is never larger than the Euclidean one, $d^{1-\frac{2}{p}}$.
It remains to prove \eqref{eq:lp_ge2_o1_equiv}, and we distinguish two subcases.

\smallskip
\noindent\emph{Subcase 1: $p\le \hat\rho(d)$.}
Then $\rho_*=p$, so $\norm{x}_{\rho_*}=\norm{x}_p$ and \eqref{eq:lp_ge2_o1_equiv} holds with constant $1$.

\smallskip
\noindent\emph{Subcase 2: $p> \hat\rho(d)$.}
Then $\rho_*=\hat\rho(d)$ and
\begin{align*}
    \norm{x}_{\rho_*}^2 \le d^{\frac{2}{\rho_*}-\frac{2}{p}}\norm{x}_p^2
    = d^{2/\hat\rho(d)}\,d^{-2/p}\norm{x}_p^2 \le d^{2/\hat\rho(d)}\norm{x}_p^2 \le e^2\norm{x}_p^2,
\end{align*}
which yields \eqref{eq:lp_ge2_o1_equiv} with $\cO(1)=e^2$. The lower bound is \eqref{eq:lower_bound_lp_p_ge_2}.
\end{proof}

\subsection{Spaces with Schatten-\texorpdfstring{$p$}{p} Norms on \texorpdfstring{$\R^{m\times n}$}{real m-by-n matrices}}
\label{subsec:regularity_vs_equiv_schatten}

For matrices, we compare the fixed Frobenius proxy with auxiliary Schatten norms. The two $p$ regimes follow the vector case, although the smoothness bounds for the auxiliary norms differ.

Let $V\eqdef \R^{m\times n}$ and denote $r\eqdef \min\{m, n\}$. For $p\in[1, \infty]$, the Schatten-$p$ norm is
$\norm{X}_{S_p}\eqdef \norm{\sigma(X)}_p$, where $\sigma(X)\in\R^r$ is the vector of singular values of $X$
(padded with zeros when $\mathrm{rank}(X)<r$). Since Schatten norms are $\ell_p$ norms on $\sigma(X)$,
the standard comparison inequalities (see \Cref{lem:schatten_norm_comparison}) give, for $1\le \rho\le p\le\infty$,
\begin{align}\label{eq:schatten_equiv_general}
    \norm{X}_{S_p} \le \norm{X}_{S_\rho} \le r^{\frac{1}{\rho}-\frac{1}{p}}\norm{X}_{S_p}.
\end{align}
(For $X\ne0$, one may replace $r$ by $\mathrm{rank}(X)$ in \eqref{eq:schatten_equiv_general} for an
instance-dependent bound, which can be tighter when $\mathrm{rank}(X)\ll r$; the zero case is immediate.)

\textbf{The regime $1\le p<2$.}
The scaled Frobenius auxiliary norm $\|\cdot\|_{V, +}=r^{\frac{1}{p}-\frac{1}{2}}\fnorm{\cdot}$ is $1$-smooth and gives
$\kappa_V=r^{2/p-1}$ for every $r\ge1$.
The matching lower bound shows that this parameter is optimal, so allowing non-Frobenius auxiliary norms cannot improve it.

\textbf{The regime $2\le p\le\infty$.}
The Frobenius proxy is $1$-smooth and gives the squared-norm comparison factor and regularity parameter $r^{1-2/p}$.
Alternatively, for $2\le\rho\le p$ with $\rho<\infty$, the auxiliary norm $\|\cdot\|_{S_\rho}$ has comparison factor $r^{2/\rho-2/p}$ and admits the smoothness bound $\max\{2, \rho-1\}$.
Their product is the regularity bound optimized in the proposition below.
For $r\ge8$, the minimizing exponent for this bound is
$\rho_* = \min\{p, \max\{3, \hat\rho(r)\}\}$, where
$\hat\rho(r)=\ln r+\sqrt{(\ln r)^2-2\ln r}$.
If $\rho_*=p$, the auxiliary and target norms coincide and the comparison factor is $1$; otherwise the factor is at most $e^2$.
Thus the tuned squared-norm comparison factor is no larger than the Frobenius factor and is bounded independently of $r$.

\textbf{Interpretation.}
The resulting regularity bound has the optimal order $\Theta(\min\{p, \ln(er)\})$ for $r\ge8$; \Cref{rem:uniform_regularity_orders} extends this order to $1\le r<8$.
For fixed finite $p\ge2$ this order is independent of $r$, while at $p=\infty$ it is logarithmic, compared with the Frobenius proxy's linear parameter $r$.
At $p=2$, both auxiliary norms are Frobenius and one can take $\kappa_V=1$: the uniform smoothness estimate $\max\{2, \rho-1\}$ is not sharp at $\rho=2$.
As in the vector case, these geometric improvements give a smaller oracle-complexity bound only when the full statistical coefficients favor the native analysis.
We now state the precise comparison and prove it.

\begin{proposition}[Regularity vs.\ Frobenius equivalence in Schatten classes]
\label{prop:schatten_regular_vs_equiv}
Fix $m, n\in\N$, let $r\eqdef \min\{m, n\}$ and let $p\in[1, \infty]$. Let $V\eqdef \R^{m\times n}$ and equip it with
the norm $\|\cdot\|_V\eqdef \|\cdot\|_{S_p}$. Then the following hold.
\begin{enumerate}
    \item \textbf{Case $1\le p<2$: regularity recovers the Schatten-$p$--Frobenius equivalence.}
    Take the scaled Frobenius auxiliary norm
    \begin{align*}
        \|X\|_{V, +} \eqdef r^{\frac{1}{p}-\frac{1}{2}}\fnorm{X},
    \end{align*}
    which is $1$-smooth, so $\kappa_{V, +}=1$. Then \eqref{eq:reg_kappa_def_gen} holds with
    \begin{align*}
        \kappa_V \eqdef \kappa_p(m, n) = r^{\frac{2}{p}-1},
    \end{align*}
    and it is equivalent to the standard Schatten-$p$--Frobenius norm equivalence:
    \begin{align}\label{eq:schatten_p_vs_fro}
        \fnorm{X} \le \norm{X}_{S_p} \le r^{\frac{1}{p}-\frac{1}{2}}\fnorm{X}
        \qquad\Longleftrightarrow\qquad
        \sqfnorm{X} \le \norm{X}_{S_p}^{2} \le r^{\frac{2}{p}-1}\sqfnorm{X}.
    \end{align}
    Moreover, this value of $\kappa_V$ cannot be improved: the matching lower bound is
    \eqref{eq:schatten_lower_p_lt_2} in \Cref{prop:tightness_schatten_bounds}.

    \item \textbf{Case $2\le p\le\infty$: comparison with the Frobenius proxy.}
    For any $\rho\in[2, p]$ with $\rho<\infty$, take the auxiliary norm $\|\cdot\|_{V, +}\eqdef \|\cdot\|_{S_\rho}$.
    Then \eqref{eq:schatten_equiv_general} yields the compatibility chain
    \begin{align}\label{eq:reg_schatten_ge2_chain}
        \norm{X}_{S_p}^{2} \le \norm{X}_{S_\rho}^{2} \le r^{\frac{2}{\rho}-\frac{2}{p}}\norm{X}_{S_p}^{2},
    \end{align}
    and $(V, \|\cdot\|_{S_\rho})$ is $\kappa_{V, +}$-smooth with $\kappa_{V, +}=\max\{2, \rho-1\}$ for $\rho\ge 2$
    (see \Cref{example:3} and the references therein), so that \eqref{eq:reg_kappa_def_gen} holds with
    \begin{align*}
        \kappa_V(\rho)=\max\{2, \rho-1\}\,r^{\frac{2}{\rho}-\frac{2}{p}}.
    \end{align*}
    Assume now that $r\ge8$, so $\ln r\ge2$. Under this condition, optimizing over $\rho$ yields a regularity
    parameter of order $\kappa_V=\cO\bigl(\min\{p, \ln(er)\}\bigr)$. More explicitly, set
    \begin{align*}
        \hat{\rho}(r)\eqdef \ln r + \sqrt{(\ln r)^2-2\ln r},
        \qquad
        \rho_* \eqdef \min\bigl\{p, \,\max\{3, \hat{\rho}(r)\}\bigr\} .
    \end{align*}
    The Frobenius choice $\rho=2$ produces the compatibility factor $r^{1-\frac{2}{p}}$, whereas the tuned choice
    $\rho=\rho_*$ produces a factor that is no larger and is bounded by an absolute constant independent of the dimension parameter $r=\min\{m, n\}$ in both regimes:
    \begin{align}\label{eq:schatten_ge2_o1_equiv}
        \norm{X}_{S_p}^{2} \le \norm{X}_{S_{\rho_*}}^{2} \le \cO(1)\,\norm{X}_{S_p}^{2}
        \qquad\forall X\in\R^{m\times n}.
    \end{align}
    At $p=2$, both proxies coincide with the Frobenius norm and the compatibility factor is $1$.
    Its exact smoothness constant is also $1$, so $\kappa_V=1$ is admissible in this case; the uniform bound
    $\max\{2, \rho-1\}$ used above is not sharp at $\rho=2$.
    Finally, the scaling $\kappa_p(m, n)=\Theta\bigl(\min\{p, \ln(er)\}\bigr)$ is unavoidable for $r\ge 8$: the
    matching lower bound is \eqref{eq:schatten_lower_p_ge_2} in \Cref{prop:tightness_schatten_bounds}.
\end{enumerate}
\end{proposition}

\begin{proof}
As in the vector case, both cases instantiate \Cref{def:kappa_smooth_regular} on $V=\R^{m\times n}$, and the
constructions and smoothness estimates are those of \Cref{example:3}.

\smallskip
\noindent\emph{Case $1\le p<2$.}
This is the Schatten analogue of case~1 of \Cref{prop:lp_regular_vs_equiv}. Since $\|\cdot\|_{V, +}$ is a scaled
Hilbert norm, it is $1$-smooth. Applying \eqref{eq:schatten_equiv_general} with smaller exponent $p$ and larger exponent $2$ yields
$\norm{X}_{S_2}=\fnorm{X}\le \norm{X}_{S_p}\le r^{\frac{1}{p}-\frac{1}{2}}\fnorm{X}$, i.e., \eqref{eq:schatten_p_vs_fro}.
The compatibility chain \eqref{eq:reg_kappa_def_gen} then follows by construction with $\kappa_V=r^{\frac{2}{p}-1}$
and $\kappa_{V, +}=1$. The lower bound is \eqref{eq:schatten_lower_p_lt_2}.

\smallskip
\noindent\emph{Case $2\le p\le\infty$.}
The compatibility chain \eqref{eq:reg_schatten_ge2_chain} is exactly \eqref{eq:schatten_equiv_general} squared.
The smoothness bound for Schatten-$\rho$ spaces with $\rho\ge 2$, and the resulting regularity parameter, are
stated and proved in \Cref{example:3} (see also~\citep{juditsky2008large}). For the remaining claims, assume $r\ge8$.
Under this condition, the optimization over $\rho$ is the same calculus as in the $\ell_p$ case, with $d$ replaced by $r$, and leads to
$\kappa_V=\cO(\min\{p, \ln(er)\})$. It remains to justify \eqref{eq:schatten_ge2_o1_equiv}, and we distinguish two
subcases.

\smallskip
\noindent\emph{Subcase 1: $\rho_*=p$.}
Then $\norm{X}_{S_{\rho_*}}=\norm{X}_{S_p}$ and \eqref{eq:schatten_ge2_o1_equiv} holds with constant $1$.

\smallskip
\noindent\emph{Subcase 2: $\rho_*=\max\{3, \hat\rho(r)\}$.}
Since $\ln r\ge 2$, we have $\hat\rho(r)\ge \ln r$ and therefore $\rho_*\ge \ln r$, whence
$r^{2/\rho_*}\le r^{2/\ln r}=e^2$. Combining this with \eqref{eq:reg_schatten_ge2_chain} gives
\begin{align*}
    \norm{X}_{S_{\rho_*}}^{2} \le r^{\frac{2}{\rho_*}-\frac{2}{p}}\norm{X}_{S_p}^{2}
    \le r^{2/\rho_*}\norm{X}_{S_p}^{2} \le e^2\,\norm{X}_{S_p}^{2},
\end{align*}
which yields \eqref{eq:schatten_ge2_o1_equiv} with $\cO(1)=e^2$. The lower bound is
\eqref{eq:schatten_lower_p_ge_2}.
\end{proof}

\clearpage
\section{A Direct Comparison of Euclidean/Frobenius and Geometry-Native Oracles}
\label{sec:frobenius_light_tail_oracle_alt}
\label{sec:frobenius_light_tail_oracle}

There are two ways to model stochastic-gradient noise in a non-Euclidean geometry. One assumes light tails in a fixed
Euclidean or Frobenius norm and converts the resulting bound into the dual norm induced by the update geometry; the
other states the light-tail condition directly in that dual norm. This section compares them. We keep the conversion
constant $\alpha^{-1}$, the regularity constant $\kappa$ and the two noise scales separate throughout, so that one can
see which part of each bound comes from the geometry and which part comes from the perturbation model.

\paragraph{Notation.}
The index $p$ always refers to the norm on the \emph{dual} space $E^*$, exactly as in
\Cref{subsec:regularity_vs_norm_equivalence_main} and \Cref{sec:regularity_vs_norm_equivalence}: we write
$(E^*, \snorm{\cdot})=(\R^d, \norm{\cdot}_p)$ in the vector case and
$(E^*, \snorm{\cdot})=(\R^{m\times n}, \norm{\cdot}_{S_p})$ with $r\eqdef\min\{m, n\}$ in the matrix case, and we adopt
the convention $d^{2/\infty}=r^{2/\infty}=1$.

\paragraph{Overview of the comparison.}
We compare the same oracle under the two norm-based analyses. The examples use fresh independent additive perturbations: independent Rademacher signs or standard Gaussian entries in the vector case, and their diagonal matrix counterparts (\Cref{def:rademacher_oracle_alt} and \Cref{def:gaussian_oracle_alt}).
Write $D=d$ for vectors and $D=r$ for these diagonal matrix models. Under the common assumptions stated below, the two oracle bounds have the same leading term; their statistical terms are governed by $\mathsf C_*=\kappa\sigma^2$ and $\mathsf C_{\opf}=\alpha^{-2}\sigma_{\opf}^2$.

For $1\le p<2$, the squared conversion factor $\alpha^{-2}$ and the regularity constant are both $D^{2/p-1}$. Their ratio therefore cancels, leaving the comparison between the noise scales. For these dense perturbations, the ratio of squared noise scales is $D^{2/p-1}$, which is also the ratio of the native and converted coefficients. At $p=1$, the converted and native coefficients consequently have orders $D^2$ and $D^3$, respectively.
At $p=2$, the norms coincide: choosing $\sigma_{\opf}=\sigma$, $\alpha=1$ and $\kappa=1$ makes the bounds agree.

For each fixed $2<p<\infty$, the displayed native coefficient is bounded by a constant depending on $p$ times $D^{2/p}\nu^2$, while the converted coefficient grows as $D\nu^2$. Since $D^{2/p}=o(D)$ as $D\to\infty$, the native bound has smaller asymptotic dimension dependence.
At $p=\infty$, its coefficients have orders $\ln(eD)\nu^2$ for Rademacher perturbations and $(\ln(eD))^2\nu^2$ for the chosen Gaussian scales, compared with $D\nu^2$ for conversion.
Thus the ratios $\mathsf C_{\opf}/\mathsf C_*$ have orders $D/\ln(eD)$ and $D/(\ln(eD))^2$, respectively. For finite dimensions, the native coefficient can still coincide with or exceed the converted one.

These outcomes depend on the noise model. For example, the one-sparse perturbations in \Cref{rem:one_sparse_noise_alt} give equal coefficients throughout $1\le p\le2$, so the larger native coefficient for dense noise below $p=2$ is not universal.
We therefore compare the full coefficients rather than infer an advantage from $p$ alone, and claim no uniform improvement over Euclidean/Frobenius conversion.
In the vector case, dual $p=\infty$ corresponds to an $\ell_1$ update norm, while dual $p=1$ corresponds to an $\ell_\infty$ update norm.
For matrices, dual $p=\infty$ corresponds to the nuclear update norm, whose dual is the spectral norm. The analogous improvement here concerns diagonal perturbations, not independent noise in all matrix entries (\Cref{rem:diagonal_model_scope_alt}).

We now specify the oracle models and common assumptions, derive the complexity comparison, and give the detailed vector and matrix results.

\begin{assumption}[Frobenius light-tail oracle]
\label{as:stoch_frobenius_light_tail_alt}
For every $x\in E$, the stochastic gradient is unbiased,
$\ExpSub{\xi}{\nf{x;\xi}}=\nabla f(x)$, and its error satisfies
\begin{align}
\label{eq:frobenius_light_tail_main}
    \ExpSub{\xi}{
        \exp\rb{
            \frac{
                \sqfnorm{\nf{x;\xi}-\nabla f(x)}
            }{
                \sigma_{\opf}^2
            }
        }
    }
    &\le
    \exp(1)
\end{align}
for some $\sigma_{\opf}>0$.
\end{assumption}

Since $E^*$ is finite-dimensional, there are constants $\alpha, \beta>0$ with
\begin{align}
\label{eq:norm_equiv_2_light_tail_alt}
    \alpha\snorm{s}
    \le
    \fnorm{s}
    \le
    \beta\snorm{s},
    \qquad
    \forall s\in E^*.
\end{align}
Only the left inequality is used below, in the form $\snorm{s}\le\alpha^{-1}\fnorm{s}$: it is what turns a Frobenius
bound into a bound in $\snorm{\cdot}$, and $\alpha^{-1}$ is the entire price of doing so.

The converted analysis first concentrates the minibatch error in the Euclidean/Frobenius norm, whose regularity constant is $1$, and then uses
$\snorm{\gk-\nfxk}\le\alpha^{-1}\fnorm{\gk-\nfxk}$.
Consequently, its statistical coefficient is $\alpha^{-2}\sigma_{\opf}^2$, with no additional factor of $\kappa$; \Cref{lem:frobenius_to_dual_conc_alt} gives the concentration bound used below.

\paragraph{Two perturbation models.}
For $u=(u_1, \ldots, u_r)\in\R^r$ we write $\Diag_{m, n}(u)$ for
the rectangular diagonal embedding, defined entrywise for $1\le i\le m$ and $1\le j\le n$ by
\begin{align}
\label{eq:rectangular_diagonal_embedding_alt}
    \bigl[\Diag_{m, n}(u)\bigr]_{ij}
    &\eqdef
    \begin{cases}
        u_i,
        & i=j\le r, \\
        0,
        & \text{otherwise}.
    \end{cases}
\end{align}
The comparison is carried out on two canonical oracles. 
\begin{definition}[Rademacher perturbed-gradient model]
\label{def:rademacher_oracle_alt}
The oracle is $\nf{x;\xi}=\nabla f(x)+\zeta(\xi)$ with $\Exp{\zeta(\xi)}=0$. In the vector case
$\zeta\eqdef\nu\chi$ with $\chi\eqdef(\chi_1, \ldots, \chi_d)$, and in the matrix case
$\zeta\eqdef\nu\Diag_{m, n}\rb{(\chi_i)_{i=1}^r}$, where $\nu>0$ and $\{\chi_i\}_{i\ge1}$ are
independent Rademacher variables, $\Prob\{\chi_i=1\}=\Prob\{\chi_i=-1\}=\frac12$. A fresh perturbation is drawn at every oracle call.
\end{definition}

\begin{definition}[Gaussian perturbed-gradient model]
\label{def:gaussian_oracle_alt}
The oracle is $\nf{x;\xi}=\nabla f(x)+\zeta(\xi)$ with $\Exp{\zeta(\xi)}=0$. In the vector case
$\zeta\eqdef\nu\omega$ with $\omega\eqdef(\omega_1, \ldots, \omega_d)$, and in the matrix case
$\zeta\eqdef\nu\Diag_{m, n}\rb{(\omega_i)_{i=1}^r}$, where $\nu>0$ and
$\{\omega_i\}_{i\ge1}$ are i.i.d.\ standard Gaussian variables, $\omega_i\sim\mathcal N(0, 1)$. A fresh perturbation is drawn at every oracle call.
\end{definition}

Both matrix models are diagonal, so the perturbation carries exactly $r$ independent coordinates rather than $mn$.
Write $D$ for that number of independent coordinates: $D=d$ in the vector model and $D=r$ in the matrix model. The
Gaussian model additionally needs the shape factor
\begin{align}
\label{eq:gaussian_shape_factor_alt}
    \vartheta_p(D)
    &\eqdef
    \begin{cases}
        1,
        & 1\le p\le2, \\[3pt]
        \min\{p, \ln(eD)\},
        & 2<p\le\infty,
    \end{cases}
\end{align}
again with $D^{2/\infty}\eqdef1$. With this notation the two analyses can be placed side by side before we turn to the
proofs.

\subsection{Common Assumptions and Complexity Bounds}
\label{subsec:frobenius_light_tail_summary_alt}

Throughout this comparison, suppose that \Cref{as:non_euclidean_smoothness} and \Cref{as:lower_bound} hold with $L>0$.
Let $\kappa$ denote a regularity constant of the dual space $(E^*, \snorm{\cdot})$. Fix an initial point $x^0$ with
$\Delta\eqdef f(x^0)-\finf>0$, an accuracy $0<\eps\le L\Delta$, and a failure probability $\delta\in(0, 1)$.
Both routes use the update \eqref{eq:non-eu-one-step-final} with the same update norm and fixed Borel-measurable sharp
selection. Each iteration draws fresh i.i.d.\ oracle samples independently of the preceding history.
The Frobenius route uses \Cref{as:stoch_frobenius_light_tail_alt}, whereas the geometry-native route uses
\Cref{as:non_euclidean_light_tail}. In both cases, choose
\begin{align*}
    \gamma_k\equiv\frac{1}{2L},
    \qquad
    T\eqdef\left\lceil\frac{128L\Delta}{15\eps}\right\rceil.
\end{align*}
The two analyses differ only in the minibatch size they prescribe. The Frobenius route pays the conversion constant,
\begin{align}
\label{eq:frobenius_batch_alt}
    B_{\opf}
    &\eqdef
    \max\left\{
        \left\lceil
            \frac{16\alpha^{-2}\sigma_{\opf}^2}{\eps}
            \rb{\sqrt2+\sqrt{6\log(T/\delta)}}^2
        \right\rceil,
        1
    \right\},
\end{align}
whereas the geometry-native route pays the regularity constant instead,
\begin{align}
\label{eq:native_batch_alt}
    B_*
    &\eqdef
    \max\left\{
        \left\lceil
            \frac{16\sigma^2}{\eps}
            \rb{\sqrt{2\kappa}+\sqrt{6\log(T/\delta)}}^2
        \right\rceil,
        1
    \right\}.
\end{align}

To justify these choices, condition on the history $\cF_k$ before iteration $k$, which fixes $x^k$.
Apply \Cref{lem:frobenius_to_dual_conc_alt} for the Frobenius route and
\Cref{lemma:non_euclidean_large_deviations} for the geometry-native route, with
$\lambda=\sqrt{3\log(T/\delta)}$. The corresponding batch size makes the concentration threshold at most
$\sqrt{\eps}/4$, so the conditional probability of exceeding this error is at most
$\exp(-\lambda^2/3)=\delta/T$. Taking expectations of these conditional probabilities and applying a union bound gives,
for either run,
\begin{align*}
    \Prob\cb{\exists k\in\{0, \ldots, T-1\}:\snorm{\gk-\nfxk}>\frac{\sqrt{\eps}}{4}}
    \le T\exp(-\lambda^2/3)=\delta.
\end{align*}
On the complementary event, the deterministic descent estimate in the proof of
\Cref{thm:stochastic_L_smooth_final} gives
\begin{align*}
    f(\xkpo)\le f(\xk)-\frac{15}{128L}\sqsnorm{\nfxk}
    \qquad\text{whenever }\sqsnorm{\nfxk}>\eps.
\end{align*}
If this strict gradient inequality held at every iteration, summing the descent inequalities along that trajectory would yield
\begin{align*}
    f(x^T)<f(x^0)-\frac{15T\eps}{128L}\le\finf,
\end{align*}
contradicting \Cref{as:lower_bound}. Thus each run satisfies
$\min_{0\le k<T}\sqsnorm{\nabla f(x^k)}\le\eps$ with probability at least $1-\delta$.

Suppressing logarithmic factors, each minibatch size contributes a single coefficient to the statistical term,
\begin{align}
\label{eq:statistical_coefficients_alt}
    \mathsf C_{\opf}
    &\eqdef
    \alpha^{-2}\sigma_{\opf}^2,
    &
    \mathsf C_*
    &\eqdef
    \kappa\sigma^2.
\end{align}
Since $\eps\le L\Delta$, the ceiling in $T$ gives $T=\cO(L\Delta/\eps)$. The batch sizes satisfy
$B_{\opf}=\widetilde{\cO}(1+\mathsf C_{\opf}/\eps)$ and
$B_*=\widetilde{\cO}(1+\mathsf C_*/\eps)$, where the suppressed factors are logarithmic in $T$ and $1/\delta$.
Multiplying by $T$, the total numbers of oracle calls $N_{\opf}\eqdef TB_{\opf}$ and $N_*\eqdef TB_*$ therefore satisfy
\begin{align}
\label{eq:complexity_frobenius_alt}
    N_{\opf}
    &=
    \widetilde{\cO}\rb{
        \frac{L\Delta}{\eps}
        +
        \frac{L\Delta}{\eps^2}\mathsf C_{\opf}
    }\\
\label{eq:complexity_dual_native_cor_F}
    &=\widetilde{\cO}\rb{\frac{L\Delta}{\eps}+\frac{L\Delta\sigma_{\opf}^2}{\eps^2\alpha^2}}
\end{align}
and
\begin{align}
\label{eq:complexity_native_alt}
    N_*
    &=
    \widetilde{\cO}\rb{
        \frac{L\Delta}{\eps}
        +
        \frac{L\Delta}{\eps^2}\mathsf C_*
    }.
\end{align}
The first term is identical in both, so the comparison reduces, up to the suppressed logarithmic factors, to $\mathsf C_{\opf}$ against $\mathsf C_*$.
These coefficients refer to the same oracle and update geometry, but use noise scales measured in different norms. We compare the coefficients rather than interpret their ratio as an exact ratio of oracle complexities, since both analyses provide upper bounds.

\paragraph{Side-by-side comparison of the noise terms.}
The tables below place the two statistical terms next to each other, separating the regimes $1\le p<2$ and
$2<p<\infty$ and displaying the boundary cases $p=1$, $p=2$ and $p=\infty$ in full, so that every dimension exponent
can be read off directly rather than inferred. Four conventions apply throughout. We take $\nu^2=\cO(1)$ to expose the
dependence on dimension; a general $\nu$ is restored by multiplying every displayed term by $\nu^2$. A formula indexed
by $\mathsf M\in\{\mathrm{Rad}, \mathrm{Gau}\}$ holds for both perturbation models, the numerical Gaussian factor $4$
being absorbed into $\widetilde{\cO}$; the two models are written separately only where they genuinely differ. The
symbols $\kappa_{\ell_p}(d)$ and $\kappa_{S_p}(r)$ are the admissible regularity parameters established in
\Cref{subsec:frobenius_light_tail_proofs_alt}. Finally, for every fixed finite $p$ the two perturbation models carry
the same dependence on dimension up to constants, and they part company only at $p=\infty$, where the admissible
Gaussian scale costs one additional logarithmic factor.

One exponent is worth unpacking before the tables, since it is the only one that is not read off a single quantity.
For $1\le p<2$ the regularity factor is $\kappa=D^{2/p-1}$. In the Rademacher model, the native noise scale is
$\sigma^2=D^{2/p}\nu^2$, so the exact coefficient is
\begin{align*}
    \mathsf C_*^{\mathrm{Rad}}
    &=
    \kappa\sigma^2\\
    &=
    D^{2/p-1}\cdot D^{2/p}\nu^2\\
    &=
    D^{4/p-1}\nu^2.
\end{align*}
In the Gaussian model, the chosen admissible scale $\sigma^2=4D^{2/p}\nu^2$ instead gives
\begin{align*}
    \mathsf C_*^{\mathrm{Gau}}
    &=\kappa\sigma^2
    =D^{2/p-1}\cdot4D^{2/p}\nu^2
    =4D^{4/p-1}\nu^2.
\end{align*}
Thus both models have the dimension exponent $\nfr{4}{p}-1$; only the absolute factor differs.

\newpage
\noindent\textbf{$\ell_p^d$ spaces.}
Read the two columns against each other row by row. The left column never improves on $d$ and degrades to $d^{2/p}$
below $p=2$. The native coefficient has a larger dimension exponent for $p<2$, and the two coefficients coincide at
$p=2$. For each fixed $2<p<\infty$, the native coefficient has a smaller asymptotic order as $d\to\infty$.
At $p=\infty$, its dependence is logarithmic for Rademacher noise and squared-logarithmic for Gaussian noise,
again asymptotically smaller than linear. These comparisons do not guarantee strict improvement in every finite dimension.

\noindent
\begin{minipage}[t]{0.475\textwidth}
\small
\noindent\textbf{Euclidean/Frobenius approach}

For $1\le p<2$,
\begin{align*}
    N_{\opf, \mathrm{noise}}^{\mathsf M}
    &=
    \widetilde{\cO}\rb{
        \frac{d^{2/p}L\Delta}{\eps^2}
    }.
\end{align*}
In particular, at $p=1$,
\begin{align*}
    \left.N_{\opf, \mathrm{noise}}^{\mathsf M}\right|_{p=1}
    &=
    \widetilde{\cO}\rb{
        \frac{d^2L\Delta}{\eps^2}
    }.
\end{align*}
For every $2\le p\le\infty$,
\begin{align*}
    N_{\opf, \mathrm{noise}}^{\mathsf M}
    &=
    \widetilde{\cO}\rb{
        \frac{dL\Delta}{\eps^2}
    }.
\end{align*}
\end{minipage}
\hfill
\begin{minipage}[t]{0.475\textwidth}
\small
\noindent\textbf{Geometry-native approach}

For $1\le p<2$,
\begin{align*}
    N_{*, \mathrm{noise}}^{\mathsf M}
    &=
    \widetilde{\cO}\rb{
        \frac{d^{4/p-1}L\Delta}{\eps^2}
    }.
\end{align*}
In particular, at $p=1$,
\begin{align*}
    \left.N_{*, \mathrm{noise}}^{\mathsf M}\right|_{p=1}
    &=
    \widetilde{\cO}\rb{
        \frac{d^3L\Delta}{\eps^2}
    }.
\end{align*}
At $p=2$,
\begin{align*}
    \left.N_{*, \mathrm{noise}}^{\mathsf M}\right|_{p=2}
    &=
    \widetilde{\cO}\rb{
        \frac{dL\Delta}{\eps^2}
    }.
\end{align*}
For $2<p<\infty$,
\begin{align*}
    N_{*, \mathrm{noise}}^{\mathrm{Rad}}
    &=
    \widetilde{\cO}\rb{
        \frac{\kappa_{\ell_p}(d)d^{2/p}L\Delta}{\eps^2}
    }, \\
    N_{*, \mathrm{noise}}^{\mathrm{Gau}}
    &=
    \widetilde{\cO}\rb{
        \frac{\kappa_{\ell_p}(d)\vartheta_p(d)d^{2/p}L\Delta}{\eps^2}
    }.
\end{align*}
At $p=\infty$, for every $d\ge1$,
\begin{align*}
    \left.N_{*, \mathrm{noise}}^{\mathrm{Rad}}\right|_{p=\infty}
    &=
    \widetilde{\cO}\rb{
        \frac{\ln(ed)L\Delta}{\eps^2}
    }, \\
    \left.N_{*, \mathrm{noise}}^{\mathrm{Gau}}\right|_{p=\infty}
    &=
    \widetilde{\cO}\rb{
        \frac{(\ln(ed))^2L\Delta}{\eps^2}
    }.
\end{align*}
\end{minipage}

\newpage
\noindent\textbf{Spaces with Schatten-$p$ norms.}
The diagonal perturbations in \Cref{def:rademacher_oracle_alt} and \Cref{def:gaussian_oracle_alt} allow the vector noise-scale
calculations to carry over with $d$ replaced by $r=\min\{m, n\}$. For $u\in\R^r$ and $\nu>0$, the singular values of
$\nu\Diag_{m, n}(u)$ are the numbers $\nu|u_i|$ arranged in nonincreasing order. Consequently,
\begin{align*}
    \norm{\nu\Diag_{m, n}(u)}_{S_p}=\nu\norm{u}_p,
    \qquad 1\le p\le\infty.
\end{align*}
The resulting statistical coefficients have analogous asymptotic dimension dependence, but use the separately defined Schatten regularity parameter $\kappa_{S_p}(r)$.
It need not equal $\kappa_{\ell_p}(d)$ even when $d=r$.

\noindent
\begin{minipage}[t]{0.475\textwidth}
\small
\noindent\textbf{Euclidean/Frobenius approach}

For $1\le p<2$,
\begin{align*}
    N_{\opf, \mathrm{noise}}^{\mathsf M}
    &=
    \widetilde{\cO}\rb{
        \frac{r^{2/p}L\Delta}{\eps^2}
    }.
\end{align*}
In particular, at $p=1$,
\begin{align*}
    \left.N_{\opf, \mathrm{noise}}^{\mathsf M}\right|_{p=1}
    &=
    \widetilde{\cO}\rb{
        \frac{r^2L\Delta}{\eps^2}
    }.
\end{align*}
For every $2\le p\le\infty$,
\begin{align*}
    N_{\opf, \mathrm{noise}}^{\mathsf M}
    &=
    \widetilde{\cO}\rb{
        \frac{rL\Delta}{\eps^2}
    }.
\end{align*}
\end{minipage}
\hfill
\begin{minipage}[t]{0.475\textwidth}
\small
\noindent\textbf{Geometry-native approach}

For $1\le p<2$,
\begin{align*}
    N_{*, \mathrm{noise}}^{\mathsf M}
    &=
    \widetilde{\cO}\rb{
        \frac{r^{4/p-1}L\Delta}{\eps^2}
    }.
\end{align*}
In particular, at $p=1$,
\begin{align*}
    \left.N_{*, \mathrm{noise}}^{\mathsf M}\right|_{p=1}
    &=
    \widetilde{\cO}\rb{
        \frac{r^3L\Delta}{\eps^2}
    }.
\end{align*}
At $p=2$,
\begin{align*}
    \left.N_{*, \mathrm{noise}}^{\mathsf M}\right|_{p=2}
    &=
    \widetilde{\cO}\rb{
        \frac{rL\Delta}{\eps^2}
    }.
\end{align*}
For $2<p<\infty$,
\begin{align*}
    N_{*, \mathrm{noise}}^{\mathrm{Rad}}
    &=
    \widetilde{\cO}\rb{
        \frac{\kappa_{S_p}(r)r^{2/p}L\Delta}{\eps^2}
    }, \\
    N_{*, \mathrm{noise}}^{\mathrm{Gau}}
    &=
    \widetilde{\cO}\rb{
        \frac{\kappa_{S_p}(r)\vartheta_p(r)r^{2/p}L\Delta}{\eps^2}
    }.
\end{align*}
At $p=\infty$, for every $r\ge1$,
\begin{align*}
    \left.N_{*, \mathrm{noise}}^{\mathrm{Rad}}\right|_{p=\infty}
    &=
    \widetilde{\cO}\rb{
        \frac{\ln(er)L\Delta}{\eps^2}
    }, \\
    \left.N_{*, \mathrm{noise}}^{\mathrm{Gau}}\right|_{p=\infty}
    &=
    \widetilde{\cO}\rb{
        \frac{(\ln(er))^2L\Delta}{\eps^2}
    }.
\end{align*}
\end{minipage}

\subsection{Detailed Results and Proofs}
\label{subsec:frobenius_light_tail_proofs_alt}

We first establish the Euclidean/Frobenius conversion bound used in the summary.

\begin{lemma}[Frobenius-to-dual concentration]
\label{lem:frobenius_to_dual_conc_alt}
Let $\{\eta^j\}_{j=1}^B$ be i.i.d. random elements of $E^*$ such that $\Exp{\eta^j}=0$ and
$\Exp{\exp\rb{\frac{\sqfnorm{\eta^j}}{\sigma_{\opf}^2}}}\le\exp(1)$.
Define $\bar\eta\eqdef B^{-1}\sum_{j=1}^B\eta^j$.
Then, for every $\lambda\ge0$,
\begin{align}
\label{eq:prob_concentration_alt}
    \Prob\cb{
        \snorm{\bar\eta}
        \ge
        \frac1\alpha
        \rb{\sqrt2+\sqrt2\lambda}
        \frac{\sigma_{\opf}}{\sqrt B}
    }
    &\le
    \exp\rb{-\lambda^2/3}.
\end{align}
\end{lemma}

\begin{proof}
The Euclidean space $(E^*, \fnorm{\cdot})$ is $1$-regular. Applying
\Cref{lemma:non_euclidean_large_deviations} in this space gives
\begin{align*}
    \Prob\cb{
        \fnorm{\bar\eta}
        \ge
        \rb{\sqrt2+\sqrt2\lambda}
        \frac{\sigma_{\opf}}{\sqrt B}
    }
    &\le
    \exp\rb{-\lambda^2/3}.
\end{align*}
For any $t\ge0$, the left inequality in
\eqref{eq:norm_equiv_2_light_tail_alt} implies the event inclusion
\begin{align*}
    \cb{\snorm{\bar\eta}\ge t}
    &\subseteq
    \cb{\fnorm{\bar\eta}\ge\alpha t}.
\end{align*}
Taking
\begin{align*}
    t
    &\eqdef
    \frac1\alpha
    \rb{\sqrt2+\sqrt2\lambda}
    \frac{\sigma_{\opf}}{\sqrt B}
\end{align*}
proves \eqref{eq:prob_concentration_alt}.
\end{proof}

We next establish the Gaussian light-tail scale used for both norm families.

\begin{lemma}[Gaussian light tails in $\ell_p$]
\label{lem:gaussian_light_tail_scale_alt}
Let $\omega\sim\mathcal N(0, I_D)$ and $p\in[1, \infty]$. For every $\nu>0$,
\begin{align}
\label{eq:gaussian_light_tail_scale_alt}
    \Exp{
        \exp\rb{
            \frac{
                \norm{\nu\omega}_p^2
            }{
                4\vartheta_p(D)D^{2/p}\nu^2
            }
        }
    }
    &\le
    \exp(1).
\end{align}
\end{lemma}

\begin{proof}
We first consider $1\le p\le2$. The norm comparison
\begin{align*}
    \norm{\omega}_p^2
    &\le
    D^{\frac2p-1}\norm{\omega}_2^2
\end{align*}
gives
\begin{align*}
    \Exp{
        \exp\rb{
            \frac{\norm{\nu\omega}_p^2}{4D^{2/p}\nu^2}
        }
    }
    &\le
    \Exp{
        \exp\rb{
            \frac{\norm{\omega}_2^2}{4D}
        }
    }.
\end{align*}
Since $\norm{\omega}_2^2$ has a chi-squared distribution with $D$ degrees of freedom,
\begin{align*}
    \Exp{
        \exp\rb{
            \frac{\norm{\omega}_2^2}{4D}
        }
    }
    &=
    \rb{1-\frac{1}{2D}}^{-D/2}.
\end{align*}
For $x\in[0, 1)$, the inequality $-\ln(1-x)\le \nfr{x}{1-x}$ yields
\begin{align*}
    -\frac D2
    \ln\rb{1-\frac{1}{2D}}
    &\le
    \frac14
    \frac1{1-\frac{1}{2D}}
    \le
    \frac12
    <
    1.
\end{align*}
Thus \eqref{eq:gaussian_light_tail_scale_alt} holds for $1\le p\le2$.

Now let $2<p\le\infty$ and set
\begin{align*}
    m_p
    &\eqdef
    \Exp{\norm{\omega}_p}.
\end{align*}
For finite $p$, Jensen's inequality and the Gaussian moment formula give
\begin{align*}
    m_p
    &\le
    \rb{\Exp{\norm{\omega}_p^p}}^{1/p}\\
    &=
    D^{1/p}\rb{\Exp{|\omega_1|^p}}^{1/p}\\
    &=
    D^{1/p}
    \rb{
        \frac{2^{p/2}\Gamma\rb{(p+1)/2}}{\sqrt\pi}
    }^{1/p}\\
    &\le
    D^{1/p}\sqrt p.
\end{align*}
The last inequality is the standard Gaussian moment bound
$\rb{\Exp{|\omega_1|^p}}^{1/p}\le\sqrt p$. A second bound follows from
\begin{align*}
    \norm{\omega}_p
    &\le
    D^{1/p}\norm{\omega}_\infty.
\end{align*}
For any $t>0$, we first have
\begin{align*}
    \Exp{\exp\rb{t\norm{\omega}_\infty}}
    &\le
    \sum_{i=1}^D\Exp{\exp\rb{t|\omega_i|}}\\
    &\le
    2D\exp\rb{t^2/2}.
\end{align*}
The first inequality uses
\begin{align*}
    \exp\rb{t\max_{1\le i\le D}|\omega_i|}
    &\le
    \sum_{i=1}^D\exp\rb{t|\omega_i|},
\end{align*}
and the second uses
\begin{align*}
    \Exp{\exp\rb{t|\omega_i|}}
    &\le
    \Exp{\exp\rb{t\omega_i}}
    +
    \Exp{\exp\rb{-t\omega_i}}\\
    &=
    2\exp\rb{t^2/2}.
\end{align*}
Consequently,
\begin{align*}
    \Exp{\norm{\omega}_\infty}
    &\le
    \frac1t
    \ln\Exp{\exp\rb{t\norm{\omega}_\infty}}\\
    &\le
    \frac{\ln(2D)}t+\frac t2.
\end{align*}
Taking $t=\sqrt{2\ln(2D)}$ gives
\begin{align*}
    \Exp{\norm{\omega}_\infty}
    &\le
    \sqrt{2\ln(2D)}.
\end{align*}
Combining the two estimates, including the second one at $p=\infty$, gives
\begin{align}
\label{eq:gaussian_lp_mean_alt}
    m_p^2
    &\le
    \min\{p, 2\ln(2D)\}D^{2/p}.
\end{align}

Set
\begin{align*}
    H_p(D)
    &\eqdef
    \min\{p, \ln(eD)\}D^{2/p}.
\end{align*}
If $D=1$, then $H_p(1)=1$ and $\norm{\omega}_p=|\omega_1|$ for every $p>2$. Therefore,
\begin{align*}
    \Exp{\exp\rb{\frac{\norm{\omega}_p^2}{4H_p(1)}}}
    &=
    \Exp{\exp\rb{\frac{\omega_1^2}{4}}}\\
    &=
    \sqrt2\\
    &<
    \exp(1).
\end{align*}
It remains to consider $D\ge2$.

For $p>2$, the map $z\mapsto\norm{z}_p$ is $1$-Lipschitz with respect to $\norm{\cdot}_2$, because
\begin{align*}
    \abs{\norm{u}_p-\norm{v}_p}
    &\le
    \norm{u-v}_p
    \le
    \norm{u-v}_2.
\end{align*}
The moment-generating-function form of Gaussian concentration, obtained from the Gaussian log-Sobolev inequality by
the Herbst argument, therefore gives, for
\begin{align*}
    X_p
    &\eqdef
    \norm{\omega}_p-m_p,
\end{align*}
the bound
\begin{align}
\label{eq:gaussian_lp_centered_mgf_alt}
    \Exp{\exp\rb{sX_p}}
    &\le
    \exp\rb{s^2/2},
    \qquad
    s\in\R.
\end{align}
Let $\gamma\sim\mathcal N(0, 1)$ be independent of $\omega$. Gaussian linearization and
\eqref{eq:gaussian_lp_centered_mgf_alt} give
\begin{align*}
    \Exp{\exp\rb{\frac{\norm{\omega}_p^2}{4H_p(D)}}}
    &=
    \bb E_{\omega, \gamma}\left[
    \exp\rb{
        \frac{\gamma\norm{\omega}_p}{\sqrt{2H_p(D)}}
    }\right]\\
    &=
    \bb E_{\gamma}\left[
    \exp\rb{
        \frac{\gamma m_p}{\sqrt{2H_p(D)}}
    }
    \bb E_{\omega}\left[
        \exp\rb{
            \frac{\gamma X_p}{\sqrt{2H_p(D)}}
        }
    \right]\right]\\
    &\le
    \bb E_{\gamma}\left[
    \exp\rb{
        \frac{\gamma m_p}{\sqrt{2H_p(D)}}
        +
        \frac{\gamma^2}{4H_p(D)}
    }\right]\\
    &=
    \rb{1-\frac1{2H_p(D)}}^{-1/2}
    \exp\rb{
        \frac{m_p^2}{4H_p(D)-2}
    }.
\end{align*}

Suppose first that $2<p\le\ln(eD)$. Then $H_p(D)=pD^{2/p}>2$, and
\eqref{eq:gaussian_lp_mean_alt} gives $m_p^2\le H_p(D)$. Hence,
\begin{align*}
    \ln\Exp{\exp\rb{\frac{\norm{\omega}_p^2}{4H_p(D)}}}
    &\le
    -\frac12\ln\rb{1-\frac1{2H_p(D)}}
    +
    \frac{m_p^2}{4H_p(D)-2}\\
    &<
    \frac12\ln\rb{\frac43}
    +
    \frac13\\
    &<
    1.
\end{align*}

Suppose next that $p>\ln(eD)$. Since $D\ge2$,
\begin{align*}
    \ln(2D)
    &\le
    \ln(eD).
\end{align*}
Therefore, \eqref{eq:gaussian_lp_mean_alt} gives $m_p^2\le2H_p(D)$. Moreover,
\begin{align*}
    H_p(D)
    &=
    \ln(eD)D^{2/p}\\
    &\ge
    \ln(2e)\\
    &>
    \frac32.
\end{align*}
It follows that
\begin{align*}
    \ln\Exp{\exp\rb{\frac{\norm{\omega}_p^2}{4H_p(D)}}}
    &\le
    -\frac12\ln\rb{1-\frac1{2H_p(D)}}
    +
    \frac{m_p^2}{4H_p(D)-2}\\
    &<
    \frac12\ln\rb{\frac32}
    +
    \frac34\\
    &<
    1.
\end{align*}
Rescaling $\omega$ by $\nu$ proves \eqref{eq:gaussian_light_tail_scale_alt} for $p>2$.
\end{proof}

\subsubsection{\texorpdfstring{$\ell_p^d$ Spaces}{ell-p Spaces in Dimension d}}
\label{subsec:frobenius_vs_dual_lp_alt}

Let $(E^*, \snorm{\cdot})=(\R^d, \norm{\cdot}_p)$ with $p\in[1, \infty]$. The sharp one-sided Euclidean-to-$\ell_p$
constant is
\begin{align}
\label{eq:alpha_lp_alt}
    \alpha_p(d)^{-2}
    &\eqdef
    \begin{cases}
        d^{\frac2p-1},
        & 1\le p<2, \\[3pt]
        1,
        & 2\le p\le\infty.
    \end{cases}
\end{align}
Indeed,
\begin{align*}
    \norm{z}_p
    &\le
    d^{\frac1p-\frac12}\norm{z}_2,
    && 1\le p<2,
\end{align*}
whereas
\begin{align*}
    \norm{z}_p
    &\le
    \norm{z}_2,
    && 2\le p\le\infty.
\end{align*}

Write $\kappa_{\ell_p}(d)$ for the following admissible regularity parameter of $(\R^d, \norm{\cdot}_p)$:
\begin{align}
\label{eq:kappa_lp_alt}
    \kappa_{\ell_p}(d)
    &\eqdef
    \begin{cases}
        d^{\frac2p-1},
        & 1\le p<2, \\[5pt]
        \displaystyle
        \min_{\substack{2\le\rho\le p\\ \rho<\infty}}
        (\rho-1)d^{\frac2\rho-\frac2p},
        & 2\le p\le\infty.
    \end{cases}
\end{align}
This is the parameter constructed in \Cref{example:2}. It satisfies
\begin{align}
\label{eq:kappa_lp_upper_alt}
    \kappa_{\ell_p}(d)
    &\le
    2e\min\{p, \ln(ed)\},
    \qquad
    2\le p\le\infty.
\end{align}
The explicit factor $2e$ follows by taking $\rho=p$ when $p\le2\ln(ed)$ and $\rho=2\ln(ed)$ when
$p>2\ln(ed)$. The use of $\ln(ed)$ keeps this exact bound valid also at $d=1$.

\begin{proposition}[$\ell_p^d$ geometry under Rademacher perturbations]
\label{prop:frobenius_vs_dual_lp_rademacher_alt}
Let the oracle be as in \Cref{def:rademacher_oracle_alt}. Then the smallest admissible light-tail scales are
\begin{align}
\label{eq:sigma_lp_rademacher_alt}
    \sigma_{\opf}^2
    &=
    d\nu^2,
    &
    \sigma^2
    &=
    d^{2/p}\nu^2.
\end{align}
With $\alpha^{-2}=\alpha_p(d)^{-2}$ and $\kappa=\kappa_{\ell_p}(d)$, the coefficients in
\eqref{eq:statistical_coefficients_alt} are
\begin{align}
\label{eq:coeff_lp_rademacher_alt}
    \mathsf C_{\opf}^{\mathrm{Rad}}
    &\eqdef
    \begin{cases}
        d^{2/p}\nu^2,
        & 1\le p<2, \\[3pt]
        d\nu^2,
        & 2\le p\le\infty,
    \end{cases}
    \\
    \mathsf C_*^{\mathrm{Rad}}
    &\eqdef
    \begin{cases}
        d^{\frac4p-1}\nu^2,
        & 1\le p<2, \\[3pt]
        \kappa_{\ell_p}(d)d^{2/p}\nu^2,
        & 2\le p\le\infty.
    \end{cases}
\end{align}
The coefficients agree at $p=2$. For $p<2$, the native coefficient contains one additional factor
$d^{2/p-1}$. For every fixed $p>2$,
\begin{align*}
    \mathsf C_*^{\mathrm{Rad}}
    &\le
    2e\min\{p, \ln(ed)\}d^{2/p}\nu^2,
\end{align*}
while $\mathsf C_{\opf}^{\mathrm{Rad}}=d\nu^2$. At $p=\infty$ and $d\ge1$, the displayed dependences are
$\cO(\ln(ed)\nu^2)$ and $d\nu^2$, respectively.
\end{proposition}

\begin{proof}
Every coordinate of $\zeta=\nu\chi$ has absolute value $\nu$. Therefore,
\begin{align*}
    \norm{\zeta}_2^2
    &=
    \nu^2\sum_{i=1}^d\chi_i^2\\
    &=
    d\nu^2.
\end{align*}
For finite $p$,
\begin{align*}
    \norm{\zeta}_p^2
    &=
    \rb{\sum_{i=1}^d|\nu\chi_i|^p}^{2/p}\\
    &=
    \rb{d\nu^p}^{2/p}\\
    &=
    d^{2/p}\nu^2.
\end{align*}
At $p=\infty$,
\begin{align*}
    \norm{\zeta}_\infty^2
    &=
    \max_{1\le i\le d}|\nu\chi_i|^2\\
    &=
    \nu^2,
\end{align*}
which agrees with $d^{2/\infty}=1$.

Let $\widetilde\sigma_{\opf}>0$ and $\widetilde\sigma>0$ be arbitrary candidate light-tail scales for the Euclidean
and $\ell_p$ norms, respectively. Since both norms of $\zeta$ are deterministic,
\begin{align*}
    \Exp{\exp\rb{\frac{\norm{\zeta}_2^2}{\widetilde\sigma_{\opf}^2}}}
    &\le
    \exp(1)
    &\Longleftrightarrow
    \widetilde\sigma_{\opf}^2
    \ge
    d\nu^2,
    \\
    \Exp{\exp\rb{\frac{\norm{\zeta}_p^2}{\widetilde\sigma^2}}}
    &\le
    \exp(1)
    &\Longleftrightarrow
    \widetilde\sigma^2
    \ge
    d^{2/p}\nu^2.
\end{align*}
This proves \eqref{eq:sigma_lp_rademacher_alt} and shows that both scales are minimal.

For $1\le p<2$, substituting \eqref{eq:alpha_lp_alt}, \eqref{eq:kappa_lp_alt}, and
\eqref{eq:sigma_lp_rademacher_alt} gives
\begin{align*}
    \mathsf C_{\opf}^{\mathrm{Rad}}
    &=
    d^{\frac2p-1}\cdot d\nu^2
    =
    d^{2/p}\nu^2,
    \\
    \mathsf C_*^{\mathrm{Rad}}
    &=
    d^{\frac2p-1}\cdot d^{2/p}\nu^2
    =
    d^{\frac4p-1}\nu^2.
\end{align*}
For $2\le p\le\infty$, $\alpha_p(d)^{-2}=1$, so
\begin{align*}
    \mathsf C_{\opf}^{\mathrm{Rad}}
    &=
    d\nu^2,
    \\
    \mathsf C_*^{\mathrm{Rad}}
    &=
    \kappa_{\ell_p}(d)d^{2/p}\nu^2.
\end{align*}
The bound \eqref{eq:kappa_lp_upper_alt} gives the stated consequences.
\end{proof}

\begin{proposition}[$\ell_p^d$ geometry under Gaussian perturbations]
\label{prop:frobenius_vs_dual_lp_gaussian_alt}
Let the oracle be as in \Cref{def:gaussian_oracle_alt}. Using the explicit universal constant from
\Cref{lem:gaussian_light_tail_scale_alt}, the scales
\begin{align}
\label{eq:sigma_lp_gaussian_alt}
    \sigma_{\opf}^2
    &\eqdef
    4d\nu^2,
    &
    \sigma^2
    &\eqdef
    4\vartheta_p(d)d^{2/p}\nu^2
\end{align}
are admissible. The Frobenius scale is also of the smallest possible order: every admissible scale
$\widetilde\sigma_{\opf}$ satisfies $\widetilde\sigma_{\opf}^2\ge d\nu^2$.

For the choices in \eqref{eq:sigma_lp_gaussian_alt}, the coefficients are
\begin{align}
\label{eq:coeff_lp_gaussian_alt}
    \mathsf C_{\opf}^{\mathrm{Gau}}
    &\eqdef
    4
    \begin{cases}
        d^{2/p}\nu^2,
        & 1\le p<2, \\[3pt]
        d\nu^2,
        & 2\le p\le\infty,
    \end{cases}
    \\
    \mathsf C_*^{\mathrm{Gau}}
    &=
    4
    \begin{cases}
        d^{\frac4p-1}\nu^2,
        & 1\le p<2, \\[3pt]
        \kappa_{\ell_p}(d)\vartheta_p(d)d^{2/p}\nu^2,
        & 2\le p\le\infty.
    \end{cases}
\end{align}
At $p=2$, both coefficients are $4d\nu^2$. For every fixed $p>2$,
\begin{align*}
    \mathsf C_*^{\mathrm{Gau}}
    &\le
    8e
    \min\{p, \ln(ed)\}^2d^{2/p}\nu^2,
\end{align*}
whereas $\mathsf C_{\opf}^{\mathrm{Gau}}=4d\nu^2$. At $p=\infty$ and $d\ge1$, the displayed dependences are
$\cO((\ln(ed))^2\nu^2)$ and $d\nu^2$, respectively.
\end{proposition}

\begin{proof}
The Frobenius claim follows from \Cref{lem:gaussian_light_tail_scale_alt} with $D=d$ and $p=2$. To see that its
order cannot be smaller, let $\widetilde\sigma_{\opf}$ be any admissible Frobenius scale. Jensen's inequality gives
\begin{align*}
    \exp(1)
    &\ge
    \Exp{\exp\rb{\frac{\norm{\nu\omega}_2^2}{\widetilde\sigma_{\opf}^2}}}\\
    &\ge
    \exp\rb{
        \frac{\Exp{\norm{\nu\omega}_2^2}}{\widetilde\sigma_{\opf}^2}
    }.
\end{align*}
Since
\begin{align*}
    \Exp{\norm{\nu\omega}_2^2}
    &=
    d\nu^2,
\end{align*}
we obtain
\begin{align*}
    \widetilde\sigma_{\opf}^2
    &\ge
    d\nu^2.
\end{align*}

The native scale follows from the same lemma with $D=d$ and the given value of $p$. It remains to substitute the two
scales into \eqref{eq:statistical_coefficients_alt}. If $1\le p<2$, then
\begin{align*}
    \mathsf C_{\opf}^{\mathrm{Gau}}
    &=
    d^{\frac2p-1}\cdot 4d\nu^2
    =
    4d^{2/p}\nu^2,
    \\
    \mathsf C_*^{\mathrm{Gau}}
    &=
    d^{\frac2p-1}\cdot 4d^{2/p}\nu^2
    =
    4d^{\frac4p-1}\nu^2.
\end{align*}
If $2\le p\le\infty$, then
\begin{align*}
    \mathsf C_{\opf}^{\mathrm{Gau}}
    &=
    4d\nu^2,
    \\
    \mathsf C_*^{\mathrm{Gau}}
    &=
    4\kappa_{\ell_p}(d)\vartheta_p(d)d^{2/p}\nu^2.
\end{align*}
Using \eqref{eq:kappa_lp_upper_alt} proves the final two statements.
\end{proof}

\begin{remark}[One-sparse perturbations]
\label{rem:one_sparse_noise_alt}
Fix $\nu>0$. Let $\chi$ be Rademacher and $I$ be uniform on $[D]$, independently of $\chi$, and let
$e_1, \ldots, e_D$ be the standard basis of $\R^D$. Define
\begin{align*}
    \zeta&\eqdef\nu\chi e_I &&\text{in the vector case, }D=d,\\
    \zeta&\eqdef\nu\Diag_{m, n}(\chi e_I) &&\text{in the matrix case, }D=r=\min\{m, n\}.
\end{align*}
The oracle returns $\nabla f(x)+\zeta$ with a fresh independent pair $(\chi, I)$ at every call.
In both cases, $\Exp{\zeta}=0$, and the Euclidean/Frobenius norm and the native $\ell_p$/Schatten-$p$ norm
of $\zeta$ equal $\nu$ almost surely for every $1\le p\le\infty$.
Thus the light-tail condition reduces to $\exp\rb{\frac{\nu^2}{\sigma^2}}\le e$, and the minimal squared light-tail scales
are $\sigma_{\opf}^2=\sigma^2=\nu^2$. For $1\le p\le2$, the chosen conversion and regularity factors both equal
$D^{2/p-1}$, so
\begin{align*}
    \mathsf C_{\opf}=\mathsf C_*=D^{2/p-1}\nu^2.
\end{align*}
Sparsity alone does not ensure that the coefficients coincide for other perturbation models.
\end{remark}

\subsubsection{Spaces with Schatten-\texorpdfstring{$p$}{p} Norms on \texorpdfstring{$\R^{m\times n}$}{real m-by-n matrices}}
\label{subsec:frobenius_vs_dual_schatten_alt}

Let $(E^*, \snorm{\cdot})=(\R^{m\times n}, \norm{\cdot}_{S_p})$, where $p\in[1, \infty]$ and
$r\eqdef\min\{m, n\}$. The Frobenius norm is $\fnorm{\cdot}=\norm{\cdot}_{S_2}$. Applying the vector norm comparison
to the singular-value vector gives
\begin{align}
\label{eq:alpha_schatten_alt}
    \alpha_{S_p}(r)^{-2}
    &\eqdef
    \begin{cases}
        r^{\frac2p-1},
        & 1\le p<2, \\[3pt]
        1,
        & 2\le p\le\infty.
    \end{cases}
\end{align}
Indeed,
\begin{align*}
    \norm{Z}_{S_p}
    &\le
    r^{\frac1p-\frac12}\fnorm{Z},
    && 1\le p<2,
\end{align*}
whereas
\begin{align*}
    \norm{Z}_{S_p}
    &\le
    \fnorm{Z},
    && 2\le p\le\infty.
\end{align*}

Choose an admissible regularity parameter $\kappa_{S_p}(r)$ as follows:
\begin{align}
\label{eq:kappa_schatten_alt}
    \kappa_{S_p}(r)
    &\eqdef
    \begin{cases}
        r^{\frac2p-1},
        & 1\le p<2, \\[3pt]
        1,
        & p=2, \\[5pt]
        \displaystyle
        \min_{\substack{2\le\rho\le p\\ \rho<\infty}}
        \max\{2, \rho-1\}r^{\frac2\rho-\frac2p},
        & 2<p\le\infty.
    \end{cases}
\end{align}
The $p=2$ line uses the exact Hilbert-space constant. The remaining lines follow from \Cref{example:3}. In particular,
\begin{align}
\label{eq:kappa_schatten_upper_alt}
    \kappa_{S_p}(r)
    &\le
    2e\min\{p, \ln(er)\},
    \qquad
    2\le p\le\infty.
\end{align}
As in the vector case, for $p>2$ take $\rho=p$ when $p\le2\ln(er)$ and $\rho=2\ln(er)$ otherwise. The factor
$2e$ is valid also at $r=1$.

\begin{proposition}[Schatten-$p$ geometry under diagonal Rademacher perturbations]
\label{prop:frobenius_vs_dual_schatten_rademacher_alt}
Let the oracle be as in \Cref{def:rademacher_oracle_alt}. Then the smallest admissible light-tail scales are
\begin{align}
\label{eq:sigma_schatten_rademacher_alt}
    \sigma_{\opf}^2
    &=
    r\nu^2,
    &
    \sigma^2
    &=
    r^{2/p}\nu^2.
\end{align}
With $\alpha^{-2}=\alpha_{S_p}(r)^{-2}$ and $\kappa=\kappa_{S_p}(r)$, the statistical coefficients are
\begin{align}
\label{eq:coeff_schatten_rademacher_alt}
    \mathsf C_{\opf}^{\mathrm{Rad}}
    &\eqdef
    \begin{cases}
        r^{2/p}\nu^2,
        & 1\le p<2, \\[3pt]
        r\nu^2,
        & 2\le p\le\infty,
    \end{cases}
    \\
    \mathsf C_*^{\mathrm{Rad}}
    &\eqdef
    \begin{cases}
        r^{\frac4p-1}\nu^2,
        & 1\le p<2, \\[3pt]
        \kappa_{S_p}(r)r^{2/p}\nu^2,
        & 2\le p\le\infty.
    \end{cases}
\end{align}
The coefficients agree at $p=2$. For $p<2$, the native coefficient contains one additional factor
$r^{2/p-1}$. For every fixed $p>2$,
\begin{align*}
    \mathsf C_*^{\mathrm{Rad}}
    &\le
    2e\min\{p, \ln(er)\}r^{2/p}\nu^2,
\end{align*}
whereas $\mathsf C_{\opf}^{\mathrm{Rad}}=r\nu^2$. At $p=\infty$ and $r\ge1$, the displayed dependences are
$\cO(\ln(er)\nu^2)$ and $r\nu^2$, respectively.
\end{proposition}

\begin{proof}
The matrix
\begin{align*}
    \zeta
    &=
    \nu\Diag_{m, n}\rb{(\chi_i)_{i=1}^r}
\end{align*}
has singular values
\begin{align*}
    \sigma_i(\zeta)
    &=
    \nu|\chi_i|
    =
    \nu,
    \qquad
    i=1, \ldots, r.
\end{align*}
Therefore,
\begin{align*}
    \fnorm{\zeta}^2
    &=
    \sum_{i=1}^r\sigma_i(\zeta)^2\\
    &=
    r\nu^2.
\end{align*}
For finite $p$,
\begin{align*}
    \norm{\zeta}_{S_p}^2
    &=
    \rb{\sum_{i=1}^r\sigma_i(\zeta)^p}^{2/p}\\
    &=
    \rb{r\nu^p}^{2/p}\\
    &=
    r^{2/p}\nu^2.
\end{align*}
At $p=\infty$,
\begin{align*}
    \norm{\zeta}_{S_\infty}^2
    &=
    \max_{1\le i\le r}\sigma_i(\zeta)^2\\
    &=
    \nu^2.
\end{align*}
Let $\widetilde\sigma_{\opf}>0$ and $\widetilde\sigma>0$ be arbitrary candidate light-tail scales for the Frobenius
and Schatten-$p$ norms, respectively. Since both norms of $\zeta$ are deterministic,
\begin{align*}
    \Exp{\exp\rb{\frac{\fnorm{\zeta}^2}{\widetilde\sigma_{\opf}^2}}}
    &\le
    \exp(1)
    \\
    &\Longleftrightarrow
    \exp\rb{\frac{r\nu^2}{\widetilde\sigma_{\opf}^2}}
    \le
    \exp(1)
    \\
    &\Longleftrightarrow
    \widetilde\sigma_{\opf}^2
    \ge
    r\nu^2,
\end{align*}
and, with the convention $r^{2/\infty}=1$,
\begin{align*}
    \Exp{\exp\rb{\frac{\norm{\zeta}_{S_p}^2}{\widetilde\sigma^2}}}
    &\le
    \exp(1)
    \\
    &\Longleftrightarrow
    \exp\rb{\frac{r^{2/p}\nu^2}{\widetilde\sigma^2}}
    \le
    \exp(1)
    \\
    &\Longleftrightarrow
    \widetilde\sigma^2
    \ge
    r^{2/p}\nu^2.
\end{align*}
Thus the choices $\sigma_{\opf}^2=r\nu^2$ and $\sigma^2=r^{2/p}\nu^2$ in
\eqref{eq:sigma_schatten_rademacher_alt} are admissible and are the smallest possible scales.

For $1\le p<2$, substitution gives
\begin{align*}
    \mathsf C_{\opf}^{\mathrm{Rad}}
    &=
    r^{\frac2p-1}\cdot r\nu^2
    =
    r^{2/p}\nu^2,
    \\
    \mathsf C_*^{\mathrm{Rad}}
    &=
    r^{\frac2p-1}\cdot r^{2/p}\nu^2
    =
    r^{\frac4p-1}\nu^2.
\end{align*}
For $2\le p\le\infty$, $\alpha_{S_p}(r)^{-2}=1$, and hence
\begin{align*}
    \mathsf C_{\opf}^{\mathrm{Rad}}
    &=
    r\nu^2,
    \\
    \mathsf C_*^{\mathrm{Rad}}
    &=
    \kappa_{S_p}(r)r^{2/p}\nu^2.
\end{align*}
The bound \eqref{eq:kappa_schatten_upper_alt} gives the stated consequences.
\end{proof}

\begin{proposition}[Schatten-$p$ geometry under diagonal Gaussian perturbations]
\label{prop:frobenius_vs_dual_schatten_gaussian_alt}
Let the oracle be as in \Cref{def:gaussian_oracle_alt}. Using the explicit universal constant from
\Cref{lem:gaussian_light_tail_scale_alt}, the scales
\begin{align}
\label{eq:sigma_schatten_gaussian_alt}
    \sigma_{\opf}^2
    &\eqdef
    4r\nu^2,
    &
    \sigma^2
    &\eqdef
    4\vartheta_p(r)r^{2/p}\nu^2
\end{align}
are admissible. Every admissible Frobenius scale $\widetilde\sigma_{\opf}$ satisfies
$\widetilde\sigma_{\opf}^2\ge r\nu^2$, so the displayed Frobenius scale has the smallest possible order.

For the choices in \eqref{eq:sigma_schatten_gaussian_alt}, the coefficients are
\begin{align}
\label{eq:coeff_schatten_gaussian_alt}
    \mathsf C_{\opf}^{\mathrm{Gau}}
    &\eqdef
    4
    \begin{cases}
        r^{2/p}\nu^2,
        & 1\le p<2, \\[3pt]
        r\nu^2,
        & 2\le p\le\infty,
    \end{cases}
    \\
    \mathsf C_*^{\mathrm{Gau}}
    &=
    4
    \begin{cases}
        r^{\frac4p-1}\nu^2,
        & 1\le p<2, \\[3pt]
        \kappa_{S_p}(r)\vartheta_p(r)r^{2/p}\nu^2,
        & 2\le p\le\infty.
    \end{cases}
\end{align}
At $p=2$, both coefficients are $4r\nu^2$. For every fixed $p>2$,
\begin{align*}
    \mathsf C_*^{\mathrm{Gau}}
    &\le
    8e
    \min\{p, \ln(er)\}^2r^{2/p}\nu^2,
\end{align*}
whereas $\mathsf C_{\opf}^{\mathrm{Gau}}=4r\nu^2$. At $p=\infty$ and $r\ge1$, the displayed dependences are
$\cO((\ln(er))^2\nu^2)$ and $r\nu^2$, respectively.
\end{proposition}

\begin{proof}
For
\begin{align*}
    \zeta
    &=
    \nu\Diag_{m, n}\rb{(\omega_i)_{i=1}^r},
\end{align*}
the singular values are the numbers $\nu|\omega_i|$, $i=1, \ldots, r$, arranged in nonincreasing order.
Since the corresponding vector norms are invariant under permutations, we obtain
\begin{align*}
    \fnorm{\zeta}
    &=
    \nu\norm{\omega}_2,
    \\
    \norm{\zeta}_{S_p}
    &=
    \nu\norm{\omega}_p,
\end{align*}
where $\omega=(\omega_1, \ldots, \omega_r)\sim\mathcal N(0, I_r)$.

Applying \Cref{lem:gaussian_light_tail_scale_alt} with $D=r$ and $p=2$ proves that
$4r\nu^2$ is an admissible Frobenius scale. If $\widetilde\sigma_{\opf}$ is any admissible Frobenius scale,
then Jensen's inequality gives
\begin{align*}
    \exp(1)
    &\ge
    \Exp{\exp\rb{\frac{\fnorm{\zeta}^2}{\widetilde\sigma_{\opf}^2}}}\\
    &\ge
    \exp\rb{
        \frac{\Exp{\fnorm{\zeta}^2}}{\widetilde\sigma_{\opf}^2}
    }.
\end{align*}
Since
\begin{align*}
    \Exp{\fnorm{\zeta}^2}
    &=
    r\nu^2,
\end{align*}
we obtain
\begin{align*}
    \widetilde\sigma_{\opf}^2
    &\ge
    r\nu^2.
\end{align*}

Applying the same lemma with $D=r$ and the stated value of $p$ proves that the native scale in
\eqref{eq:sigma_schatten_gaussian_alt} is admissible. For $1\le p<2$, substitution yields
\begin{align*}
    \mathsf C_{\opf}^{\mathrm{Gau}}
    &=
    r^{\frac2p-1}\cdot 4r\nu^2
    =
    4r^{2/p}\nu^2,
    \\
    \mathsf C_*^{\mathrm{Gau}}
    &=
    r^{\frac2p-1}\cdot 4r^{2/p}\nu^2
    =
    4r^{\frac4p-1}\nu^2.
\end{align*}
For $2\le p\le\infty$, $\alpha_{S_p}(r)^{-2}=1$, and hence
\begin{align*}
    \mathsf C_{\opf}^{\mathrm{Gau}}
    &=
    4r\nu^2,
    \\
    \mathsf C_*^{\mathrm{Gau}}
    &=
    4\kappa_{S_p}(r)\vartheta_p(r)r^{2/p}\nu^2.
\end{align*}
Using \eqref{eq:kappa_schatten_upper_alt} proves the final statements.
\end{proof}

\begin{remark}[Scope of the matrix perturbation models]
\label{rem:diagonal_model_scope_alt}
The two matrix perturbations are diagonal, with all $r$ singular-value coordinates active. They are not matrices with
independent noise in all $mn$ entries. The singular values of an entrywise random matrix have a different distribution,
so neither proposition applies to that model without a new calculation of both light-tail scales.
\end{remark}

\clearpage
\section{Fundamental Properties of Smooth and Regular Spaces}\label{sec:proofs_regularity}
This section collects general results for finite-dimensional smooth and regular normed spaces. We first state a lemma that provides a smooth auxiliary norm uniformly equivalent to a regular norm, and then recall equivalent characterizations of smoothness in terms of the squared norm and its conjugate. The $\ell_p^d$ and Schatten-$p$ examples, with matching upper and lower regularity bounds, are developed in the next section, \Cref{sec:reg_examples}.

The following lemma supplies the smooth auxiliary norm used in our proof of the concentration inequality (\Cref{app:proof_key_concentration}).

\begin{lemma}[{\citet[Lemma~3]{juditsky2008large}}]
\label{lemfact}
Let $(V, \|\cdot\|_V)$ be a finite-dimensional $\kappa_V$-regular space in the sense of Definition~\ref{def:kappa_smooth_regular}.
Then there exists a $\kappa_V$-smooth norm $\|\cdot\|_{V, +}$ on $V$ such that
\begin{align}\label{eq:reg_to_smooth_equiv_general}
    \|u\|_V^2 \le \|u\|_{V, +}^2 \le 2\|u\|_V^2, \quad \forall u\in V.
\end{align}
\end{lemma}

Next, we recall standard equivalent characterizations of smoothness in the primal space. For $\phi(u)\eqdef\frac12\|u\|_{V, +}^2$, these properties relate the smoothness of the norm $\|\cdot\|_{V, +}$ to the strong convexity of the Fenchel conjugate $\phi^*$ and the Lipschitz continuity of the gradient $\nabla\phi$.

\begin{proposition}[{\citet[Proposition~3.2]{juditsky2008large}}]
\label{prop:kappa_smooth_characterizations}
Let $(V, \|\cdot\|_{V, +})$ be a finite-dimensional normed space with dual space $V^*$ and induced dual norm $\|\cdot\|_{V, +, *}$.
Define $\phi(u)\eqdef \frac12\|u\|_{V, +}^2$ and $\phi^*(s)\eqdef \frac12\|s\|_{V, +, *}^2$.
Then the following six properties are equivalent:
\begin{enumerate}[label=(\roman*)]
    \item The space $(V, \|\cdot\|_{V, +})$ is $\kappa_{V, +}$-smooth.

    \item The function $\phi$ is everywhere differentiable, and its gradient satisfies
    \begin{align*}
        \langle \nabla\phi(u) - \nabla\phi(v), \, u-v\rangle \le \kappa_{V, +}\|u-v\|_{V, +}^2, \quad \forall u, v\in V.
    \end{align*}

    \item The function $\phi$ is continuously differentiable, and the mapping $\nabla\phi:V\to V^*$ is $\kappa_{V, +}$-Lipschitz continuous:
    \begin{align*}
        \|\nabla\phi(u)-\nabla\phi(v)\|_{V, +, *} \le \kappa_{V, +}\|u-v\|_{V, +}, \quad \forall u, v\in V.
    \end{align*}

    \item The set-valued subdifferential operator $\partial\phi^*:V^*\rightrightarrows V$ is strongly monotone with constant $\nfr{1}{\kappa_{V, +}}$:
    \begin{align*}
        \langle s_1-s_2, \, u_1-u_2\rangle \ge \frac{1}{\kappa_{V, +}}\|s_1-s_2\|_{V, +, *}^2,
    \end{align*}
    for all $s_1, s_2\in V^*$ and all $u_1\in\partial\phi^*(s_1)$, $u_2\in\partial\phi^*(s_2)$.

    \item The inverse mapping $(\partial\phi^*)^{-1}=\partial\phi$ is single-valued and $\kappa_{V, +}$-Lipschitz continuous, which can be written as
    \begin{align*}
        \|u_1-u_2\|_{V, +} \ge \frac{1}{\kappa_{V, +}}\|s_1-s_2\|_{V, +, *},
    \end{align*}
    for all $s_1, s_2\in V^*$ and all $u_1\in\partial\phi^*(s_1)$, $u_2\in\partial\phi^*(s_2)$.

    \item The function $\phi^*$ is strongly convex with constant $\nfr{1}{\kappa_{V, +}}$ with respect to $\|\cdot\|_{V, +, *}$:
    \begin{align*}
        \phi^*(s_2) \ge \phi^*(s_1) + \langle u_1, \, s_2-s_1\rangle + \frac{1}{2\kappa_{V, +}}\|s_2-s_1\|_{V, +, *}^2,
    \end{align*}
    for all $s_1, s_2\in V^*$ and all $u_1\in\partial\phi^*(s_1)$.
\end{enumerate}
\end{proposition}

\clearpage
\section{Examples of \texorpdfstring{$\kappa$}{kappa}-Regular Spaces and Sharp Bounds on \texorpdfstring{$\kappa$}{kappa}}\label{sec:reg_examples}

This section establishes the universal regularity bound and gives explicit upper and matching lower bounds for the norm families used in the main text. For $\ell_p^d$ and Schatten-$p$ spaces, write $D=d$ and $D=r=\min\{m, n\}$, respectively. The optimal parameter equals $D^{2/p-1}$ for $1\le p<2$ and has order $\min\{p, \ln(eD)\}$ for $2\le p\le\infty$. Thus the latter regime admits a bound independent of $D$ for each fixed finite $p$, with logarithmic growth in the branch $p\ge\ln(eD)$, including $p=\infty$.

\Cref{tab:regularity_examples} summarizes the bounds and auxiliary norms using the main-text convention $V=E^*$. We begin with the universal bound via John's ellipsoid, followed by the $\ell_p^d$ and Schatten-$p$ constructions and their matching lower bounds.

\begin{table}[H]
\centering
\begin{threeparttable}
\caption{Examples of $\kappa$-regular spaces for $V=E^*$ with corresponding regularity parameters $\kappa$ and auxiliary norm $\spnorm{\cdot}$.}
\label{tab:regularity_examples}
\small
\setlength{\tabcolsep}{4pt}
\begin{tabular}{@{}p{0.28\textwidth}p{0.185\textwidth}>{\raggedright\arraybackslash}p{0.245\textwidth}p{0.105\textwidth}p{0.105\textwidth}@{}}
\toprule
Space $(E^*, \snorm{\cdot})$
& $\kappa$ scaling\tnote{b}
& Auxiliary norm $\|\cdot\|_{*, +}$
& Upper bound
& Lower bound \\
\midrule

\multirow{3}{=}{Any $d$-dimensional normed space, $d\ge1$\tnote{a}}
& \multirow{3}{*}{$\Theta(d)$}
& Inner-product norm whose unit ball is the maximum-volume ellipsoid contained in $\mathbb{B}_*=\{s\in E^*:\snorm{s}\le1\}$ (the John ellipsoid).
& \multirow{3}{*}{Prop.~\ref{example:1}}
& \multirow{3}{*}{Prop.~\ref{prop:tightness_universal_bound}} \\

\midrule

$(\R^d, \|\cdot\|_p)$, $1\le p<2$
& $\Theta\!\left(d^{\frac{2}{p}-1}\right)$
& $d^{\frac{1}{p}-\frac{1}{2}}\|\cdot\|_2$ (scaled Euclidean norm)
& \multirow{3}{*}{Prop.~\ref{example:2}}
& \multirow{3}{*}{Prop.~\ref{prop:tightness_lp_bounds}} \\

\multirow{2}{*}{$(\R^d, \|\cdot\|_p)$, $2\le p\le \infty$}
& \multirow{2}{*}{$\Theta\!\bigl(\min\{p, \ln(ed)\}\bigr)$}
& $\|\cdot\|_{\rho_*}$, where $\rho_*=2$ for $1\le d<8$, and $\rho_*=\min\{p, \hat\rho(d)\}$ for $d\ge8$.\tnote{c}
&
& \\

\midrule

$(\R^{m\times n}, \|\cdot\|_{S_p})$, $1\le p<2$, $r\eqdef \min\{m, n\}$
& $\Theta\!\left(r^{\frac{2}{p}-1}\right)$
& $r^{\frac{1}{p}-\frac{1}{2}}\|\cdot\|_{S_2}$
& \multirow{4}{*}{Prop.~\ref{example:3}}
& \multirow{4}{*}{Prop.~\ref{prop:tightness_schatten_bounds}} \\

\multirow{3}{=}{$(\R^{m\times n}, \|\cdot\|_{S_p})$, $2\le p\le \infty$, $r\eqdef \min\{m, n\}$}
& \multirow{3}{*}{$\Theta\!\bigl(\min\{p, \ln(er)\}\bigr)$}
& $\|\cdot\|_{S_{\rho_*}}$, where $\rho_*=\min\{p, 3\}$ for $1\le r<8$, and $\rho_*=\min\{p, \max\{3, \hat\rho(r)\}\}$ for $r\ge8$.\tnote{d}
&
& \\

\bottomrule
\end{tabular}
\begin{tablenotes}[flushleft]
\footnotesize
\item[a] The universal bound applies to every nonzero finite-dimensional normed space, including the dual space $(E^*, \snorm{\cdot})$, and is sharp in general. The structured norms below can have substantially smaller $\kappa$, although $\ell_1^d$ attains the universal bound $\kappa=d$.
\item[b] For $p\ge2$, the upper and lower bounds agree up to absolute constants, absorbed by $\Theta(\cdot)$. Throughout the manuscript, order estimates use $\ln(eD)=1+\ln D$, with $D=d$ or $r$, to cover all positive dimensions (\Cref{rem:uniform_regularity_orders}). Exact formulas and explicit-constant inequalities retain $\ln D$ under their stated dimension conditions.
\item[c] Here
$\hat\rho(d)=\ln d+\sqrt{(\ln d)^2-2\ln d}$ for $d\ge8$.
This choice minimizes $(\rho-1)d^{2/\rho-2/p}$ over
$2\le\rho\le p$, $\rho<\infty$; in particular, $\rho_*$ remains finite when $p=\infty$.
\item[d] Here $\hat\rho(r)=\ln r+\sqrt{(\ln r)^2-2\ln r}$ for $r\ge8$.
This choice minimizes $\max\{2, \rho-1\}r^{2/\rho-2/p}$ over $2\le\rho\le p$, $\rho<\infty$, and remains finite when $p=\infty$.
For $1\le r<8$, this bound is nonincreasing on $[2, 3]$ and increasing on $[3, \infty)$, giving $\rho_*=\min\{p, 3\}$.
\end{tablenotes}
\end{threeparttable}
\end{table}

\subsection{Universal Bound via John's Ellipsoid}\label{sub:proof_of_the_example_1}

The universal regularity bound follows from the symmetric form of John's ellipsoid theorem~\citep{john1948extremum}. We give a proof based on a convex optimization problem for the ellipsoid's shape matrix.

\begin{theorem}[John's Ellipsoid Theorem for Symmetric Bodies]
\label{thm:john_ellipsoid}
Let $V$ be a $d$-dimensional real vector space and let $B\subset V$ be a compact, convex body with non-empty interior containing the origin.
Assume $B$ is centrally symmetric, i.e., if $x\in B$ then $-x\in B$.
Then:
\begin{enumerate}[label=(\roman*)]
    \item There exists a unique ellipsoid $\cE_{\mathrm{in}}$ of maximal volume contained in $B$ (the John ellipsoid).
    \item The body $B$ is \enquote{sandwiched} by this ellipsoid and its scaled-up version:
    \begin{align*}
        \cE_{\mathrm{in}} \subseteq B \subseteq \sqrt{d}\,\cE_{\mathrm{in}}.
    \end{align*}
\end{enumerate}
\end{theorem}

\begin{proof}
The case $d=0$ is immediate, so assume $d\ge1$. Fix an identification $V\simeq\R^d$, and let $\mathcal B_2$ denote the Euclidean unit ball. Volume ratios and the claimed inclusions are invariant under invertible linear changes of coordinates. Symmetry and convexity imply $0\in\operatorname{int}B$, so there exist $\varepsilon, R>0$ such that
\begin{align*}
    \varepsilon\mathcal B_2\subseteq B\subseteq R\mathcal B_2.
\end{align*}

\paragraph{Centering and the optimal shape matrix.}
Every ellipsoid of positive volume can be written uniquely as $c+Q\mathcal B_2$, where $c\in\R^d$ and $Q\in S_{++}^d$ is symmetric positive definite. If $c+Q\mathcal B_2\subseteq B$, symmetry gives $-c+Q\mathcal B_2\subseteq B$, and convexity then yields
\begin{align*}
    Q\mathcal B_2
    =\tfrac12\bigl[(c+Q\mathcal B_2)+(-c+Q\mathcal B_2)\bigr]
    \subseteq B.
\end{align*}
Centering therefore preserves feasibility and volume, without any appeal to uniqueness.

Consider the set
\begin{align*}
    \mathcal Q\eqdef\{Q\in S_+^d:Q\mathcal B_2\subseteq B\},
\end{align*}
where $S_+^d$ is the cone of symmetric positive semidefinite matrices. This set is closed and convex. Indeed, for $Q_1, Q_2\in\mathcal Q$, $t\in[0, 1]$, and $z\in\mathcal B_2$, convexity of $B$ gives $((1-t)Q_1+tQ_2)z\in B$. Moreover, $\mathcal Q$ is bounded because every $Q\in\mathcal Q$ satisfies $\|Q\|_{\mathrm{op}}\le R$. Thus $\mathcal Q$ is compact.
Since $\varepsilon I_d\in\mathcal Q$, the maximum of $\det Q$ over $\mathcal Q$ is positive and is attained at a positive definite matrix $Q^*$. The volume of $Q\mathcal B_2$ is $\det Q$ times the volume of $\mathcal B_2$, so $Q^*\mathcal B_2$ is a maximal-volume ellipsoid among all feasible ellipsoids by the centering argument.

The centered maximizer is unique: maximizing $\det Q$ on positive definite matrices is equivalent to maximizing the strictly concave function $\log\det Q$ over the convex set $\mathcal Q\cap S_{++}^d$. Strict concavity follows from
\begin{align*}
    D^2(\log\det Q)[H, H]
    =-\tr(Q^{-1}HQ^{-1}H)<0
\end{align*}
for every nonzero symmetric matrix $H$. We prove below that an optimal ellipsoid cannot have a nonzero center.

\paragraph{A finite contact decomposition.}
Normalize by the linear map $T\eqdef(Q^*)^{-1}$ and write $\widetilde B\eqdef T(B)$. The body $\widetilde B$ is centrally symmetric, contains $\mathcal B_2$, and has $\mathcal B_2$ as its unique centered maximal-volume ellipsoid.
Define its polar body by
\begin{align*}
    \widetilde B^\circ
    \eqdef\{y\in\R^d:\langle y, x\rangle\le1\text{ for all }x\in\widetilde B\}.
\end{align*}
It is compact and contained in $\mathcal B_2$. By the separating-hyperplane characterization of a closed convex body, for a symmetric positive definite matrix $Q$,
\begin{align*}
    Q\mathcal B_2\subseteq\widetilde B
    \quad\Longleftrightarrow\quad
    \|Qy\|_2\le1\quad\text{for all }y\in\widetilde B^\circ.
\end{align*}
Here we used $\sup_{z\in\mathcal B_2}\langle y, Qz\rangle=\|Qy\|_2$.
Consequently, $I_d$ solves the finite-dimensional convex problem
\begin{align*}
    \min_{Q\in S_{++}^d}-\log\det Q
    \quad\text{subject to}\quad
    g(Q)\eqdef\max_{y\in\widetilde B^\circ}\|Qy\|_2^2-1\le0.
\end{align*}
The function $g$ is finite, continuous, and convex, since it is the maximum of continuous convex quadratic functions indexed by a compact set. Slater's condition holds: $Q=\tfrac12 I_d$ is positive definite and satisfies $g(Q)\le-\tfrac34$.
The convex KKT theorem therefore supplies a multiplier $\mu\ge0$ and a subgradient $G\in\partial g(I_d)$ such that the Lagrangian
\begin{align*}
    \mathcal L(Q, \mu)=-\log\det Q+\mu g(Q)
\end{align*}
satisfies stationarity and complementary slackness:
\begin{align*}
    -I_d+\mu G=0,
    \qquad \mu g(I_d)=0.
\end{align*}
In particular, $\mu>0$ and $g(I_d)=0$.

Equip the space of symmetric matrices with the pairing $\langle H, K\rangle=\tr(HK)$. The subdifferential rule for a maximum of differentiable convex functions over a compact index set gives
\begin{align*}
    \partial g(I_d)
    =\operatorname{conv}\{2yy^\top:y\in\widetilde B^\circ, \ \|y\|_2=1\},
\end{align*}
because the gradient of $Q\mapsto\|Qy\|_2^2$ at $I_d$ is $2yy^\top$. The maximum is attained, and these gradients depend continuously on $y$, so the compact-index maximum rule applies. By Carath\'eodory's theorem in the finite-dimensional space of symmetric matrices, $G$ is a finite convex combination of these active gradients. Absorbing the factor $2\mu$ into the coefficients and discarding zero coefficients gives contact points $y_1, \dots, y_k$ and weights $\lambda_1, \dots, \lambda_k>0$ satisfying
\begin{align}
    \label{eq:kkt_identity_john_detailed}
    I_d=\sum_{i=1}^k\lambda_i y_i y_i^\top.
\end{align}
Each $y_i$ belongs to $\mathcal B_2\subseteq\widetilde B$ and to $\widetilde B^\circ$, and $\langle y_i, y_i\rangle=1$. Thus
\begin{align}
    \label{eq:contact_condition_john_detailed}
    y_i\in\partial\widetilde B\cap\partial\mathcal B_2,
    \qquad \|y_i\|_2=1.
\end{align}
Taking traces in \eqref{eq:kkt_identity_john_detailed} yields
\begin{align}
    \label{eq:sum_of_multipliers_john_detailed}
    \sum_{i=1}^k\lambda_i=d.
\end{align}

At each contact point $y_i$, a supporting hyperplane of $\widetilde B$ exists and also supports $\mathcal B_2$, because $\mathcal B_2\subseteq\widetilde B$. The Euclidean ball has the unique supporting hyperplane $\{x:\langle x, y_i\rangle=1\}$ at $y_i$. Hence
\begin{align}
    \label{eq:supporting_hyperplane_property_john}
    \langle x, y_i\rangle\le1,
    \qquad x\in\widetilde B, \quad i=1, \dots, k.
\end{align}
By central symmetry, the same inequality applied to $-x$ gives $|\langle x, y_i\rangle|\le1$.

\paragraph{Uniqueness among all ellipsoids.}
Let $c+Q\mathcal B_2\subseteq B$ be any maximal-volume ellipsoid. Centering it preserves its volume, so uniqueness of the centered maximizer gives $Q=Q^*$. After applying $T$, this ellipsoid is $a+\mathcal B_2\subseteq\widetilde B$, where $a=Tc$.
For each $i$, the two points $a+y_i$ and $a-y_i$ belong to $\widetilde B$. Applying the two-sided supporting inequality to them gives $\langle a, y_i\rangle\le0$ and $\langle a, y_i\rangle\ge0$, respectively. Thus $\langle a, y_i\rangle=0$ for every $i$, and \eqref{eq:kkt_identity_john_detailed} implies
\begin{align*}
    \|a\|_2^2=\sum_{i=1}^k\lambda_i\langle a, y_i\rangle^2=0.
\end{align*}
Therefore $c=0$, proving uniqueness among all ellipsoids and identifying $\cE_{\mathrm{in}}=Q^*\mathcal B_2$.

\paragraph{The sandwich inequality.}
For every $x\in\widetilde B$, the contact decomposition and the supporting inequalities give
\begin{align*}
    \|x\|_2^2
    =\sum_{i=1}^k\lambda_i\langle x, y_i\rangle^2
    \le\sum_{i=1}^k\lambda_i
    =d.
\end{align*}
Hence $\mathcal B_2\subseteq\widetilde B\subseteq\sqrt d\,\mathcal B_2$. Applying $T^{-1}$ yields
$\cE_{\mathrm{in}}\subseteq B\subseteq\sqrt d\,\cE_{\mathrm{in}}$.
\end{proof}

We now proceed with a dimension-dependent regularity guarantee that holds for \emph{every} $d$-dimensional normed space. It will serve as a baseline for comparison with the sharper, structure-dependent bounds stated next.


\begin{proposition}[Universal Regularity Bound]\label{example:1}
Let $d\ge1$. Any $d$-dimensional normed space $(V, \|\cdot\|_V)$ is $d$-regular in the sense of Definition~\ref{def:kappa_smooth_regular}, i.e., it is $\kappa_V$-regular with the choice $\kappa_V=d$.
\end{proposition}
\begin{proof}
To simplify notation throughout this proof, we omit the subscript $V$ and write $\| \cdot \| \eqdef \|\cdot\|_V$, $\|\cdot\|_+ \eqdef \|\cdot\|_{V, +}$, $\kappa_{+} \eqdef \kappa_{V, +}$ and $\kappa \eqdef \kappa_{V}$.
The proof proceeds by explicitly constructing the required smooth norm and verifying that it satisfies the conditions of Definition~\ref{def:kappa_smooth_regular} with $\kappa=d$. The construction relies on the geometric sandwich inequality provided by John's Ellipsoid Theorem.

\textbf{Step 1: Applying John's Theorem to the Unit Ball.}
Let $(V, \|\cdot\|)$ be a $d$-dimensional normed space. Its unit ball, $B=\{x\in V:\|x\|\le 1\}$, is a centrally symmetric, compact, convex body containing the origin. By John's Ellipsoid Theorem (\Cref{thm:john_ellipsoid}), there exists a unique ellipsoid of maximal volume $\cE_{\mathrm{in}}$ contained in $B$, such that
\begin{align*}
    \cE_{\mathrm{in}} \subseteq B \subseteq \sqrt{d}\cE_{\mathrm{in}}.
\end{align*}

\textbf{Step 2: From Geometric Inclusion to an Algebraic Norm Inequality.}
Any ellipsoid centered at the origin can be represented as the unit ball of a Euclidean norm. Let us define the Euclidean norm $\|\cdot\|_{+}$ on $V$ as the one whose unit ball is precisely the John ellipsoid $\cE_{\mathrm{in}}$, i.e., $\cE_{\mathrm{in}}=\{x\in V:\|x\|_{+}\le 1\}$.

We can now translate the geometric inclusions into algebraic inequalities between the norms $\|\cdot\|$ and $\|\cdot\|_{+}$:
\begin{itemize}
    \item The inclusion $\cE_{\mathrm{in}} \subseteq B$ means that if $\|x\|_{+}\le 1$, then $\|x\|\le 1$. By homogeneity of norms, this implies $\|x\| \le \|x\|_{+}$ for all $x\in V$.
    \item The inclusion $B \subseteq \sqrt{d}\cE_{\mathrm{in}}$ means that if $\|x\|\le 1$, then $\|x\|_{+}\le \sqrt{d}$. By homogeneity, this implies $\|x\|_{+}\le \sqrt{d}\|x\|$ for all $x\in V$.
\end{itemize}
Combining these two results, we have the norm equivalence inequality
$\|x\| \le \|x\|_{+} \le \sqrt{d}\|x\|$ for all $x\in V$.
Squaring this expression gives the form required by the definition of regularity:
\begin{align}
    \label{eq:john_norm_inequality_squared}
    \|x\|^2 \le \|x\|_{+}^2 \le d\|x\|^2, \quad \forall x\in V.
\end{align}

\textbf{Step 3: Verifying the Conditions for $d$-Regularity.}
To show that $(V, \|\cdot\|)$ is $d$-regular, we must find a $\kappa_{+}$-smooth auxiliary norm $\|\cdot\|_{+}$ and a constant $\kappa\ge \kappa_{+}$ such that
\begin{align*}
\|x\|^2 \le \|x\|_{+}^2 \le \frac{\kappa}{\kappa_{+}}\|x\|^2.
\end{align*}

Let us choose the Euclidean norm $\|\cdot\|_{+}$ from Step 2. As a Euclidean norm, the space $(V, \|\cdot\|_{+})$ is $1$-smooth. To see this, note that for $\phi_+(x)\eqdef \frac12\|x\|_{+}^2$, there exists an inner product $\langle\cdot, \cdot\rangle_{+}$ with $\|x\|_{+}^2=\langle x, x\rangle_{+}$, and the descent inequality holds with equality for $\kappa_{+}=1$:
\begin{align*}
    \phi_+(x+y)
    &= \frac12\|x+y\|_{+}^2
     = \frac12\|x\|_{+}^2 + \langle x, y\rangle_{+} + \frac12\|y\|_{+}^2 \\
    &= \phi_+(x) + \langle \nabla\phi_+(x), y\rangle + \frac12\|y\|_{+}^2,
\end{align*}
where $\nabla\phi_+(x)\in V^*$ is the linear functional $y\mapsto \langle x, y\rangle_{+}$. Therefore, we can set $\kappa_{+}=1$.

Now we check the compatibility inequality from Definition~\ref{def:kappa_smooth_regular}:
\begin{align*}
    \|x\|_{+}^2 \le \frac{\kappa}{\kappa_{+}}\|x\|^2.
\end{align*}
Substituting $\kappa_{+}=1$ and comparing with \eqref{eq:john_norm_inequality_squared}, we see a direct match if we set $\kappa=d$:
\begin{align*}
    \|x\|_{+}^2 \le d\|x\|^2.
\end{align*}
The constructed norm $\|\cdot\|_{+}$ is $1$-smooth (so $\kappa_{+}=1$) and satisfies the required inequality for $\kappa=d$. Therefore, by Definition~\ref{def:kappa_smooth_regular}, the space $(V, \|\cdot\|_V)$ is $d$-regular.
\end{proof}

While Proposition~\ref{example:1} establishes that $\kappa_V=d$ is sufficient for any finite-dimensional space, one might ask whether a sub-linear dependence is achievable universally. The following proposition answers this in the negative by demonstrating that for the $\ell_1$ norm, the constant of regularity must scale linearly with the dimension. Thus, the universal bound is tight up to absolute constants.

\begin{proposition}[Tightness of the universal $d$-regularity bound]
\label{prop:tightness_universal_bound}
Fix $d\in \N$ and consider the normed space $(V, \|\cdot\|_V)$ with $V\eqdef \R^d$ and $\|\cdot\|_V\eqdef \|\cdot\|_1$.
If $(V, \|\cdot\|_V)$ is $\kappa_V$-regular in the sense of Definition~\ref{def:kappa_smooth_regular}, then necessarily
\begin{align}\label{eq:universal_lower_bound_kappa_ge_d}
    \kappa_V \ge d.
\end{align}
In particular, the universal upper bound of Proposition~\ref{example:1} is tight (up to absolute constants): there exist $d$-dimensional normed spaces whose constant of regularity is of order $d$.
\end{proposition}

\begin{proof}

Let $\{\eps_i\}_{i=1}^d$ be i.i.d.\ Rademacher random variables, i.e.\ $\Prob\{\eps_i=1\}=\Prob\{\eps_i=-1\}=\tfrac12$, and let $\{e_i\}_{i=1}^d$ be the standard basis of $\R^d$.
Define the filtration $\cF_i\eqdef\sigma(\eps_1, \dots, \eps_i)$ and the sequence
\begin{align*}
    \xi_i \eqdef \eps_i e_i, \qquad i=1, \dots, d.
\end{align*}
Then $\{\xi_i\}_{i=1}^d$ is an $\{\cF_i\}$-martingale-difference sequence: it is $\cF_i$-measurable and
\begin{align*}
    \Exp{\xi_i\mid \cF_{i-1}}
    = e_i \Exp{\eps_i\mid \cF_{i-1}}
    = e_i \Exp{\eps_i}
    = 0
    \quad \text{a.s.}
\end{align*}
Moreover, $\|\xi_i\|_1=\|e_i\|_1=1$, hence $\Exp{\|\xi_i\|_1^2}=1$ and therefore
\begin{align}\label{eq:sum_xi_i_l1}
    \sum_{i=1}^d \Exp{\|\xi_i\|_1^2} = d.
\end{align}

Let $S_d\eqdef \sum_{i=1}^d \xi_i = (\eps_1, \dots, \eps_d)$. Since $|\eps_i|=1$ a.s., we have deterministically
\begin{align}\label{eq:Sd_l1_deterministic}
    \|S_d\|_1 = \sum_{i=1}^d |\eps_i| = d,
\end{align}
and consequently
\begin{align}\label{eq:ESd_l1_sq}
    \Exp{\|S_d\|_1^2} = d^2.
\end{align}

On the other hand, since $(\R^d, \|\cdot\|_1)$ is assumed $\kappa_V$-regular and $\{\xi_i\}$ is a martingale-difference sequence, the second-moment inequality of \Cref{lem:second_moment_kappa_regular} yields
\begin{align}\label{eq:second_moment_bound_l1}
    \Exp{\|S_d\|_1^2} \le \kappa_V \sum_{i=1}^d \Exp{\|\xi_i\|_1^2}.
\end{align}
Combining \eqref{eq:second_moment_bound_l1} with \eqref{eq:sum_xi_i_l1} and \eqref{eq:ESd_l1_sq} gives
\begin{align*}
    d^2 \le \kappa_V d,
\end{align*}
which implies $\kappa_V\ge d$. This is exactly \eqref{eq:universal_lower_bound_kappa_ge_d}.
\end{proof}

\subsection{\texorpdfstring{$\ell_p^d$ Spaces}{ell-p Spaces in Dimension d}}\label{sub:proof_of_the_example_2}

We now determine essentially tight regularity constants for $\ell_p^d$ normed spaces $(V, \|\cdot\|_V)$ with $V=\R^d$ and $\|\cdot\|_V=\|\cdot\|_p$, distinguishing between the regime $p\ge 2$ (where $\kappa_V$ can be as small as logarithmic in $d$) and the regime $p<2$ (where $\kappa_V$ necessarily grows polynomially in $d$).



\begin{proposition}[$\ell_p^d$ spaces]\label{example:2}
Fix $d\in\N$ and $p\in[1, \infty]$. Let $V\eqdef \R^d$ and equip it with the norm $\|\cdot\|_V\eqdef \|\cdot\|_p$, where
\begin{align*}
    \|x\|_p \eqdef \Big(\sum_{i=1}^d |x_i|^p\Big)^{1/p} \quad (1\le p<\infty),
    \qquad
    \|x\|_\infty \eqdef \max_{1\le i\le d}|x_i|.
\end{align*}
Then the space $(V, \|\cdot\|_V)$ is $\kappa_V$-regular in the sense of Definition~\ref{def:kappa_smooth_regular}, with admissible regularity parameter $\kappa_V\eqdef \kappa_p(d)$ given as follows:
\begin{enumerate}
    \item Let $2\le p\le \infty$. Then
    \begin{align*}
        \kappa_V \eqdef \kappa_p(d)
        = \min_{\substack{2\le \rho\le p\\ \rho<\infty}} (\rho-1)\,d^{\frac{2}{\rho} - \frac{2}{p}}.
    \end{align*}
    When $p=\infty$, we use $1/p=0$ and the auxiliary exponent ranges over $[2, \infty)$.
    Furthermore, if $d\ge 8$, then
    \begin{align*}
        \kappa_V \eqdef \kappa_p(d) \le \min\{p-1, \; 2e^2\ln(d)\},
    \end{align*}
    with the corresponding auxiliary norm
    \begin{align*}
        \|\cdot\|_{V, +} \eqdef \|\cdot\|_{\rho_*},
        \qquad
        \rho_* \eqdef \min\!\left\{p, \; \ln d + \sqrt{(\ln d)^2-2\ln d}\right\}.
    \end{align*}

    \item Let $1\le p<2$. Then
    \begin{align*}
        \kappa_V \eqdef \kappa_p(d) = d^{\frac{2}{p}-1},
    \end{align*}
    with the corresponding auxiliary norm
    \begin{align*}
        \|x\|_{V, +} \eqdef d^{\frac{1}{p}-\frac{1}{2}}\|x\|_2,
        \qquad x\in\R^d.
    \end{align*}
\end{enumerate}
\end{proposition}

\begin{proof}
To simplify notation throughout this proof, we write
\begin{align*}
    \|x\| \eqdef \|x\|_V = \|x\|_p,
    \qquad
    \kappa \eqdef \kappa_V.
\end{align*}
Whenever an auxiliary norm is introduced, we denote it by
\begin{align*}
    \|x\|_+ \eqdef \|x\|_{V, +},
    \qquad
    \kappa_+ \eqdef \kappa_{V, +}.
\end{align*}
The letters $x, h\in V=\R^d$ denote generic vectors; this is only a local renaming of the dummy variables in Definition~\ref{def:kappa_smooth_regular}.

\begin{enumerate}
\item
We start with Case 1: $2 \le p \le +\infty$.

The proof proceeds in three steps. First, we establish that for any $2\le\rho<\infty$, the space $(\R^d, \norm{\cdot}_\rho)$ is $(\rho-1)$-smooth. Second, we use this result with a standard norm equivalence inequality to derive the regularity constant for $(\R^d, \norm{\cdot}_p)$.
Lastly, we provide a bound depending only on $p$ and $d$.

\paragraph{Step 1: Smoothness of $(\R^d, \|\cdot\|_\rho)$ for $2\le\rho<\infty$.}
Fix $\rho\in[2, p]$ with $\rho<\infty$ and use $\|x\|_+ \eqdef \|x\|_\rho$ as the candidate auxiliary norm on $V$.
Define
\begin{align*}
    \phi_+(x) \eqdef \frac12\|x\|_+^2 = \frac12\|x\|_\rho^2,
    \qquad
    \psi(x) \eqdef 2\phi_+(x)=\|x\|_\rho^2.
\end{align*}
To prove that $(V, \|\cdot\|_+)$ is $\kappa_+$-smooth with $\kappa_+=\rho-1$, it suffices to show that
\begin{align*}
    \inner{\nabla^2\psi(x)h}{h} \le 2(\rho-1)\|h\|_+^2 = 2(\rho-1)\|h\|_\rho^2
\end{align*}
for all $x\neq 0$ and all $h\in V$.
The case $\rho=2$ is Euclidean, hence $1$-smooth, so we assume $\rho>2$.

Since $\rho>2$, the function $\psi$ is twice continuously differentiable on $\R^d\setminus\{0\}$. For $x\neq0$, its gradient is
\begin{align*}
    \nabla \psi(x) = 2\left(\sum_i |x_i|^\rho\right)^{\frac{2}{\rho}-1} \cdot \left(|x_1|^{\rho-1}\sign(x_1), \dots, |x_d|^{\rho-1}\sign(x_d)\right)^\top.
\end{align*}
At the origin,
\begin{align*}
    \frac{|\psi(h)-\psi(0)|}{\|h\|_\rho}
    =\|h\|_\rho\longrightarrow0
    \qquad\text{as }h\to0, \ h\neq0,
\end{align*}
so $\psi$ is differentiable at $0$ with $\nabla\psi(0)=0$. Moreover, the gradient formula gives
\begin{align*}
    \|\nabla\psi(x)\|_{\rho/(\rho-1)}
    &=2\left(\sum_i|x_i|^\rho\right)^{\frac{2}{\rho}-1}
       \left(\sum_i|x_i|^\rho\right)^{\frac{\rho-1}{\rho}} \\
    &=2\|x\|_\rho\longrightarrow0
    \qquad\text{as }x\to0, \ x\neq0.
\end{align*}
Thus the gradient is also continuous at the origin, and $\psi\in C^1(\R^d)$.

Since $\psi(x)$ is homogeneous of degree $2$, it suffices to verify the bound for vectors on the unit sphere, i.e., where $\norm{x}_\rho=1$. For such an $x$, the Hessian action $\inner{\nabla^2 \psi(x)h}{h}$ is:
\begin{align*}
    \inner{\nabla^2 \psi(x)h}{h} &= 2\underbrace{(2-\rho)}_{\le 0} \left(\sum_i |x_i|^\rho\right)^{\frac{2}{\rho}-2} \left(\sum_i |x_i|^{\rho-1}\sign(x_i)h_i\right)^2 \\
    &\quad + 2\underbrace{\left(\sum_i |x_i|^\rho\right)}_{=1}^{\frac{2}{\rho}-1}\sum_i (\rho-1)|x_i|^{\rho-2}h_i^2.
\end{align*}
Since $\rho > 2$, the first term is non-positive. Dropping it and substituting $\sum_i |x_i|^\rho = 1$ gives the upper bound:
\begin{align*}
    \inner{\nabla^2 \psi(x)h}{h} &\le 2(\rho-1)\sum_{i=1}^d |x_i|^{\rho-2}h_i^2.
\end{align*}
We now apply Hölder's inequality to the sum on the right-hand side. Recall that for any two vectors $a, b \in \R^d$ and conjugate exponents $q, r > 1$ (i.e., $\frac{1}{q} + \frac{1}{r} = 1$), the inequality states:
$\sum_{i=1}^d |a_i b_i| \le \left( \sum_{i=1}^d |a_i|^q \right)^{1/q} \left( \sum_{i=1}^d |b_i|^r \right)^{1/r}$.
In our case, we set $a_i = |x_i|^{\rho-2}$ and $b_i = h_i^2$. Since both terms are non-negative, the absolute value on the left-hand side is not needed. The conjugate exponents are $q = \frac{\rho}{\rho-2}$ and $r = \frac{\rho}{2}$. Applying the inequality yields:
\begin{align*}
    \sum_{i=1}^d |x_i|^{\rho-2}h_i^2 &\le \left( \sum_{i=1}^d \rb{|x_i|^{\rho-2}}^q \right)^{1/q} \left( \sum_{i=1}^d \rb{h_i^2}^r \right)^{1/r} \\
    &= \left( \sum_{i=1}^d |x_i|^{\rho} \right)^{(\rho-2)/\rho} \left( \sum_{i=1}^d |h_i|^{\rho} \right)^{2/\rho} \\
    &= \rb{\norm{x}_\rho^\rho}^{(\rho-2)/\rho} \rb{\norm{h}_\rho^\rho}^{2/\rho} = \norm{x}_\rho^{\rho-2} \sqnorm{h}_\rho.
\end{align*}
Since we assumed $\norm{x}_\rho=1$, this gives $\sum_{i=1}^d |x_i|^{\rho-2}h_i^2 \le \sqnorm{h}_\rho$. Substituting back yields:
\begin{align*}
    \inner{\nabla^2 \psi(x)h}{h} \le 2(\rho-1)\sqnorm{h}_\rho = 2(\rho-1)\|h\|_+^2.
\end{align*}
This confirms that the Hessian bound holds for every $x\neq0$. To obtain the global smoothness inequality, fix $x, h\in\R^d$. The case $h=0$ is immediate, so assume $h\neq0$ and define $F(t)\eqdef\psi(x+th)$ for $t\in[0, 1]$. Then $F\in C^1([0, 1])$, and whenever $x+th\neq0$,
\begin{align*}
    F''(t)=\inner{\nabla^2\psi(x+th)h}{h}
    \le 2(\rho-1)\|h\|_\rho^2.
\end{align*}
The segment can pass through the origin at at most one parameter value $t_0$. Integrating this bound on intervals avoiding $t_0$, taking limits using continuity of $F'$, and adding the two inequalities when the interval crosses $t_0$ gives
\begin{align*}
    F'(t)-F'(0)\le2(\rho-1)\|h\|_\rho^2t,
    \qquad t\in[0, 1].
\end{align*}
Integrating once more yields
\begin{align*}
    \psi(x+h)-\psi(x)
    &=\int_0^1 F'(t)\,dt \\
    &\le\inner{\nabla\psi(x)}{h}+(\rho-1)\|h\|_\rho^2.
\end{align*}
Since $\phi_+=\psi/2$, this is equivalent to
\begin{align*}
    \phi_+(x+h)
    \le\phi_+(x)+\inner{\nabla\phi_+(x)}{h}
       +\frac{\rho-1}{2}\|h\|_+^2,
    \qquad x, h\in V.
\end{align*}
Therefore, the auxiliary norm $\|\cdot\|_+ = \|\cdot\|_\rho$ is $\kappa_+$-smooth with
\begin{align*}
    \kappa_+ = \rho-1.
\end{align*}

\paragraph{Step 2: From Smoothness to Regularity.}
To establish the regularity of $(V, \|\cdot\|)$ for $p\in[2, \infty]$, we use $\|\cdot\|_+ = \|\cdot\|_\rho$ with $\rho\in[2, p]$ and $\rho<\infty$ as the auxiliary smooth norm.
The standard equivalence inequality for $\ell_p$-norms states that for any $x\in\R^d$ and $\rho\le p$,
\begin{align*}
    \|x\| = \|x\|_p \le \|x\|_\rho = \|x\|_+ \le d^{\frac{1}{\rho}-\frac{1}{p}}\|x\|_p = d^{\frac{1}{\rho}-\frac{1}{p}}\|x\|.
\end{align*}
Squaring yields
\begin{align*}
    \|x\|^2 \le \|x\|_+^2 \le d^{\frac{2}{\rho}-\frac{2}{p}}\|x\|^2,
\end{align*}
which is exactly the compatibility chain required by Definition~\ref{def:kappa_smooth_regular}.

From Step 1, the auxiliary norm $\|\cdot\|_+$ is $\kappa_+$-smooth with
\begin{align*}
    \kappa_+ = \rho-1.
\end{align*}
Therefore, Definition~\ref{def:kappa_smooth_regular} gives
\begin{align*}
    \kappa
    = \kappa_+\, d^{\frac{2}{\rho}-\frac{2}{p}}
    = (\rho-1)d^{\frac{2}{\rho}-\frac{2}{p}}.
\end{align*}
The tightest regularity constant supplied by this construction is found by minimizing this expression over all $\rho\in[2, p]$ with $\rho<\infty$. This leads to the formula in the example statement.

To find the tightest bound on the regularity constant, we must solve
\begin{align}\label{eq:kappa_p_d}
    \kappa_V \eqdef \kappa_p(d) = \min_{\substack{2\le \rho\le p\\ \rho<\infty}} \kappa(\rho),
    \qquad
    \kappa(\rho) \eqdef (\rho-1)d^{\frac{2}{\rho}-\frac{2}{p}}.
\end{align}
The minimum is attained: for $p<\infty$, the objective is continuous on $[2, p]$; for $p=\infty$, it is continuous on $[2, \infty)$ and tends to $\infty$ as $\rho\to\infty$.
The term $d^{-2/p}$ is a constant with respect to $\rho$, so we can focus on minimizing the function $g(\rho) \eqdef (\rho-1)d^{2/\rho}$.
\paragraph{Step 3: A $\rho$-Independent Bound on the Regularity Constant.}
We compute the first and second derivatives of $g(\rho)$ with respect to $\rho$. 
\begin{align*}
    g'(\rho) &= d^{2/\rho} \left( 1 - \frac{2 (\rho-1)\ln {d}}{\rho^2} \right). \\
    g''(\rho) &= d^{2/\rho} \frac{4\ln {d}}{\rho^4} \left( (\rho-1)\ln {d} - \rho \right).
\end{align*}

$g''(\rho)$ is non-negative if and only if $(\ln d)(\rho-1) - \rho \ge 0$, which is equivalent to the condition $\rho \ge \frac{\ln d}{\ln d - 1}$. We show this holds for any $\rho \ge 2$ when $d \ge 8$.

Since we assume $d \ge 8$, it follows that $d > e^2 \approx 7.39$. This implies $\ln d > 2$.
From $\ln d > 2$, we have $2(\ln d - 1) = 2\ln d - 2 > \ln d$, which rearranges to:
\begin{align*}
    2 > \frac{\ln d}{\ln d - 1}.
\end{align*}
At the same time, $d \ge 8 > e$ implies $\ln d > 1$, so $\ln d - 1 > 0$, which ensures that $\frac{\ln d}{\ln d - 1} > 1$.
Combining these gives the sandwich inequality:
\begin{align*}
    2 > \frac{\ln d}{\ln d - 1} > 1.
\end{align*}
Therefore, for any $\rho \ge 2$, we have:
\begin{align*}
    \rho \ge 2 > \frac{\ln d}{\ln d - 1}.
\end{align*}

This confirms that $g''(\rho) \ge 0$, and thus $g(\rho)$ is convex for all $\rho \ge 2$ under the mild condition $d \ge 8$.

\paragraph{Finding the Optimal $\rho^*$.}
Since we have established that the function $g(\rho)$ is convex on our interval of interest for $d \ge 8$, we can find its global minimizer by setting the first derivative to zero:
\begin{align*}
    g'(\rho) = d^{2/\rho} \left( 1 - \frac{2(\ln d)(\rho-1)}{\rho^2} \right) = 0.
\end{align*}
This requires the term in the parenthesis to be zero, which leads to a quadratic equation for $\rho$:
\begin{align*}
    \rho^2 = 2(\ln d)(\rho-1) \implies \rho^2 - (2\ln d)\rho + 2\ln d = 0.
\end{align*}
The solutions to this quadratic equation are:
\begin{align*}
    \rho^* = \ln d \pm \sqrt{(\ln d)^2 - 2\ln d}.
\end{align*}
For the solutions to be real, the discriminant must be non-negative, which requires $(\ln d)^2 - 2\ln d \ge 0$, or $\ln d \ge 2$. This condition, $d \ge e^2 \approx 7.4$, is consistent with our assumption that $d \ge 8$.

The quadratic equation yields two roots. 

To determine which of these corresponds to a local minimum, we apply the second derivative test. 
Note that $\mathrm{sign}(g''(\rho)) = \mathrm{sign}\left( (\ln d)(\rho-1) - \rho \right)$.

At either critical point $\rho^*\in\{\rho_{-}, \rho_{+}\}$, the first-order condition gives $\rho^{*2}=2(\ln d)(\rho^*-1)$. We can therefore simplify the expression by substituting $(\ln d)(\rho^*-1)=\frac{\rho^{*2}}{2}$.
Plugging this into our sign-determining term:
\begin{align*}
    (\ln d)(\rho^*-1) - \rho^* = \frac{\rho^{*2}}{2} - \rho^* = \frac{\rho^*(\rho^*-2)}{2}.
\end{align*}
Therefore, the sign of the second derivative at a critical point $\rho^*$ is the same as the sign of $\rho^*-2$.

Let us now analyze the location of the two roots relative to the value 2. We are working under the condition $\ln d \ge 2$ (since $d \ge 8 > e^2$).
\begin{itemize}
    \item \textbf{For the smaller root, $\rho_{-}$}: We check if $\rho_{-}$ is greater or less than 2.
    \begin{align*}
        \ln d - 2 &> \sqrt{(\ln d)^2 - 2\ln d} \\
        (\ln d - 2)^2 &> (\ln d)^2 - 2\ln d \quad (\text{squaring is valid as } \ln d > 2) \\
        (\ln d)^2 - 4\ln d + 4 &> (\ln d)^2 - 2\ln d \\
        4 &> 2\ln d \\
        2 &> \ln d.
    \end{align*}
    This is a contradiction to our premise that $\ln d \ge 2$. Therefore, the inequality must be reversed, which means $\rho_{-} \le 2$.

    \item \textbf{For the larger root, $\rho_{+}$}: Since $\ln d \ge 2$, the term under the square root is non-negative. Therefore:
    \begin{align*}
        \rho_{+} = \ln d + \sqrt{(\ln d)^2 - 2\ln d} \ge \ln d \ge 2.
    \end{align*}
\end{itemize}
Now we apply our simplified second derivative test:
\begin{itemize}
    \item At $\rho_{-}$, since $\rho_{-} \le 2$, we have $\rho_{-}-2 \le 0$. This implies $g''(\rho_{-}) \le 0$, so $\rho_{-}$ corresponds to a \textbf{local maximum}.
    \item At $\rho_{+}$, since $\rho_{+} \ge 2$, we have $\rho_{+}-2 \ge 0$. This implies $g''(\rho_{+}) \ge 0$, so $\rho_{+}$ corresponds to a \textbf{local minimum}.
\end{itemize}
Since $g$ is convex on $[2, \infty)$, its minimizer on this interval is the larger root $\rho_{+}$.

\paragraph{Analyzing the Location of the Minimum.}
The optimal regularity bound within this construction depends on where this unconstrained minimizer $\rho^*$ falls relative to the domain $\{\rho:2\le\rho\le p, \ \rho<\infty\}$.
\begin{itemize}
    \item \textbf{If $\rho^* > p$}: This occurs when $p$ is small compared to $\ln d$. Since $g(\rho)$ is convex, its minimum over $[2, p]$ must occur at the boundary. As $g'(\rho) < 0$ for $\rho < \rho^*$, the function is decreasing on $[2, p]$, so the minimum is achieved at the right endpoint, $\rho = p$. In this case, the regularity constant is $\kappa(p) = (p-1)d^{2/p-2/p} = p-1$.

    \item \textbf{If $2 \le \rho^* \le p$}: The minimum is achieved at the unconstrained minimizer $\rho^*$. We now find an upper bound for the value of $\kappa(\rho^*)$.
    \begin{align*}
        \kappa(\rho^*) = (\rho^*-1) d^{\frac{2}{\rho^*}-\frac{2}{p}}.
    \end{align*}
    Since $\rho^*\le p$ and $d\ge1$, we have $d^{-2/\rho^*}\le d^{-2/p}\le1$. Using the upper bound $d^{-2/p}\le1$ gives
    \begin{align*}
        \kappa(\rho^*) \le (\rho^*-1) d^{\frac{2}{\rho^*}}.
    \end{align*}
    From the first-order optimality condition, we have the identity $2(\ln d)(\rho^*-1) = (\rho^*)^2$. This allows us to substitute $(\rho^*-1) = \frac{(\rho^*)^2}{2\ln d}$.
    \begin{align*}
        \kappa(\rho^*) \le \frac{(\rho^*)^2}{2\ln d} \cdot d^{\frac{2}{\rho^*}} = \frac{(\rho^*)^2}{2\ln d} \exp\left(\frac{2\ln d}{\rho^*}\right).
    \end{align*}
    We know that $ \ln d \le \rho^* = \ln d + \sqrt{(\ln d)^2 - 2\ln d} \le 2\ln d$. Substituting these bounds:
    \begin{align*}
        \kappa(\rho^*) \le \frac{(2\ln d)^2}{2\ln d} \exp\left(\frac{2\ln d}{\ln d}\right) = (2\ln d) \cdot e^2.
    \end{align*}
    
\end{itemize}

For $p<\infty$, combining the endpoint bound $p-1$ with the bound $2e^2\ln d$ gives the estimate below. For $p=\infty$, the finite minimizer $\rho^*$ gives the bound $2e^2\ln d$ directly:
\begin{align*}
    \kappa_p(d) = \min_{\substack{2\le \rho\le p\\ \rho<\infty}} \kappa(\rho) \le \min\{p-1, 2 e^2 (\ln d)\}.
\end{align*}
\item We now proceed with Case 2: $1 \le p < 2$.

We prove that $(V, \|\cdot\|)$ is $\kappa$-regular with
\begin{align}\label{eq:kappa_lp_p_lt_2}
    \kappa \eqdef \kappa_V = \kappa_p(d) = d^{\frac{2}{p}-1}.
\end{align}
In particular, the regularity constant is polynomial in $d$ for $p<2$.

Fix $p\in[1, 2)$ and let $d\ge 1$.

\paragraph{Step 1: Construct a smooth auxiliary norm.}
Define the auxiliary norm on $V$ by
\begin{align}\label{eq:def_aux_norm_lp}
    \|x\|_+ \eqdef \|x\|_{V, +} \eqdef d^{\frac{1}{p}-\frac{1}{2}}\|x\|_2,
    \qquad x\in\R^d.
\end{align}
Let $a\eqdef d^{\frac{1}{p}-\frac{1}{2}}>0$ and define
\begin{align*}
    \phi_+(x)\eqdef\frac12\sqnorm{x}_+
    =\frac{a^2}{2}\sqnorm{x}_2.
\end{align*}
This is a quadratic function and is continuously differentiable on $\R^d$.
Using the standard Euclidean inner product to identify $V$ with $V^*$, we obtain
\begin{align*}
    \nabla\phi_+(x)=a^2x.
\end{align*}
For any $x, y\in\R^d$, expanding the square gives
\begin{align*}
    \phi_+(x+y) &= \frac{a^2}{2}\sqnorm{x+y}_2 \\
    &= \frac{a^2}{2}\left(\sqnorm{x}_2+2\inner{x}{y}+\sqnorm{y}_2\right) \\
    &= \phi_+(x)+\inner{a^2x}{y}+\frac{a^2}{2}\sqnorm{y}_2 \\
    &= \phi_+(x)+\inner{\nabla\phi_+(x)}{y}+\frac12\sqnorm{y}_+.
\end{align*}
Thus the inequality in \Cref{def:kappa_smooth_regular}(i) holds with equality for $\kappa_+=1$, so $(V, \norm{\cdot}_+)$ is $1$-smooth.


\paragraph{Step 2: Verify the compatibility inequalities.}
We claim that for every $x\in\R^d$,
\begin{align}\label{eq:lp_l2_equiv_p_lt_2}
    \norm{x}_2 \le \norm{x}_p \le d^{\frac{1}{p}-\frac{1}{2}}\norm{x}_2.
\end{align}

\smallskip
\noindent\emph{Proof of the left inequality in \eqref{eq:lp_l2_equiv_p_lt_2}.}
If $x=0$ it is trivial. Otherwise, let $a_i \eqdef \nfr{|x_i|}{\norm{x}_p}$. Then $\sum_{i=1}^d a_i^p = 1$ and each $a_i\in[0, 1]$. Since $p<2$ and $a_i\in[0, 1]$, we have $a_i^2 \le a_i^p$. Hence
\begin{align*}
    \norm{x}_2^2 &= \sum_{i=1}^d |x_i|^2 \\
    &= \norm{x}_p^2 \sum_{i=1}^d a_i^2 \\
    &\le \norm{x}_p^2 \sum_{i=1}^d a_i^p \\
    &= \norm{x}_p^2,
\end{align*}
which yields $\norm{x}_2 \le \norm{x}_p$.

\smallskip
\noindent\emph{Proof of the right inequality in \eqref{eq:lp_l2_equiv_p_lt_2}.}
Let $r \eqdef \nfr{2}{p} > 1$ and $r' \eqdef \nfr{2}{2-p}$ be its conjugate exponent. Using H\"older's inequality,
\begin{align*}
    \sum_{i=1}^d |x_i|^p &= \sum_{i=1}^d (|x_i|^2)^{p/2}\cdot 1 \\
    &\le \left(\sum_{i=1}^d |x_i|^2\right)^{p/2}\left(\sum_{i=1}^d 1^{r'}\right)^{1/r'} \\
    &= \norm{x}_2^p \, d^{1-\frac{p}{2}}.
\end{align*}
Taking $p$-th roots gives $\norm{x}_p \le d^{\frac{1}{p}-\frac{1}{2}}\norm{x}_2$.

\smallskip
Combining the two inequalities proves \eqref{eq:lp_l2_equiv_p_lt_2}. Now, by the definition \eqref{eq:def_aux_norm_lp} and \eqref{eq:lp_l2_equiv_p_lt_2},
\begin{align*}
    \norm{x}_p \le d^{\frac{1}{p}-\frac{1}{2}}\norm{x}_2 = \norm{x}_+,
\end{align*}
so $\norm{x}_p^2 \le \norm{x}_+^2$. Also, using $\norm{x}_2 \le \norm{x}_p$ (the left inequality),
\begin{align*}
    \norm{x}_+^2 = d^{\frac{2}{p}-1}\norm{x}_2^2 \le d^{\frac{2}{p}-1}\norm{x}_p^2.
\end{align*}
Therefore, for all $x\in\R^d$,
\begin{align}\label{eq:compat_p_lt_2}
    \norm{x}_p^2 \le \norm{x}_+^2 \le d^{\frac{2}{p}-1}\norm{x}_p^2.
\end{align}

\paragraph{Step 3: Conclude $\kappa$-regularity.}
We have shown that the auxiliary norm $(V, \|\cdot\|_+)$ is $1$-smooth, i.e., $\kappa_+=1$, and that
\eqref{eq:compat_p_lt_2} gives the compatibility chain
\begin{align*}
    \|x\|^2 \le \|x\|_+^2 \le d^{\frac{2}{p}-1}\|x\|^2.
\end{align*}
Therefore, by Definition~\ref{def:kappa_smooth_regular}, the space $(V, \|\cdot\|)$ is $\kappa$-regular with
\begin{align*}
    \kappa = d^{\frac{2}{p}-1}.
\end{align*}
Equivalently,
\begin{align*}
    \kappa_V = \kappa_p(d) = d^{\frac{2}{p}-1}.
\end{align*}
This establishes the upper bound in \eqref{eq:kappa_lp_p_lt_2}.\qedhere
\end{enumerate}

\end{proof}


\begin{remark}[Optimality of the Euclidean auxiliary norm for $p<2$ within the scaled $\ell_\rho$ family]
\leavevmode\phantomsection 
\label{rem:optimality_euclidean_aux_p_lt_2}
Fix $d\in\N$ and $p\in[1, 2)$. Let $V\eqdef \R^d$ and equip it with the norm $\|\cdot\|_V\eqdef \|\cdot\|_p$.
For every $\rho\in[2, \infty)$, consider the scaled auxiliary norm
\begin{align}\label{eq:aux_lp_rho_scaled}
    \|x\|_{V, +, \rho} \eqdef d^{\frac{1}{p}-\frac{1}{\rho}}\|x\|_\rho,
    \qquad x\in\R^d.
\end{align}
Then $(V, \|\cdot\|_V)$ is $\kappa_V(\rho)$-regular with respect to this auxiliary norm, where
\begin{align}\label{eq:kappa_rho_family}
    \kappa_V(\rho) \eqdef (\rho-1)d^{\frac{2}{p}-\frac{2}{\rho}}.
\end{align}
Moreover, the function $\rho\mapsto \kappa_V(\rho)$ is minimized over $\rho\in[2, \infty)$ at $\rho=2$.
Hence, within this family of scaled $\ell_\rho$ auxiliary norms, the optimal choice is the Euclidean norm
\begin{align*}
    \|x\|_{V, +} \eqdef d^{\frac{1}{p}-\frac{1}{2}}\|x\|_2.
\end{align*}
\end{remark}

\begin{proof}
To simplify notation throughout this proof, we write
\begin{align*}
    \|x\| \eqdef \|x\|_V = \|x\|_p,
    \qquad
    \|x\|_{+, \rho} \eqdef \|x\|_{V, +, \rho}.
\end{align*}

Fix $\rho\in[2, \infty)$. Since $p<\rho$, the standard norm equivalence inequalities for $\ell_p$ spaces give
\begin{align*}
    \|x\|_\rho \le \|x\| \le d^{\frac{1}{p}-\frac{1}{\rho}}\|x\|_\rho,
    \qquad \forall x\in\R^d.
\end{align*}
Substituting the definition \eqref{eq:aux_lp_rho_scaled}, we obtain
\begin{align}\label{eq:lp_equiv_scaled}
    \|x\| \le \|x\|_{+, \rho} \le d^{\frac{1}{p}-\frac{1}{\rho}}\|x\|,
    \qquad \forall x\in\R^d.
\end{align}
Squaring gives the compatibility chain
\begin{align*}
    \|x\|^2 \le \|x\|_{+, \rho}^2 \le d^{\frac{2}{p}-\frac{2}{\rho}}\|x\|^2,
    \qquad \forall x\in\R^d.
\end{align*}

By Step~1 in the proof of Proposition~\ref{example:2}, the norm $\|\cdot\|_\rho$ is $(\rho-1)$-smooth for every $\rho\in[2, \infty)$.
Since scaling a norm by a positive constant does not change its smoothness parameter, the scaled norm $\|\cdot\|_{+, \rho}$ is also $(\rho-1)$-smooth.
Hence $(V, \|\cdot\|_V)$ is $\kappa_V(\rho)$-regular with respect to the auxiliary norm $\|\cdot\|_{V, +, \rho}$, where
\begin{align*}
    \kappa_V(\rho)
    = (\rho-1)d^{\frac{2}{p}-\frac{2}{\rho}}.
\end{align*}

To identify the best choice within this family, compare an arbitrary $\rho>2$ against $\rho=2$:
\begin{align*}
    \frac{\kappa_V(\rho)}{\kappa_V(2)}
    = \frac{(\rho-1)d^{\frac{2}{p}-\frac{2}{\rho}}}{1\cdot d^{\frac{2}{p}-1}}
    = (\rho-1)d^{1-\frac{2}{\rho}}.
\end{align*}
Since $\rho>2$, we have $\rho-1>1$ and $1-\frac{2}{\rho}>0$.
Also, for every $d\ge 1$, we have $d^{1-\frac{2}{\rho}}\ge 1$.
Therefore,
\begin{align*}
    \frac{\kappa_V(\rho)}{\kappa_V(2)} > 1,
    \qquad \forall \rho>2.
\end{align*}
Thus $\kappa_V(\rho)$ is minimized at $\rho=2$, which yields the Euclidean auxiliary norm
\begin{align*}
    \|x\|_{V, +} \eqdef d^{\frac{1}{p}-\frac{1}{2}}\|x\|_2.
\end{align*}
\end{proof}

We now complement the upper bounds from Proposition~\ref{example:2} with matching lower bounds on the admissible regularity parameter $\kappa_V$ for $\ell_p^d$ spaces. For $1\le p<2$, the optimal parameter is $d^{2/p-1}$. For $2\le p\le\infty$, its optimal order is $\min\{p, \ln(ed)\}$, uniformly in $d\ge1$; see \Cref{rem:uniform_regularity_orders}. Consequently, the optimal parameter remains bounded as $d$ varies with a fixed finite $p\ge2$, and equals $1$ at $p=2$. The logarithmic branch of the minimum corresponds to $p\ge\ln(ed)$, including $p=\infty$, where logarithmic growth is unavoidable up to absolute constants.

\begin{proposition}[Lower bounds for $\ell_p^d$ spaces]
\label{prop:tightness_lp_bounds}
Fix $d\in\N$ and $p\in[1, \infty]$. Let $V\eqdef \R^d$ and equip it with the norm $\|\cdot\|_V\eqdef \|\cdot\|_p$.
If the space $(V, \|\cdot\|_V)$ is $\kappa_V$-regular, then the following hold:
\begin{enumerate}
    \item If $1\le p<2$, then
    \begin{align}\label{eq:lower_bound_lp_p_lt_2}
        \kappa_V \ge d^{\frac{2}{p}-1}.
    \end{align}

    \item If $2\le p\le \infty$ and $d\ge 8$, then
    \begin{align}\label{eq:lower_bound_lp_p_ge_2}
        \kappa_V \ge \frac{1}{8}\min\{p-1, \ln d\}.
    \end{align}
\end{enumerate}
\end{proposition}
\begin{proof}
To simplify notation throughout this proof, we write
\begin{align*}
    \|x\| \eqdef \|x\|_V = \|x\|_p,
    \qquad
    \kappa \eqdef \kappa_V.
\end{align*}

\begin{enumerate}
\item \textbf{Case 1: $1\le p<2$.}

Let $\{\eps_i\}_{i=1}^d$ be i.i.d.\ Rademacher random variables, and let $\{e_i\}_{i=1}^d$ be the standard basis of $\R^d$.
Define
\begin{align*}
    \cF_i \eqdef \sigma(\eps_1, \dots, \eps_i),
    \qquad
    \xi_i \eqdef \eps_i e_i,
    \qquad
    i=1, \dots, d.
\end{align*}
Then $\{\xi_i\}_{i=1}^d$ is an $\{\cF_i\}$-martingale-difference sequence: it is $\cF_i$-measurable and
\begin{align*}
    \Exp{\xi_i \mid \cF_{i-1}}
    = e_i \Exp{\eps_i \mid \cF_{i-1}}
    = 0
    \qquad \text{a.s.}
\end{align*}
Moreover,
\begin{align*}
    \|\xi_i\| = \|e_i\|_p = 1,
    \qquad i=1, \dots, d.
\end{align*}

Let
\begin{align*}
    S_d \eqdef \sum_{i=1}^d \xi_i = (\eps_1, \dots, \eps_d).
\end{align*}
Since $|\eps_i|=1$ almost surely, we have deterministically
\begin{align*}
    \|S_d\|
    = \|S_d\|_p
    = \left(\sum_{i=1}^d |\eps_i|^p\right)^{1/p}
    = d^{1/p}.
\end{align*}
Hence
\begin{align*}
    \Exp{\|S_d\|^2} = d^{2/p},
    \qquad
    \sum_{i=1}^d \Exp{\|\xi_i\|^2} = \sum_{i=1}^d 1 = d.
\end{align*}
Since $(V, \|\cdot\|_V)$ is $\kappa$-regular and $\{\xi_i\}$ is a martingale-difference sequence,
\Cref{lem:second_moment_kappa_regular} yields
\begin{align*}
    \Exp{\|S_d\|^2}
    \le \kappa \sum_{i=1}^d \Exp{\|\xi_i\|^2}
    = \kappa d.
\end{align*}
Combining with $\Exp{\|S_d\|^2}=d^{2/p}$ gives
\begin{align*}
    d^{2/p} \le \kappa d,
\end{align*}
and therefore
\begin{align*}
    \kappa \ge d^{\frac{2}{p}-1}.
\end{align*}
This proves \eqref{eq:lower_bound_lp_p_lt_2}.

\item \textbf{Case 2: $2\le p\le\infty$ and $d\ge 8$.}

We first treat the subcase $2\le p<\infty$.

\paragraph{Step 1: Parameter choice and an explicit family of vectors.}
Let
\begin{align*}
    n \eqdef \left\lfloor \min\{p, \log_2 d\}\right\rfloor.
\end{align*}
Since $d\ge 8$, we have $\log_2 d\ge 3$ and hence $n\ge 2$.
Define
\begin{align*}
    \Omega \eqdef \{\pm1\}^n,
    \qquad
    |\Omega|=2^n \le d.
\end{align*}
Fix an arbitrary enumeration
\begin{align*}
    \Omega = \{\omega^{(1)}, \omega^{(2)}, \dots, \omega^{(2^n)}\},
    \qquad
    \omega^{(j)}=(\omega^{(j)}_1, \dots, \omega^{(j)}_n)\in\{\pm1\}^n.
\end{align*}
We identify the first $2^n$ coordinates of $\R^d$ with the index set $\{1, \dots, 2^n\}$ and set all remaining coordinates equal to $0$.

For each $i\in\{1, \dots, n\}$, define a vector $u_i\in\R^d$ by
\begin{align*}
    (u_i)_j \eqdef 2^{-n/p}\,\omega^{(j)}_i
    \quad \text{for } j=1, \dots, 2^n,
    \qquad
    (u_i)_j \eqdef 0
    \quad \text{for } j=2^n+1, \dots, d.
\end{align*}
Then each $u_i$ has exactly $2^n$ nonzero entries, each equal to $\pm 2^{-n/p}$, and therefore
\begin{align*}
    \|u_i\|^p
    = \|u_i\|_p^p
    = \sum_{j=1}^{2^n} \left|2^{-n/p}\omega^{(j)}_i\right|^p
    = 2^n \cdot 2^{-n}
    = 1.
\end{align*}
Hence
\begin{align*}
    \|u_i\| = 1,
    \qquad i=1, \dots, n.
\end{align*}

\paragraph{Step 2: A martingale-difference sequence.}
Let $\{\eps_i\}_{i=1}^n$ be i.i.d.\ Rademacher random variables, define
\begin{align*}
    \cF_i \eqdef \sigma(\eps_1, \dots, \eps_i),
    \qquad
    \xi_i \eqdef \eps_i u_i,
    \qquad
    i=1, \dots, n.
\end{align*}
Then $\{\xi_i\}_{i=1}^n$ is an $\{\cF_i\}$-martingale-difference sequence and
\begin{align*}
    \|\xi_i\| = \|u_i\| = 1,
    \qquad i=1, \dots, n,
\end{align*}
so
\begin{align*}
    \sum_{i=1}^n \Exp{\|\xi_i\|^2} = n.
\end{align*}

\paragraph{Step 3: Lower bound on $\Exp{\|S_n\|^2}$.}
Let
\begin{align*}
    S_n \eqdef \sum_{i=1}^n \xi_i.
\end{align*}
For each coordinate $j\in\{1, \dots, 2^n\}$, we have
\begin{align*}
    (S_n)_j
    = \sum_{i=1}^n \eps_i (u_i)_j
    = 2^{-n/p}\sum_{i=1}^n \eps_i\,\omega^{(j)}_i.
\end{align*}
Therefore,
\begin{align*}
    \|S_n\|^p
    &= \|S_n\|_p^p \\
    &= \sum_{j=1}^{2^n}\left|2^{-n/p}\sum_{i=1}^n \eps_i\,\omega^{(j)}_i\right|^p \\
    &= 2^{-n}\sum_{j=1}^{2^n}\left|\sum_{i=1}^n \eps_i\,\omega^{(j)}_i\right|^p \\
    &= \Expu{\omega\sim \mathrm{Unif}(\Omega)}{\left|\sum_{i=1}^n \eps_i\,\omega_i\right|^p},
\end{align*}
where the last equality holds because $\{\omega^{(j)}\}_{j=1}^{2^n}$ is merely an enumeration of $\Omega$.

Since $\omega\sim\mathrm{Unif}(\Omega)$ is uniform on $\{\pm1\}^n$, for every fixed realization of $\eps=(\eps_1, \dots, \eps_n)\in\{\pm1\}^n$ we have
\begin{align*}
    \Expu{\omega\sim\mathrm{Unif}(\Omega)}{\left|\sum_{i=1}^n \eps_i\omega_i\right|^p}
    =
    \Expu{\omega\sim\mathrm{Unif}(\Omega)}{\left|\sum_{i=1}^n \omega_i\right|^p},
\end{align*}
because the coordinatewise sign-flip map
\begin{align*}
    \omega \mapsto (\eps_1\omega_1, \dots, \eps_n\omega_n)
\end{align*}
is a bijection of $\Omega$ and hence preserves the uniform law.
Therefore,
\begin{align*}
    \|S_n\|
    =
    \left(
    \Expu{\omega\sim\mathrm{Unif}(\Omega)}{\left|\sum_{i=1}^n \omega_i\right|^p}
    \right)^{1/p}.
\end{align*}
In particular, $\|S_n\|$ is deterministic, and thus
\begin{align*}
    \Exp{\|S_n\|^2}
    = \|S_n\|^2
    =
    \left(
    \Expu{\omega\sim \mathrm{Unif}(\Omega)}{\left|\sum_{i=1}^n \omega_i\right|^p}
    \right)^{2/p}.
\end{align*}

We now lower bound this quantity using a single atom of the distribution.
Let
\begin{align*}
    Z_n \eqdef \sum_{i=1}^n \omega_i,
    \qquad
    \omega\sim \mathrm{Unif}(\Omega).
\end{align*}
Since $\Prob\{Z_n=n\}=2^{-n}$, we obtain
\begin{align*}
    \Expu{\omega\sim \mathrm{Unif}(\Omega)}{|Z_n|^p}
    \ge n^p \Prob\{Z_n=n\}
    = n^p 2^{-n}.
\end{align*}
Therefore,
\begin{align*}
    \|S_n\|
    = \left(
    \Expu{\omega\sim \mathrm{Unif}(\Omega)}{|Z_n|^p}
    \right)^{1/p}
    \ge n\cdot 2^{-n/p}.
\end{align*}
By construction $n\le p$, hence $2^{-n/p}\ge \nfr{1}{2}$, and consequently
\begin{align*}
    \Exp{\|S_n\|^2}
    = \|S_n\|^2
    \ge \frac{n^2}{4}.
\end{align*}

\paragraph{Step 4: Conclude via the second-moment bound.}
Up to this point, all expectations involving $\omega$ were taken with respect to the auxiliary randomness
$\omega\sim \mathrm{Unif}(\Omega)$. We now switch back to the \emph{original} probability space on which the
martingale-difference sequence $\{\xi_i\}_{i=1}^n$ (equivalently, the Rademacher signs
$\{\eps_i\}_{i=1}^n$) is defined. In particular, in the remainder of the proof,
$\Exp{\cdot}$ denotes expectation with respect to the randomness of $(\eps_1, \dots, \eps_n)$.

Since $(V, \|\cdot\|_V)$ is $\kappa$-regular and $\{\xi_i\}$ is a martingale-difference sequence,
\Cref{lem:second_moment_kappa_regular} yields
\begin{align*}
    \Exp{\|S_n\|^2}
    \le \kappa \sum_{i=1}^n \Exp{\|\xi_i\|^2}
    = \kappa n.
\end{align*}
Combining this with the lower bound $\Exp{\|S_n\|^2}\ge \nfr{n^2}{4}$ gives
\begin{align*}
    \kappa \ge \frac{n}{4}.
\end{align*}

Finally, since $n=\lfloor\min\{p, \log_2 d\}\rfloor$ and $d\ge 8$, we have
\begin{align*}
    n
    &\ge \min\{p-1, \log_2 d-1\} \\
    &\ge \frac{1}{2}\min\{p-1, \log_2 d\} \\
    &\ge \frac{1}{2}\min\{p-1, \ln d\}.
\end{align*}
Therefore,
\begin{align*}
    \kappa
    \ge \frac{n}{4}
    \ge \frac{1}{8}\min\{p-1, \ln d\},
\end{align*}
which proves \eqref{eq:lower_bound_lp_p_ge_2} for $2\le p<\infty$.

\paragraph{Subcase $p=\infty$.}
Let
\begin{align*}
    n \eqdef \lfloor \log_2 d \rfloor,
    \qquad
    \Omega \eqdef \{\pm 1\}^n,
    \qquad
    |\Omega|=2^n \le d.
\end{align*}
Fix an enumeration
\begin{align*}
    \Omega=\{\omega^{(1)}, \dots, \omega^{(2^n)}\},
    \qquad
    \omega^{(j)}=(\omega^{(j)}_1, \dots, \omega^{(j)}_n)\in\{\pm1\}^n.
\end{align*}
As before, we identify the first $2^n$ coordinates of $\R^d$ with the index set $\{1, \dots, 2^n\}$ and set all remaining coordinates equal to $0$.

For each $i\in\{1, \dots, n\}$, define $u_i\in\R^d$ by
\begin{align*}
    (u_i)_j \eqdef \omega^{(j)}_i
    \quad \text{for } j=1, \dots, 2^n,
    \qquad
    (u_i)_j \eqdef 0
    \quad \text{for } j=2^n+1, \dots, d.
\end{align*}
Then every nonzero entry of $u_i$ equals $\pm1$, and hence
\begin{align*}
    \|u_i\| = \|u_i\|_\infty = 1,
    \qquad i=1, \dots, n.
\end{align*}

Let $\{\eps_i\}_{i=1}^n$ be i.i.d.\ Rademacher random variables and define
\begin{align*}
    \cF_i \eqdef \sigma(\eps_1, \dots, \eps_i),
    \qquad
    \xi_i \eqdef \eps_i u_i,
    \qquad
    i=1, \dots, n.
\end{align*}
Then $\{\xi_i\}_{i=1}^n$ is an $\{\cF_i\}$-martingale-difference sequence and
\begin{align*}
    \sum_{i=1}^n \Exp{\|\xi_i\|^2}
    = \sum_{i=1}^n \Exp{\|\xi_i\|_\infty^2}
    = n.
\end{align*}

Now consider
\begin{align*}
    S_n \eqdef \sum_{i=1}^n \xi_i.
\end{align*}
For each $j\in\{1, \dots, 2^n\}$, we have
\begin{align*}
    (S_n)_j
    = \sum_{i=1}^n \eps_i (u_i)_j
    = \sum_{i=1}^n \eps_i\,\omega^{(j)}_i.
\end{align*}
Hence
\begin{align*}
    \|S_n\|
    = \|S_n\|_\infty
    = \max_{1\le j\le 2^n}\left|(S_n)_j\right|
    = \max_{\omega\in\Omega}\left|\sum_{i=1}^n \eps_i\omega_i\right|.
\end{align*}
For any fixed realization of $\eps=(\eps_1, \dots, \eps_n)\in\{\pm1\}^n$, we may choose $\omega=\eps$.
Then
\begin{align*}
    \sum_{i=1}^n \eps_i\omega_i
    = \sum_{i=1}^n \eps_i^2
    = n,
\end{align*}
so the maximum is at least $n$.
On the other hand, for every $\omega\in\{\pm1\}^n$,
\begin{align*}
    \left|\sum_{i=1}^n \eps_i\omega_i\right|
    \le \sum_{i=1}^n |\eps_i\omega_i|
    = n,
\end{align*}
so the maximum is also at most $n$.
Therefore,
\begin{align*}
    \|S_n\| = n
    \qquad \text{deterministically, and hence} \qquad
    \Exp{\|S_n\|^2} = n^2.
\end{align*}

Finally, since $(V, \|\cdot\|_V)$ is $\kappa$-regular and $\{\xi_i\}$ is a martingale-difference sequence,
\Cref{lem:second_moment_kappa_regular} yields
\begin{align*}
    \Exp{\|S_n\|^2}
    \le \kappa \sum_{i=1}^n \Exp{\|\xi_i\|^2}
    = \kappa n.
\end{align*}
Combining with $\Exp{\|S_n\|^2}=n^2$ gives
\begin{align*}
    \kappa \ge n.
\end{align*}
Since $n=\lfloor\log_2 d\rfloor \ge \frac12\ln d$ for $d\ge 8$, we obtain
\begin{align*}
    \kappa \ge \frac12\ln d,
\end{align*}
which is stronger than \eqref{eq:lower_bound_lp_p_ge_2} in the case $p=\infty$, because
\begin{align*}
    \min\{p-1, \ln d\} = \ln d
    \qquad \text{when } p=\infty.\qedhere
\end{align*}
\end{enumerate}
\end{proof}

\begin{remark}[Prior work]
\citet[Example~1.2, p.~139]{dumbgen2010nemirovski} give a logarithmic lower bound for the second-moment constant in $\ell_\infty^d$; their Lemma~3.6 (p.~146) gives corresponding bounds for the squared type-$2$ constant. These results concern moment inequalities, and imply lower bounds on our regularity parameter through \Cref{lem:second_moment_kappa_regular}.

For $2\le p\le\infty$, the general $\ell_p^d$ order follows with an additional argument. Let $K_p(d)$ be the smallest constant such that
\begin{align*}
    \Exp{\Bigl\|\sum_i X_i\Bigr\|_p^2}
    \le K_p(d)\sum_i\Exp{\|X_i\|_p^2}
\end{align*}
for every finite sequence of independent, mean-zero random vectors in $\R^d$ with finite second moments. The cited results, together with $K_\infty(m)\ge1$, give $K_\infty(m)\ge c\ln(em)$ for all $m\ge1$ and an absolute constant $c>0$. Embedding $\R^m$ as a coordinate subspace of $\R^d$ and using $\|z\|_\infty\le\|z\|_p\le m^{1/p}\|z\|_\infty$ gives, for every $1\le m\le d$,
\begin{align*}
    \kappa_V\ge K_p(d)\ge K_p(m)\ge m^{-2/p}K_\infty(m)
    \ge c\,m^{-2/p}\ln(em).
\end{align*}
For finite $p\ge2$, choose $m=\min\{d, \lfloor e^p\rfloor\}$; then $m^{2/p}\le e^2$ and $\ln(em)\ge\min\{p, \ln(ed)\}$. For $p=\infty$, take $m=d$ and use $m^{-2/p}=1$. Thus the lower-bound order $\min\{p, \ln(ed)\}$ follows, with an unspecified absolute constant.

The explicit factor $1/8$ in \eqref{eq:lower_bound_lp_p_ge_2} is supplied by our direct proof above. That proof was obtained independently, before consulting this reference; we include it for completeness and to keep the exposition self-contained within our framework.
\end{remark}

\subsection{Spaces with Schatten-\texorpdfstring{$p$}{p} Norms on \texorpdfstring{$\R^{m\times n}$}{real m-by-n matrices}}\label{sub:proof_of_the_example_3}

We next record the analogous regularity bounds for Schatten-$p$ norms on matrix spaces. Since the Schatten-$p$ norm is simply the $\ell_p$ norm of the singular-value vector, the behavior closely parallels the vector case: for $p\ge 2$, one may take a regularity parameter that grows at most logarithmically in the effective dimension $r\eqdef \min\{m, n\}$, whereas for $p<2$, the natural admissible parameter grows polynomially in $r$.

\begin{proposition}[Schatten-$p$ norms for matrices]\label{example:3}
Fix $m, n\in\N$ and let $r\eqdef \min\{m, n\}$. Let $V\eqdef \R^{m\times n}$ and equip it with the norm
$\|\cdot\|_V \eqdef \|\cdot\|_{S_p}$, where
\begin{align*}
    \|X\|_{S_p}
    \eqdef
    \|\sigma(X)\|_p
    =
    \left(\sum_{i=1}^{r}\sigma_i(X)^p\right)^{1/p},
    \qquad p\in[1, \infty),
\end{align*}
and
\begin{align*}
    \|X\|_{S_\infty}
    \eqdef
    \max_{1\le i\le r}\sigma_i(X).
\end{align*}
Then the space $(V, \|\cdot\|_V)$ is $\kappa_V$-regular in the sense of Definition~\ref{def:kappa_smooth_regular}, with admissible regularity parameter $\kappa_V\eqdef \kappa_p(m, n)$ given as follows:
\begin{enumerate}
\item Let $2\le p\le \infty$. Then
\begin{align*}
    \kappa_V \eqdef \kappa_p(m, n)
    =
    \min_{\substack{2\le \rho\le p\\ \rho<\infty}}
    \left(\max\{2, \rho-1\}\, r^{\frac{2}{\rho}-\frac{2}{p}}\right).
\end{align*}
For $r\ge8$, this parameter is realized by the auxiliary norm
\begin{align*}
    \|X\|_{V, +} \eqdef \|X\|_{S_{\rho_*}},
    \qquad
    \rho_* \eqdef
    \min\!\left\{
        p, \;
        \max\!\left\{
            3, \;
            \ln r+\sqrt{(\ln r)^2-2\ln r}
        \right\}
    \right\}.
\end{align*}
Under the same condition $r\ge8$,
\begin{align*}
    \kappa_V \eqdef \kappa_p(m, n) = \mathcal{O}\bigl(\min\{p, \ln(er)\}\bigr).
\end{align*}

    \item Let $1\le p<2$. Then
    \begin{align*}
        \kappa_V \eqdef \kappa_p(m, n)
        =
        r^{\frac{2}{p}-1}.
    \end{align*}
\end{enumerate}
\end{proposition}

\begin{proof}
To simplify notation throughout this proof, we write
\begin{align*}
    \|X\|
    \eqdef
    \|X\|_V
    =
    \|X\|_{S_p},
    \qquad
    \kappa
    \eqdef
    \kappa_V.
\end{align*}
Whenever an auxiliary norm is introduced, we denote it by
\begin{align*}
    \|X\|_+
    \eqdef
    \|X\|_{V, +},
    \qquad
    \kappa_+
    \eqdef
    \kappa_{V, +}.
\end{align*}

We consider separately the cases $2\le p\le\infty$ and $1\le p<2$.

\begin{enumerate}[label=(\roman*)]
\item \textbf{Case $2\le p\le\infty$.}
This case is established in \citep{juditsky2008large}; see the proof of Example~3.3 therein.
In our notation, this yields that $(V, \|\cdot\|)$ is $\kappa$-regular with
\begin{align*}
    \kappa
    \eqdef
    \kappa_p(m, n)
    =
    \min_{\substack{2\le \rho\le p\\ \rho<\infty}}
    \left(\max\{2, \rho-1\}\, r^{\frac{2}{\rho}-\frac{2}{p}}\right).
\end{align*}
We now derive the minimizing exponent for $r\ge8$. Set $a\eqdef\ln r>2$ and define, for finite $\rho\ge2$,
\begin{align*}
    F(\rho)\eqdef\max\{2, \rho-1\}\,r^{\frac{2}{\rho}-\frac{2}{p}},
\end{align*}
with the convention $1/p=0$ when $p=\infty$. On $2\le\rho\le3$, we have $F(\rho)=2r^{2/\rho-2/p}$ and
\begin{align*}
    \frac{d}{d\rho}\ln F(\rho)=-\frac{2a}{\rho^2}<0,
    \qquad 2<\rho<3.
\end{align*}
Thus $F$ is strictly decreasing on $[2, 3]$; in particular, if $2\le p\le3$, its minimum over $[2, p]$ is attained at $\rho=p$.

On $\rho\ge3$, we have $F(\rho)=(\rho-1)r^{2/\rho-2/p}$, so
\begin{align*}
    \frac{d}{d\rho}\ln F(\rho)
    =\frac{1}{\rho-1}-\frac{2a}{\rho^2}
    =\frac{\rho^2-2a\rho+2a}{\rho^2(\rho-1)},
    \qquad \rho>3.
\end{align*}
The numerator has roots
\begin{align*}
    \rho_- = a-\sqrt{a^2-2a},
    \qquad
    \rho_+ = a+\sqrt{a^2-2a}.
\end{align*}
Since $a>2$ and the numerator at $\rho=2$ equals $4-2a<0$, these roots satisfy $\rho_-<2<\rho_+$. Consequently, if $\rho_+>3$, then $F$ decreases on $[3, \rho_+]$ and increases on $[\rho_+, \infty)$; if $\rho_+\le3$, it increases on $[3, \infty)$. Together with continuity at $\rho=3$ and the decrease on $[2, 3]$, this shows that the minimum over all finite $\rho\ge2$ is attained at $\max\{3, \rho_+\}$. Imposing $\rho\le p$ therefore gives
\begin{align*}
    \rho_*
    =\min\!\left\{p, \max\!\left\{3, \ln r+\sqrt{(\ln r)^2-2\ln r}\right\}\right\}.
\end{align*}
For $p=\infty$, this yields the finite exponent $\rho_* = \max\{3, \rho_+\}$, so the auxiliary norm $\|\cdot\|_{S_{\rho_*}}$ remains within the construction above.

For completeness, the order bound follows by choosing $\rho=p$ when $2\le p<2a$, which gives $\kappa\le\max\{2, p-1\}\le p<2a$. When $p\ge2a$, including $p=\infty$, the admissible choice $\rho=2a$ gives
\begin{align*}
    \kappa\le(2a-1)\exp\!\left(1-\frac{2a}{p}\right)\le2ea.
\end{align*}
For every finite $p\ge2$, the choice $\rho=p$ also gives $\kappa\le p$. Hence, for $r\ge8$,
\begin{align*}
    \kappa
    \eqdef
    \kappa_p(m, n)
    =
    \mathcal{O}\bigl(\min\{p, \ln(er)\}\bigr).
\end{align*}

\item \textbf{Case $1\le p<2$.}

\emph{Upper bound.}
Let $\|\cdot\|_{S_2}$ be the Frobenius (Schatten-$2$) norm and define the auxiliary norm
\begin{align*}
    \|X\|_+
    \eqdef
    r^{\frac{1}{p}-\frac{1}{2}}\|X\|_{S_2}.
\end{align*}
This is a scaled Hilbert norm, hence $1$-smooth, so we may take
\begin{align*}
    \kappa_+ = 1.
\end{align*}
Moreover, for every $X\in \R^{m\times n}$,
\begin{align*}
    \|X\|
    =
    \|X\|_{S_p}
    =
    \|\sigma(X)\|_p
    \le
    r^{\frac{1}{p}-\frac{1}{2}}\|\sigma(X)\|_2
    =
    r^{\frac{1}{p}-\frac{1}{2}}\|X\|_{S_2}
    =
    \|X\|_+,
\end{align*}
and since $p<2$ implies $\|X\|_{S_2}\le \|X\|_{S_p}$, we also have
\begin{align*}
    \|X\|_+^2
    =
    r^{\frac{2}{p}-1}\|X\|_{S_2}^2
    \le
    r^{\frac{2}{p}-1}\|X\|_{S_p}^2
    =
    r^{\frac{2}{p}-1}\|X\|^2.
\end{align*}
Thus
\begin{align*}
    \|X\|^2
    \le
    \|X\|_+^2
    \le
    r^{\frac{2}{p}-1}\|X\|^2,
    \qquad
    \forall X\in V.
\end{align*}
Therefore, by Definition~\ref{def:kappa_smooth_regular}, the space $(V, \|\cdot\|)$ is $\kappa$-regular with
\begin{align*}
    \kappa
    \eqdef
    \kappa_p(m, n)
    =
    r^{\frac{2}{p}-1}.
\end{align*}
\end{enumerate}
\end{proof}


We now complement Proposition~\ref{example:3} with corresponding lower bounds for the matrix setting. With $r\eqdef\min\{m, n\}$, the optimal parameter is $r^{2/p-1}$ for $1\le p<2$, and its optimal order is $\min\{p, \ln(er)\}$ for $2\le p\le\infty$, uniformly in $r\ge1$; see \Cref{rem:uniform_regularity_orders}. Consequently, the optimal parameter remains bounded as $r$ varies with a fixed finite $p\ge2$, and equals $1$ at $p=2$. The logarithmic branch of the minimum corresponds to $p\ge\ln(er)$, including $p=\infty$, where logarithmic growth is unavoidable up to absolute constants.

\begin{proposition}[Lower bounds for Schatten classes]
\label{prop:tightness_schatten_bounds}
Fix $m, n\in\N$, let $r\eqdef \min\{m, n\}$, let $V\eqdef \R^{m\times n}$, and equip it with the norm $\|\cdot\|_V\eqdef \|\cdot\|_{S_p}$.
If the space $(V, \|\cdot\|_V)$ is $\kappa_V$-regular, then the following hold:
\begin{enumerate}
    \item If $1\le p<2$, then
    \begin{align}\label{eq:schatten_lower_p_lt_2}
        \kappa_V \ge r^{\frac{2}{p}-1}.
    \end{align}

    \item If $2\le p\le \infty$ and $r\ge 8$, then
    \begin{align}\label{eq:schatten_lower_p_ge_2}
        \kappa_V \ge \frac18\min\{p-1, \ln r\}.
    \end{align}
\end{enumerate}
\end{proposition}

\begin{proof}
To simplify notation throughout this proof, we write
\begin{align*}
    \|X\| \eqdef \|X\|_V = \|X\|_{S_p},
    \qquad
    \kappa \eqdef \kappa_V.
\end{align*}

\begin{enumerate}
\item \textbf{Case 1: $1\le p<2$.}

Let $\{\eps_i\}_{i=1}^{r}$ be independent Rademacher random variables and define the filtration
\begin{align*}
    \cF_i \eqdef \sigma(\eps_1, \dots, \eps_i),
    \qquad
    i=1, \dots, r.
\end{align*}
For each $i=1, \dots, r$, let $E_{ii}\in\R^{m\times n}$ denote the matrix with a $1$ in position $(i, i)$ and zeros elsewhere, and set
\begin{align*}
    \xi_i \eqdef \eps_i E_{ii},
    \qquad
    i=1, \dots, r.
\end{align*}
Then $\{\xi_i\}_{i=1}^r$ is an $\{\cF_i\}$-martingale-difference sequence, since each $\xi_i$ is $\cF_i$-measurable and
\begin{align*}
    \Exp{\xi_i \mid \cF_{i-1}}
    =
    E_{ii}\Exp{\eps_i \mid \cF_{i-1}}
    =
    0
    \qquad \text{a.s.}
\end{align*}

Next, note that $E_{ii}$ has exactly one singular value equal to $1$ and all remaining singular values equal to $0$, hence
\begin{align*}
    \|\xi_i\|
    =
    |\eps_i|\|E_{ii}\|_{S_p}
    =
    1,
    \qquad
    \Exp{\|\xi_i\|^2}
    =
    1.
\end{align*}

Moreover, for
\begin{align*}
    S_r \eqdef \sum_{i=1}^r \xi_i,
\end{align*}
we have
\begin{align*}
    S_r
    =
    \sum_{i=1}^r \eps_i E_{ii}.
\end{align*}
Recall that $E_{ii}\in\R^{m\times n}$ is the matrix with a single $1$ in entry $(i, i)$ and zeros elsewhere.
Therefore, $S_r$ is the rectangular diagonal matrix whose first $r$ diagonal entries are $(\eps_1, \dots, \eps_r)$ and whose remaining entries are zero:
\begin{align*}
    (S_r)_{ab}
    =
    \begin{cases}
        \eps_a, & \text{if } a=b\le r, \\
        0, & \text{otherwise}.
    \end{cases}
\end{align*}
For brevity, we denote this by
\begin{align*}
    S_r = \mathrm{Diag}_{m\times n}(\eps_1, \dots, \eps_r).
\end{align*}

The singular values of $S_r$ are $|\eps_1|, \dots, |\eps_r|$ (and zeros thereafter), so deterministically
\begin{align*}
    \|S_r\|
    =
    \|S_r\|_{S_p}
    =
    \left(\sum_{i=1}^r |\eps_i|^p\right)^{1/p}
    =
    r^{1/p}.
\end{align*}
Hence
\begin{align*}
    \Exp{\|S_r\|^2}
    =
    r^{2/p},
    \qquad
    \sum_{i=1}^r \Exp{\|\xi_i\|^2}
    =
    r.
\end{align*}
Finally, since $(V, \|\cdot\|_V)$ is $\kappa$-regular and $\{\xi_i\}$ is a martingale-difference sequence, \Cref{lem:second_moment_kappa_regular} gives
\begin{align*}
    \Exp{\|S_r\|^2}
    \le
    \kappa \sum_{i=1}^r \Exp{\|\xi_i\|^2}
    =
    \kappa r.
\end{align*}
Combining this with $\Exp{\|S_r\|^2}=r^{2/p}$ yields
\begin{align*}
    r^{2/p} \le \kappa r,
\end{align*}
that is,
\begin{align*}
    \kappa \ge r^{\frac{2}{p}-1}.
\end{align*}
This proves \eqref{eq:schatten_lower_p_lt_2}.

\item \textbf{Case 2: $2\le p\le \infty$ and $r\ge 8$.}

We first treat the subcase $2\le p<\infty$.

\paragraph{Step 1: Parameter choice and an explicit family of diagonal matrices.}
Let
\begin{align*}
    N \eqdef \left\lfloor \min\{p, \log_2 r\}\right\rfloor.
\end{align*}
Since $r\ge 8$, we have $\log_2 r\ge 3$, so in particular $N\ge 2$ and
\begin{align*}
    2^N \le r \le \min\{m, n\}.
\end{align*}
Define the discrete cube
\begin{align*}
    \Omega \eqdef \{\pm1\}^N,
    \qquad
    |\Omega|=2^N.
\end{align*}
Fix an arbitrary enumeration
\begin{align*}
    \Omega=\{\omega^{(1)}, \dots, \omega^{(2^N)}\},
    \qquad
    \omega^{(j)}=(\omega^{(j)}_1, \dots, \omega^{(j)}_N)\in\{\pm1\}^N.
\end{align*}

For each $i\in\{1, \dots, N\}$, define a diagonal matrix $U_i\in\R^{m\times n}$ by setting
\begin{align*}
    (U_i)_{ab} \eqdef 0,
    \qquad
    \text{for } a\neq b,
\end{align*}
and
\begin{align*}
    (U_i)_{jj}
    \eqdef
    2^{-N/p}\,\omega^{(j)}_i,
    \qquad
    j=1, \dots, 2^N,
\end{align*}
while
\begin{align*}
    (U_i)_{jj}
    \eqdef
    0,
    \qquad
    j=2^N+1, \dots, r.
\end{align*}
Thus $U_i$ has exactly $2^N$ nonzero diagonal entries, each equal to $\pm 2^{-N/p}$.

Since $U_i$ is diagonal, its singular values are the absolute values of its diagonal entries (and zeros thereafter). Hence
\begin{align*}
    \|U_i\|^p
    =
    \|U_i\|_{S_p}^p
    =
    \sum_{j=1}^{2^N}\left|2^{-N/p}\omega^{(j)}_i\right|^p
    =
    2^N\cdot 2^{-N}
    =
    1.
\end{align*}
Therefore,
\begin{align*}
    \|U_i\| = 1,
    \qquad
    i=1, \dots, N.
\end{align*}

\paragraph{Step 2: A martingale-difference sequence.}
Let $\{\eps_i\}_{i=1}^N$ be i.i.d.\ Rademacher random variables, let
\begin{align*}
    \cF_i \eqdef \sigma(\eps_1, \dots, \eps_i),
    \qquad
    \xi_i \eqdef \eps_i U_i,
    \qquad
    i=1, \dots, N.
\end{align*}
Then $\{\xi_i\}_{i=1}^N$ is an $\{\cF_i\}$-martingale-difference sequence and
\begin{align*}
    \|\xi_i\| = \|U_i\| = 1,
    \qquad
    i=1, \dots, N,
\end{align*}
so
\begin{align*}
    \sum_{i=1}^N \Exp{\|\xi_i\|^2} = N.
\end{align*}

\paragraph{Step 3: Lower bound on $\Exp{\|S_N\|^2}$.}
Let
\begin{align*}
    S_N \eqdef \sum_{i=1}^N \xi_i.
\end{align*}
Then $S_N$ is diagonal, and for each $j\in\{1, \dots, 2^N\}$,
\begin{align*}
    (S_N)_{jj}
    =
    \sum_{i=1}^N \eps_i (U_i)_{jj}
    =
    2^{-N/p}\sum_{i=1}^N \eps_i\,\omega^{(j)}_i.
\end{align*}
Therefore,
\begin{align*}
    \|S_N\|^p
    &=
    \|S_N\|_{S_p}^p \\
    &=
    \sum_{j=1}^{2^N} |(S_N)_{jj}|^p \\
    &=
    \sum_{j=1}^{2^N}\left|2^{-N/p}\sum_{i=1}^N \eps_i\,\omega^{(j)}_i\right|^p \\
    &=
    2^{-N}\sum_{j=1}^{2^N}\left|\sum_{i=1}^N \eps_i\,\omega^{(j)}_i\right|^p \\
    &=
    \Expu{\omega\sim\mathrm{Unif}(\Omega)}{\left|\sum_{i=1}^N \eps_i \omega_i\right|^p},
\end{align*}
where the last equality holds because $\{\omega^{(j)}\}_{j=1}^{2^N}$ is an enumeration of $\Omega$.

Since $\omega\sim\mathrm{Unif}(\Omega)$ is uniform on $\{\pm1\}^N$, for every fixed realization of $\eps=(\eps_1, \dots, \eps_N)\in\{\pm1\}^N$ we have
\begin{align*}
    \Expu{\omega\sim\mathrm{Unif}(\Omega)}{\left|\sum_{i=1}^N \eps_i\omega_i\right|^p}
    =
    \Expu{\omega\sim\mathrm{Unif}(\Omega)}{\left|\sum_{i=1}^N \omega_i\right|^p},
\end{align*}
because the coordinatewise sign-flip map
\begin{align*}
    \omega\mapsto (\eps_1\omega_1, \dots, \eps_N\omega_N)
\end{align*}
is a bijection of $\Omega$ and therefore preserves the uniform law. Consequently,
\begin{align*}
    \|S_N\|
    =
    \left(
    \Expu{\omega\sim\mathrm{Unif}(\Omega)}{\left|\sum_{i=1}^N \omega_i\right|^p}
    \right)^{1/p},
\end{align*}
so $\|S_N\|$ is deterministic. In particular,
\begin{align*}
    \Exp{\|S_N\|^2}
    =
    \|S_N\|^2
    =
    \left(
    \Expu{\omega\sim\mathrm{Unif}(\Omega)}{\left|\sum_{i=1}^N \omega_i\right|^p}
    \right)^{2/p}.
\end{align*}

We now lower bound this quantity using a single atom. Let
\begin{align*}
    Z_N \eqdef \sum_{i=1}^N \omega_i,
    \qquad
    \omega\sim\mathrm{Unif}(\Omega).
\end{align*}
Since $\Prob\{Z_N=N\}=2^{-N}$, we obtain
\begin{align*}
    \Expu{\omega\sim\mathrm{Unif}(\Omega)}{|Z_N|^p}
    \ge
    N^p \Prob\{Z_N=N\}
    =
    N^p 2^{-N}.
\end{align*}
Hence
\begin{align*}
    \|S_N\|
    =
    \left(
    \Expu{\omega\sim\mathrm{Unif}(\Omega)}{|Z_N|^p}
    \right)^{1/p}
    \ge
    N\cdot 2^{-N/p}.
\end{align*}
By construction, $N\le p$, so $2^{-N/p}\ge \nfr{1}{2}$, and thus
\begin{align*}
    \Exp{\|S_N\|^2}
    =
    \|S_N\|^2
    \ge
    \frac{N^2}{4}.
\end{align*}

\paragraph{Step 4: Conclude via the second-moment bound.}
Up to this point, all expectations involving $\omega$ were taken with respect to the auxiliary randomness
$\omega\sim\mathrm{Unif}(\Omega)$. We now switch back to the original probability space on which the martingale-difference sequence $\{\xi_i\}_{i=1}^N$ is defined. In particular, in the remainder of the proof, $\Exp{\cdot}$ denotes expectation with respect to the randomness of $\eps$.

Since $(V, \|\cdot\|_V)$ is $\kappa$-regular and $\{\xi_i\}$ is a martingale-difference sequence, \Cref{lem:second_moment_kappa_regular} gives
\begin{align*}
    \Exp{\|S_N\|^2}
    \le
    \kappa \sum_{i=1}^N \Exp{\|\xi_i\|^2}
    =
    \kappa N.
\end{align*}
Combining this with $\Exp{\|S_N\|^2}\ge \nfr{N^2}{4}$ yields
\begin{align*}
    \kappa \ge \frac{N}{4}.
\end{align*}

Finally, since
\begin{align*}
    N = \lfloor \min\{p, \log_2 r\}\rfloor,
\end{align*}
we have
\begin{align*}
    N \ge \min\{p-1, \log_2 r-1\}.
\end{align*}
Since $r\ge 8$, we have $\log_2 r\ge 3$, and therefore
\begin{align*}
    \log_2 r - 1 \ge \frac12 \log_2 r \ge \frac12 \ln r.
\end{align*}
Hence
\begin{align*}
    \kappa
    \ge
    \frac{N}{4}
    \ge
    \frac18 \min\{p-1, \ln r\},
\end{align*}
which proves \eqref{eq:schatten_lower_p_ge_2} for $2\le p<\infty$.

\paragraph{Subcase $p=\infty$.}
Let
\begin{align*}
    N \eqdef \lfloor \log_2 r\rfloor,
    \qquad
    \Omega \eqdef \{\pm1\}^N,
    \qquad
    |\Omega|=2^N.
\end{align*}
Fix the same enumeration
\begin{align*}
    \Omega=\{\omega^{(1)}, \dots, \omega^{(2^N)}\},
    \qquad
    \omega^{(j)}=(\omega^{(j)}_1, \dots, \omega^{(j)}_N)\in\{\pm1\}^N.
\end{align*}
Define diagonal matrices $U_i\in\R^{m\times n}$ by
\begin{align*}
    (U_i)_{ab}
    \eqdef
    0,
    \qquad
    \text{for } a\neq b,
\end{align*}
and
\begin{align*}
    (U_i)_{jj}
    \eqdef
    \omega^{(j)}_i,
    \qquad
    j=1, \dots, 2^N,
\end{align*}
while
\begin{align*}
    (U_i)_{jj}
    \eqdef
    0,
    \qquad
    j=2^N+1, \dots, r.
\end{align*}
Then every nonzero diagonal entry of $U_i$ equals $\pm1$, hence
\begin{align*}
    \|U_i\|
    =
    \|U_i\|_{S_\infty}
    =
    1,
    \qquad
    i=1, \dots, N.
\end{align*}

With
\begin{align*}
    \cF_i \eqdef \sigma(\eps_1, \dots, \eps_i),
    \qquad
    \xi_i \eqdef \eps_i U_i,
    \qquad
    i=1, \dots, N,
\end{align*}
we again obtain an $\{\cF_i\}$-martingale-difference sequence satisfying
\begin{align*}
    \sum_{i=1}^N \Exp{\|\xi_i\|^2}
    =
    \sum_{i=1}^N \Exp{\|\xi_i\|_{S_\infty}^2}
    =
    N.
\end{align*}

For
\begin{align*}
    S_N \eqdef \sum_{i=1}^N \xi_i,
\end{align*}
we have, for each $j\in\{1, \dots, 2^N\}$,
\begin{align*}
    (S_N)_{jj}
    =
    \sum_{i=1}^N \eps_i (U_i)_{jj}
    =
    \sum_{i=1}^N \eps_i\,\omega^{(j)}_i.
\end{align*}
Since $S_N$ is diagonal,
\begin{align*}
    \|S_N\|
    =
    \|S_N\|_{S_\infty}
    =
    \max_{1\le j\le 2^N} |(S_N)_{jj}|
    =
    \max_{\omega\in\Omega}\left|\sum_{i=1}^N \eps_i\omega_i\right|.
\end{align*}
For any fixed realization of $\eps\in\{\pm1\}^N$, choosing $\omega=\eps$ gives
\begin{align*}
    \sum_{i=1}^N \eps_i\omega_i
    =
    \sum_{i=1}^N \eps_i^2
    =
    N,
\end{align*}
so $\|S_N\|\ge N$. On the other hand, for every $\omega\in\Omega$,
\begin{align*}
    \left|\sum_{i=1}^N \eps_i\omega_i\right|
    \le
    \sum_{i=1}^N |\eps_i\omega_i|
    =
    N,
\end{align*}
so $\|S_N\|\le N$. Therefore,
\begin{align*}
    \|S_N\| = N
    \qquad \text{deterministically, and hence} \qquad
    \Exp{\|S_N\|^2} = N^2.
\end{align*}

Applying \Cref{lem:second_moment_kappa_regular} gives
\begin{align*}
    N^2
    =
    \Exp{\|S_N\|^2}
    \le
    \kappa \sum_{i=1}^N \Exp{\|\xi_i\|^2}
    =
    \kappa N,
\end{align*}
so
\begin{align*}
    \kappa \ge N.
\end{align*}
Since
\begin{align*}
    N = \lfloor \log_2 r\rfloor \ge \frac12\ln r
\end{align*}
for $r\ge 8$, this yields
\begin{align*}
    \kappa \ge \frac12\ln r,
\end{align*}
which is stronger than \eqref{eq:schatten_lower_p_ge_2} in the case $p=\infty$, since
\begin{equation*}
    \min\{p-1, \ln r\} = \ln r
    \qquad \text{when } p=\infty.\qedhere
\end{equation*}
\end{enumerate}
\end{proof}

\begin{remark}[Uniform logarithmic orders]
\label{rem:uniform_regularity_orders}
For $2\le p\le\infty$, the optimal regularity parameters of $\ell_p^d$ and Schatten-$p$ spaces have orders
\begin{align*}
    \Theta\bigl(\min\{p, \ln(ed)\}\bigr)
    \qquad\text{and}\qquad
    \Theta\bigl(\min\{p, \ln(er)\}\bigr),
\end{align*}
respectively, with absolute constants, for every $d\ge1$ and $r\ge1$.
Indeed, for $D\ge8$, where $D=d$ or $D=r$, the preceding upper and lower bounds give these orders because $p-1$ is comparable to $p$ for $p\ge2$ and $\ln(eD)=1+\ln D$ is comparable to $\ln D$.
For $1\le d<8$, taking $\rho=2$ in \Cref{example:2} gives $\kappa_p(d)\le d^{1-2/p}\le7$.
For $1\le r<8$, taking $\rho=2$ in \Cref{example:3} gives $\kappa_p(m, n)\le2r^{1-2/p}\le14$.
Together with the universal lower bound $\kappa\ge1$ and $1\le\min\{p, \ln(eD)\}\le\ln(7e)$ for $1\le D<8$, these bounds cover the remaining dimensions.
The uniform notation for orders does not change the exact minimizers $\rho_*$ or the estimates with explicit constants stated under dimension restrictions.
\end{remark}

\clearpage

\section{Analysis of \texorpdfstring{\algname[]{SGD}}{SGD} Under Dual-Norm Light-Tail Assumption}

This section develops the analysis under \Cref{as:non_euclidean_light_tail}, from minibatch concentration to oracle and wall-clock complexity.
In \Cref{app:proof_key_concentration}, we prove \Cref{lemma:non_euclidean_large_deviations} following the martingale large-deviation approach of \citet{juditsky2008large}.
Regularity provides a smooth auxiliary norm on $E^*$ in which we control the minibatch noise; comparison with $\snorm{\cdot}$ then yields concentration in the induced dual norm.

In \Cref{sub:proof_of_the_theorem_stochastic_l_smooth_final}, we apply this concentration bound conditionally at each iterate and use a union bound to control the minibatch error simultaneously over the $T$ iterations.
Combining this error bound with descent under non-Euclidean smoothness gives the high-probability convergence guarantee and the oracle complexity.
As in \Cref{app:notation} and \Cref{cor:complexity_dual_native_light_tail}, $N_*=BT$ counts the stochastic-gradient oracle calls of serial minibatch \algname{SGD}: $B$ calls at each of the $T$ iterations.
We use $\widetilde{\cO}(\cdot)$ with the convention of \Cref{app:notation}, suppressing logarithmic factors in both $T$ and $\nfr{1}{\delta}$, where $\delta\in(0, 1)$ is the failure probability.

Finally, \Cref{sec:non_eu_updaten} combines the same iteration guarantee with a bound on minibatch collection time to establish the wall-clock complexity of \algname{Rennala NSGD} under fixed worker computation times.


\subsection{Proof of the Key Concentration Inequality}
\label{app:proof_key_concentration}

\textbf{How $\kappa$-regularity enables minibatch concentration.}
For independent centered noise with finite second moments, the Euclidean second-moment identity
$\Exp{\sqsnorm{\bar{\eta}}}=\frac{1}{B^2}\sum_{j=1}^{B}\Exp{\sqsnorm{\eta^j}}$
follows by expanding the squared norm and cancelling the cross terms. This identity need not hold for a norm not induced by an inner product. Even then, norm equivalence gives $\cO(1/B)$ second-moment decay in any fixed finite-dimensional space, with a geometry-dependent constant; $\kappa$-regularity bounds this constant by $\kappa$ (\Cref{lem:second_moment_kappa_regular}).

The high-probability argument additionally controls exponential moments of a potential built from the squared norm. Smoothness gives a quadratic upper bound on this potential (\Cref{lemma:V_beta_descent}), and conditional centering removes the linear term associated with each martingale increment. However, the $\ell_1$, $\ell_\infty$, nuclear, and spectral norms are generally nonsmooth away from the origin, apart from low-dimensional cases in which they are Euclidean. Their squared norms therefore need not provide the required bound.

$\kappa$-regularity resolves this difficulty by supplying a $\kappa_+$-smooth auxiliary norm and the comparison \eqref{eq:reg_kappa_def}. Together, these allow the martingale argument to yield concentration in the original dual norm, with the geometric dependence captured by $\kappa$. A suitably scaled Euclidean norm is always an admissible $1$-smooth
choice, but other auxiliary norms can give significantly smaller
regularity parameters. The regimes for dual $\ell_p$ and Schatten-$p$ norms are described in \Cref{subsec:regularity_vs_norm_equivalence_main}, with orders valid in all positive dimensions summarized in \Cref{rem:uniform_regularity_orders}. As discussed there, improvement is not uniform across $p$, and a smaller regularity parameter alone does not guarantee a smaller oracle-complexity bound.

We now collect the auxiliary lemmas needed to prove \Cref{lemma:non_euclidean_large_deviations}.

\begin{lemma}[Fenchel Conjugate of the Restricted Squared Norm]
\label{lemma:fenchel_of_V}
Let $\widehat{\phi}: E \to \R \cup \{+\infty\}$ be the squared norm function restricted to the primal unit ball $\mathbb{B} = \{x \in E \mid \norm{x} \le 1\}$:
\begin{align*}
    \widehat{\phi}(x) = \begin{cases}
        \frac{1}{2}\sqnorm{x}, & \text{if } \norm{x} \le 1 \\
        +\infty, & \text{otherwise}
    \end{cases}.
\end{align*}
Its Fenchel conjugate, $(\widehat{\phi})^*(s) = \sup_{x \in E} \{\inner{s}{x} - \widehat{\phi}(x)\}$, is precisely the potential function $V(s)$:
\begin{align*}
    (\widehat{\phi})^*(s) = V(s) \eqdef \begin{cases}
        \frac{1}{2}\sqsnorm{s}, & \text{if } \snorm{s} \le 1 \\
        \snorm{s} - \frac{1}{2}, & \text{if } \snorm{s} > 1
    \end{cases}.
\end{align*}
\end{lemma}
\begin{proof}
Since $E$ is finite-dimensional, the primal unit ball $\mathbb{B}$ is compact. The function
$x\mapsto\inner{s}{x}-\frac12\sqnorm{x}$ is continuous, so
\begin{align*}
    (\widehat{\phi})^*(s)
    = \max_{x\in \mathbb{B}}\left\{\inner{s}{x}-\frac12\sqnorm{x}\right\}.
\end{align*}
If $s=0$, this maximum is zero, attained at $x=0$, and it equals
$\max_{0\le r\le1}\{-r^2/2\}=0$.

Now suppose $s\ne0$. For every $x\in \mathbb{B}$, the duality inequality gives
\begin{align*}
    \inner{s}{x}-\frac12\sqnorm{x}
    \le \snorm{s}\norm{x}-\frac12\sqnorm{x}
    \le \max_{0\le r\le1}\left\{r\snorm{s}-\frac{r^2}{2}\right\}.
\end{align*}
The continuous linear functional $u\mapsto\inner{s}{u}$ attains its maximum
$\snorm{s}>0$ on $\mathbb{B}$. Any maximizer $u$ has $\norm{u}=1$: otherwise, rescaling
$u$ to the unit sphere would increase the pairing and contradict maximality.
For every $r\in[0, 1]$, the point $x=ru$ therefore belongs to $\mathbb{B}$ and satisfies
\begin{align*}
    \inner{s}{ru}-\frac12\sqnorm{ru}
    = r\snorm{s}-\frac{r^2}{2}.
\end{align*}
This proves the reverse bound. Together with the case $s=0$, we obtain
\begin{align*}
    (\widehat{\phi})^*(s)
    = \max_{0\le r\le1}\left\{r\snorm{s}-\frac{r^2}{2}\right\}.
\end{align*}
Completing the square shows that the scalar maximum is attained at
$r=\min\{\snorm{s}, 1\}$. Hence
\begin{align*}
    (\widehat{\phi})^*(s)
    = \begin{cases}
        \frac12\sqsnorm{s}, & \text{if } \snorm{s}\le1, \\
        \snorm{s}-\frac12, & \text{if } \snorm{s}>1,
    \end{cases}
    = V(s).
\end{align*}
At $\snorm{s}=1$, both expressions equal $\frac12$, attained at $r=1$.
\end{proof}


\begin{lemma}[Descent Property of the Potential Function $V_\beta$]
\label{lemma:V_beta_descent}
Let the space $(E^*, \snorm{\cdot})$ be $\kappa$-smooth. 
Define the function $V_\beta: E^* \to \R$ as:
\begin{align*}
    V_\beta(s) \eqdef \beta V(s/\beta), \quad \text{where} \quad V(s) = \begin{cases}
        \frac{1}{2}\sqsnorm{s}, & \text{if } \snorm{s} \le 1 \\
        \snorm{s} - \frac{1}{2}, & \text{if } \snorm{s} > 1
    \end{cases}.
\end{align*}
Then for any $\beta > 0$, the potential function $V_\beta(s)$ is continuously differentiable and satisfies the inequality:
\begin{align*}
    V_\beta(s+\eta) \le V_\beta(s) + \inner{\nabla V_\beta(s)}{\eta} + \frac{\kappa}{2\beta}\sqsnorm{\eta}, \quad \forall s, \eta \in E^*.
\end{align*}
\end{lemma}

\begin{proof}
The proof strategy is to first establish that the gradient of the base potential function, $\nabla V(s)$, is $\kappa$-Lipschitz continuous. The desired descent inequality is then a direct consequence of this property. We adapt the proof technique from the appendix of \citet{juditsky2008large}.

\paragraph{Step 1: Establishing the $\kappa$-Lipschitz Continuity of $\nabla V(s)$.}
Let $\phi(x)\eqdef\frac12\sqnorm{x}$, whose Fenchel conjugate is $\phi^*(s)=\frac12\sqsnorm{s}$.
Apply \Cref{prop:kappa_smooth_characterizations} with $V=E^*$, $\|\cdot\|_{V, +}=\snorm{\cdot}$, and $\kappa_{V, +}=\kappa$.
Under the canonical identification $V^*=E^{**}\simeq E$, the induced dual norm is $\norm{\cdot}$.
The proposition's squared-norm function is therefore $\phi^*$, and its Fenchel conjugate is $\phi$.
By implication (i) $\Rightarrow$ (vi), the assumed $\kappa$-smoothness of $(E^*, \snorm{\cdot})$ makes $\phi$ $\frac1\kappa$-strongly convex with respect to $\norm{\cdot}$.

Define
\begin{align*}
    \widehat{\phi}(x)=
    \begin{cases}
        \frac12\sqnorm{x}, & \text{if }\norm{x}\le1,\\
        +\infty, & \text{otherwise}.
    \end{cases}
\end{align*}
Since the unit ball $\mathbb{B}=\{x\in E:\norm{x}\le1\}$ is closed and convex, $\widehat{\phi}$ is proper and closed, and restricting $\phi$ to $\mathbb{B}$ preserves its $\frac1\kappa$-strong convexity.
By \Cref{lemma:fenchel_of_V}, its Fenchel conjugate is $V=(\widehat{\phi})^*$.

The conjugate of a proper closed $\mu$-strongly convex function is differentiable, with a $\frac1\mu$-Lipschitz gradient in the corresponding dual norms.
Applying this result with $\mu=\frac1\kappa$ shows that $V$ is continuously differentiable and
\begin{align*}
    \norm{\nabla V(s_1)-\nabla V(s_2)}
    \le \kappa\snorm{s_1-s_2},
    \qquad s_1, s_2\in E^*.
\end{align*}

\paragraph{Step 2: From Lipschitz Gradient to the Descent Inequality.}
The desired inequality is a standard consequence of a function having a Lipschitz continuous gradient. For any function $F$ with a $\kappa$-Lipschitz gradient, the descent lemma states:
\begin{align*}
    F(s+\eta) \le F(s) + \inner{\nabla F(s)}{\eta} + \frac{\kappa}{2}\sqsnorm{\eta}.
\end{align*}
Applying this directly to our function $V(s)$, which we have just shown to have a $\kappa$-Lipschitz gradient, gives:
\begin{align*}
    V(s+\eta) \le V(s) + \inner{\nabla V(s)}{\eta} + \frac{\kappa}{2}\sqsnorm{\eta}.
\end{align*}
This proves the lemma for the base case where $\beta=1$.

\paragraph{Step 3: Scaling the Result for $V_\beta$.}
The final step is to extend the result to the scaled function $V_\beta(s) = \beta V(\nfr{s}{\beta})$. Using the chain rule, the gradient is $\nabla V_\beta(s) = \nabla V(\nfr{s}{\beta})$.

We start with the inequality proven in Step 2, but apply it to the points $\nfr{s}{\beta}$ and $\nfr{\eta}{\beta}$:
\begin{align*}
    V(s/\beta + \eta/\beta) \le V(s/\beta) + \inner{\nabla V(s/\beta)}{\eta/\beta} + \frac{\kappa}{2}\sqsnorm{\eta/\beta}.
\end{align*}
Multiplying the entire inequality by the positive scalar $\beta$:
\begin{align*}
    \beta V\left(\frac{s+\eta}{\beta}\right) &\le \beta V(s/\beta) + \beta \inner{\nabla V(s/\beta)}{\eta/\beta} + \beta \frac{\kappa}{2\beta^2}\sqsnorm{\eta} \\
    V_\beta(s+\eta) &\le V_\beta(s) + \inner{\nabla V(s/\beta)}{\eta} + \frac{\kappa}{2\beta}\sqsnorm{\eta}.
\end{align*}
Finally, substituting the identity for the gradient, $\nabla V(\nfr{s}{\beta}) = \nabla V_\beta(s)$, we arrive at the desired conclusion:
\begin{align*}
    V_\beta(s+\eta) \le V_\beta(s) + \inner{\nabla V_\beta(s)}{\eta} + \frac{\kappa}{2\beta}\sqsnorm{\eta}.
\end{align*}
This completes the proof.
\end{proof}

\begin{lemma}[Scalar Concentration for Sub-Gaussian Martingales \citep{juditsky2008large}]
\label{lemma:scalar_concentration}
Let $\{\psi_j\}_{j=1}^B$ be a sequence of real-valued random variables adapted to a filtration $\{\cF_j\}_{j=0}^B$. Assume that for constants $\mu_j \ge 0$ and $\nu_j > 0$, the following conditions hold almost surely for all $j=1, \dots, B$:
\begin{enumerate}[label=(\roman*)]
    \item $\Exp{\psi_j | \cF_{j-1}} \le \mu_j$.
    \item $\Exp{\exp\rb{\frac{\psi_j^2}{\nu_j^2}} | \cF_{j-1}} \le \exp(1)$.
\end{enumerate}
Then for any $\lambda \ge 0$, the sum $\Psi_B = \sum_{j=1}^B \psi_j$ satisfies:
\begin{align*}
    \Prob\cb{\Psi_B > \sum_{j=1}^B \mu_j + \lambda\sqrt{\sum_{j=1}^B \nu_j^2}} \le \exp(-\lambda^2/3).
\end{align*}
\end{lemma}

\begin{lemma}[Lower Bound on the Potential Function]
\label{lemma:potential_function_bound}
For any $\beta > 0$, the following inequality gives a lower bound on the potential $V_\beta(s)$:
\begin{align*}
    \snorm{s} \le \frac{\beta}{2} + V_\beta(s), \quad \forall s \in E^*.
\end{align*}
\end{lemma}
\begin{proof}
The proof is most easily done in two steps. First, we prove the inequality for the base function $V(s)$ (i.e., for $\beta=1$). Second, we use the scaling property of $V_\beta(s)$ to generalize the result.

\paragraph{Step 1: Proof for the Base Case ($\beta=1$).}
Our goal is to prove that for any $s \in E^*$, the following holds:
\begin{align*}
    \snorm{s} \le \frac{1}{2} + V(s).
\end{align*}
We analyze this by considering the two cases from the piecewise definition of $V(s)$.

\begin{itemize}
    \item \textbf{Case 1: $\snorm{s} \le 1$.}\\
    In this case, the definition of $V(s)$ is $V(s) = \frac{1}{2}\sqsnorm{s}$. We need to show that:
    \begin{align*}
        \snorm{s} \le \frac{1}{2} + \frac{1}{2}\sqsnorm{s}.
    \end{align*}
    Let $t = \snorm{s}$. The inequality becomes $t \le \frac{1}{2} + \frac{1}{2}t^2$. We can rearrange this into a simple quadratic inequality:
    \begin{align*}
        \frac{1}{2}t^2 - t + \frac{1}{2} &\ge 0 \\
        t^2 - 2t + 1 &\ge 0 \\
        (t-1)^2 &\ge 0.
    \end{align*}
    A squared real number is always non-negative, so this inequality is always true. Thus, the property holds for this case.

    \item \textbf{Case 2: $\snorm{s} > 1$.}\\
    In this case, the definition of $V(s)$ is $V(s) = \snorm{s} - \frac{1}{2}$. We need to show that:
    \begin{align*}
        \snorm{s} \le \frac{1}{2} + \left(\snorm{s} - \frac{1}{2}\right).
    \end{align*}
    The right-hand side simplifies to $\snorm{s}$. The inequality becomes:
    \begin{align*}
        \snorm{s} \le \snorm{s}.
    \end{align*}
    This holds with equality. Thus, the property also holds for this case.
\end{itemize}
Since the inequality holds in both cases, we have proven that $\snorm{s} \le \frac{1}{2} + V(s)$ for all $s \in E^*$.

\paragraph{Step 2: Scaling the Result to $V_\beta(s)$.}
We now use the result from Step 1, which is valid for any vector in $E^*$. Let us apply it to the specific vector $\nfr{s}{\beta}$:
\begin{align*}
    \snorm{s/\beta} \le \frac{1}{2} + V(s/\beta).
\end{align*}
By the absolute homogeneity property of norms, $\snorm{\nfr{s}{\beta}} = \frac{1}{\beta}\snorm{s}$. Substituting this in gives:
\begin{align*}
    \frac{1}{\beta}\snorm{s} \le \frac{1}{2} + V(s/\beta).
\end{align*}
Since $\beta > 0$, we can multiply the entire inequality by $\beta$ without changing its direction:
\begin{align*}
    \snorm{s} \le \frac{\beta}{2} + \beta V(s/\beta).
\end{align*}
By definition, $V_\beta(s) = \beta V(\nfr{s}{\beta})$. Substituting this into the right-hand side gives the final result:
\begin{align*}
    \snorm{s} \le \frac{\beta}{2} + V_\beta(s).
\end{align*}
This completes the proof.
\end{proof}

\lightTailConcentration*

\begin{proofof}{Lemma~\ref{lemma:non_euclidean_large_deviations}}\label{proofof:non_euclidean_large_deviations}
Let $\kappa$ be the regularity constant specified in the lemma. By \Cref{lemfact}, there exists a $\kappa$-smooth norm $\spnorm{\cdot}$ on $E^*$ such that
\begin{align}\label{eq:reg_to_smooth_equiv_dual}
    \snorm{s}^2 \le \spnorm{s}^2 \le 2\snorm{s}^2 \quad \forall s\in E^*.
\end{align}
Define $\sigma_+^2 \eqdef 2\sigma^2$. Then for every $j$,
\begin{align*}
    \frac{\spnorm{\eta^j}^2}{\sigma_+^2} \le \frac{2\snorm{\eta^j}^2}{2\sigma^2} = \frac{\snorm{\eta^j}^2}{\sigma^2},
\end{align*}
and hence by the assumed light-tail condition,
\begin{align}\label{eq:light_tail_plus_norm}
    \Exp{\exp\left(\frac{\spnorm{\eta^j}^2}{\sigma_+^2}\right)} \le \Exp{\exp\left(\frac{\snorm{\eta^j}^2}{\sigma^2}\right)} \le e.
\end{align}

\paragraph{Step 1: set up a martingale difference sequence.}
Let $\cF_j \eqdef \sigma(\eta^1, \ldots, \eta^j)$ and define the partial sums $S_0\equiv 0$, $S_j \eqdef \sum_{t=1}^j \eta^t$. Then $(S_j)_{j=0}^B$ is a martingale and $(\eta^j)_{j=1}^B$ is a martingale-difference sequence w.r.t.\ $(\cF_j)$.

\paragraph{Step 2: the potential function.}
For $\beta>0$, define $V_\beta : E^*\to\R$ by
\begin{align*}
    V_\beta(s) \eqdef \beta V(s/\beta), \qquad
    V(u) \eqdef 
    \begin{cases}
        \frac12\spnorm{u}^2, & \spnorm{u}\le 1, \\
        \spnorm{u}-\frac12, & \spnorm{u}>1.
    \end{cases}
\end{align*}
We use the following three properties:
\begin{enumerate}[label=(\alph*)]
\item $V_\beta$ is $1$-Lipschitz w.r.t.\ $\spnorm{\cdot}$. Indeed, write
\begin{align*}
    V_\beta(s)=h_\beta(\spnorm{s}), \qquad
    h_\beta(r)=
    \begin{cases}
        \frac{r^2}{2\beta}, & 0\le r\le\beta, \\
        r-\frac{\beta}{2}, & r>\beta.
    \end{cases}
\end{align*}
The two branches have matching derivatives at $r=\beta$, and
$h_\beta'(r)=\min\{r/\beta, 1\}\in[0, 1]$ for $r>0$.
Thus $h_\beta$ is $1$-Lipschitz on $[0, \infty)$. Combining this with the reverse triangle inequality gives
\begin{align}\label{eq:psi_lipschitz}
    |V_\beta(s+\eta)-V_\beta(s)|
    \le \bigl|\spnorm{s+\eta}-\spnorm{s}\bigr|
    \le \spnorm{\eta} \quad \forall s, \eta\in E^*.
\end{align}
\item By \Cref{lemma:V_beta_descent}, applied to the $\kappa$-smooth space $(E^*, \spnorm{\cdot})$, $V_\beta$ is continuously differentiable and satisfies the descent inequality
\begin{align}\label{eq:Vbeta_descent}
    V_\beta(s+\eta)\le V_\beta(s)+\inner{\nabla V_\beta(s)}{\eta}+\frac{\kappa}{2\beta}\spnorm{\eta}^2 \quad \forall s, \eta\in E^*.
\end{align}
\item By \Cref{lemma:potential_function_bound}, applied to the same auxiliary norm, for all $s\in E^*$,
\begin{align}\label{eq:Vbeta_lower_bound}
    V_\beta(s)\ge \spnorm{s}-\frac{\beta}{2} \qquad\text{equivalently}\qquad \spnorm{s}\le V_\beta(s)+\frac{\beta}{2}.
\end{align}
\end{enumerate}

\paragraph{Step 3: define increments and verify scalar conditions.}
Define
\begin{align*}
    \psi_j \eqdef V_\beta(S_j)-V_\beta(S_{j-1}), \qquad j=1, \ldots, B.
\end{align*}
First, \eqref{eq:light_tail_plus_norm} and Jensen's inequality imply the second moment bound:
\begin{align}\label{eq:second_moment_from_light_tail}
    \Exp{\frac{\spnorm{\eta^j}^2}{\sigma_+^2}} \le \log \Exp{\exp\left(\frac{\spnorm{\eta^j}^2}{\sigma_+^2}\right)} \le 1 \quad\Longrightarrow\quad \Exp{\spnorm{\eta^j}^2} \le \sigma_+^2.
\end{align}
Since $S_j$ is $\cF_j$-measurable and $V_\beta$ is continuous, $\psi_j$ is $\cF_j$-measurable.
The continuity of $\nabla V_\beta$ also makes $\nabla V_\beta(S_{j-1})$ $\cF_{j-1}$-measurable.
Differentiability and \eqref{eq:psi_lipschitz} give, for every $s, \eta\in E^*$,
\begin{align*}
    \left|\inner{\nabla V_\beta(s)}{\eta}\right|
    = \lim_{t\to0}\frac{|V_\beta(s+t\eta)-V_\beta(s)|}{|t|}
    \le \spnorm{\eta}.
\end{align*}
Equivalently, $\nabla V_\beta(s)$ has norm at most $1$ in the norm dual to $\spnorm{\cdot}$.
Consequently, by Cauchy--Schwarz and \eqref{eq:second_moment_from_light_tail},
\begin{align*}
    \Exp{\left|\inner{\nabla V_\beta(S_{j-1})}{\eta^j}\right|}
    \le \Exp{\spnorm{\eta^j}}
    \le \sqrt{\Exp{\spnorm{\eta^j}^2}}
    \le \sigma_+<\infty.
\end{align*}
Also, \eqref{eq:psi_lipschitz} implies $\Exp{|\psi_j|}\le\Exp{\spnorm{\eta^j}}\le\sigma_+$.
Thus the increment, the mixed term, and the squared-noise term are all integrable.
We may therefore take conditional expectations in \eqref{eq:Vbeta_descent} and pull out the $\cF_{j-1}$-measurable gradient.
Since $\eta^j$ is independent of $\cF_{j-1}$ and has mean zero, the mixed term has conditional mean zero, and we obtain
\begin{align*}
    \Exp{\psi_j\mid \cF_{j-1}} &= \Exp{V_\beta(S_{j-1}+\eta^j)-V_\beta(S_{j-1})\mid \cF_{j-1}}\\
    &\le \Exp{\inner{\nabla V_\beta(S_{j-1})}{\eta^j} +\frac{\kappa}{2\beta}\spnorm{\eta^j}^2 \mid \cF_{j-1}}\\
    &= \frac{\kappa}{2\beta}\Exp{\spnorm{\eta^j}^2}.
\end{align*}
Therefore,
\begin{align}\label{eq:mu_bound}
    \Exp{\psi_j\mid \cF_{j-1}} \le \frac{\kappa\sigma_+^2}{2\beta} \qquad\text{for all } j.
\end{align}
Also, by \eqref{eq:psi_lipschitz} and \eqref{eq:light_tail_plus_norm},
\begin{align}\label{eq:psi_subgaussian}
    \Exp{\exp\left(\frac{\psi_j^2}{\sigma_+^2}\right)\mid\cF_{j-1}} \le \Exp{\exp\left(\frac{\spnorm{\eta^j}^2}{\sigma_+^2}\right)} \le e.
\end{align}

\paragraph{Step 4: apply scalar concentration to $\sum_{j=1}^B\psi_j=V_\beta(S_B)$.}
Although \Cref{lemma:scalar_concentration} is stated with a strict tail event, the same bound holds with $\ge$.
For $\lambda>0$, apply the lemma with $0\le\lambda_n<\lambda$ and $\lambda_n\uparrow\lambda$.
The strict tail events decrease to the non-strict event at $\lambda$, so continuity of probability from above and
$\exp(-\lambda_n^2/3)\to\exp(-\lambda^2/3)$ yield the claim.
For $\lambda=0$, the bound is trivial because its right-hand side is $1$.
No continuity assumption on the distribution is needed.

Applying \Cref{lemma:scalar_concentration} (the sub-Gaussian case) with $\mu_j=\nfr{\kappa\sigma_+^2}{2\beta}$ and $\nu_j=\sigma_+$ yields that for all $\lambda\ge 0$,
\begin{align*}
    \Prob\cb{V_\beta(S_B) \ge \frac{B\kappa\sigma_+^2}{2\beta}+\lambda\sqrt{B}\,\sigma_+} \le \exp(-\lambda^2/3).
\end{align*}
Using \eqref{eq:Vbeta_lower_bound}, we obtain
\begin{align*}
    \Prob\cb{\spnorm{S_B} \ge \frac{\beta}{2}+\frac{B\kappa\sigma_+^2}{2\beta}+\lambda\sqrt{B}\,\sigma_+} \le \exp(-\lambda^2/3).
\end{align*}
Optimizing the deterministic term in $\beta$ gives $\beta^\star=\sqrt{B\kappa}\,\sigma_+$ and thus
\begin{align*}
    \Prob\cb{\spnorm{S_B} \ge (\sqrt{\kappa}+\lambda)\sqrt{B}\,\sigma_+} \le \exp(-\lambda^2/3).
\end{align*}

\paragraph{Step 5: return to $\snorm{\cdot}$ and to the average.}
Since $\snorm{S_B}\le \spnorm{S_B}$ and $\sigma_+=\sqrt{2}\sigma$,
\begin{align*}
    \Prob\cb{\snorm{S_B} \ge (\sqrt{2\kappa}+\sqrt{2}\lambda)\sqrt{B}\,\sigma} \le \exp(-\lambda^2/3).
\end{align*}
Finally, $\bar\eta=\nfr{S_B}{B}$ and norms are homogeneous, so $\snorm{\bar\eta} = \nfr{\snorm{S_B}}{B}$. Dividing the threshold by $B$ yields
\begin{align*}
    \Prob\cb{\snorm{\bar\eta} \ge \left(\sqrt{2\kappa}+\sqrt{2}\lambda\right)\frac{\sigma}{\sqrt{B}}} \le \exp(-\lambda^2/3).
\end{align*}
This completes the proof.
\end{proofof}


\subsection{Proof of the Convergence Rate Theorem}\label{sub:proof_of_the_theorem_stochastic_l_smooth_final}

\lightTailConvergence*

\begin{proof}[Proof of \Cref{thm:stochastic_L_smooth_final}]
Fix the initial point $x^0$ and define the filtration
\begin{align*}
    \cF_k \eqdef \sigma\!\left(x^0, \xi_j^t:\ 0\le t<k, \ 1\le j\le B\right), \qquad k\ge0.
\end{align*}
By induction from the update \eqref{eq:non-eu-one-step-final} and the fixed Borel-measurable sharp selection in \Cref{subsec:non_eu_update}, $\xk$ is $\cF_k$-measurable.
As specified in the serial algorithm, the new batch $\{\xi_j^k\}_{j=1}^B$ consists of i.i.d.\ samples from $\cD$ and is independent of $\cF_k$.
Let $\etajk \eqdef \nf{\xk; \xijk} - \nfxk$ be the noise of the $j$-th sample at iteration $k$.
Conditionally on $\cF_k$, the query $\xk$ is fixed, so these noise vectors are independent and identically distributed. By \Cref{as:non_euclidean_light_tail}, they satisfy, almost surely,
\begin{align*}
    \Exp{\etajk\mid\cF_k}=0,
    \qquad
    \Exp{\exp\left(\frac{\sqsnorm{\etajk}}{\sigma^2}\right)\mid\cF_k}\le e.
\end{align*}
The minibatch error is $\frac{1}{B}\sum_{j=1}^B \etajk = \gk - \nfxk$.
Thus \Cref{lemma:non_euclidean_large_deviations} applies conditionally on $\cF_k$ and gives, for each $k$ and $\lambda\ge0$, almost surely,
\begin{align*}
    \Prob\cb{\snorm{\gk - \nfxk} \ge (\sqrt{2\kappa} + \sqrt{2}\lambda) \frac{\sigma}{\sqrt{B}} \mid \cF_k} \le \exp(-\lambda^2/3).
\end{align*}
Let $A_k(\lambda)$ denote the event in this conditional probability. Taking expectations and using the tower property yields
\begin{align*}
    \Prob\cb{A_k(\lambda)}
    =\Exp{\Prob\cb{A_k(\lambda)\mid\cF_k}}
    \le\exp(-\lambda^2/3).
\end{align*}
We may now apply the unconditional union bound over the prescribed iterations $k\in\{0, \dots, T-1\}$; independence of the events $A_k(\lambda)$ is not needed:
\begin{align*}
    \Prob\cb{\bigcup_{k=0}^{T-1} A_k(\lambda)}
    \le \sum_{k=0}^{T-1}\Prob\cb{A_k(\lambda)}
    \le \sum_{k=0}^{T-1}\exp(-\lambda^2/3)
    = T\exp(-\lambda^2/3).
\end{align*}
We want this failure probability to be at most $\delta$. By setting $T \exp(-\nfr{\lambda^2}{3}) \le \delta$, we find that choosing $\lambda = \sqrt{3\log(\nfr{T}{\delta})}$ suffices. With this choice, we can conclude that with probability at least $1-\delta$, the following holds for all $k \in \{0, \dots, T-1\}$:
\begin{align}
    \label{eq:high_prob_bound_intermediate}
    \snorm{\gk - \nfxk} \le (\sqrt{2\kappa} + \sqrt{2}\lambda) \frac{\sigma}{\sqrt{B}}.
\end{align}
We choose the error threshold $\frac{\sqrt{\eps}}{4}$ and select the batch size $B$ so that the right-hand side of \eqref{eq:high_prob_bound_intermediate} is at most this threshold. Thus, we require:
\begin{align*}
    (\sqrt{2\kappa} + \sqrt{2}\sqrt{3\log(T/\delta)}) \frac{\sigma}{\sqrt{B}} \le \frac{\sqrt{\eps}}{4}.
\end{align*}
Solving for $B$ yields the choice:
\begin{align*}
    B = \max\left\{ \left\lceil \frac{16\sigma^2 \rb{\sqrt{2\kappa} + \sqrt{6\log(T/\delta)}}^2}{\eps} \right\rceil, 1 \right\}.
\end{align*}
With these choices of $\lambda$ and $B$, we conclude that with probability at least $1-\delta$, we have for all $k \in \{0, \dots, T-1\}$:
\begin{align}
    \label{eq:JqcRXoaKBYDLWQq_non_euclidean}
    \snorm{\gk - \nfxk} \leq \frac{\sqrt{\eps}}{4}.
\end{align}

Fix any realization on which \eqref{eq:JqcRXoaKBYDLWQq_non_euclidean} holds for every $0\le k<T$; these realizations form an event of probability at least $1-\delta$. All subsequent descent inequalities are evaluated along this fixed realization. We proceed by contradiction to show that there must exist some $k \in \{0, \dots, T - 1\}$ such that $\sqsnorm{\nfxk} \leq \eps$.

Assume the contrary, i.e., that for all $k \in \{0, \dots, T - 1\}$:
\begin{align}
    \label{eq:contradiction_assumption}
    \sqsnorm{\nfxk} > \eps.
\end{align}
This implies $\snorm{\nfxk} > \sqrt{\eps}$. Combining this with the error bound from \eqref{eq:JqcRXoaKBYDLWQq_non_euclidean}, we get:
\begin{align}
    \label{eq:grad_error_vs_norm}
    \snorm{\gk - \nfxk} \leq \frac{\sqrt{\eps}}{4} < \frac{\snorm{\nfxk}}{4}.
\end{align}
By the triangle inequality for the dual norm, we can bound the norm of the stochastic gradient $\gk$:
\begin{align}
    \label{eq:stoch_grad_norm_bounds}
    \frac{3}{4} \snorm{\nfxk} \leq \snorm{\gk} \leq \frac{5}{4}\snorm{\nfxk}.
\end{align}
We now have all the necessary components to analyze the one-step progress of the algorithm. The proof proceeds by combining the standard descent lemma for $L$-smooth functions with the high-probability bounds on the stochastic gradient error.

Let us analyze the descent from iterate $\xk$ to $\xkpo$. Since the function $f$ is $L$-smooth, we have:
\begin{align*}
    f(\xkpo) &\le f(\xk) + \inner{\nfxk}{\xkpo - \xk} + \frac{L}{2}\sqnorm{\xkpo - \xk} \\
    &= f(\xk) - \gamma_k \inner{\nfxk}{(\gk)^\sharp} + \frac{L\gamma_k^2}{2}\sqnorm{(\gk)^\sharp} \\
    &= f(\xk) - \gamma_k \inner{\nfxk}{(\gk)^\sharp} + \frac{L\gamma_k^2}{2}\sqsnorm{\gk}.
\end{align*}
In the last step, we used the fundamental property that $\sqnorm{s^\sharp} = \sqsnorm{s}$. We now rewrite the dual pairing term by adding and subtracting $\gk$:
\begin{align*}
    -\inner{\nfxk}{(\gk)^\sharp} = \inner{\gk - \nfxk}{(\gk)^\sharp} - \inner{\gk}{(\gk)^\sharp}.
\end{align*}
Using the identity $\inner{s}{s^\sharp} = \sqsnorm{s}$, this becomes:
\begin{align*}
    -\inner{\nfxk}{(\gk)^\sharp} = \inner{\gk - \nfxk}{(\gk)^\sharp} - \sqsnorm{\gk}.
\end{align*}
Substituting this back into the descent inequality, we get:
\begin{align*}
    f(\xkpo) &\le f(\xk) + \gamma_k \inner{\gk - \nfxk}{(\gk)^\sharp} - \gamma_k \sqsnorm{\gk} + \frac{L\gamma_k^2}{2}\sqsnorm{\gk} \\
    &= f(\xk) + \gamma_k \inner{\gk - \nfxk}{(\gk)^\sharp} - \gamma_k(1 - L\gamma_k/2)\sqsnorm{\gk}.
\end{align*}
We now apply the generalized Young's inequality, $\inner{u}{v} \le \frac{1}{2{s}}\snorm{u}^2 + \frac{{s}}{2}\norm{v}^2$, to the dual pairing term with $u = \gk - \nfxk$, $v = (\gk)^\sharp$, and a parameter ${s} > 0$:
\begin{align*}
    \inner{\gk - \nfxk}{(\gk)^\sharp} &\le \frac{1}{2{s}}\sqsnorm{\gk - \nfxk} + \frac{{s}}{2}\sqnorm{(\gk)^\sharp} \\
    &= \frac{1}{2{s}}\sqsnorm{\gk - \nfxk} + \frac{{s}}{2}\sqsnorm{\gk}.
\end{align*}
Plugging this bound into our main inequality gives:
\begin{align*}
    f(\xkpo) &\le f(\xk) + \gamma_k \left( \frac{1}{2{s}}\sqsnorm{\gk - \nfxk} + \frac{{s}}{2}\sqsnorm{\gk} \right) - \gamma_k(1 - L\gamma_k/2)\sqsnorm{\gk} \\
    &= f(\xk) + \frac{\gamma_k}{2{s}}\sqsnorm{\gk - \nfxk} - \gamma_k \left(1 - \frac{L\gamma_k}{2} - \frac{{s}}{2}\right)\sqsnorm{\gk}.
\end{align*}
This inequality holds for any ${s} > 0$. To substitute the lower bound $\sqsnorm{\gk}\ge\frac{9}{16}\sqsnorm{\nfxk}$ into its negative term, we require
\begin{align*}
    1-\frac{L\gamma_k}{2}-\frac{s}{2}\ge0.
\end{align*}
With the prescribed stepsize $\gamma_k=\frac{1}{2L}$, this condition is $\frac34-\frac{s}{2}\ge0$. Hence we restrict the Young parameter to $0<s\le\frac32$ from now on. On this interval, the coefficient multiplying $\sqsnorm{\gk}$ is nonpositive, so we may use the high-probability bounds from \eqref{eq:grad_error_vs_norm} and \eqref{eq:stoch_grad_norm_bounds} to obtain
\begin{align*}
    f(\xkpo) &\le f(\xk) + \frac{\gamma_k}{2{s}} \left(\frac{1}{4}\snorm{\nfxk}\right)^2 - \gamma_k \left(1 - \frac{L\gamma_k}{2} - \frac{{s}}{2}\right)\left(\frac{3}{4}\snorm{\nfxk}\right)^2 \\
    &= f(\xk) + \frac{\gamma_k}{32{s}}\sqsnorm{\nfxk} - \frac{9\gamma_k}{16} \left(1 - \frac{L\gamma_k}{2} - \frac{{s}}{2}\right)\sqsnorm{\nfxk} \\
    &= f(\xk) - \gamma_k \left[ \frac{9}{16}\left(1 - \frac{L\gamma_k}{2} - \frac{{s}}{2}\right) - \frac{1}{32{s}} \right] \sqsnorm{\nfxk}.
\end{align*}
To ensure descent, the step size $\gamma_k$ and the parameter ${s}$ must make the term in the square brackets positive. Using the constant step size $\gamma_k\equiv\gamma=\frac{1}{2L}$ prescribed in the theorem, the expression becomes:
\begin{align*}
    f(\xkpo) \le f(\xk) - \frac{1}{2L} \left[ \frac{9}{16}\left(1 - \frac{1}{4} - \frac{{s}}{2}\right) - \frac{1}{32{s}} \right] \sqsnorm{\nfxk}.
\end{align*}
To obtain the strongest possible descent guarantee from this inequality, we maximize the term in the square brackets over the admissible interval $0<s\le\frac32$. Define
\begin{align*}
    h({s}) \eqdef \frac{9}{16}\left(1 - \frac{1}{4} - \frac{{s}}{2}\right) - \frac{1}{32{s}}.
\end{align*}
Its derivative is
\begin{align*}
    h'({s}) = -\frac{9}{32} + \frac{1}{32{s}^2}.
\end{align*}
Setting $h'({s}) = 0$ gives $9{s}^2 = 1$, whose positive solution is ${s}^* = \frac{1}{3}$. Since $h''(s)=-\frac{1}{16s^3}<0$ for $s>0$, this is the unique maximizer on the admissible interval $0<s\le\frac32$. In particular, $s^*=\frac13$ is admissible and gives $1-\frac{L\gamma_k}{2}-\frac{s^*}{2}=\frac{7}{12}>0$.
We now substitute this optimal value ${s}^* = \nfr{1}{3}$ back into our one-step descent inequality:
\begin{align*}
    h(1/3) = \frac{9}{16}\left(1 - \frac{1}{4} - \frac{1}{6}\right) - \frac{3}{32} = \frac{15}{64}.
\end{align*}
Substituting this optimized constant back into the full inequality gives our final one-step progress guarantee:
\begin{align*}
    f(\xkpo) \le f(\xk) - \frac{1}{2L} \frac{15}{64} \sqsnorm{\nfxk}.
\end{align*}
Summing this inequality over $k=0, \dots, T-1$ along the fixed realization and using \eqref{eq:contradiction_assumption}, we obtain
\begin{align*}
    f(x^T)
    &\le f(x^0)-\frac{15}{128L}\sum_{k=0}^{T-1}\sqsnorm{\nfxk} \\
    &< f(x^0)-\frac{15T\eps}{128L} \\
    &\le f(x^0)-\Delta=\finf.
\end{align*}
The last inequality uses the prescribed iteration count $T=\left\lceil\frac{128L\Delta}{15\eps}\right\rceil$ and $\Delta=f(x^0)-\finf$.
This contradicts the lower bound $f(x^T)\ge\finf$ from \Cref{as:lower_bound}.
Thus every realization in the good event contains an iterate with $\sqsnorm{\nfxk}\le\eps$ for some $0\le k<T$.
Since this event has probability at least $1-\delta$, we conclude that
\begin{align*}
    \Prob\cb{\min_{0\le k<T}\sqsnorm{\nfxk}\le\eps}\ge1-\delta.
\end{align*}
This completes the proof.
\end{proof}

\lightTailComplexity*

\begin{proof}[Proof of \Cref{cor:complexity_dual_native_light_tail}]
By \Cref{thm:stochastic_L_smooth_final}, the required
number of iterations is
\begin{align*}
    T
    =
    \left\lceil
        \frac{128L\Delta}{15\eps}
    \right\rceil.
\end{align*}
Since $\eps\le L\Delta$, this implies
\begin{align*}
    T
    =
    \mathcal{O}\left(\frac{L\Delta}{\eps}\right).
\end{align*}

The same theorem prescribes the minibatch size
\begin{align*}
    B
    =
    \max\left\{
        \left\lceil
            \frac{
                16\sigma^2
                \left(
                    \sqrt{2\kappa}
                    +
                    \sqrt{6\log(T/\delta)}
                \right)^2
            }{\eps}
        \right\rceil,
        1
    \right\}.
\end{align*}
Using $(a+b)^2\le 2a^2+2b^2$, we have
\begin{align*}
    \left(
        \sqrt{2\kappa}
        +
        \sqrt{6\log(T/\delta)}
    \right)^2
    \le
    4\kappa+12\log(T/\delta).
\end{align*}
Consequently,
\begin{align*}
    B
    =
    \widetilde{\mathcal{O}}\left(
        1+\frac{\kappa\sigma^2}{\eps}
    \right),
\end{align*}
where the suppressed factors are logarithmic in $T$ and $\nfr{1}{\delta}$.
Therefore,
\begin{align*}
    N_*
    =BT
    &=
    \widetilde{\mathcal{O}}\left(
        \frac{L\Delta}{\eps}
        \left(
            1+\frac{\kappa\sigma^2}{\eps}
        \right)
    \right) \\
    &=
    \widetilde{\mathcal{O}}\left(
        \frac{L\Delta}{\eps}
        +
        \frac{\kappa L\Delta\sigma^2}{\eps^2}
    \right),
\end{align*}
which proves \eqref{eq:complexity_dual_native_cor}. The stated success
probability and stationarity guarantee follow directly from
\Cref{thm:stochastic_L_smooth_final}.
\end{proof}

\subsection{Time-Complexity Analysis of \texorpdfstring{\algname[]{Rennala NSGD}}{Rennala NSGD} Under the Light-Tailed Assumption}\label{sec:non_eu_updaten}

We explain the role of fresh gradients and the modified reply order in \Cref{sec:rennala_nsgd}, then prove the wall-clock guarantee under the fixed computation model.

\textbf{Why discard stale gradients?}
The server may update the iterate while a worker is still computing, so a gradient can be outdated when it arrives.
As discussed in \Cref{app:related_time_complexity}, a gradient computed at $x^{k-\delta_k}$ need not be an unbiased estimator of $\nabla f(\xk)$.
Moreover, to the best of our knowledge, the available wall-clock bounds for classical \algname{ASGD}, which applies every arriving gradient, are suboptimal in general relative to the lower bound of \citet{tyurin2023optimal} for their zero-respecting algorithm class and specified unbiased bounded-variance oracle model with fixed worker computation times.
\algname{Rennala SGD} avoids staleness by using only gradients computed at the current iterate; in that Euclidean oracle model, its time complexity matches the lower bound up to a constant factor~\citep{tyurin2023optimal}.
For our analysis, this restriction also connects the parallel method to the serial one: conditional on the history, the accepted gradients are i.i.d.\ oracle responses at $\xk$.
The first part of the proof below establishes this property and applies the serial convergence guarantee.

\textbf{Updating before replying.}
We slightly modify Method~4 of \citet{tyurin2023optimal}, whose server replies to the worker that completes the minibatch before updating the iterate.
Our server updates first and then sends the new iterate and its index to that worker, allowing it to start a fresh computation immediately.
Other workers finish their pending computations and receive the latest iterate when their messages are processed.
Since each message is fully processed before the next is handled, including any update and reply, a simultaneous return processed after an update is checked against the new index.
This change preserves the fresh-minibatch property and the wall-clock upper bound: a worker still computing at an earlier iterate needs at most $\tau_i$ additional time before receiving the current one, while the worker that completed the minibatch starts immediately.
The timing argument below accounts for both cases.

With these protocol details in place, we restate and prove the wall-clock guarantee.

\lightTailWallClock*

We will also prove that, for any positive integers $T$ and $B$, the time needed to complete $T$ iterations satisfies
\begin{align}\label{eq:time_exact_bound}
    T_{\mathrm{total}} \le 2T\min_{S\in[N]}\left[\left(\sum_{i=1}^S\frac{1}{\tau_i}\right)^{-1}(B+S)\right].
\end{align}
This deterministic estimate will also be used in the noise-model comparison in \Cref{app:rennala_time_discussion}.

\begin{proof}[Proof of \Cref{thm:main_time_complexity}]
The argument has two parts. We first verify that the stochastic gradients accepted during
an iteration form a fresh minibatch at the current iterate, so that the serial guarantee of
\Cref{thm:stochastic_L_smooth_final} applies verbatim; we then bound the wall-clock time
required to complete $T$ such iterations under the fixed computation model.

\paragraph{Step 1: the accepted gradients form a fresh minibatch.}
Fix an iteration $k$, and let $\cF_k$ denote the complete algorithmic history immediately
after $\xk$ is formed and before it is sent to any worker, including the states of workers
with computations or messages pending. For $k>0$, this is just before the reply to the
worker that completed the preceding minibatch; for $k=0$, it is before the initial replies.
Thus $\xk$ is $\cF_k$-measurable, and every sample used for a gradient tagged with $k$ is
drawn after this history is fixed.

The server accepts a stochastic gradient during iteration $k$ only if it is tagged with the
current iteration index $k$. Hence, every accepted gradient is computed at $\xk$. Whenever a
worker receives $(\xk, k)$, it draws a fresh sample from the common data distribution,
independently of the past and of the samples used by other oracle calls, and computes a stochastic gradient at $\xk$.

Under the fixed computation model, the order of worker returns is determined by the constants
$\{\tau_i\}_{i=1}^N$, the worker states recorded in $\cF_k$, and any independent tie-breaking randomness.
By the tie-breaking requirement in \Cref{sec:rennala_nsgd}, this order does not depend on the fresh oracle samples or their gradient values.
Thus selecting the first $B$ returns tagged with $k$ preserves the sampling law: conditional on $\cF_k$,
the accepted samples $\{\xijk\}_{j=1}^B$ are i.i.d.\ with distribution $\cD$.
Since $\xk$ is $\cF_k$-measurable, the accepted gradients are conditionally i.i.d.\ given $\cF_k$,
with the conditional distribution of $\nf{\xk;\xi}$ for a fresh $\xi\sim\cD$.
This conditional distribution can depend on the history through $\xk$.

Writing the accepted minibatch as $\gk \eqdef \frac{1}{B}\sum_{j=1}^B \nf{\xk;\xijk}$, we
consequently have $\Expb{\gk}{\cF_k} = \nfxk$, and the zero-mean noise vectors
$\etajk \eqdef \nf{\xk;\xijk}-\nfxk$ satisfy \Cref{as:non_euclidean_light_tail} conditionally
on $\cF_k$. Applying \Cref{lemma:non_euclidean_large_deviations} conditionally with $x=\xk$
therefore gives, for every $\lambda>0$,
\begin{align}
\label{eq:rennala_lt_conditional_minibatch}
    \Prob\cb{\snorm{\gk - \nfxk} \ge \rb{\sqrt{2\kappa} + \sqrt{2}\lambda}\frac{\sigma}{\sqrt{B}} \mid \cF_k} \le \exp(-\lambda^2/3).
\end{align}
Thus the iterates produced by \algname{Rennala NSGD} satisfy the same conditional minibatch
oracle properties as the serial minibatch method analyzed in
\Cref{thm:stochastic_L_smooth_final}, so that theorem applies.

\paragraph{Step 2: iteration complexity.}
By \Cref{thm:stochastic_L_smooth_final}, with the choice of $T$ and $B$ in the statement, \algname{Rennala NSGD} finds an $\eps$-stationary point after at most $T$ iterations with probability at least $1-\delta$. Our subsequent wall-clock time analysis is adapted from the proof of Theorem~7.5 in \citep{tyurin2023optimal} to the non-Euclidean setting. Thus, it remains to upper bound the wall-clock time required to complete $T$ iterations under the fixed computation model where worker $i$ returns one stochastic gradient every $\tau_i$ seconds. Without loss of generality, assume $0<\tau_1\le \dots \le \tau_N$.

\paragraph{Time to collect one minibatch.}
Fix an iteration index $k$. In iteration $k$, the server forms the minibatch estimator $g^k$ by collecting $B$ stochastic gradients computed at the current iterate $x^k$ (i.e., gradients whose iteration index equals $k$), ignoring all stale gradients.

Choose nonnegative integer target counts $B_i$ for the workers. These are sufficient quotas
for bounding the collection time, not necessarily the numbers of gradients actually accepted
from each worker before the minibatch is completed.

At the instant $\xk$ is formed, each worker has at most one computation or message pending
from an earlier iteration. Completing and processing it takes at most $\tau_i$ additional
time, after which the worker receives the current iterate. The worker that completed the
preceding minibatch instead receives $\xk$ immediately and starts a fresh computation;
at $k=0$, all workers receive the initial iterate without first completing a stale computation.
As long as iteration $k$ has not ended, every subsequent reply sends $\xk$ again.
Consequently, unless the minibatch has already been completed, worker $i$ can deliver
$B_i>0$ fresh gradients within $\tau_i(1+B_i)$ time after $\xk$ is formed. The extra
$\tau_i$ is unnecessary for a worker that starts immediately, but remains a valid upper
bound. Define
\begin{align*}
    t_i \eqdef
    \begin{cases}
        \tau_i(1+B_i), & B_i>0, \\
        0, & B_i=0.
    \end{cases}
\end{align*}
Since workers run in parallel, a sufficient condition to collect a batch of size $B$ is the existence of $t'\ge 0$ and integers $B_1, \dots, B_N\in\mathbb{N}_0$ such that
\begin{align}\label{eq:aux_batch_time_integer_proof}
    t' \ge \max_{i\in[N]} t_i, \qquad \sum_{i=1}^N B_i \ge B.
\end{align}
Indeed, if the iteration had not ended by time $t'$, each worker would have delivered its
quota, giving at least $B$ accepted gradients, a contradiction. Thus the modified reply
order admits the same collection-time bound as the original protocol.

\paragraph{Relaxation and rounding.}
We relax \eqref{eq:aux_batch_time_integer_proof} by allowing $B_i\in\R_+$. If $(t', B_1, \dots, B_N)$ is feasible for the relaxed problem, then $(2t', \lceil B_1\rceil, \dots, \lceil B_N\rceil)$ is feasible for the integer problem. Indeed, $\sum_i \lceil B_i\rceil \ge \sum_i B_i \ge B$, and for any $i$ with $B_i>0$,
\begin{align*}
    \tau_i(1+\lceil B_i\rceil)\le \tau_i(2+B_i)\le 2\tau_i(1+B_i)\le 2t'.
\end{align*}
(For $B_i=0$, rounding leaves $B_i$ and $t_i$ equal to zero.) Hence it suffices to construct a feasible solution to the relaxed problem and then pay a factor $2$.

\paragraph{Constructing a feasible solution.}
For any $S\in[N]$, define
\begin{align*}
    t'(S) \eqdef \left(\sum_{i=1}^S \frac{1}{\tau_i}\right)^{-1}(B+S),
\end{align*}
and let $S^\star\in\argmin_{S\in[N]} t'(S)$ be the smallest minimizer. Consider the choice
\begin{align*}
    B_i \eqdef
    \begin{cases}
        \frac{t'(S^\star)}{\tau_i}-1, & i\le S^\star, \\
        0, & i>S^\star.
    \end{cases}
\end{align*}
Then
\begin{align*}
    \sum_{i=1}^N B_i &= \sum_{i=1}^{S^\star}\left(\frac{t'(S^\star)}{\tau_i}-1\right) \\
    &= t'(S^\star)\sum_{i=1}^{S^\star}\frac{1}{\tau_i}-S^\star \\
    &= (B+S^\star)-S^\star = B,
\end{align*}
so the batch-size constraint holds with equality.
Moreover, one can verify that $t'(S^\star)>\tau_{S^\star}$, hence $B_i>0$ for all $i\le S^\star$.
\paragraph{Showing that $B_i>0$ for all $i\le S^\star$.}
Recall that
\begin{align*}
    t'(S) \eqdef \left(\sum_{i=1}^S \frac{1}{\tau_i}\right)^{-1}(B+S), \qquad S^\star\in\argmin_{S\in[N]} t'(S) \quad \text{ is the \emph{smallest} minimizer,}
\end{align*}
and we set
\begin{align*}
    B_i \eqdef 
    \begin{cases}
        \frac{t'(S^\star)}{\tau_i}-1, & i\le S^\star, \\
        0, & i> S^\star.
    \end{cases}
\end{align*}
We prove that $t'(S^\star)>\tau_{S^\star}$, which will imply $B_i>0$ for all $i\le S^\star$.

\smallskip
\noindent\emph{Case 1: $S^\star=1$.}
Then
\begin{align*}
    t'(1) = \tau_1(B+1) > \tau_1 = \tau_{S^\star}.
\end{align*}

\smallskip
\noindent\emph{Case 2: $S^\star>1$.}
Since $S^\star$ is the \emph{smallest} minimizer of $t'(\cdot)$, we must have
\begin{align*}
    t'(S^\star) < t'(S^\star-1).
\end{align*}
Writing $A_S \eqdef \sum_{i=1}^S \frac{1}{\tau_i}$, this inequality is equivalent to
\begin{align*}
    A_{S^\star}^{-1}(B+S^\star) < A_{S^\star-1}^{-1}(B+S^\star-1),
\end{align*}
and hence (multiplying by $A_{S^\star}A_{S^\star-1}>0$)
\begin{align*}
    A_{S^\star-1}(B+S^\star) < A_{S^\star}(B+S^\star-1).
\end{align*}
Using $A_{S^\star}=A_{S^\star-1}+\frac{1}{\tau_{S^\star}}$, we obtain
\begin{align*}
    A_{S^\star-1}(B+S^\star) &< \left(A_{S^\star-1}+\frac{1}{\tau_{S^\star}}\right)(B+S^\star-1) \\
    &= A_{S^\star-1}(B+S^\star-1) + \frac{B+S^\star-1}{\tau_{S^\star}}.
\end{align*}
Subtracting $A_{S^\star-1}(B+S^\star-1)$ from both sides yields
\begin{align*}
    A_{S^\star-1} < \frac{B+S^\star-1}{\tau_{S^\star}},
\end{align*}
and therefore
\begin{align*}
    A_{S^\star} = A_{S^\star-1}+\frac{1}{\tau_{S^\star}} < \frac{B+S^\star-1}{\tau_{S^\star}}+\frac{1}{\tau_{S^\star}} = \frac{B+S^\star}{\tau_{S^\star}}.
\end{align*}
Finally,
\begin{align*}
    t'(S^\star) = \frac{B+S^\star}{A_{S^\star}} > \tau_{S^\star}.
\end{align*}

\smallskip
Since $\tau_1\le \dots \le \tau_{S^\star}$, for every $i\le S^\star$ we have
\begin{align*}
    B_i = \frac{t'(S^\star)}{\tau_i}-1 \ge \frac{t'(S^\star)}{\tau_{S^\star}}-1 > 0.
\end{align*}
Thus indeed $B_i>0$ for all $i\le S^\star$.

Finally,
\begin{align*}
    \max_{i\in[N]} t_i &= \max_{i\le S^\star}\tau_i(1+B_i) \\
    &= \max_{i\le S^\star}\tau_i\cdot \frac{t'(S^\star)}{\tau_i} \\
    &= t'(S^\star),
\end{align*}
so $(t'(S^\star), B_1, \dots, B_N)$ is feasible for the relaxed problem. By the rounding argument, the server can therefore collect a batch of size $B$ within at most $2t'(S^\star)$ seconds in every iteration.


\paragraph{Total time.}
Thus, the wall-clock time to execute $T$ iterations is at most
\begin{align}\label{eq:time_exact_bound_proof}
    T_{\mathrm{total}} \le 2T\cdot \min_{S\in[N]} \left[ \left(\sum_{i=1}^S \frac{1}{\tau_i}\right)^{-1}(B+S) \right].
\end{align}
Combining this deterministic runtime bound with \Cref{thm:stochastic_L_smooth_final} (which guarantees success after $T$ iterations with probability at least $1-\delta$) yields the high-probability time guarantee \eqref{eq:time_exact_bound}.

\paragraph{From the exact bound to a $\widetilde{\mathcal O}$ complexity bound.}
Fix $S\in[N]$ and define the average inverse time (average rate) of the fastest $S$ workers by
\begin{align*}
    \bar r_S \eqdef \frac{1}{S}\sum_{i=1}^S \frac{1}{\tau_i}.
\end{align*}
Then we can rewrite
\begin{align*}
    \left(\sum_{i=1}^S \frac{1}{\tau_i}\right)^{-1}(B+S) &= \frac{B+S}{\sum_{i=1}^S \frac{1}{\tau_i}} = \frac{S}{\sum_{i=1}^S \frac{1}{\tau_i}}\left(1+\frac{B}{S}\right) \\
    &= \bar r_S^{-1}\left(1+\frac{B}{S}\right).
\end{align*}
Plugging this into \eqref{eq:time_exact_bound_proof} gives
\begin{align}\label{eq:time_factorized}
    T_{\mathrm{total}} \le 2T\cdot \min_{S\in[N]}\bar r_S^{-1}\left(1+\frac{B}{S}\right) = 2\min_{S\in[N]}\bar r_S^{-1}\left(T+\frac{TB}{S}\right).
\end{align}

Next, we substitute the iteration and batch-size choices from \Cref{thm:stochastic_L_smooth_final}:
\begin{align*}
    T = \left\lceil \frac{128L\Delta}{15\eps}\right\rceil \implies T = \mathcal{O}\left(\frac{L\Delta}{\eps}\right),
\end{align*}
and
\begin{align*}
    B = \max\left\{ \left\lceil \frac{16\sigma^2\left(\sqrt{2\kappa}+\sqrt{6\log(T/\delta)}\right)^2}{\eps} \right\rceil, 1\right\}.
\end{align*}
Using $(a+b)^2\le 2a^2+2b^2$ with $a=\sqrt{2\kappa}$ and $b=\sqrt{6\log(\nfr{T}{\delta})}$, we obtain
\begin{align*}
    \left(\sqrt{2\kappa}+\sqrt{6\log(T/\delta)}\right)^2 \le 2(2\kappa)+2\bigl(6\log(T/\delta)\bigr) = 4\kappa+12\log(T/\delta).
\end{align*}
Therefore,
\begin{align}\label{eq:B_bound_for_time}
    B \le 1+\frac{16\sigma^2\left(4\kappa+12\log(T/\delta)\right)}{\eps} = 1+\frac{64\kappa\sigma^2}{\eps} +\frac{192\sigma^2\log(T/\delta)}{\eps}.
\end{align}
In particular, using $\kappa\ge1$ and hiding the logarithmic factor $\log(\nfr{T}{\delta})$ inside $\widetilde{\mathcal O}(\cdot)$, we have
\begin{align*}
    B = \widetilde{\mathcal O}\left(1+\frac{\kappa\sigma^2}{\eps}\right).
\end{align*}
Consequently,
\begin{align*}
    \frac{TB}{S}
    &= \widetilde{\mathcal O}\left(\frac{T}{S}+\frac{\kappa T\sigma^2}{S\eps}\right) \\
    &= \widetilde{\mathcal O}\left(\frac{L\Delta}{S\eps}+\frac{\kappa L\Delta\sigma^2}{S\eps^2}\right).
\end{align*}
Since $S\ge1$, the deterministic contribution $T/S$ is at most $T$ and can be absorbed into the leading $T$ term in \eqref{eq:time_factorized}. Substituting these bounds into \eqref{eq:time_factorized} gives
\begin{align*}
    T_{\mathrm{total}} = \widetilde{\mathcal O}\left( \min_{S\in[N]} \left[ \left(\frac{1}{S}\sum_{i=1}^S\frac{1}{\tau_i}\right)^{-1} \left( \frac{L\Delta}{\eps} + \frac{\kappa L\Delta\sigma^2}{S\eps^2} \right) \right] \right),
\end{align*}
which is the claimed $\widetilde{\mathcal O}$ time complexity bound.

\end{proof}

\subsection{Discussion}\label{app:rennala_time_discussion}
We now interpret \eqref{eq:time_simplified_bound}, explain how the update norm and noise model affect it, and compare its Euclidean specialization with existing guarantees.

\textbf{A system factor and a problem factor.}
For each $S\in[N]$, the bound \eqref{eq:time_simplified_bound} is the product of two factors. The first, $\bigl(\frac{1}{S}\sum_{i=1}^S\frac{1}{\tau_i}\bigr)^{-1}$, is the harmonic mean of the computation times of the $S$ fastest workers, and it depends only on the system. The second, $\frac{L\Delta}{\eps}+\frac{\kappa L\Delta\sigma^2}{S\eps^2}$, is the oracle complexity \eqref{eq:complexity_dual_native_cor} with its statistical term divided by $S$, and apart from $S$ it depends only on the problem and the update norm. Enlarging $S$ shrinks the statistical term but admits slower workers into the harmonic mean, and the minimum over $S$ selects the best balance. Running the algorithm does not require knowing either the minimizing
number of workers or their computation times $\tau_i$.
The prescribed $T$ and $B$ do not depend on these quantities:
the server simply accepts the first $B$ fresh gradients to arrive.

\textbf{Where the update norm enters.}
The update norm $\norm{\cdot}$ leaves the system factor untouched and acts only through the second factor, but there it affects every constant except $\Delta$. The smoothness constant $L$ is defined through the pair $(\norm{\cdot}, \snorm{\cdot})$ in \Cref{as:non_euclidean_smoothness}, the noise level $\sigma$ is measured in $\snorm{\cdot}$ in \Cref{as:non_euclidean_light_tail}, the accuracy $\eps$ refers to stationarity measured in $\snorm{\cdot}$, and $\kappa$ is the regularity constant of $(E^*, \snorm{\cdot})$. Bounds obtained for different update norms therefore certify stationarity in different dual norms, and they should not be compared as if they measured the same quantity.

\textbf{Native versus converted noise in wall-clock time.}
The comparison of \Cref{subsubsec:native_noise_improvement} carries over from oracle complexity to wall-clock time.
Under the comparison assumptions of \Cref{subsec:frobenius_light_tail_summary_alt} and the fixed computation model
of \Cref{thm:main_time_complexity}, apply \Cref{lem:frobenius_to_dual_conc_alt} conditionally to each accepted minibatch
and use $B_{\opf}$ from \eqref{eq:frobenius_batch_alt}. The same union bound and descent argument give the required
stationarity guarantee. The deterministic minibatch-collection bound \eqref{eq:time_exact_bound} applies with
$B=B_{\opf}$. Since $B_{\opf}=\widetilde{\cO}(1+\mathsf C_{\opf}/\eps)$, substituting into that bound gives
the form \eqref{eq:time_simplified_bound} with the full native coefficient $\kappa\sigma^2$ replaced by
$\mathsf C_{\opf}=\alpha^{-2}\sigma_{\opf}^2$; the system factor is unchanged.
With the update norm, target accuracy, and worker times fixed, a smaller
statistical coefficient reduces the noise-dependent term for every
choice of $S$. Taking the minimum over $S$ therefore cannot increase
the resulting expression. Thus, a smaller statistical coefficient
also gives a wall-clock upper bound that is no larger and can be
significantly smaller, up to the suppressed logarithmic factors.

\textbf{The Euclidean case.}
When the update norm is Euclidean, $\snorm{\cdot}=\norm{\cdot}_2$ and $\kappa=1$, and \eqref{eq:time_simplified_bound} coincides, up to logarithmic factors, with the time complexity of \algname{Rennala SGD} \citep{tyurin2023optimal}.\footnote{Their lower bound applies to
zero-respecting algorithms (Definition~6.2) under their unbiased
bounded-variance oracle model (Definition~6.3), with fixed worker
computation times. They give a matching upper bound
for \algname{Rennala SGD}.} The assumptions differ, however: our light-tail condition is stronger than bounded variance, and our guarantee holds with high probability rather than in expectation. We therefore read this agreement as a consistency check with the Euclidean theory rather than as an optimality result, and we make no optimality claim for a general update norm.

\clearpage

\section{Impact Statement}\label{app:impact_statement}
This paper presents work whose goal is to advance the field of Machine
Learning. There are many potential societal consequences of our work, none
of which we feel must be specifically highlighted here.

\end{document}